\documentclass{article}

\PassOptionsToPackage{numbers}{natbib}

\usepackage[main,final]{neurips_2026}

\usepackage[utf8]{inputenc} %
\usepackage[T1]{fontenc}    %
\usepackage{url}            %
\usepackage{booktabs}       %
\usepackage{array,longtable} %
\usepackage{amsfonts}       %
\usepackage{nicefrac}       %
\usepackage{microtype}      %
\usepackage{xcolor}         %
\usepackage[pdftex]{graphicx}
\usepackage{amsthm}
\newtheorem{proposition}{Proposition}

\newtheorem{lemma}{Lemma}
\usepackage{amsmath}
\usepackage{amssymb}
\usepackage{needspace}
\usepackage{float}
\usepackage{hyperref}       %
\hypersetup{hidelinks,
  pdftitle={Training-Free Generative Sampling via Moment-Matched Score Smoothing},
  pdfauthor={Zhenyu Yao and Daniel Paulin}}
\usepackage{algorithm}
\usepackage{algpseudocode}
\newcommand{\sym}{\operatorname{sym}}
\newcommand{\nsteps}{n_{\mathrm{steps}}}
\algrenewcommand\algorithmicrequire{\textbf{Input:}}
\algrenewcommand\algorithmicensure{\textbf{Output:}}
\title{
Training-Free Generative Sampling via Moment-Matched Score Smoothing}

\author{%
  Zhenyu Yao \quad Daniel Paulin \\
  College of Computing and Data Science\\
  Nanyang Technological University\\
  \texttt{YAOZ0012@e.ntu.edu.sg \quad daniel.paulin@ntu.edu.sg} \\
}

\begin{document}

\maketitle

\begin{abstract}
    Diffusion models generate samples by denoising along the score of a perturbed target distribution. In practice, one trains a neural diffusion model, which is computationally expensive. Recent work suggests that score matching implicitly smooths the empirical score, and that this smoothing bias promotes generalization by capturing low-dimensional data geometry. We propose moment-matched score-smoothed overdamped Langevin dynamics (MM-SOLD), a training-free interacting particle sampler that enforces the target moments throughout the sampling trajectory. We prove that the one-particle marginal of the ideal constrained Gibbs equilibrium converges, as the number of particles grows, to a Gibbs--Boltzmann density obtained by exponentially tilting the naive score-smoothed target. Its mean and covariance match the training GMM, even for rank-deficient data. We additionally establish sampling consistency, statistical target-error bounds for bounded data, and consistency of a safeguarded BAOAB-based tilting estimator. Experiments on 2D distributions and latent-space image generation show that MM-SOLD enables fast, robust, latent sampling on CPUs using fixed pretrained autoencoders, with sample fidelity and diversity competitive with neural diffusion baselines.

\end{abstract}

\section{Introduction}
    Diffusion models generate samples by reversing a forward noising process whose reverse drift depends on the score of the perturbed data distribution \cite{sohl2015deep,ho2020denoising,song2019generative,song2020score,haussmann1986time}. They achieve high fidelity across many domains \cite{ho2022cascaded,dhariwal2021diffusion,rombach2022high,kong2020diffwave,ho2022imagen,xu2022geodiff,niu2020permutation}, but training a neural score model is expensive. On a finite dataset, the ideal empirical score-matching objective has a closed-form minimizer: the score of the Gaussian mixture obtained by noising the empirical measure \cite{hyvarinen2005estimation,vincent2011connection}. Exact reverse diffusion with this score returns the empirical measure at the zero-noise endpoint, hence memorizes \cite{biroli2024dynamical}.

Neural score models nevertheless generalize beyond their training samples \cite{pidstrigach2022score,bonnaire2025diffusion,yoon2023diffusion}. Their sampling targets and convergence have received extensive study \cite{chen2023score,chen2022sampling,de2022convergence,strasman2024analysis}, alongside faster samplers and latent representations \cite{salimans2022progressive,song2020denoising,vahdat2021score}. Recent work interprets part of the neural inductive bias as smoothing of the empirical score \cite{ma2021linear,vardi2023implicit,farghly2025diffusion,chen2025interpolation,gabriel2025kernel}. Suitable smoothing can adapt to low-dimensional data geometry, but excessive smoothing can cause barycentric collapse. Manifold-adapted kernels require geometric information that is generally unavailable. Closed-form diffusion models ($\sigma$-CFDM) \cite{scarvelis2023closed} offer training-free generation through score smoothing, but their barycentric structure can limit diversity.

We introduce moment-matched score-smoothed overdamped Langevin dynamics (MM-SOLD). It couples particles through a centered, scaled Stiefel constraint, matching their mean and covariance to prescribed moments at every iteration. This preserves global second-order geometry while allowing local interpolation. Our contributions are: (i) this constrained sampler and a scalable nearest-neighbor score estimator; (ii) a characterization of its large-particle equilibrium and qualitative sampling consistency for a full-score QR--Euler variant; (iii) KL improvement relative to the Gaussian-smoothed data law and local finite-sample target-error bounds for bounded data; (iv) plug-in and dual tilting estimators, with an almost-sure consistency theorem for safeguarded BAOAB-based dual estimation; and (v) experiments on 2D distributions, latent image generation, and classification. The latent sampler requires no score-network training; image experiments use fixed pretrained autoencoders.

Section~\ref{sec:Score_smoothing_and_closed-form_diffusion_models} reviews score smoothing, Section~\ref{sec:Moment-matched_score_smoothing} presents MM-SOLD and its guarantees, and Section~\ref{sec:Experiments} reports experiments. Proofs and implementation details are in the appendix. Code is available at \url{https://github.com/GitZhenyu2019/MM-SOLD}.

\section{Score smoothing and closed-form diffusion models}
\label{sec:Score_smoothing_and_closed-form_diffusion_models}
    \subsection{Concepts and formulations}
Given independent samples $\{x_i\}_{i=1}^N$ from $\pi_{\mathrm{data}}$ on $\mathbb R^d$, diffusion models approximate the time reversal \cite{haussmann1986time} of
\begin{equation}
\label{eq:forward_SDE}
\mathrm dX_t=-\alpha X_t\,\mathrm dt+\sqrt2\,\mathrm dB_t,
\qquad X_0\sim\pi_{\mathrm{data}},\quad \alpha\ge0.
\end{equation}
Writing $p_t$ for the density of $X_t$, the reverse process initialized from $Y_0\sim p_T$ satisfies, for $0\le t<T$,
\begin{equation}
\label{eq:reverse_SDE}
\mathrm dY_t=(\alpha Y_t+2\nabla\log p_{T-t}(Y_t))\,\mathrm dt+\sqrt2\,\mathrm dB_t.
\end{equation}
The score is usually learned by score matching \cite{hyvarinen2005estimation}. Starting instead from the empirical law $\hat\pi_{\mathrm{data}}=N^{-1}\sum_i\delta_{x_i}$ gives the Gaussian mixture
\[
\hat p_t(z)=\frac1N\sum_{i=1}^N\mathcal N(z\mid e^{-\alpha t}x_i,\beta_t I_d),
\qquad
\beta_t=\begin{cases}(1-e^{-2\alpha t})/\alpha,&\alpha>0,\\2t,&\alpha=0.\end{cases}
\]
For $\alpha>0$ this converges to $\mathcal N(0,\alpha^{-1}I_d)$ as $t\to\infty$.
The idealized empirical score-matching loss
$\hat\ell_{\mathrm{sm}}(s)=\int_0^T\mathbb E\|s(t,X_t)-\nabla\log\hat p_t(X_t)\|^2\,\mathrm dt$
is minimized by $\nabla\log\hat p_t$ almost everywhere; denoising score matching targets the same score \cite{vincent2011connection}. Explicitly,
\begin{equation}
\label{eq:closed_form_empirical_score}
\nabla\log\hat p_t(z)=\beta_t^{-1}(c_t(z)-z),
\end{equation}
\begin{equation}
\label{eq:closed_form_ct}
c_t(z)=\sum_{i=1}^N w_i^t(z)e^{-\alpha t}x_i,
\qquad w_i^t(z)=\frac{e^{-\|z-e^{-\alpha t}x_i\|^2/(2\beta_t)}}{\sum_j e^{-\|z-e^{-\alpha t}x_j\|^2/(2\beta_t)}}.
\end{equation}
At small $t$, the weights concentrate on nearby training points. Exact empirical reverse diffusion therefore returns the training measure in its zero-noise limit.

For a probability kernel $k$, define the smoothed score
\begin{equation}
\label{eq:smoothed_score}
s^k(t,z)=\int\nabla\log\hat p_t(y)\,k_z(\mathrm dy).
\end{equation}
For additive smoothing, $k_z=\mathrm{Law}(z+\sigma\varepsilon)$, with $\mathbb E\varepsilon=0$ and $\mathbb E\|\varepsilon\|^2<\infty$, this is a convolution and
\begin{equation}
\label{eq:additive_smoothed_score}
s^{k,\sigma}(t,z)=\beta_t^{-1}\bigl(\mathbb E c_t(z+\sigma\varepsilon)-z\bigr).
\end{equation}
An unbiased Monte Carlo estimate uses independent smoothing draws:
\begin{equation}
\label{eq:MC_smoothed_score}
\widehat s^{k,\sigma}(t,z)=\beta_t^{-1}\left(\frac1M\sum_{m=1}^M c_t(z+\sigma\varepsilon_m)-z\right).
\end{equation}
The $\sigma$-CFDM sampler \cite{scarvelis2023closed} combines Gaussian score smoothing with a straight-flow ODE \cite{liu2022flow}, using a linear interpolation schedule rather than the Ornstein--Uhlenbeck schedule above. Its barycentric sampling structure motivates our moment correction.

The finite-mixture score has linear growth, so differentiation under the smoothing expectation is valid:
\begin{equation}
\label{eq:exchanged_smoothed_score}
s^{k,\sigma}(t,z)=\nabla\mathbb E_\varepsilon\log\hat p_t(z+\sigma\varepsilon).
\end{equation}
At a \emph{fixed} noise level $\tau>0$, Langevin corrector dynamics \cite{song2020score} with this drift has invariant density
\begin{equation}
\label{eq:naive_score-smoothed_sampling_target_in_time}
\hat p_\tau^{k,\sigma}(z)\propto\exp\bigl(\mathbb E_\varepsilon\log\hat p_\tau(z+\sigma\varepsilon)\bigr).
\end{equation}
This identifies a fully equilibrated corrector target, not the general endpoint law of time-dependent reverse diffusion. Motivated by small-noise mixtures, we fix component standard deviation $\delta>0$ and centers $x_i$, and study
\begin{equation}
\label{eq:naive_score-smoothed_sampling_target}
\hat p^{\delta,k,\sigma}(z)\propto\exp\bigl(\mathbb E_\varepsilon\log\hat p^\delta(z+\sigma\varepsilon)\bigr),
\qquad \hat p^\delta=\frac1N\sum_i\mathcal N(x_i,\delta^2I_d).
\end{equation}
This log-domain smoothing can adapt to manifold geometry under suitable assumptions and bandwidth choices \cite{farghly2025diffusion}. We use Gaussian smoothing noise $\varepsilon\sim\mathcal N(0,I_d)$ henceforth.

\subsection{Sensitivity to smoothing scale}
Figure~\ref{fig:smoothing_scale_circle} (top row) illustrates the smoothing-scale sensitivity of \eqref{eq:naive_score-smoothed_sampling_target} on a 2D circle: too little smoothing leaves memorization, too much replaces local interpolation by barycentric averaging and collapses the target. The distortion is substantially more severe for sparse, curved manifolds in high-dimensional ambient spaces. Manifold-adapted kernels can alleviate it \cite{farghly2025diffusion} but require geometric information unavailable in practice, so existing training-free closed-form diffusion models treat this sensitivity as fundamental and rely on dataset-specific tuning.
\begin{figure}[t]
    \centering
    {\textit{Naive density \eqref{eq:naive_score-smoothed_sampling_target}}}\\[1pt]
    \begin{minipage}{0.24\textwidth}\centering\includegraphics[width=\textwidth, trim=70 60 40 40, clip]{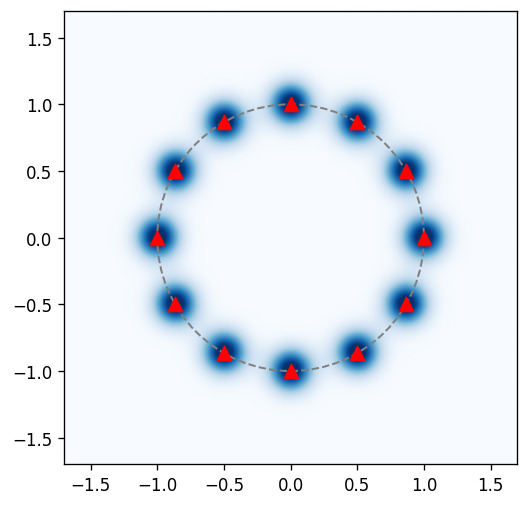}\end{minipage}\hfill
    \begin{minipage}{0.24\textwidth}\centering\includegraphics[width=\textwidth, trim=70 60 40 40, clip]{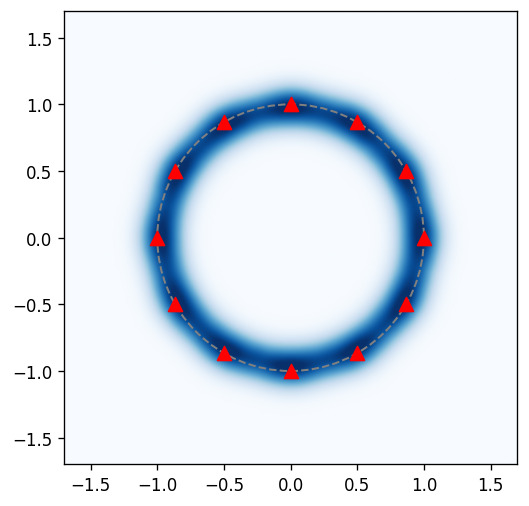}\end{minipage}\hfill
    \begin{minipage}{0.24\textwidth}\centering\includegraphics[width=\textwidth, trim=70 60 40 40, clip]{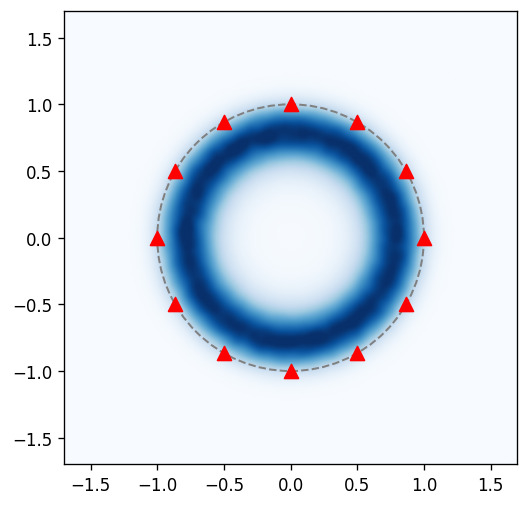}\end{minipage}\hfill
    \begin{minipage}{0.24\textwidth}\centering\includegraphics[width=\textwidth, trim=70 60 40 40, clip]{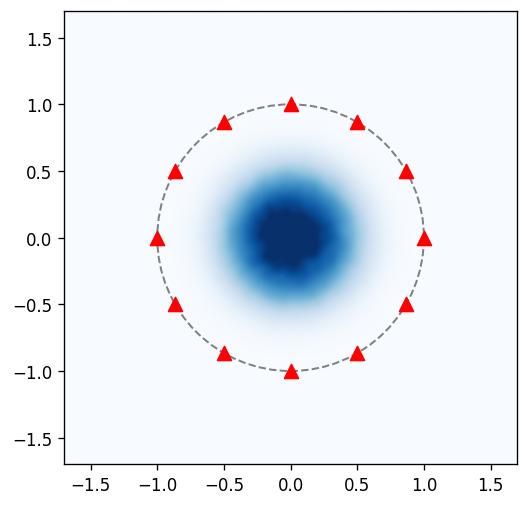}\end{minipage}

    \vspace{2pt}
    {\textit{Moment-matched density \eqref{eq:moment_matched_target}}}\\[1pt]
    \begin{minipage}{0.24\textwidth}\centering\includegraphics[width=\textwidth, trim=70 60 40 40, clip]{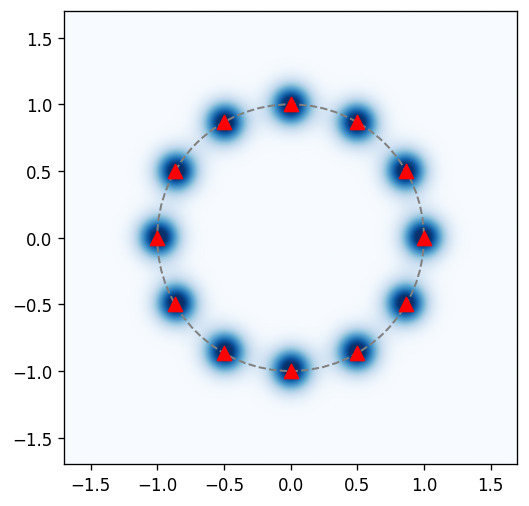}\end{minipage}\hfill
    \begin{minipage}{0.24\textwidth}\centering\includegraphics[width=\textwidth, trim=70 60 40 40, clip]{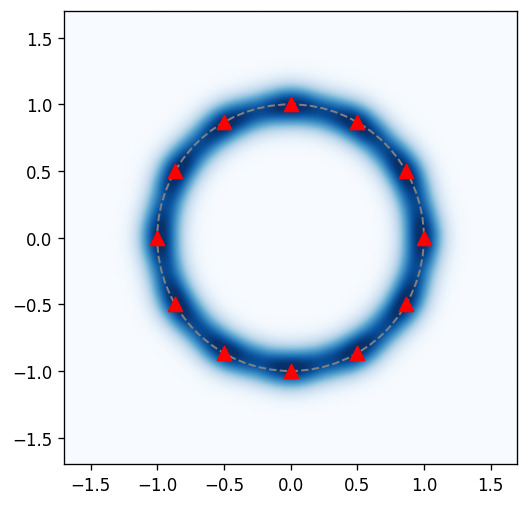}\end{minipage}\hfill
    \begin{minipage}{0.24\textwidth}\centering\includegraphics[width=\textwidth, trim=70 60 40 40, clip]{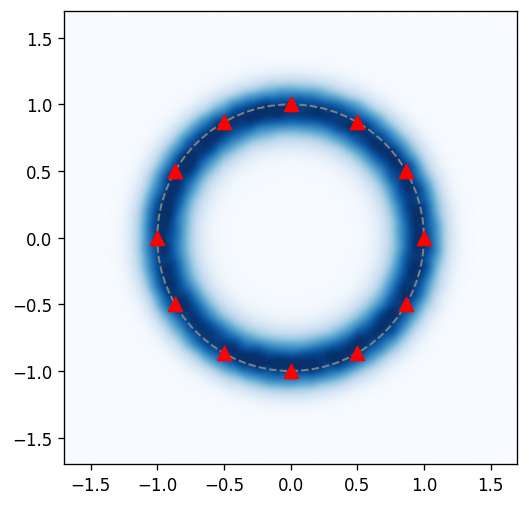}\end{minipage}\hfill
    \begin{minipage}{0.24\textwidth}\centering\includegraphics[width=\textwidth, trim=70 60 40 40, clip]{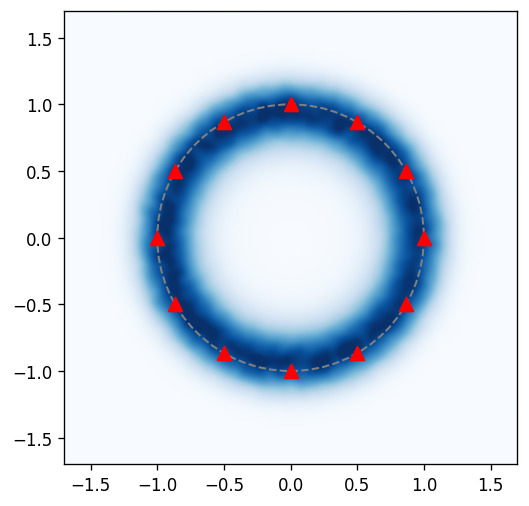}\end{minipage}
    \caption{Naive (top) and moment-matched (bottom) densities under isotropic Gaussian smoothing on a 2D circle, with increasing bandwidth $\sigma$ (left to right: $\sigma=0.05, 0.20, 0.45, 0.65$). Red triangles denote training samples; the dashed curve indicates the underlying unit circle. The moment-matched correction restores the circular geometry across all bandwidths.}
    \label{fig:smoothing_scale_circle}
\end{figure}

\section{Moment-matched score smoothing}
\label{sec:Moment-matched_score_smoothing}
    \subsection{Algorithmic description}
Recall that the smoothed potential in density \eqref{eq:naive_score-smoothed_sampling_target} is
\begin{equation}
V(z):=-\mathbb E_{\varepsilon\sim\rho_\varepsilon}\big[\log \hat{p}^\delta(z+\sigma\varepsilon)\big],
\quad
g(z):=\nabla V(z),  
\end{equation}
so the density is the invariant measure of unconstrained overdamped Langevin dynamics:
\begin{equation}
\label{eq:unconstrained_OLD}
\mathrm d Z_t=-g(Z_t)\,\mathrm dt+\sqrt{2}\,\mathrm dB_t.  
\end{equation}
For training vectors $x_i\in\mathbb R^d$, set $\mu^*=N^{-1}\sum_i x_i$ and $\widehat\Sigma_N=N^{-1}\sum_i(x_i-\mu^*)(x_i-\mu^*)^\top$. Throughout the theory and experiments, MM-SOLD enforces the GMM covariance $\Sigma^*=\widehat\Sigma_N+\delta^2I_d$ in the operating coordinates. Since $\delta>0$, this is positive definite even for rank-deficient data. Write $\Sigma^*=L^*(L^*)^\top$ for its Cholesky factorization.

For $P$ particles $Z^1,\dots,Z^P\in\mathbb R^d$, write $Z\in\mathbb R^{P\times d}$ for the row-stacked particle matrix. We require the empirical mean and covariance of these particles to satisfy
$\hat\mu_P(Z)=\mu^*$ and $\hat\Sigma_P(Z)=\Sigma^*$
at every iteration. To simplify the computation, we reformulate these constraints by mapping to:
\begin{equation}
Y=(Z-\mathbf 1_P{\mu^*}^\top)(L^*)^{-\top},
\quad
Z=\mathbf 1_P{\mu^*}^\top+Y(L^*)^\top,
\end{equation}
where $\mathbf 1_P\in\mathbb R^P$ is the all-ones vector.
Then the moment constraints in $Y$-space become:
\begin{equation}
\mathbf 1_P^\top Y=0,
\quad
Y^\top Y=PI_d,
\end{equation}
so the particle matrix mapped in $Y$-space evolves on a centered and scaled Stiefel-type manifold:
\begin{equation}
\mathcal M_P:=
\left\{
Y\in\mathbb R^{P\times d}:\;
\mathbf 1_P^\top Y=0,\;
Y^\top Y=PI_d
\right\},
\end{equation}
which requires the number of particles $P\ge d+1$ by construction.

Writing $Y=\sqrt P\,UX$ with $U\in\mathbb R^{P\times(P-1)}$ orthonormal columns spanning $\mathbf 1_P^\perp$ and $X\in V_{P-1,d}:=\{X\in\mathbb R^{(P-1)\times d}:X^\top X=I_d\}$, we let $\mathcal{U}_{\mathcal M_P}$ denote the pushforward of the Stiefel Haar measure under $X\mapsto\sqrt P\,UX$---the natural uniform measure on $\mathcal M_P$ \cite{james1954normal,mardia1977uniform,chikuse2003statistics}.
\begin{algorithm}[t]
\caption{MM-SOLD: moment-matched score-smoothed overdamped Langevin dynamics}
\label{alg:mmsold-lm}
\begin{algorithmic}[1]
\Require Training set $\{x_i\}_{i=1}^N$, GMM component standard deviation $\delta$, smoothing bandwidth $\sigma$, particle number $P$, step size $h$, number of iterations $\nsteps$.
\State Compute $\mu^*=N^{-1}\sum_i x_i$ and $\widehat\Sigma_N=N^{-1}\sum_i(x_i-\mu^*)(x_i-\mu^*)^\top$.
\State Set $\Sigma^*=\widehat\Sigma_N+\delta^2I_d$ and compute $L^*$ with $\Sigma^*=L^*(L^*)^\top$.
\State Sample initial particles $Z_0^1,\dots,Z_0^P$ independently from $\hat p^\delta$.
\State $Y_0\gets\mathcal R\!\left((Z_0-\mathbf 1_P{\mu^*}^\top)(L^*)^{-\top}\right)$.
\State Sample $\Xi_{\mathrm{prev}}\in\mathbb R^{P\times d}$ with i.i.d.\ $\mathcal N(0,1)$ entries.
\For{$k=0,1,\dots,\nsteps-1$}
    \State $Z_k \gets \mathbf 1_P{\mu^*}^\top+Y_k(L^*)^\top$.
    \State $G_k^Z \gets [\,g(Z_k^1)^\top;\dots;g(Z_k^P)^\top\,]$, $\quad G_k^Y \gets G_k^ZL^*$.
    \State Sample $\Xi_k\in\mathbb R^{P\times d}$ with i.i.d.\ $\mathcal N(0,1)$ entries.
    \State $\widetilde G_k \gets \Pi_{Y_k}(G_k^Y)$, $\widetilde\Xi_{\mathrm{prev}} \gets \Pi_{Y_k}(\Xi_{\mathrm{prev}})$, $\widetilde\Xi_k \gets \Pi_{Y_k}(\Xi_k)$.
    \State $\widetilde Y_{k+1}\gets Y_k-h\widetilde G_k+\sqrt{h/2}\,(\widetilde\Xi_{\mathrm{prev}}+\widetilde\Xi_k)$.
    \State $Y_{k+1}\gets \mathcal R(\widetilde Y_{k+1})$.
    \State $\Xi_{\mathrm{prev}}\gets \Xi_k$.
\EndFor
\State \Return Generated samples: \(Z_{\nsteps}=\mathbf 1_P{\mu^*}^\top+Y_{\nsteps}(L^*)^\top\).
\end{algorithmic}
\end{algorithm}

In whitened coordinates, the one-particle potential is
$\widetilde V(y)=V(\mu^*+L^*y)$. The total potential and constrained Gibbs law are
\begin{equation}
\label{eq:finite_P_constrained_gibbs_main}
F_P(Y)=\sum_{i=1}^P\widetilde V(y_i),\qquad
\mu_P(dY)\propto e^{-F_P(Y)}\mathcal U_{\mathcal M_P}(dY).
\end{equation}

The steps of MM-SOLD are summarized in Algorithm \ref{alg:mmsold-lm}. Each iteration evaluates the smoothed score in $Z$-space, pulls it back to $Y$-space via $G^Y=G^Z L^*$, and projects drift and noise onto the tangent space of $\mathcal M_P$ via $\Pi_Y(A)=\Pi_Y^{\mathrm{St}}\bigl(\Pi_{\mathrm{ctr}}(A)\bigr)$, where $\Pi_{\mathrm{ctr}}(A)=A-\tfrac1P\mathbf 1_P\mathbf 1_P^\top A$ centers and $\Pi_Y^{\mathrm{St}}(A)=A-Y\,\sym\!\left(\tfrac1P Y^\top A\right)$ with $\sym(B):=\frac12(B+B^\top)$ enforces tangency to $Y^\top Y=PI_d$. The update uses the Leimkuhler--Matthews (LM) scheme \cite{leimkuhler2013rational} reusing the previous step's noise, motivated by its lower bias in unconstrained settings. Our consistency theorem below concerns the Euler--Maruyama variant. Since the projected step may leave $\mathcal M_P$, we retract via centered reduced-QR, $\mathcal R(\widetilde Y_{k+1})=\sqrt{P}Q$, with the standard diagonal sign correction. Derivation details are in Appendix \ref{app:MM-SOLD_details}.

\subsection{Evaluations on 2D}
\label{sec:eval_on_2D}
We show the behavior of MM-SOLD on two 2D distributions, named ``Checkerboard'' and ``Two Spirals''. For each distribution, we uniformly draw 500 samples as the training dataset. We use $P=5{,}000$ particles in MM-SOLD, generating 5{,}000 samples. The sampler runs for 3{,}000 iterations with step size $5\times 10^{-4}$. As a baseline we compare with $\sigma$-CFDM \cite{scarvelis2023closed}---to our knowledge, the only existing training-free score-smoothing sampler shown to produce high-quality novel samples. 
Figure \ref{fig:2d_checkerboard_spirals} presents results obtained with isotropic Gaussian smoothing kernels, under varying smoothing bandwidth $\sigma$ and numbers $M$ of Monte Carlo noise samples used to estimate the smoothed score. Since $\sigma$-CFDM generates convex combinations of barycenters of training subsets, it is sensitive to both $\sigma$ and $M$. For small $\sigma$ and $M$, its samples remain close to the training data. As $\sigma$ increases, they spread toward the convex hull of the training set and produce many off-support samples. Increasing $M$ makes samples concentrate around some local centroids, leading to sampling degradation. In contrast, MM-SOLD uses a different sampling paradigm with moment-matched particles, making it more robust to $\sigma$ and $M$ while better preserving sample diversity and distributional fidelity.
\begin{figure}[t]    
\centering    
\setlength{\tabcolsep}{1.5pt}    
\renewcommand{\arraystretch}{0.8}    
\begin{tabular}{cccc}        
\includegraphics[width=0.242\textwidth, trim=29 24 7 20.5, clip]{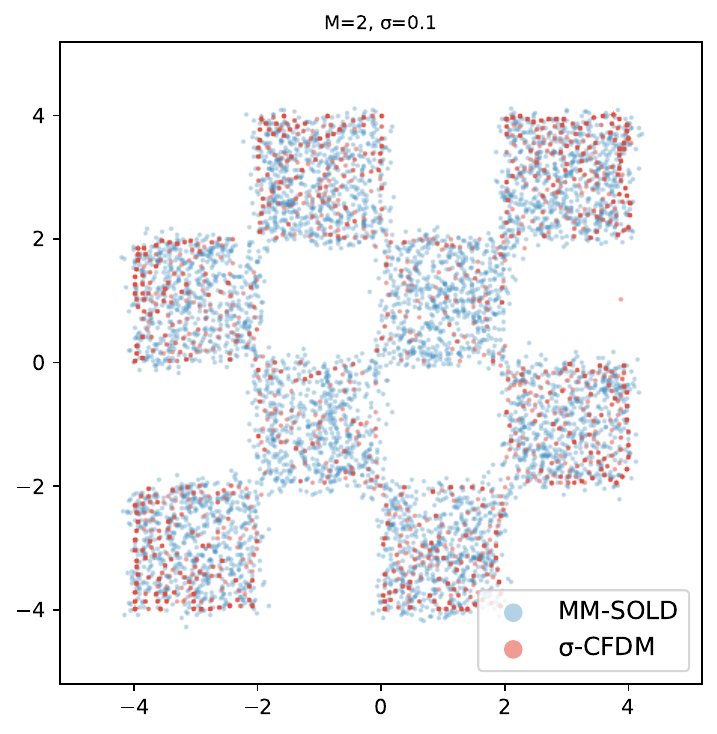} &        
\includegraphics[width=0.242\textwidth, trim=29 24 7 20.5, clip]{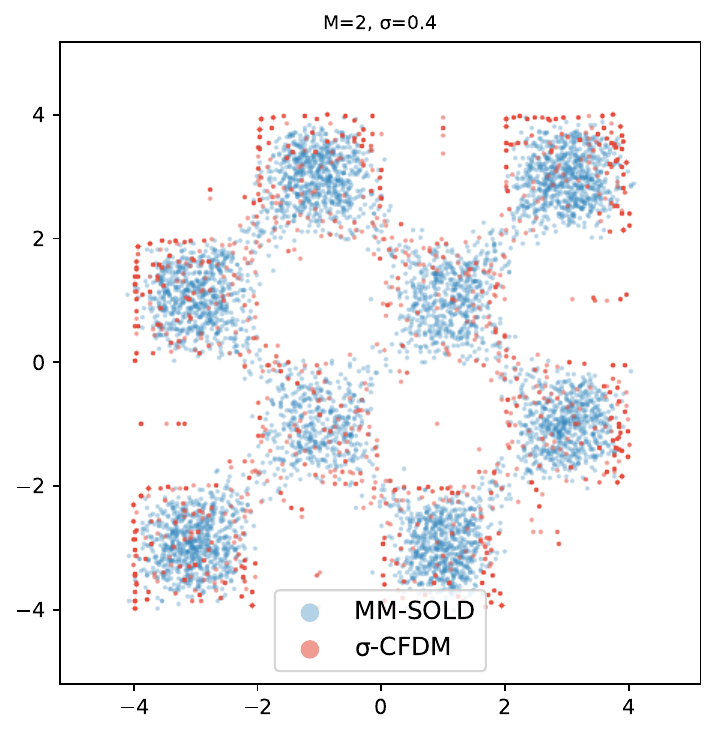} &        
\includegraphics[width=0.242\textwidth, trim=29 24 7 20.5, clip]{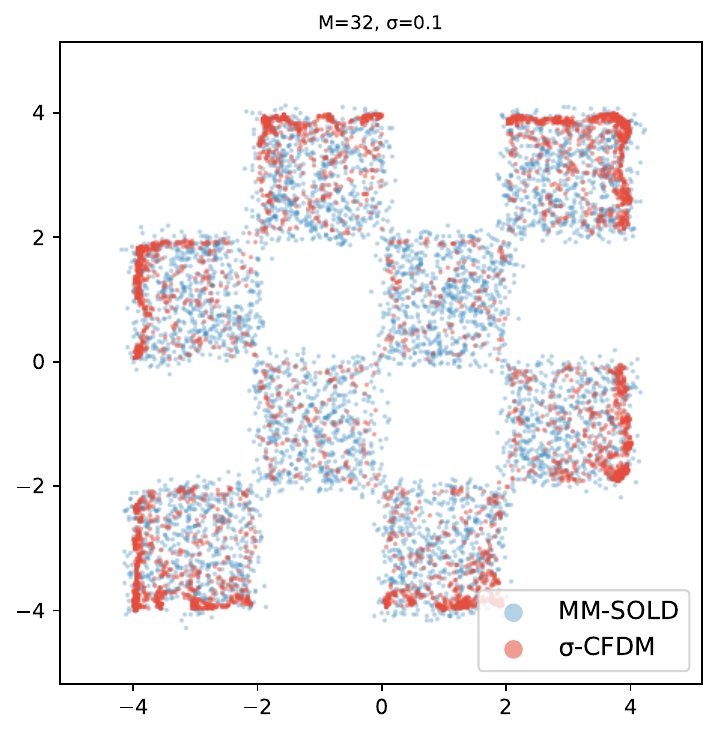} &        
\includegraphics[width=0.242\textwidth, trim=29 24 7 20.5, clip]{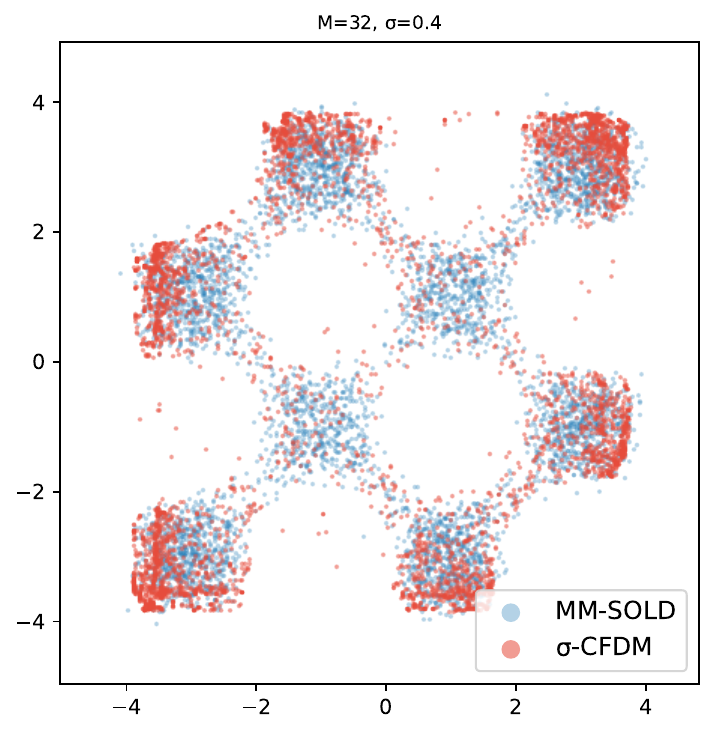} \\[-1mm]
{\scriptsize $M=2,\sigma=0.1$} &
{\scriptsize $M=2,\sigma=0.4$} &
{\scriptsize $M=32,\sigma=0.1$} &
{\scriptsize $M=32,\sigma=0.4$} \\[1mm]
\includegraphics[width=0.242\textwidth, trim=29 24 7 20.5, clip]{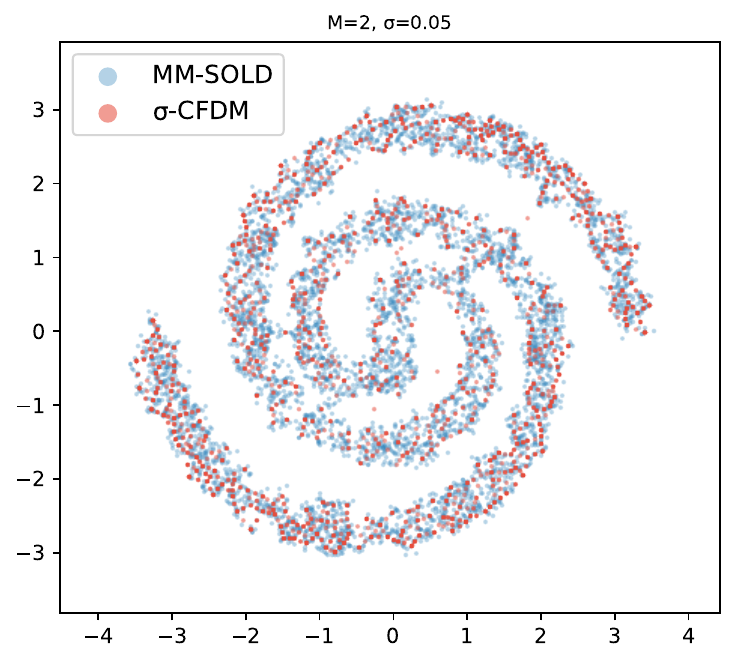} &        
\includegraphics[width=0.242\textwidth, trim=29 24 7 20.5, clip]{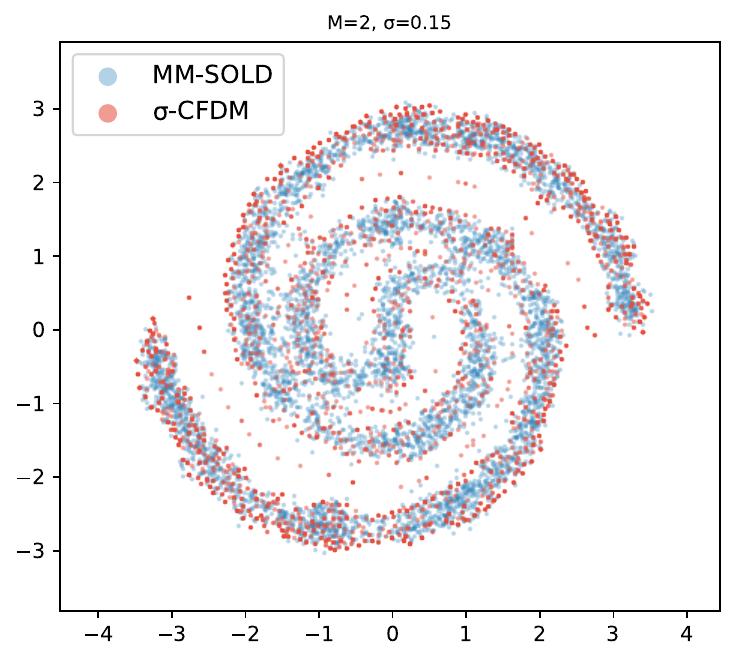} &        
\includegraphics[width=0.242\textwidth, trim=29 24 7 20.5, clip]{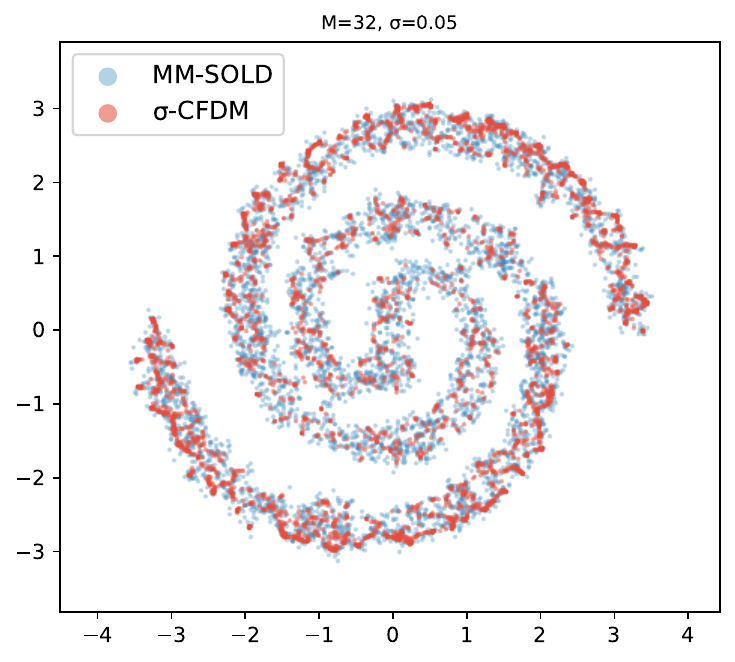} &        
\includegraphics[width=0.242\textwidth, trim=29 24 7 20.5, clip]{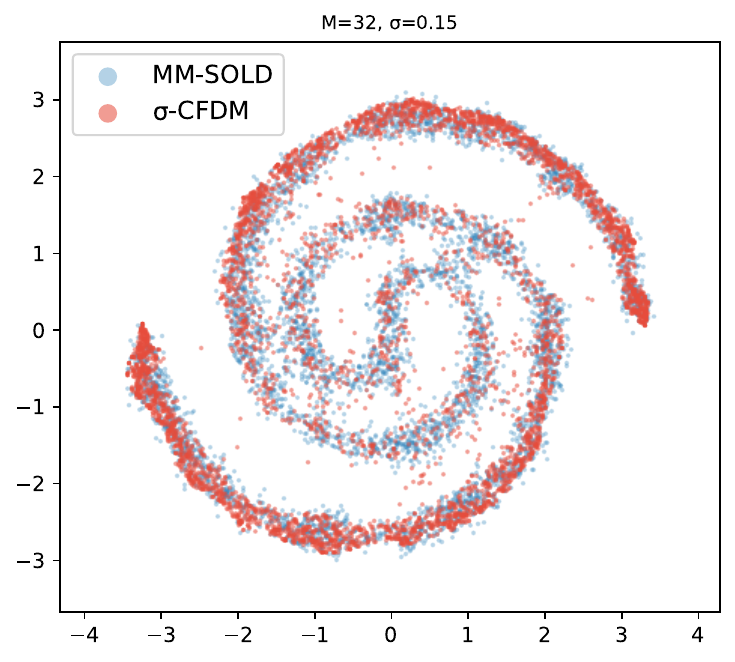} \\[-1mm]
{\scriptsize $M=2,\sigma=0.05$} &
{\scriptsize $M=2,\sigma=0.15$} &
{\scriptsize $M=32,\sigma=0.05$} &
{\scriptsize $M=32,\sigma=0.15$}
\end{tabular}    
\caption{Comparison between MM-SOLD and $\sigma$-CFDM on 2D. Top row: ``Checkerboard''; bottom row: ``Two Spirals''. Blue dots denote MM-SOLD and red dots denote $\sigma$-CFDM.}    
\label{fig:2d_checkerboard_spirals}
\end{figure}

\subsection{Target in the large-particle limit}
A natural question is what distribution MM-SOLD targets at stationarity with sufficient sampling iterations and small step sizes. For unconstrained overdamped Langevin dynamics \eqref{eq:unconstrained_OLD}, its stationary law $\hat{p}^{\delta,k,\sigma}$ can be characterized as the minimizer of the free energy \cite{jordan1998variational,pavliotis2014stochastic}:
\begin{equation}
\label{eq:free_energy}
\mathcal F(\rho)=\int V(z)\rho(z)\mathrm dz+\int \rho(z)\log\rho(z)\mathrm dz.
\end{equation}
The large-particle equilibrium as $P\to\infty$ is described by the same free energy restricted to the GMM moment constraints:
\begin{equation}
\mathcal C(\mu^*,\Sigma^*)
:=
\left\{
\rho:\int \rho(z)\mathrm dz=1,\;
\int z\rho(z)\mathrm dz=\mu^*,\;
\int (z-\mu^*)(z-\mu^*)^\top\rho(z)\mathrm dz=\Sigma^*
\right\}.
\end{equation}

\begin{proposition}[Moment-matched limiting target]
\label{prop:large_particle_target}
Assume that the score smoothed GMM potential is defined with
\(\delta>0\), \(\sigma\ge0\), and Gaussian smoothing noise
\(\varepsilon\sim\mathcal N(0,I_d)\). With the GMM moments
$(\mu^*,\Sigma^*)$ defined above, for each \(P\ge d+1\), let \(\mu_P\) denote the
ideal constrained Gibbs equilibrium on \(\mathcal M_P\) defined in
\eqref{eq:finite_P_constrained_gibbs_main}. Let \(\mu_P^{Z,(1)}\) be its
one-particle marginal mapped back to \(Z\)-space by
$z_i=\mu^*+L^*y_i$.
Then, as \(P\to\infty\), $\mu_P^{Z,(1)}\Rightarrow
\rho_N$,
where
\begin{equation}
\rho_N
=
\arg\min_{\rho\in\mathcal C(\mu^*,\Sigma^*)}
\mathcal F(\rho).
\label{eq:constrained_free_energy}
\end{equation}
Equivalently,
\begin{equation}
\rho_N(z)
=
\frac1{Z(\lambda,\Lambda)}
\exp\left(
-V(z)-\lambda^\top(z-\mu^*)
-\frac12(z-\mu^*)^\top\Lambda(z-\mu^*)
\right),
\label{eq:moment_matched_target}
\end{equation}
for some \(\lambda\in\mathbb R^d\) and symmetric
\(\Lambda\in\mathrm{Sym}_d\). With \(g(z):=\nabla V(z)\) and
\(\sym(A):=\frac12(A+A^\top)\), the tilting parameters satisfy
\begin{equation}
\lambda
=
-\mathbb E_{\rho_N}[g(Z)],
\qquad
\Sigma^*\Lambda+\Lambda\Sigma^*
=
2\left(I_d-\sym(C)\right),
\quad
C:=
\mathbb E_{\rho_N}
\left[(Z-\mu^*)g(Z)^\top\right].
\label{eq:tilting_parameter_identities}
\end{equation}
When $\sigma=0$, $\rho_N=\hat p^\delta$ and $\lambda=\Lambda=0$.
\end{proposition}

\paragraph{Sampling and statistical guarantees.}
The proof of Proposition~\ref{prop:large_particle_target} is in Appendix~\ref{app:proof_large_particle_target}. Write $p_N=\hat\pi_{\rm data}$ for the empirical data law and $\rho_p$ for the population counterpart of $\rho_N$, when it exists. For a sampler output $\widehat Z$,
\begin{equation}
\label{eq:main-error-decomposition}
W_1(\mathcal L(\widehat Z),\pi_{\rm data})
\le \underbrace{W_1(\mathcal L(\widehat Z),\rho_N)}_{\text{algorithmic error}}
+\underbrace{W_1(\rho_N,\rho_p)}_{\text{finite-data error}}
+\underbrace{W_1(\rho_p,\pi_{\rm data})}_{\text{population bias}}.
\end{equation}
Proposition~\ref{prop:approx_mmsold} makes the first term arbitrarily small with sufficiently many particles ($P\ge d+2$), long simulation, and small steps for QR--Euler--Maruyama with the full antithetic score estimator and fixed even Monte Carlo budget, uniformly over initial laws. For bounded data, Proposition~\ref{prop:finite_sample_target_bound} bounds the second term locally by $C W_1(p_N,\pi_{\rm data})$, giving dimension-dependent empirical Wasserstein rates. Proposition~\ref{prop:population_projection_bias} proves KL improvement relative to the Gaussian-smoothed data law and bounds population bias by a score-smoothing term plus at most $\delta\sqrt d$. This includes singular data laws, but need not give vanishing bias at fixed bandwidths. The results give no complexity bound or guarantee for LM with a fixed nearest-neighbor budget.

Since the naive score-smoothed target is $\hat{p}^{\delta,k,\sigma}(z)\propto e^{-V(z)}$, \eqref{eq:constrained_free_energy} equivalently minimizes $\mathrm{KL}(\rho\,\|\,\hat{p}^{\delta,k,\sigma})$ over $\mathcal C(\mu^*,\Sigma^*)$---so $\rho_N$ is the linear-and-quadratic exponential tilting of the naive target closest to it in KL among all moment-matched distributions.

The tilting parameters in~\eqref{eq:tilting_parameter_identities} involve expectations under the unknown moment-matched target. In practice, we consider two estimators, detailed in Appendix~\ref{app:tilting-parameter-estimation}. A computationally cheap plug-in estimator replaces these expectations by empirical averages over the training set:
$\widehat\lambda=-\tfrac{1}{N}\sum_{i=1}^N g(x_i)$ and
$\widehat C=\tfrac{1}{N}\sum_{i=1}^N(x_i-\mu^*)g(x_i)^\top$, with $\widehat\Lambda$ obtained from
\((\Sigma^*+\zeta I_d)\widehat\Lambda+\widehat\Lambda(\Sigma^*+\zeta I_d)
= 2(I_d-\sym(\widehat C))\)
where $\zeta>0$ is used for numerical stability. This estimator is generally biased. A consistent alternative minimizes the convex dual objective
\begin{equation}
\mathcal D(\lambda,\Lambda)
=
\log Z(\lambda,\Lambda)
+
\frac{1}{2}\operatorname{tr}(\Lambda\Sigma^*), \, Z(\lambda,\Lambda)
=
\int
e^{
-V(z)
-\lambda^\top(z-\mu^*)
-\frac{1}{2}(z-\mu^*)^\top\Lambda(z-\mu^*)
}
dz.
\label{eq:appendix-dual-objective-main}
\end{equation}
Proposition~\ref{prop:baoab-dual-consistency} proves almost-sure convergence of a projected dual estimator on a compact normalizable parameter set containing the optimum, using independent BAOAB chains with growing simulation times and sample budgets and shrinking steps. Its output laws converge in $W_2$ to the moment-matched target. This covers exact or full unbiased antithetic forces with the stated force-storage rule, but not the experimental Lion updates or fixed-budget nearest-neighbor scores.

For normalizable estimated tilts, one can run independent chains targeting the explicit density~\eqref{eq:moment_matched_target}; however, this requires estimating the $O(d^2)$-dimensional matrix $\Lambda$, which is difficult in high dimensions. MM-SOLD avoids this estimation problem by enforcing the moments directly at the particle level. We evaluate the plug-in and dual estimators empirically in Section~\ref{exp:handwritten_digit_classification} and in the generation experiments.

\subsection{Nearest-neighbor score estimation}
\label{sec:nn_score_estimation}
Evaluating the negative smoothed score $g(z)=\nabla V(z)$ requires a Monte Carlo estimate of $\nabla\log\hat p^\delta$ at perturbed queries $z+\sigma\varepsilon$, and a direct sum over all $N$ training samples per perturbation is prohibitive at image-generation scale. We therefore use a local nearest-neighbor estimator: for each query we keep the $K$ nearest training samples and draw $L$ additional samples uniformly without replacement from the remainder to estimate the full numerator and denominator without bias; their ratio remains biased, following the same high-level idea as $\sigma$-CFDM \cite{scarvelis2023closed}. Two design choices specific to MM-SOLD: (i) antithetic perturbations $z\pm\sigma\varepsilon$ to reduce Monte Carlo variance, and (ii) when $K+L<d$, sampling the smoothing noise only through its inner products with the local samples using their Gram matrix, which matches the ambient-space softmax law while avoiding full $d$-dimensional Gaussian draws. Full algorithm and analysis are in Appendix \ref{app:nn_score}.

\section{Experiments}
\label{sec:Experiments}
    \subsection{Handwritten digits: generation}
\label{exp:handwritten_digit_generation}

We use MM-SOLD to generate digit-8 images in the NRAE latent space from Section~\ref{exp:handwritten_digit_classification}, using $1{,}000$ training images and $300$ held-out test images. All methods operate in a common top-$40$ PCA subspace of the $100$-dimensional NRAE latent space, explaining more than $99.8\%$ of the pooled within-class training variance; generated latents are lifted back to $100$ dimensions by setting the discarded $60$ PCA coordinates to the digit-8 sample mean before decoding.

We compare with $\sigma$-CFDM~\cite{scarvelis2023closed}, KNN-PFlow~\cite{niedoba2024nearestneighbourscoreestimators}, Sliced-Wasserstein Flow (SWF)~\cite{liutkus2019sliced}, plug-in and dual (SG) BAOAB samplers of~\eqref{eq:moment_matched_target}, and a latent DDPM~\cite{rombach2022high}. Smoothed-score methods use $\delta=c_\delta\mathrm{NN}/\sqrt{40}$, $\sigma=c_\sigma\mathrm{NN}$, and search $c_\delta,c_\sigma\in\{0.10,0.20,\ldots,1.0\}$, $M\in\{2,32\}$, where $\mathrm{NN}$ is the training nearest-neighbor scale.

MM-SOLD uses $1{,}000$ LM steps with $h=0.2\delta^2$. We measure fidelity by polynomial-kernel KID~\cite{sutherland2018demystifying} in digit-classifier features, diversity by Recall, and proximity to training latents by DupRate. Each configuration generates $300$ samples, with mean and standard deviation over five runs. Full baseline settings, grids, metric definitions, and timing procedures are in Appendix~\ref{app:handwritten_generation_details}. Reported costs exclude the shared pretrained autoencoder and decoding.

Figure~\ref{fig:digit8_generation_samples} shows samples at $(c_\delta,c_\sigma,M)=(0.4,0.6,32)$. Table~\ref{tab:digit8_generation_results} reports \emph{separate metric-wise best} configurations: MM-SOLD has the lowest reported KID and highest Recall, and a zero-DupRate configuration. Selection reuses the test set, so these comparisons are exploratory. The optima need not occur jointly or at the separately timed configuration. KNN-PFlow has DupRate $1.000$, while SWF avoids this duplication but has the worst KID. The BAOAB results support directly sampling an estimated tilted target; MM-SOLD avoids the associated parameter optimization. Appendix~\ref{app:handwritten_generation_details} includes full-grid comparisons and moment diagnostics.

\begin{figure}[t]
\centering
\setlength{\tabcolsep}{2pt}
\renewcommand{\arraystretch}{0.9}
\begin{tabular}{cccc}
\includegraphics[width=0.24\textwidth]{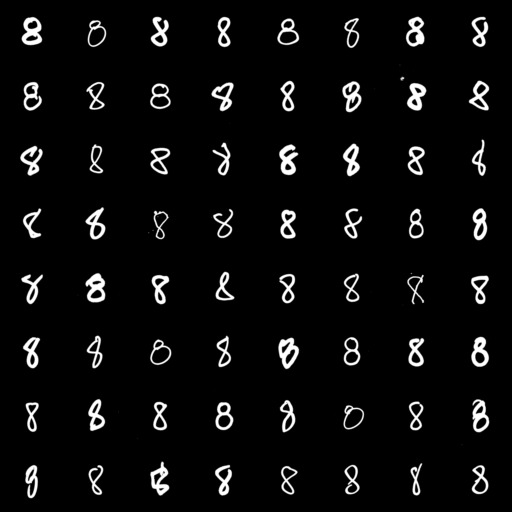} &
\includegraphics[width=0.24\textwidth]{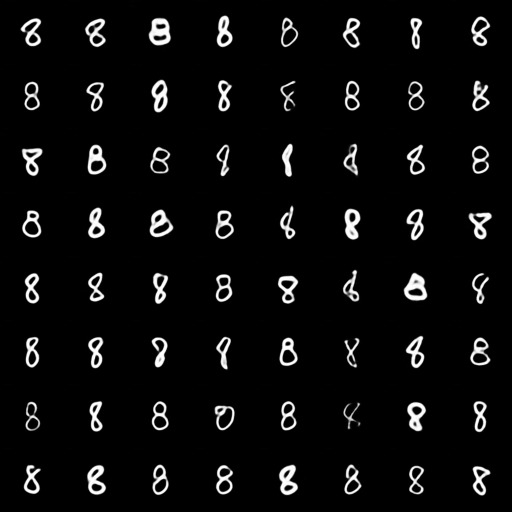} &
\includegraphics[width=0.24\textwidth]{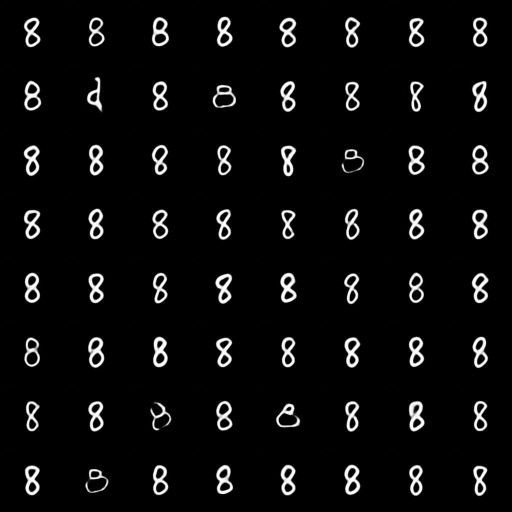} &
\includegraphics[width=0.24\textwidth]{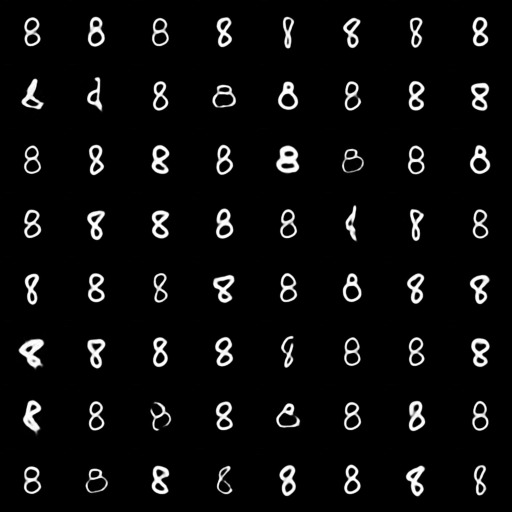} \\
{\scriptsize Real samples} &
{\scriptsize MM-SOLD} &
{\scriptsize Plug-in BAOAB} &
{\scriptsize SG-BAOAB} \\[3pt]
\includegraphics[width=0.24\textwidth]{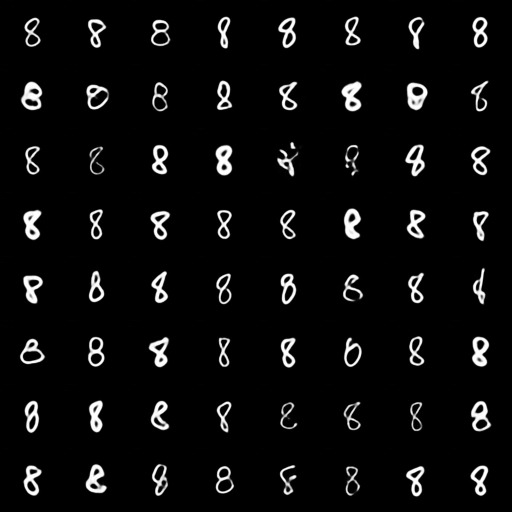} &
\includegraphics[width=0.24\textwidth]{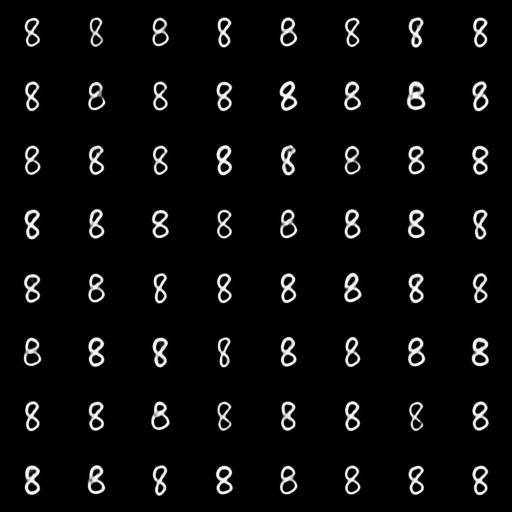} &
\includegraphics[width=0.24\textwidth]{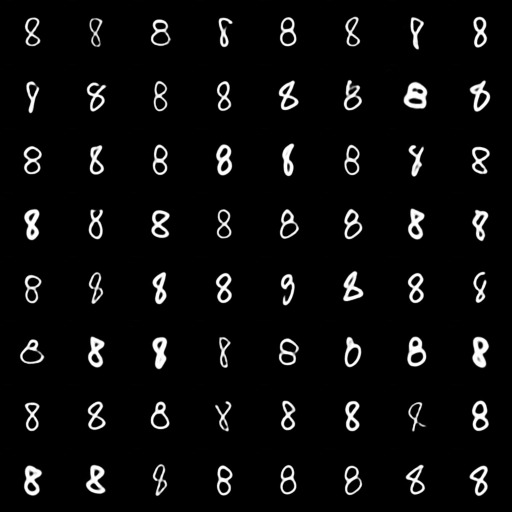} &
\includegraphics[width=0.24\textwidth]{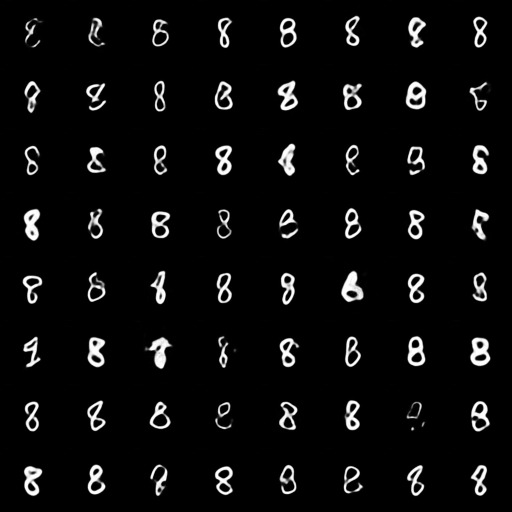} \\
{\scriptsize Latent DDPM} &
{\scriptsize $\sigma$-CFDM} &
{\scriptsize KNN-PFlow} &
{\scriptsize SWF}
\end{tabular}
\caption{Representative digit-8 samples. All samples are produced in the top-$40$ NRAE latent subspace and decoded after setting the discarded $60$ PCA coordinates to the digit-8 sample mean.}
\label{fig:digit8_generation_samples}
\end{figure}

\begin{table}[t]
\centering
\caption{Generation quality and computational cost on handwritten digit-8. For MM-SOLD, $\sigma$-CFDM, Plug-in BAOAB, and SG-BAOAB, we report the best result over the searched $(c_\delta,c_\sigma,M)$ configurations for each metric.}
\label{tab:digit8_generation_results}
\scriptsize
\setlength{\tabcolsep}{2.0pt}
\renewcommand{\arraystretch}{0.88}
\begin{tabular}{lccccccc}
\toprule
Metric
& MM-SOLD
& $\sigma$-CFDM
& KNN-PFlow
& SWF
& Plug-in BAOAB
& SG-BAOAB
& Latent DDPM \\
\midrule
KID $\downarrow$
& \textbf{0.0305$\pm$0.0090} & 0.0465$\pm$0.0119 & 0.0421$\pm$0.0120 & 0.2015$\pm$0.0192 & 0.0439$\pm$0.0093 & 0.0431$\pm$0.0155 & 0.1025$\pm$0.0234 \\
Recall $\uparrow$
& \textbf{0.881$\pm$0.010} & 0.865$\pm$0.018 & 0.869$\pm$0.016 & 0.817$\pm$0.017 & 0.875$\pm$0.011 & 0.876$\pm$0.018 & 0.859$\pm$0.024 \\
DupRate $\downarrow$
& \textbf{0.000$\pm$0.000} & 0.001$\pm$0.003 & 1.000$\pm$0.000 & \textbf{0.000$\pm$0.000} & \textbf{0.000$\pm$0.000} & \textbf{0.000$\pm$0.000} & 0.013$\pm$0.006 \\
Train/opt. time $\downarrow$
& \textbf{0 h} & \textbf{0 h} & \textbf{0 h} & \textbf{0 h} & \textbf{0 h} & 0.004 h & 0.21 h \\
Sample time $\downarrow$
& 17.91 ms & 37.46 ms & 23.75 ms & 22.98 ms & \textbf{15.97 ms} & 16.16 ms & 16.44 ms \\
Device
& CPU & CPU & CPU & CPU & CPU & CPU & RTX 6000 \\
\bottomrule
\end{tabular}
\end{table}

\subsection{Handwritten digits: classification}
\label{exp:handwritten_digit_classification}

We test whether the moment-matched target density \eqref{eq:moment_matched_target} faithfully approximates the data distribution in a high-dimensional setting, using a minimum expected cost of misclassification (ECM) \cite{bishop2006pattern} experiment as a benchmark. With equal misclassification costs and a uniform class prior, a minimum ECM classifier picks the class with the largest class-conditional density. We use a $10$-class handwritten digits dataset \cite{beaulac2022introducing} encoded into $100$-dimensional latents by a pretrained Nuclear Norm-Regularized AutoEncoder (NRAE) \cite{scarvelis2024nuclear}, with $1{,}000/58/300$ train/validation/test latents per class.

For each digit class $c$, we fit two moment-matched class-conditional densities of the form~\eqref{eq:moment_matched_target} in the $100$-dimensional NRAE latent space, using either the plug-in or dual tilting estimator from Appendix~\ref{app:tilting-parameter-estimation}. For both estimators, we calibrate a validation-set bias $b_c$ to absorb the unknown normalizing constant and classify by $\hat y(z)=\arg\min_c\{E_c(z)+b_c\}$, where $E_c(z)=V_c(z)+\lambda_c^\top(z-\mu_c^*)+\tfrac12(z-\mu_c^*)^\top\Lambda_c(z-\mu_c^*)$.

We compare with two MLP baselines in the same latent space: one trained on the original $1{,}000$ latents per class, and one trained after appending $9{,}000$ MM-SOLD samples per class. For all moment-matched ECM classifiers and MM-SOLD augmentation, we fix $M=32$, set $\delta_c=c_\delta\mathrm{NN}_c/\sqrt{100}$ and $\sigma_c=c_\sigma\mathrm{NN}_c$, where $\mathrm{NN}_c$ is the class-$c$ average nearest-neighbor distance, and select classifier scales on the validation grid $\{0.05,0.10,\ldots,0.50\}^2$. Augmentation reuses the dual classifier's scales and uses step size $h=0.2\delta_c^2$. Full preprocessing, architecture, training, and dual-estimator details are in Appendix~\ref{app:handwritten_digit_details}.

Table~\ref{tab:handwritten_digit_classification} shows that both moment-matched ECM classifiers substantially improve over the original-data MLP, increasing overall accuracy from $93.77\%$ to $97.20\%$ with plug-in tilting and $97.30\%$ with dual tilting. MM-SOLD augmentation gives the best accuracy, $97.43\%$, showing that its samples are also useful for discriminative data augmentation.

\begin{table}[t]
\centering
\caption{Handwritten digit classification in the $100$-dimensional NRAE latent space. We report per-class and overall test accuracy (\%). MLP+Aug. denotes MLP trained with MM-SOLD augmented samples, and MM-ECM denotes the minimum ECM classifier based on~\eqref{eq:moment_matched_target}.}
\label{tab:handwritten_digit_classification}
\footnotesize
\setlength{\tabcolsep}{2.2pt}
\renewcommand{\arraystretch}{0.92}
\begin{tabular}{@{}lccccccccccc@{}}
\toprule
Method & 0 & 1 & 2 & 3 & 4 & 5 & 6 & 7 & 8 & 9 & Overall \\
\midrule
MLP
& 96.7 & 96.7 & 88.3 & 96.3 & 95.7 & 92.0 & 96.0 & 89.3 & 94.7 & 92.0 & 93.77 \\
MLP+Aug.
& 97.7 & \textbf{98.7} & 97.3 & 98.0 & 96.3 & 98.0 & \textbf{99.0} & 96.3 & 96.3 & \textbf{96.7} & \textbf{97.43} \\
MM-ECM (plug-in)
& \textbf{99.7} & 93.3 & \textbf{97.7} & \textbf{98.3} & 93.7 & 98.3 & 98.7 & \textbf{98.7} & 97.0 & \textbf{96.7} & 97.20 \\
MM-ECM (dual)
& 98.7 & 94.0 & 97.0 & 95.7 & \textbf{96.7} & \textbf{99.7} & 97.7 & \textbf{98.7} & \textbf{99.3} & 95.7 & 97.30 \\
\bottomrule
\end{tabular}
\end{table}

\subsection{CelebA-HQ: generation}
\label{exp:celebAHQ_generation}

We further evaluate MM-SOLD on CelebA-HQ~\cite{karras2017progressive} at $256\times256$ resolution, using $27{,}000$ training and $3{,}000$ held-out test images. Images are encoded by a fixed pretrained Diffusion Autoencoder (DiffAE)~\cite{preechakul2022diffusion} into $512$-dimensional latents. All training-free and BAOAB-based samplers operate in an unwhitened top-$190$ PCA subspace explaining $90.1\%$ of the training latent variance; generated top-$190$ latents are completed by independently sampling the discarded $322$ PCA coordinates from their empirical Gaussian tail, adding back the global mean, and decoding.

We use the same six baselines as for digits; the DDPM operates in the full $512$-dimensional latent space. Smoothed-score methods use $\delta=c_\delta\mathrm{NN}/\sqrt{190}$ and $\sigma=c_\sigma\mathrm{NN}$, with the same grid as for digits. MM-SOLD uses $1{,}000$ LM steps, $h=0.2\delta^2$, and nearest-neighbor budgets $K=L=50$. In this experiment its prescribed covariance includes the GMM component variance, $\widehat\Sigma_{190}+\delta^2I_{190}$. All methods share the fixed DiffAE decoder, using $20$-step DPM-Solver++(2M)~\cite{lu2025dpm}. Appendix~\ref{app:celebahq_generation_details} gives all baseline and optimization settings.

We select configurations by RBF-mixture MMD$^2$ in raw DiffAE latent space, then compute Inception-space FID~\cite{heusel2017gans}, polynomial-kernel KID, and corresponding ArcFace metrics~\cite{deng2019arcface} after decoding. RBF MMD, Recall, and DupRate provide further comparisons. Table~\ref{tab:celebahq_generation_results} gives mean $\pm$ standard deviation over three runs of $3{,}000$ samples; kernel statistics are scaled by $10^3$. The same held-out set is used for selection and evaluation, so these grid-selected comparisons are exploratory. Appendix~\ref{app:celebahq_generation_details} specifies kernels, reference floors, and the interpretation of negative MMD$^2$ values. Timings cover latent sampling, excluding autoencoder training and decoding; CPU and GPU times use different hardware.

Figure~\ref{fig:celebahq_generation_samples} shows samples at $(c_\delta,c_\sigma,M)=(0.1,0.4,32)$. KNN-PFlow attains the best Inception-space metrics but DupRate $1.0000$. Among the methods with low measured duplication, MM-SOLD attains competitive image metrics without score training or tilting optimization, while SG-BAOAB improves over plug-in estimation. Low DupRate rules out proximity under the stated threshold, rather than proving absence of memorization. Additional grids, nearest-neighbor comparisons, moment diagnostics, and ablations appear in Appendices~\ref{app:celebahq_generation_details} and~\ref{app:ablation_studies}.

\begin{figure}[t]
\centering
\setlength{\tabcolsep}{2pt}
\renewcommand{\arraystretch}{0.9}
\begin{tabular}{cccc}
\includegraphics[width=0.24\textwidth]{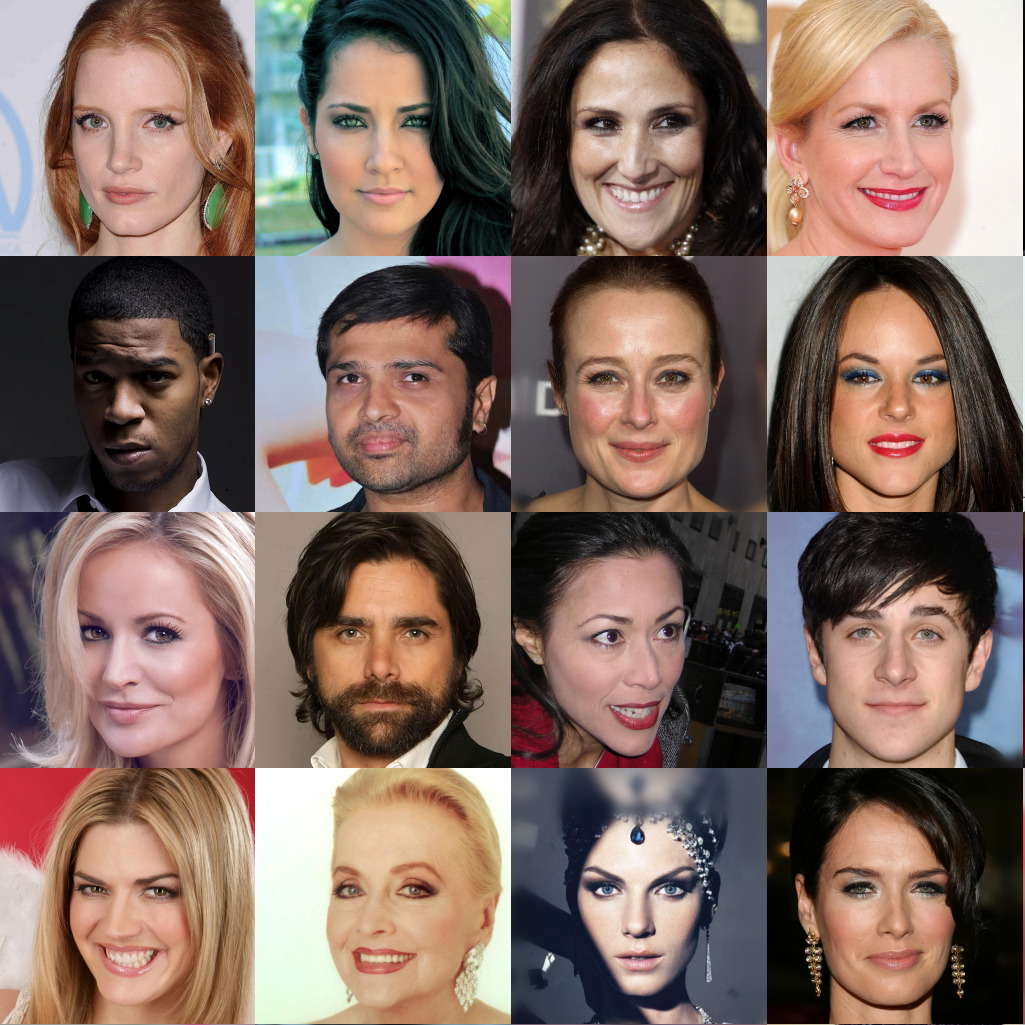} &
\includegraphics[width=0.24\textwidth]{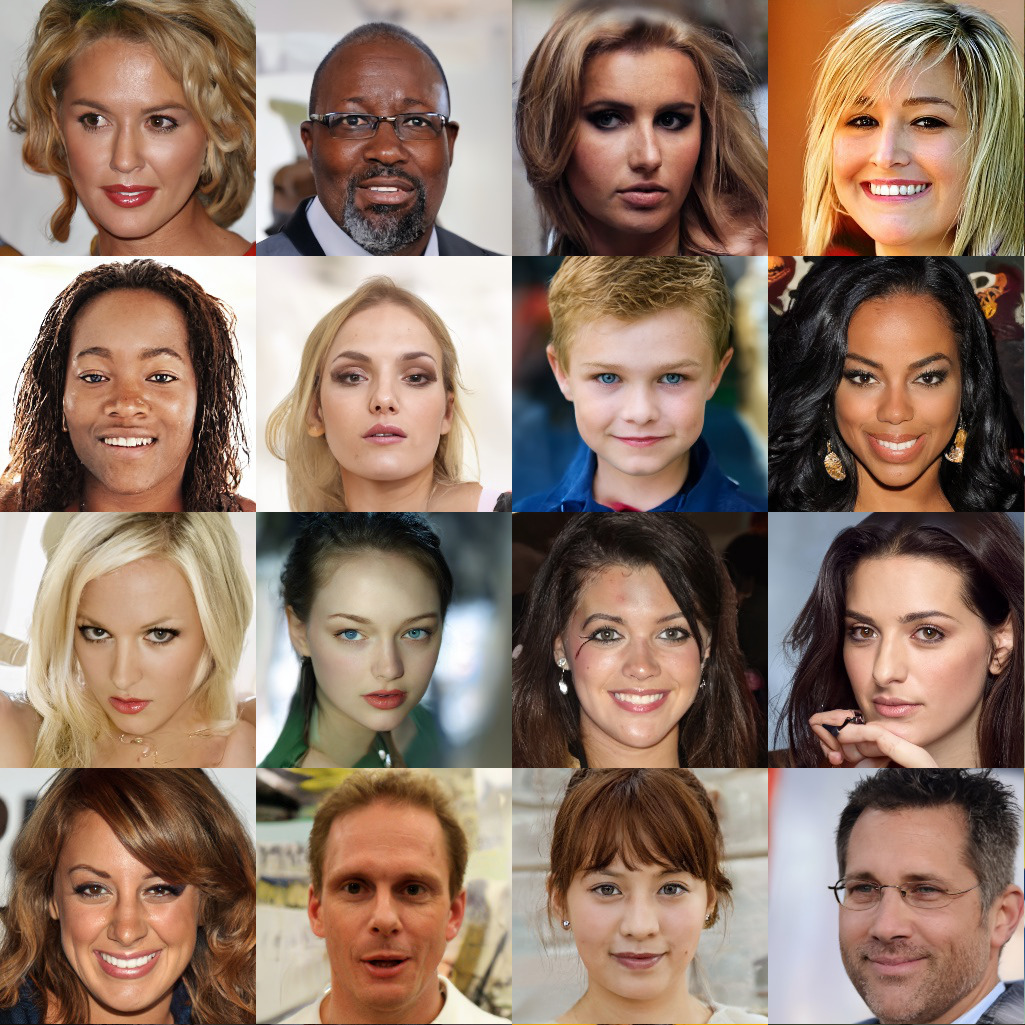} &
\includegraphics[width=0.24\textwidth]{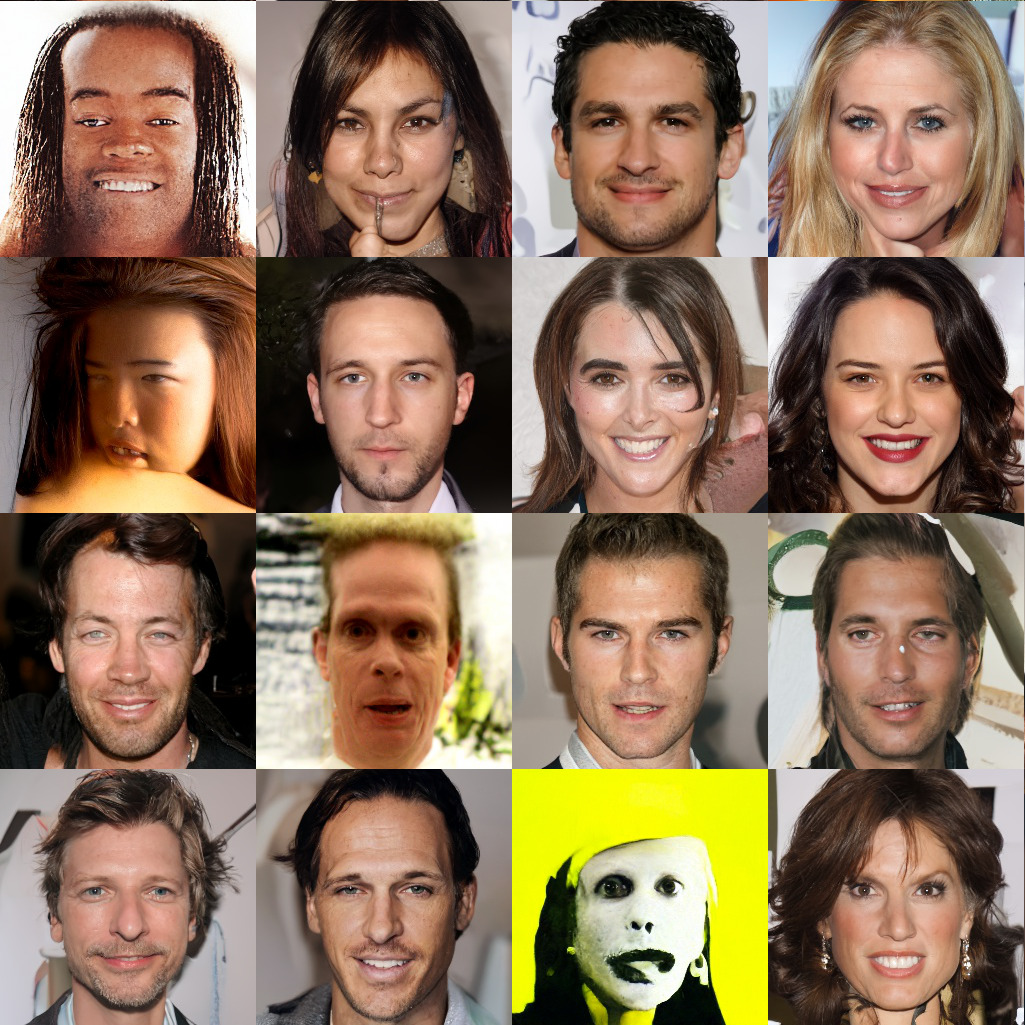} &
\includegraphics[width=0.24\textwidth]{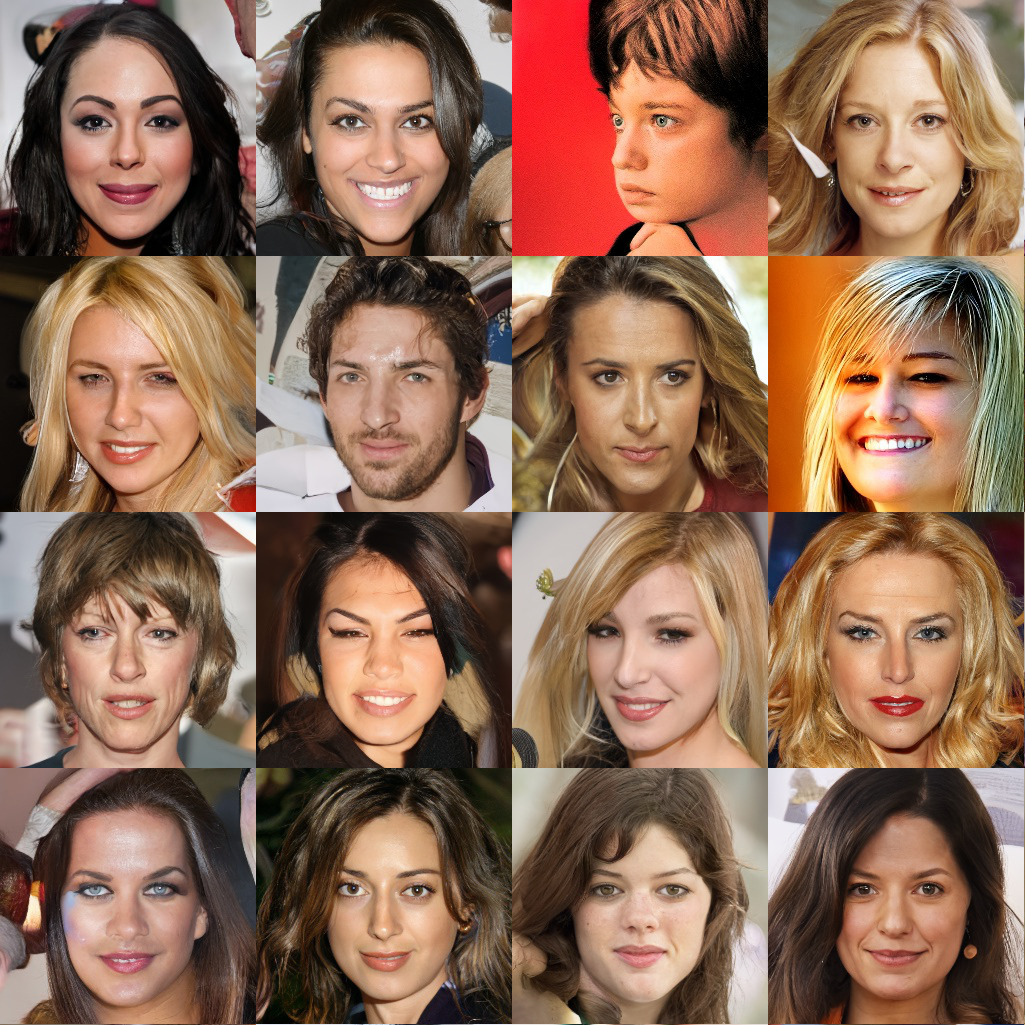} \\
{\scriptsize Real samples} &
{\scriptsize MM-SOLD} &
{\scriptsize Plug-in BAOAB} &
{\scriptsize SG-BAOAB} \\[3pt]
\includegraphics[width=0.24\textwidth]{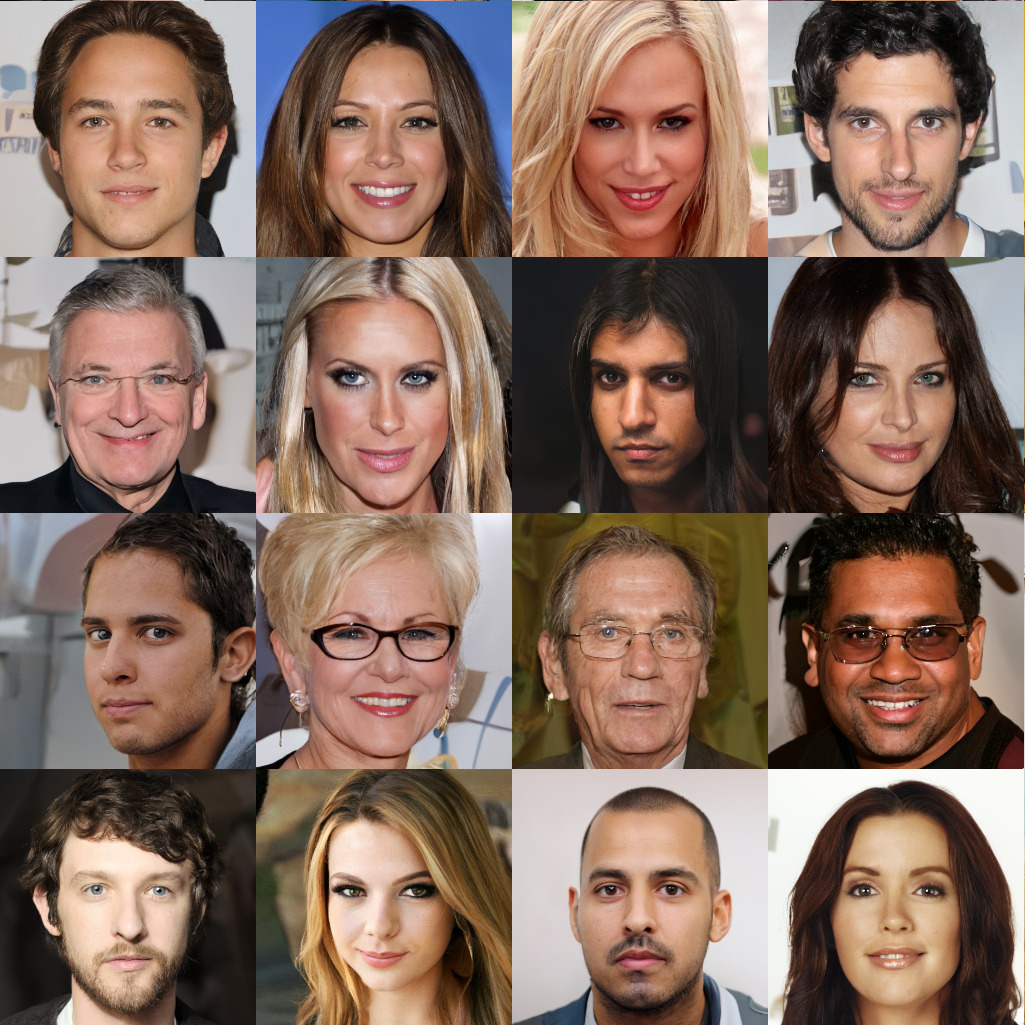} &
\includegraphics[width=0.24\textwidth]{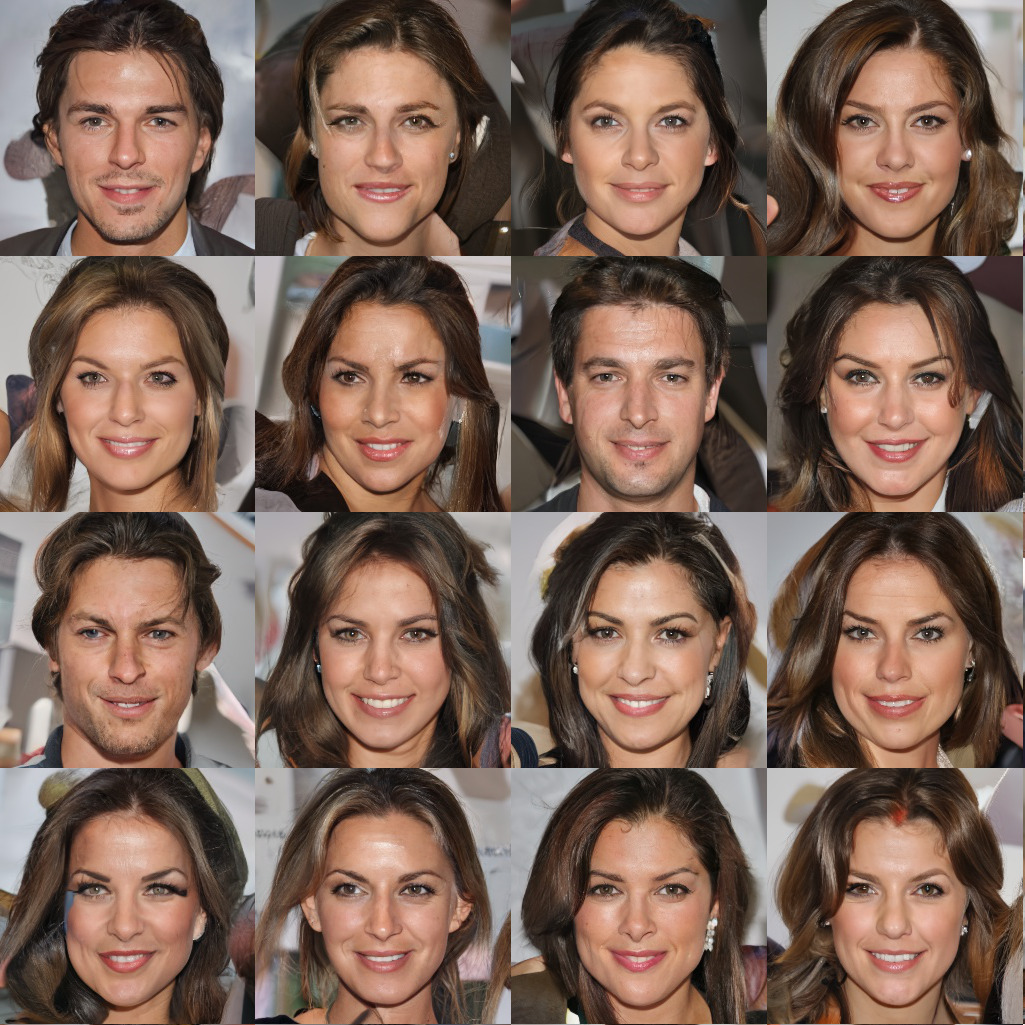} &
\includegraphics[width=0.24\textwidth]{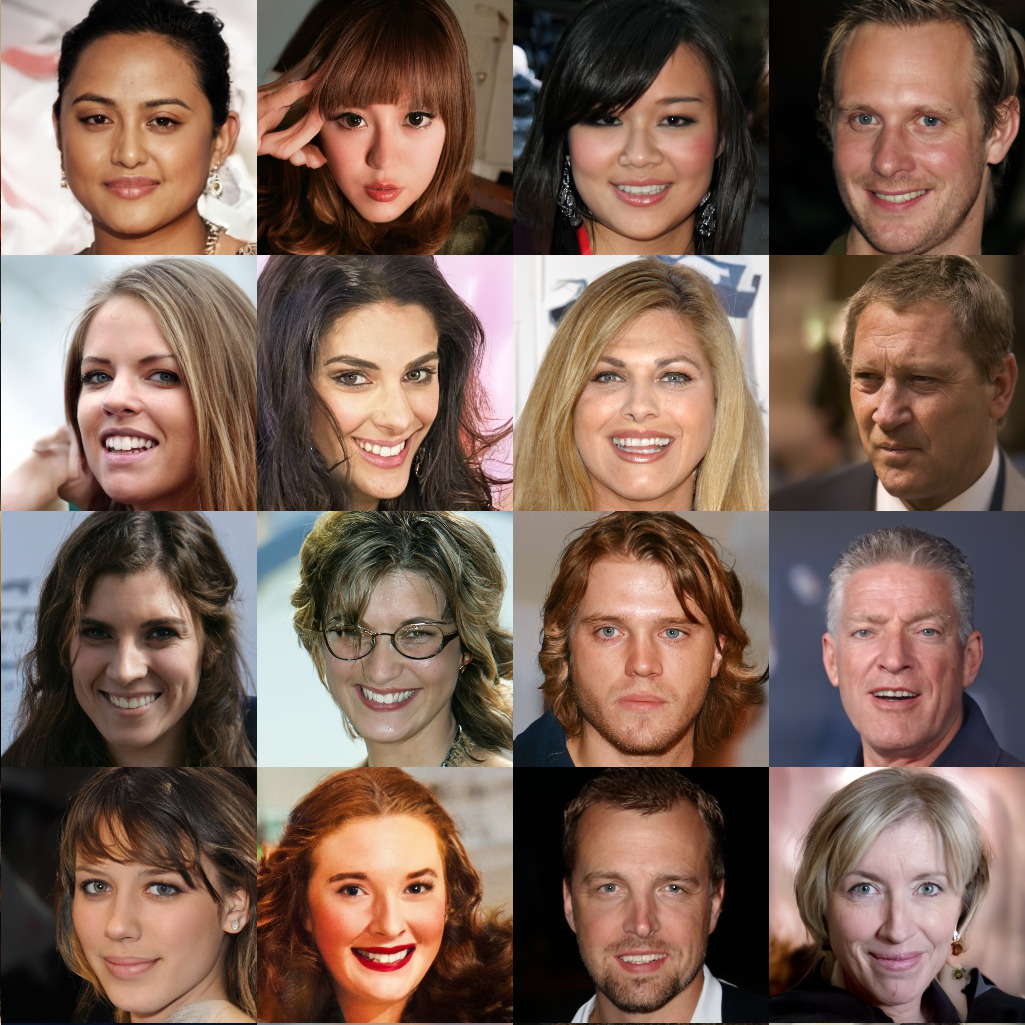} &
\includegraphics[width=0.24\textwidth]{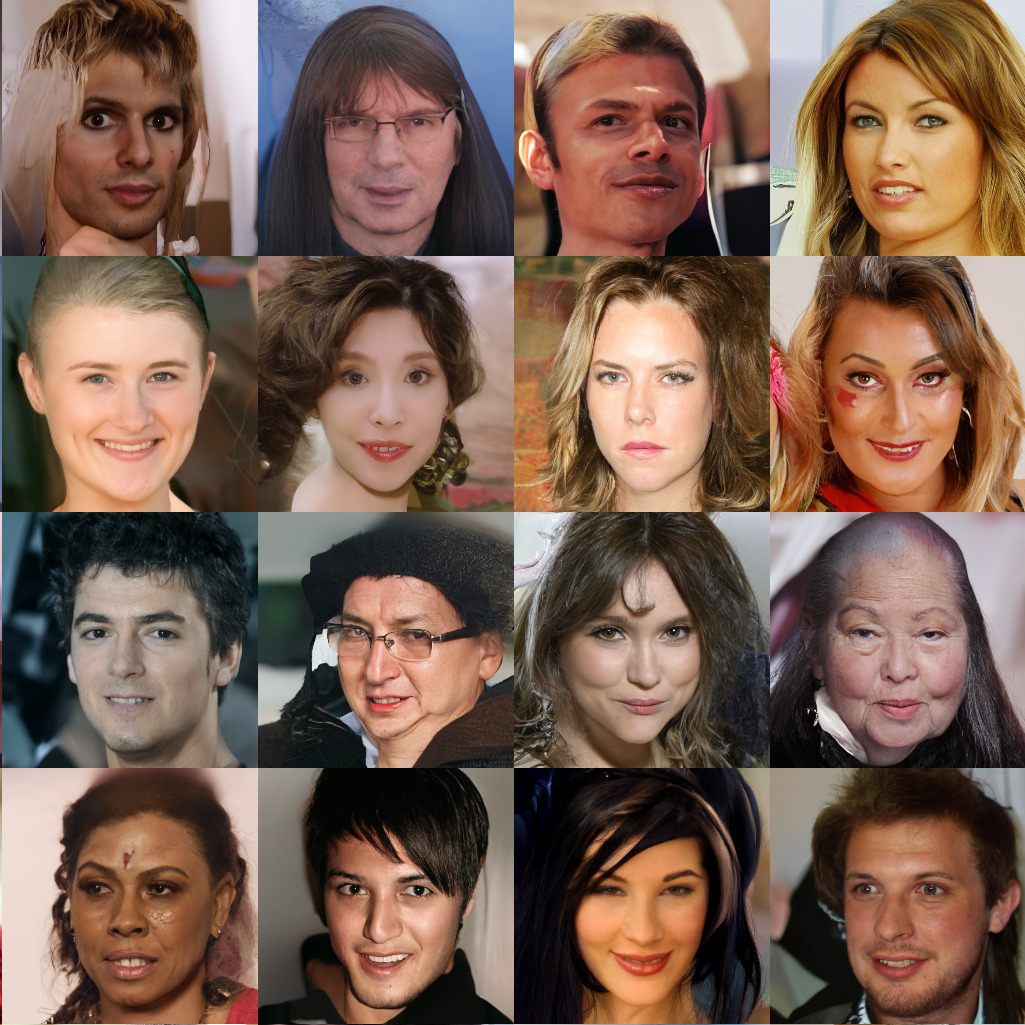} \\
{\scriptsize Latent DDPM} &
{\scriptsize $\sigma$-CFDM} &
{\scriptsize KNN-PFlow} &
{\scriptsize SWF}
\end{tabular}
\caption{Representative CelebA-HQ samples. Training-free and BAOAB samples use the unwhitened top-$190$ DiffAE subspace, with discarded $322$ PCA coordinates sampled from an independent Gaussian tail before $20$-step DPM-Solver++(2M) decoding.}
\label{fig:celebahq_generation_samples}
\end{figure}

\begin{table}[t]
\centering
\caption{Generation quality and cost on CelebA-HQ. Image-space and ArcFace FID/KID are computed only for the latent-MMD$^2$-selected configuration, after decoding $3{,}000$ generated latents with the fixed DiffAE decoder using $20$-step DPM-Solver++(2M).}
\label{tab:celebahq_generation_results}
\scriptsize
\setlength{\tabcolsep}{1.8pt}
\renewcommand{\arraystretch}{0.88}
\resizebox{\textwidth}{!}{%
\begin{tabular}{@{}lccccccc@{}}
\toprule
Metric
& MM-SOLD
& $\sigma$-CFDM
& KNN-PFlow
& SWF
& Plug-in BAOAB
& SG-BAOAB
& Latent DDPM \\
\midrule
FID (image) $\downarrow$
& 12.439$\pm$0.033 & 13.328$\pm$0.087 & \textbf{12.335$\pm$0.036} & 25.199$\pm$0.328 & 12.804$\pm$0.140 & 12.626$\pm$0.132 & 16.568$\pm$0.089 \\
KID (image, $\times10^3$) $\downarrow$
& 2.17$\pm$0.18 & 2.92$\pm$0.09 & \textbf{1.75$\pm$0.13} & 12.36$\pm$0.26 & 2.37$\pm$0.17 & 2.17$\pm$0.16 & 5.86$\pm$0.09 \\
MMD (image, $\times10^3$) $\downarrow$
& 3.67$\pm$0.02 & 5.59$\pm$0.17 & \textbf{2.90$\pm$0.13} & 18.19$\pm$0.46 & 4.04$\pm$0.11 & 3.68$\pm$0.07 & 10.92$\pm$0.04 \\
FID (ArcFace) $\downarrow$
& 0.0884$\pm$0.0005 & 0.0922$\pm$0.0008 & \textbf{0.0879$\pm$0.0001} & 0.1061$\pm$0.0003 & 0.0889$\pm$0.0003 & 0.0883$\pm$0.0002 & 0.1009$\pm$0.0004 \\
KID (ArcFace, $\times10^3$) $\downarrow$
& 0.0573$\pm$0.0018 & 0.0692$\pm$0.0015 & 0.0574$\pm$0.0010 & 0.0834$\pm$0.0016 & 0.0595$\pm$0.0008 & \textbf{0.0570$\pm$0.0006} & 0.0970$\pm$0.0021 \\
MMD (ArcFace, $\times10^3$) $\downarrow$
& 3.79$\pm$0.10 & 4.45$\pm$0.09 & \textbf{3.71$\pm$0.08} & 5.49$\pm$0.04 & 3.87$\pm$0.04 & 3.72$\pm$0.03 & 6.43$\pm$0.07 \\
MMD$^2$ (latent, $\times10^3$) $\downarrow$
& \textbf{-0.27$\pm$0.01} & 2.20$\pm$0.36 & 0.80$\pm$0.05 & 0.45$\pm$0.04 & -0.01$\pm$0.05 & 0.01$\pm$0.03 & 19.63$\pm$0.31 \\
Recall $\uparrow$
& 0.9932$\pm$0.0006 & \textbf{0.9991$\pm$0.0004} & 0.8386$\pm$0.0067 & 0.3041$\pm$0.0041 & 0.9674$\pm$0.0021 & 0.9983$\pm$0.0003 & 0.9988$\pm$0.0006 \\
DupRate $\downarrow$
& \textbf{0.0000$\pm$0.0000} & \textbf{0.0000$\pm$0.0000} & 1.0000$\pm$0.0000 & \textbf{0.0000$\pm$0.0000} & \textbf{0.0000$\pm$0.0000} & \textbf{0.0000$\pm$0.0000} & \textbf{0.0000$\pm$0.0000} \\
Train/opt. time $\downarrow$
& \textbf{0 h} & \textbf{0 h} & \textbf{0 h} & \textbf{0 h} & \textbf{0 h} & 0.0815 h & 1.309 h \\
Sample time $\downarrow$
& 412.8 ms & 547.6 ms & 147.0 ms & 35.2 ms & 430.7 ms & 437.3 ms & \textbf{0.106 ms} \\
Device
& CPU & CPU & CPU & CPU & CPU & CPU & RTX 6000 \\
\bottomrule
\end{tabular}%
}
\end{table}

\section{Conclusion and discussion}
\label{sec:Conclusion_and_discussion}
    MM-SOLD combines score smoothing with GMM moment constraints to reduce barycentric distortion without training a neural score model. Its large-particle equilibrium is an exponential tilt of the smoothed target, well defined even for rank-deficient training data. Our guarantees separate sampling error, finite-data error, and population bias: they establish full-score QR--Euler sampling consistency, KL improvement relative to the Gaussian-smoothed data law and local finite-sample Wasserstein bounds for bounded data, and almost-sure consistency of safeguarded BAOAB dual estimation. The latter also gives $W_2$ convergence of the sampled laws.

The experiments support the usefulness of moment constraints for latent generation and classification, using fixed pretrained representations. The guarantees do not cover the experimental LM/Lion updates or fixed-budget nearest-neighbor scores, provide practical complexity bounds, or remove population smoothing bias. Exact covariance matching requires $P\ge d+1$; GMM covariance is automatically positive definite. The statistical bounds include bounded manifold-supported data but retain ambient-dimension rates. Extending the implementation guarantees, obtaining intrinsic-dimension rates, and improving fidelity on complex images remain open. The comparisons use searched configurations and shared pretrained decoders, so they do not establish end-to-end training or hardware savings over neural image generators.

\begin{ack}
Daniel Paulin was supported by a Nanyang Technological University Start-up Grant, project number: 024968-00001. The authors declare no competing interests.
\end{ack}
\bibliographystyle{unsrtnat}
\bibliography{References}

\clearpage
\appendix
\section{Symbols and notation}
\label{app:symbols}

Table~\ref{tab:symbols} summarizes the notation used most frequently.
Vectors are columns; particle matrices store their transposes as rows.
Symbols specific to individual proofs or experiments are defined locally.

\begin{table}[H]
\centering
\caption{Frequently used symbols and their definitions.}
\label{tab:symbols}
\small
\setlength{\tabcolsep}{6pt}
\renewcommand{\arraystretch}{1.18}
\begin{tabular}{@{}>{\raggedright\arraybackslash}p{0.23\textwidth}%
  >{\raggedright\arraybackslash}p{\dimexpr0.77\textwidth-2\tabcolsep\relax}@{}}
\toprule
Symbol & Definition \\
\midrule
$x_i$, $N$, $d$ & Training vectors $x_1,\ldots,x_N\in\mathbb R^d$, number of training samples, and sampling dimension. \\
$\pi_{\rm data}$, $p$, $p_N$ & Data law $p=\pi_{\rm data}$ and empirical law $p_N=\hat\pi_{\rm data}=N^{-1}\sum_i\delta_{x_i}$, where $\delta_x$ is a point mass at $x$. \\
$\delta$, $\hat p^\delta$ & GMM component standard deviation and training density $\hat p^\delta(z)=N^{-1}\sum_i\mathcal N(z\mid x_i,\delta^2I_d)$. \\
$\sigma$, $\varepsilon$ & Score-smoothing bandwidth and perturbation $\varepsilon\sim\mathcal N(0,I_d)$. \\
$M$ & Number of Monte Carlo perturbations per score estimate; antithetic estimation uses $M/2$ independent draws and their negatives. \\
$w_i(z)$, $c(z)$ & Posterior mixture weights $w_i(z)=e^{-\|z-x_i\|^2/(2\delta^2)}/\sum_j e^{-\|z-x_j\|^2/(2\delta^2)}$ and weighted center $c(z)=\sum_i w_i(z)x_i$. \\
$V$, $g$ & Smoothed potential $V(z)=-\mathbb E_\varepsilon\log\hat p^\delta(z+\sigma\varepsilon)$ and negative smoothed score $g(z)=\nabla V(z)$. \\
$q$ & Unconstrained target $q(z)=e^{-V(z)}/\int e^{-V(u)}\,du$, also denoted $\hat p^{\delta,k,\sigma}$. \\
\midrule
$\mu^*$, $\widehat\Sigma_N$, $\Sigma^*$ & Mean $\mu^*=N^{-1}\sum_i x_i$, center covariance $\widehat\Sigma_N=N^{-1}\sum_i(x_i-\mu^*)(x_i-\mu^*)^\top$, and prescribed GMM covariance $\Sigma^*=\widehat\Sigma_N+\delta^2I_d\succ0$. \\
$L^*$ & Cholesky factor of the target covariance: $\Sigma^*=L^*(L^*)^\top$. \\
$\rho_N$, $\rho_p$ & Moment-matched targets constructed from the empirical and population data laws, respectively; $\rho_N$ is the large-particle limit. \\
$\lambda$, $\Lambda$, $Z(\lambda,\Lambda)$ & Linear tilt, symmetric quadratic tilt, and normalizing integral in $\rho_N(z)=Z(\lambda,\Lambda)^{-1}\exp\{-V(z)-\lambda^\top(z-\mu^*)-\tfrac12(z-\mu^*)^\top\Lambda(z-\mu^*)\}$. \\
$\theta$, $\mathcal D$ & Centered parameter pair $\theta=(\lambda,\Lambda)$ and convex dual $\mathcal D(\theta)=\log Z(\lambda,\Lambda)+\tfrac12\operatorname{tr}(\Lambda\Sigma^*)$, minimized to match the prescribed moments. \\
\midrule
$P$, $Z$, $Y$ & Particle count and matrices $Z,Y\in\mathbb R^{P\times d}$ with rows $z_i^\top,y_i^\top$, related by $z_i=\mu^*+L^*y_i$. \\
$\mathcal M_P$ & Centered, whitened moment manifold $\{Y\in\mathbb R^{P\times d}:\mathbf1_P^\top Y=0,\ Y^\top Y=PI_d\}$, where $\mathbf1_P$ is the all-ones vector. \\
$\widetilde V$, $F_P$, $\mu_P$ & Whitened potential $\widetilde V(y)=V(\mu^*+L^*y)$, total potential $F_P(Y)=\sum_i\widetilde V(y_i)$, and constrained Gibbs law $\mu_P(dY)\propto e^{-F_P(Y)}\mathcal U_{\mathcal M_P}(dY)$, with $\mathcal U_{\mathcal M_P}$ uniform on $\mathcal M_P$. \\
$h$, $\nsteps$, $T$ & Numerical step size, number of steps, and physical simulation time $T=\nsteps h$. \\
$K$, $L$ & Numbers of nearest neighbors and uniformly sampled remainder points used in a local score estimate. \\
$W_1$, $W_2$, $\mathrm{KL}$ & Wasserstein distances of orders one and two with Euclidean ground distance, and Kullback--Leibler divergence. \\
\bottomrule
\end{tabular}
\end{table}

\clearpage
\section{Derivation details of MM-SOLD}
\label{app:MM-SOLD_details}

We present here some derivation details of MM-SOLD as a supplement. Throughout, $Z\in\mathbb R^{P\times d}$ denotes the row-stacked particle matrix, and
\begin{equation}
Y=(Z-\mathbf 1_P{\mu^*}^\top)(L^*)^{-\top},
\qquad
Z=\mathbf 1_P{\mu^*}^\top+Y(L^*)^\top,
\label{eq:appendix-whitening-map}
\end{equation}
where $\Sigma^*=L^*(L^*)^\top$.

\paragraph{Equivalent form of the moment constraints}
Denote the empirical mean and covariance of particles as:
\begin{equation}
\hat\mu_P(Z)=\frac{1}{P}Z^\top \mathbf 1_P,
\qquad
\hat\Sigma_P(Z)=\frac{1}{P}\bigl(Z-\mathbf 1_P\hat\mu_P(Z)^\top\bigr)^\top
\bigl(Z-\mathbf 1_P\hat\mu_P(Z)^\top\bigr).
\label{eq:appendix-empirical-moments}
\end{equation}
We claim that the constraints $\hat\mu_P(Z)=\mu^*$ and $\hat\Sigma_P(Z)=\Sigma^*$ are equivalent to
\begin{equation}
\mathbf 1_P^\top Y=0,
\qquad
Y^\top Y=PI_d.
\label{eq:appendix-centered-stiefel}
\end{equation}
Indeed, by \eqref{eq:appendix-whitening-map},
\begin{align*}
\mathbf 1_P^\top Y
&=
\mathbf 1_P^\top (Z-\mathbf 1_P{\mu^*}^\top)(L^*)^{-\top} \\
&=
\bigl(\mathbf 1_P^\top Z-P{\mu^*}^\top\bigr)(L^*)^{-\top},
\end{align*}
so $\mathbf 1_P^\top Y=0$ if and only if $\mathbf 1_P^\top Z=P{\mu^*}^\top$, i.e. $\hat\mu_P(Z)=\mu^*$. Under this mean constraint, 
\begin{align*}
\hat\Sigma_P(Z)
&=
\frac{1}{P}(Z-\mathbf 1_P{\mu^*}^\top)^\top
(Z-\mathbf 1_P{\mu^*}^\top) \\
&=
\frac{1}{P}\bigl(Y(L^*)^\top\bigr)^\top
\bigl(Y(L^*)^\top\bigr) \\
&=
\frac{1}{P}L^*Y^\top Y(L^*)^\top.
\end{align*}
Hence $\hat\Sigma_P(Z)=\Sigma^*=L^*(L^*)^\top$ if and only if
\begin{equation*}
\frac{1}{P}L^*Y^\top Y(L^*)^\top=L^*(L^*)^\top,
\end{equation*}
which, since $L^*$ is invertible, is equivalent to $Y^\top Y=PI_d$. Therefore the particles evolve on:
\begin{equation}
\mathcal M_P=
\left\{
Y\in\mathbb R^{P\times d}:\;
\mathbf 1_P^\top Y=0,\;
Y^\top Y=PI_d
\right\}.
\label{eq:appendix-manifold}
\end{equation}

\paragraph{Tangent projection}
We next justify the projection formula used in Algorithm \ref{alg:mmsold-lm}. The linear centering constraint \(\mathbf 1_P^\top Y=0\) has tangent space
\begin{equation*}
T_Y^{\mathrm{ctr}}=
\left\{
A\in\mathbb R^{P\times d}:\;
\mathbf 1_P^\top A=0
\right\},    
\end{equation*}
with orthogonal projection
\begin{equation}
\Pi_{\mathrm{ctr}}(A)=A-\frac{1}{P}\mathbf 1_P\mathbf 1_P^\top A.
\label{eq:appendix-centering-proj}
\end{equation}
Indeed, $\mathbf 1_P^\top\Pi_{\mathrm{ctr}}(A)=\mathbf 1_P^\top A-\frac{1}{P}\mathbf 1_P^\top\mathbf 1_P\,\mathbf 1_P^\top A=0$.

Now consider the constraint $Y^\top Y=PI_d$. Differentiating in direction $A$ gives
\begin{equation}
Y^\top A+A^\top Y=0,
\label{eq:appendix-stiefel-tangent}
\end{equation}
so the corresponding tangent space is
\begin{equation*}
T_Y^{\mathrm{St}}=
\left\{
A\in\mathbb R^{P\times d}:\;
Y^\top A+A^\top Y=0
\right\}.
\end{equation*}
Define
\begin{equation}
\Pi_Y^{\mathrm{St}}(A)=A-Y\,\sym\!\left(\frac{1}{P}Y^\top A\right),
\qquad
\sym(B):=\frac{1}{2}(B+B^\top).
\label{eq:appendix-stiefel-proj}
\end{equation}
Setting $S=\sym(P^{-1}Y^\top A)$, we have $S^\top=S$, and since $Y^\top Y=PI_d$,
\begin{align*}
Y^\top \Pi_Y^{\mathrm{St}}(A)+\Pi_Y^{\mathrm{St}}(A)^\top Y
&=
Y^\top A-PS+A^\top Y-PS \\
&=
Y^\top A+A^\top Y-2PS.
\end{align*}
But by the definition of $S$,
\begin{equation*}
2PS=Y^\top A+A^\top Y,   
\end{equation*}
hence $Y^\top \Pi_Y^{\mathrm{St}}(A)+\Pi_Y^{\mathrm{St}}(A)^\top Y=0$, so $\Pi_Y^{\mathrm{St}}(A)\in T_Y^{\mathrm{St}}$. The full tangent projection used in MM-SOLD is therefore
\begin{equation}
\Pi_Y(A):=\Pi_Y^{\mathrm{St}}\bigl(\Pi_{\mathrm{ctr}}(A)\bigr).
\label{eq:appendix-full-proj}
\end{equation}
Since $Y\in\mathcal M_P$ already satisfies $\mathbf 1_P^\top Y=0$, the correction term
\begin{equation*}
Y\,\sym\!\left(\frac{1}{P}Y^\top \Pi_{\mathrm{ctr}}(A)\right)
\end{equation*}
also has zero row sum, so $\Pi_Y(A)$ lies in the tangent space of $\mathcal M_P$.

\paragraph{QR retraction}
The retraction used in Algorithm \ref{alg:mmsold-lm} is obtained by first centering and then QR-normalizing. Let
\begin{equation}
A_c:=\Pi_{\mathrm{ctr}}(A)=QR
\label{eq:appendix-qr}
\end{equation}
be the reduced QR factorization of the centered matrix, and define
\begin{equation}
\mathcal R(A):=\sqrt{P}\,Q,
\label{eq:appendix-retraction}
\end{equation}
with the standard diagonal sign correction on $R$. We verify that $\mathcal R(A)\in\mathcal M_P$. Since $A_c$ is centered, $\mathbf 1_P^\top A_c=0$, hence $\mathbf 1_P^\top QR=0$. Whenever $A_c$ has full column rank, $R$ is invertible, so $\mathbf 1_P^\top Q=0$. Therefore $\mathbf 1_P^\top \mathcal R(A)=0$. Moreover,
\begin{equation*}
\mathcal R(A)^\top \mathcal R(A)
=
(\sqrt{P}Q)^\top(\sqrt{P}Q)
=
PQ^\top Q
=
PI_d.  
\end{equation*}
Thus $\mathcal R(A)\in\mathcal M_P$.

\paragraph{Pullback of the gradient}
It remains to justify the identity \(G_k^Y=G_k^ZL^*\) used in Algorithm \ref{alg:mmsold-lm}. Let
\begin{equation*}
F(Y):=\sum_{i=1}^P V(Z^i),
\qquad
Z=\mathbf 1_P{\mu^*}^\top+Y(L^*)^\top,
\end{equation*}
and write the row-stacked gradient in \(Z\)-coordinates as
\begin{equation*}
G^Z=
\begin{bmatrix}
\nabla V(Z^1)^\top\\
\vdots\\
\nabla V(Z^P)^\top
\end{bmatrix}.
\end{equation*}
For a variation $dY$, we have $dZ=dY(L^*)^\top$. Using the Frobenius pairing,
\begin{align*}
dF
&=
\operatorname{tr}\!\bigl((G^Z)^\top dZ\bigr) \\
&=
\operatorname{tr}\!\bigl((G^Z)^\top dY(L^*)^\top\bigr) \\
&=
\operatorname{tr}\!\bigl((G^ZL^*)^\top dY\bigr).
\end{align*}
Hence the gradient in $Y$-coordinates is
\begin{equation}
G^Y=G^ZL^*.
\label{eq:appendix-pullback}
\end{equation}
This is exactly the pullback formula used in MM-SOLD.

\section{Details on the nearest-neighbor estimator of the smoothed score}
\label{app:nn_score}

We give the details of the nearest-neighbor estimator used to approximate the negative smoothed score $g(z)=\nabla V(z)$. Recall that
\begin{equation*}
V(z)=-\mathbb E_{\varepsilon}\big[\log \hat p^\delta(z+\sigma\varepsilon)\big],
\qquad
g(z)=-\mathbb E_{\varepsilon}\big[\nabla \log \hat p^\delta(z+\sigma\varepsilon)\big],
\qquad
\varepsilon\sim\mathcal N(0,I_d).
\end{equation*}
For the isotropic GMM
\begin{equation*}
\hat p^\delta(y)=\frac1N\sum_{i=1}^N \mathcal N(y\mid x_i,\delta^2 I_d),
\end{equation*}
its score has the closed form
\begin{equation*}
\nabla\log \hat p^\delta(y)
=
\frac{1}{\delta^2}\left(c(y)-y\right),
\qquad
c(y)=\sum_{i=1}^N w_i(y)x_i,
\end{equation*}
where the softmax weights
\begin{equation*}
w_i(y)=
\frac{
\exp\!\left(-\|y-x_i\|^2/(2\delta^2)\right)
}{
\sum_{j=1}^N
\exp\!\left(-\|y-x_j\|^2/(2\delta^2)\right)
}.
\end{equation*}
Thus estimating $g(z)$ reduces to estimating the local posterior mean $c(y)$ at perturbed points $y=z+\sigma\varepsilon$.

\paragraph{Local corrected estimator}
Assume $0\le K<N$ and $1\le L\le N-K$, with smoothing draws independent of the subset selection. For a query point $z$, let $A_K(z)$ be the index set of its $K$ nearest training samples. Let $B_L(z)$ be sampled uniformly without replacement from $\{1,\dots,N\}\setminus A_K(z)$ with $|B_L(z)|=L$, and define
\begin{equation*}
U(z):=A_K(z)\cup B_L(z),
\qquad
m:=K+L.
\end{equation*}
For $i\in U(z)$, assign
\begin{equation*}
\omega_i=
\begin{cases}
1, & i\in A_K(z),\\
(N-K)/L, & i\in B_L(z).
\end{cases}
\end{equation*}
For any perturbed point $y$, define
\begin{align*}
\widehat T_0(y\mid z)
&=
\sum_{i\in U(z)}
\omega_i
\exp\!\left(-\frac{\|y-x_i\|^2}{2\delta^2}\right), \\
\widehat T_1(y\mid z)
&=
\sum_{i\in U(z)}
\omega_i
\exp\!\left(-\frac{\|y-x_i\|^2}{2\delta^2}\right)x_i,
\end{align*}
and
\begin{equation*}
\widehat c(y\mid z)=\frac{\widehat T_1(y\mid z)}{\widehat T_0(y\mid z)}.
\end{equation*}
The local approximation of the negative GMM score at $y$ is then
\begin{equation}
\widehat g_{\mathrm{loc}}(y\mid z)
=
\frac{1}{\delta^2}\left(y-\widehat c(y\mid z)\right).
\end{equation}

\begin{algorithm}[t]
\caption{Nearest-neighbor estimator for the negative smoothed score $g(z)$}
\label{alg:nn_score_estimator}
\begin{algorithmic}[1]
\Require Query point $z$, training set $\{x_i\}_{i=1}^N$, component standard deviation $\delta$, smoothing bandwidth $\sigma$, nearest-neighbor number $K$, remainder number $L$, even number of Monte Carlo samples $M$.
\State Find $A_K(z)$, the $K$ nearest neighbors of $z$.
\State Sample $B_L(z)$ uniformly without replacement from $\{1,\dots,N\}\setminus A_K(z)$.
\State Set $U(z)=A_K(z)\cup B_L(z)$ and weights $\omega_i=1$ for $i\in A_K(z)$, $\omega_i=(N-K)/L$ for $i\in B_L(z)$.
\State For $i\in U(z)$ set $a_i=-\|z-x_i\|^2/(2\delta^2)+\log\omega_i$ and $G_{ij}=x_i^\top x_j$.
\For{$r=1,\dots,M/2$}
    \State If $K+L<d$, draw $\eta_r\sim\mathcal N(0,G)$; otherwise draw $\varepsilon_r\sim\mathcal N(0,I_d)$ and set $\eta_{r,i}=x_i^\top\varepsilon_r$.
    \State Set $w_{r,i}^{\pm}=\operatorname{softmax}_i((a_j\pm\sigma\eta_{r,j}/\delta^2)_{j\in U(z)})$ and $c_r^{\pm}=\sum_{i\in U(z)}w_{r,i}^{\pm}x_i$.
\EndFor
\State \Return $\widehat g(z)=\delta^{-2}\left[z-M^{-1}\sum_{r=1}^{M/2}(c_r^++c_r^-)\right]$.
\end{algorithmic}
\end{algorithm}

\paragraph{Projected-space smoothing}
Let $U(z)=\{i_1,\dots,i_m\}$ with $m=K+L$, and write $\nu_a=x_{i_a}$ for $a=1,\dots,m$. For $y=z+\sigma\varepsilon$, the local softmax logits satisfy
\begin{align*}
-\frac{\|z+\sigma\varepsilon-\nu_a\|^2}{2\delta^2}
&=
-\frac{\|z-\nu_a\|^2}{2\delta^2}
+
\frac{\sigma}{\delta^2}\langle \varepsilon,\nu_a\rangle
+
\mathrm{const}(z,\varepsilon),
\end{align*}
where the last term is common to all $a$ and therefore cancels in the softmax. Hence the smoothing noise affects the local softmax only through the inner products
\begin{equation*}
\eta_a=\langle \varepsilon,\nu_a\rangle,
\qquad a=1,\dots,m.
\end{equation*}
If $\varepsilon\sim\mathcal N(0,I_d)$, then
\begin{equation*}
\eta=(\eta_1,\dots,\eta_m)^\top\sim\mathcal N(0,G),
\qquad
G_{ab}:=\langle \nu_a,\nu_b\rangle.
\end{equation*}
Thus, when $m<d$, we sample $\eta\sim\mathcal N(0,G)$ using a positive semidefinite square root of $G$ (which may be singular) and perturb the local logits by $\sigma\eta_a/\delta^2$. For the antithetic point $z-\sigma\varepsilon$, we use the same $\eta$ with the opposite sign. This gives the same local softmax perturbation law as ambient-space Gaussian smoothing, while avoiding the generation and projection of full $d$-dimensional Gaussian noises. When $m\ge d$, we instead sample $\varepsilon\sim\mathcal N(0,I_d)$ as usual in the ambient space.

\paragraph{Unbiasedness of the corrected local sums}
Define the full sums
\begin{equation*}
T_0(y)
=
\sum_{i=1}^N
\exp\!\left(-\frac{\|y-x_i\|^2}{2\delta^2}\right),
\qquad
T_1(y)
=
\sum_{i=1}^N
\exp\!\left(-\frac{\|y-x_i\|^2}{2\delta^2}\right)x_i.
\end{equation*}
Conditioned on $A_K(z)$, the corrected local sums satisfy
\begin{equation}
\mathbb E\!\left[\widehat T_0(y\mid z)\mid A_K(z)\right]=T_0(y),
\qquad
\mathbb E\!\left[\widehat T_1(y\mid z)\mid A_K(z)\right]=T_1(y).
\label{eq:nn_unbiased_local_sums}
\end{equation}
Indeed, for any scalar or vector-valued sequence $\{a_i\}_{i\notin A_K(z)}$, uniform sampling without replacement gives
\begin{equation*}
\mathbb E\left[\sum_{i\in B_L(z)}a_i\,\middle|\,A_K(z)\right]
=
\frac{L}{N-K}\sum_{i\notin A_K(z)}a_i.
\end{equation*}
Multiplication by $(N-K)/L$ recovers the full contribution from the complement of $A_K(z)$. Taking
\begin{equation*}
a_i=
\exp\!\left(-\frac{\|y-x_i\|^2}{2\delta^2}\right)
\quad\text{and}\quad
a_i=
\exp\!\left(-\frac{\|y-x_i\|^2}{2\delta^2}\right)x_i
\end{equation*}
proves \eqref{eq:nn_unbiased_local_sums}. The ratio $\widehat c(y\mid z)=\widehat T_1(y\mid z)/\widehat T_0(y\mid z)$ is generally not exactly unbiased, since it is a ratio estimator.

\paragraph{Consistency}
If $\widehat T_0(y\mid z)\to T_0(y)$ and $\widehat T_1(y\mid z)\to T_1(y)$ in probability, and $T_0(y)>0$, then by the continuous mapping theorem,
\begin{equation}
\widehat c(y\mid z)\xrightarrow[]{\mathbb P} c(y),
\qquad
\widehat g_{\mathrm{loc}}(y\mid z)
\xrightarrow[]{\mathbb P}
-\nabla\log\hat p^\delta(y).
\end{equation}
For a fixed local set $U(z)$, antithetic averaging is unbiased for the
smoothing expectation of the local estimator because $\varepsilon$ and
$-\varepsilon$ have the same distribution. The law of large numbers gives,
conditionally on $U(z)$,
\begin{equation}
\widehat g(z)
\xrightarrow[]{\mathbb P}
\frac{1}{\delta^2}
\left(z-\mathbb E_\varepsilon[\widehat c(z+\sigma\varepsilon\mid z)]\right)
\qquad(M\to\infty).
\end{equation}
This limit generally differs from $\nabla V(z)$. Consistency for the full
score additionally requires that the local posterior mean converge to
$c$ in probability under the Gaussian query law as the local budget
increases. Both posterior means lie in the convex hull of the fixed
training set, so this convergence is also in $L^1$ and passes through the
smoothing expectation. For fixed $N$, using all training points makes the
posterior mean exact. Thus increasing $M$ alone does not remove local
ratio bias.

\section{Approximate sampling by finite-step MM-SOLD}
\label{app:approx_finite_step_mmsold}

We prove consistency for the QR--Euler--Maruyama variant with the full
antithetic score estimator, conditional on a fixed training set. This
distinguishes approximation of the empirical target from statistical error.
The result does not assert consistency of the noise-reusing LM update in
Algorithm~\ref{alg:mmsold-lm}, or of a fixed-budget nearest-neighbor ratio
estimator.

\subsection{The Euler kernel and its diffusion limit}

Fix $\delta>0$, $\sigma\ge0$, $\Sigma^*=\widehat\Sigma_N+\delta^2I_d\succ0$, and an even integer $M\ge2$.
Write $F_P(Y)=\sum_{i=1}^P V(\mu^*+L^*y_i)$ and equip $\mathcal M_P$ with the
induced Frobenius metric. With $c(u)=\sum_{j=1}^N w_j(u)x_j$ evaluated using
all training points, define
\[
\widehat g_M(z)=\frac{1}{\delta^2}
\left(z-\frac1M\sum_{r=1}^{M/2}
[c(z+\sigma\varepsilon_r)+c(z-\sigma\varepsilon_r)]\right),
\qquad \varepsilon_r\stackrel{\mathrm{i.i.d.}}{\sim}\mathcal N(0,I_d).
\]
Let $\widehat G^Y$ have rows $\widehat g_M(z_i)^\top L^*$.
At every step and for every particle, use fresh smoothing draws, independent
of the standard Gaussian matrix $\Xi\in\mathbb R^{P\times d}$. Define
\begin{equation}
Y^+=R_Y\!\left(-h\Pi_Y\widehat G^Y+\sqrt{2h}\Pi_Y\Xi\right),
\qquad
R_Y(H)=\sqrt P\,\operatorname{qf}(Y+H),
\label{eq:appendix-euler-full}
\end{equation}
where $\operatorname{qf}$ is the reduced QR factor with positive diagonal
in its triangular factor. Its kernel is denoted by
$\widehat{\mathsf P}_{P,h,M}$. Initialization is on $\mathcal M_P$; an
arbitrary centered full-rank initialization can first be QR-normalized.
For tangent $H$, $(Y+H)^\top(Y+H)=PI_d+H^\top H\succ0$, so this retraction
is well defined for every proposed step.

\begin{lemma}[QR trace cancellation and finite-time consistency]
\label{lem:qr_euler_consistency}
For fixed $P\ge d+2$, let $\mathsf P_t^P$ be the diffusion semigroup with
generator
\[
\mathcal L_P f=\Delta_{\mathcal M_P}f
-\langle\operatorname{grad}F_P,\operatorname{grad}f\rangle.
\]
For every $f\in C^\infty(\mathcal M_P)$,
\begin{equation}
\widehat{\mathsf P}_{P,h,M}f=f+h\mathcal L_Pf+O_f(h^{3/2})
\quad\text{uniformly on }\mathcal M_P.
\label{eq:appendix-euler-generator}
\end{equation}
Consequently, for every fixed $T>0$ and continuous $f$,
\[
\left\|(\widehat{\mathsf P}_{P,T/\nsteps,M})^{\nsteps} f-\mathsf P_T^P f\right\|_\infty
\longrightarrow0\qquad(\nsteps\to\infty).
\]
\end{lemma}

\begin{proof}
For a symmetric matrix $S$, let $\mathsf U(S)$ be upper triangular with
$\mathsf U(S)+\mathsf U(S)^\top=S$. The Cholesky expansion of
$PI_d+H^\top H$ gives
\[
R_Y(H)=Y+H-\frac1P Y\mathsf U(H^\top H)+O(\|H\|_F^3),
\]
uniformly in $Y$ near the zero section.
If $(E_a)$ is an orthonormal tangent basis and
$c_P=P-1-(d+1)/2$, then
\[
\sum_a E_a^\top E_a=c_PI_d,
\qquad
\sum_a D^2R_Y(0)[E_a,E_a]=-\frac{c_P}{P}Y.
\]
To verify the first identity, remove the centering direction and rotate
$Y/\sqrt P$ to $[I_d;0]$ in $\mathbb R^{(P-1)\times d}$.
The tangent matrices have a skew-symmetric top block and an unrestricted
bottom block; their orthonormal bases contribute $(d-1)I_d/2$ and
$(P-1-d)I_d$, respectively. The second identity follows by linearity of
$\mathsf U$. Differentiating the constraints twice shows that the second
fundamental form is $\mathrm{II}(H,H)=-YH^\top H/P$.
Thus the trace of $D^2R_Y(0)$ equals the trace of $\mathrm{II}$, and
\[
\sum_a D^2(f\circ R_Y)(0)[E_a,E_a]=\Delta_{\mathcal M_P}f(Y).
\]
This trace identity, rather than first-order retraction consistency alone,
is needed for the Brownian limit; compare~\cite{schwarz2025efficient}.

Antithetic cancellation gives
$\mathbb E[\widehat G^Y\mid Y]=\nabla F_P(Y)$.
Moreover, $\|c(u)\|\le\max_j\|x_j\|$, so $\widehat G^Y$ is uniformly bounded
for fixed $P$, even though the smoothing noises are unbounded.
The projected Gaussian has identity covariance on the tangent space.
Taylor expansion of $f\circ R_Y$ now gives
\eqref{eq:appendix-euler-generator}: the score randomness contributes only
$O(h^2)$ to second moments, and the third absolute moment of the increment
is $O(h^{3/2})$. The expansion is local, but the probability of leaving a
fixed tangent neighborhood is $O(e^{-b/h})$ for some $b>0$.
Gaussian tail moments control both $f$ and the discarded Taylor terms
there. The local remainder is bounded by
$C_{P,M}\|f\|_{C^3}h^{3/2}$.

For smooth $f$, the exact semigroup has
$\mathsf P_h^P f=f+h\mathcal L_Pf+O_f(h^2)$.
Smoothness on the compact manifold bounds the required derivatives of
$\mathsf P_s^P f$ uniformly for $0\le s\le T$.
Writing $Q_h=\widehat{\mathsf P}_{P,h,M}$, the identity
\[
Q_h^{\nsteps} f-\mathsf P_T^P f
=\sum_{j=0}^{\nsteps-1}Q_h^j(Q_h-\mathsf P_h^P)
\mathsf P_{T-(j+1)h}^P f
\]
and the contraction property therefore bound the difference by
$C_{P,T,M,f}\sqrt h$.
Uniform approximation by smooth functions extends convergence to
continuous $f$.
\end{proof}

\begin{proposition}[Approximate sampling with full-score Euler updates]
\label{prop:approx_mmsold}
Assume the conditions of Proposition~\ref{prop:large_particle_target}.
Fix an even $M\ge2$ and set $\Gamma_P(Y)=\mu^*+L^*y_1$.
For every $\varepsilon>0$ there exist integers $P\ge d+2$, $\nsteps\ge1$ and $T>0$,
with $h=T/\nsteps$, such that
\[
\sup_{\nu_P\in\mathcal P(\mathcal M_P)}
W_1\!\left((\Gamma_P)_\#
[\nu_P(\widehat{\mathsf P}_{P,h,M})^{\nsteps}],\rho_N\right)<\varepsilon,
\]
where $\rho_N$ is the limiting target in Proposition~\ref{prop:large_particle_target}.
\end{proposition}

\begin{proof}
The manifold $\mathcal M_P$ is a scaled copy of $V_{P-1,d}$, hence compact
and connected when $P\ge d+2$. Its normalized Riemannian volume is
$\mathcal U_{\mathcal M_P}$ by orthogonal invariance.
Integration by parts identifies the invariant law of $\mathcal L_P$ as
$\mu_P(dY)\propto e^{-F_P(Y)}\mathcal U_{\mathcal M_P}(dY)$.
The smooth elliptic diffusion has a strictly positive heat kernel; at
time one its minimum on the compact product manifold is positive.
The resulting minorization implies uniform convergence to $\mu_P$ in
total variation, and therefore of its tagged marginal in $W_1$.

By Proposition~\ref{prop:large_particle_target},
$(\Gamma_P)_\#\mu_P\Rightarrow\rho_N$.
Exchangeability and the constraints give
$\mathbb E_{\mu_P}\|\Gamma_P(Y)\|^2
=\|\mu^*\|^2+\operatorname{tr}\Sigma^*$ for every $P$.
Uniform integrability of first moments upgrades this convergence to $W_1$.
Choose $P$ making this error less than $\varepsilon/3$, then choose $T$
making the mixing error less than $\varepsilon/3$ uniformly in $\nu_P$.

For these fixed $P,T$, Lemma~\ref{lem:qr_euler_consistency} makes the
numerical tagged-marginal error tend to zero uniformly in $\nu_P$.
Indeed, the restrictions of anchored $1$-Lipschitz functions to the compact
set $\Gamma_P(\mathcal M_P)$ have finite uniform nets. Apply the lemma to
their continuous pullbacks and use the contraction property to pass from
a finite net to the full Wasserstein dual class.
Choose $\nsteps$ making this last error less than $\varepsilon/3$ and apply the
triangle inequality.
\end{proof}

This is an existence result, with no particle-uniform mixing or complexity
bound. A fixed $M$ suffices because the full estimator is conditionally
unbiased. The choices of $P,T,\nsteps$ may depend on the fixed data and $M$,
but not on the initial law; the population assumptions in the next
section are not needed here. The nearest-neighbor ratio bias identified in
Appendix~\ref{app:nn_score} requires a separate error bound or an increasing
local budget; decreasing $h$ alone does not remove it.

\section{Population projection and statistical target error}
\label{app:statistical_target_error}

Write $p=\pi_{\rm data}$ and $p_N=N^{-1}\sum_{i=1}^N\delta_{x_i}$.
Throughout this section $\delta>0$, $\sigma\ge0$, and
$\varepsilon\sim\mathcal N(0,I_d)$. For a probability law $\xi$, define
\[
p_\xi^\delta=\xi*\mathcal N(0,\delta^2I_d),\qquad
V_\xi(z)=-\mathbb E_\varepsilon\log p_\xi^\delta(z+\sigma\varepsilon),
\qquad q_\xi(z)\propto e^{-V_\xi(z)}.
\]
For bounded-support $\xi$, let $\rho_\xi$ be the unique minimizer of
$\mathrm{KL}(\cdot\|q_\xi)$ with mean $\mathbb E_\xi X$ and covariance
$\operatorname{Cov}_\xi(X)+\delta^2I_d$, the moments of $p_\xi^\delta$.
Existence and uniqueness follow from Lemma~\ref{lem:moment_realization_log_gmm}.
Write $\rho_N:=\rho_{p_N}$ and $\rho_p$ for the empirical and population targets.
The decomposition
\begin{equation}
W_1(\rho_N,p)\le W_1(\rho_N,\rho_p)
+\underbrace{W_1(\rho_p,p)}_{B_{\delta,\sigma}}
\label{eq:appendix-statistical-decomposition}
\end{equation}
separates finite-data error from population bias. We give a population
comparison and a quantitative stability result under explicit additional
assumptions. Neither assumes that moment matching is a contraction in $W_1$.

\subsection{Population moment matching and smoothing bias}

Use the sufficient statistic $S(z)=(z,zz^\top)$ in
$\mathbb R^d\times\mathrm{Sym}_d$, with the Euclidean--Frobenius inner
product and norm; its dimension is $d+d(d+1)/2$. Write
$\vartheta=(b,\Theta)$ for the raw-moment tilt parameter. For a fixed
centering mean $\mu$, the density proportional to
$e^{-V-\lambda^\top(z-\mu)-\frac12(z-\mu)^\top\Lambda(z-\mu)}$
has raw parameters
\[
b=-\lambda+\Lambda\mu,\qquad \Theta=-\tfrac12\Lambda.
\]
This distinguishes $\vartheta$ from the centered parameter
$\theta=(\lambda,\Lambda)$ used in Appendix~\ref{app:tilting-parameter-estimation}.

\begin{proposition}[Population projection and a bias bound]
\label{prop:population_projection_bias}
Let $p$ be any probability law supported in $\{\|x\|\le R\}$, with mean
$\mu$ and covariance $\Sigma$. Write $p^\delta=p_p^\delta$ and
$\Sigma_\delta=\Sigma+\delta^2I_d$.
The target $\rho_p$ exists and uniquely minimizes $\mathrm{KL}(\cdot\|q_p)$
among laws with mean $\mu$ and covariance $\Sigma_\delta$. Moreover,
\begin{equation}
\mathrm{KL}(p^\delta\|q_p)
=\mathrm{KL}(p^\delta\|\rho_p)+\mathrm{KL}(\rho_p\|q_p)
\le d\sigma^2H_\delta,
\qquad H_\delta=\delta^{-2}+R^2\delta^{-4},
\label{eq:appendix-population-pythagoras}
\end{equation}
and
\begin{equation}
B_{\delta,\sigma}=W_1(\rho_p,p)
\le\delta\sqrt d+
2\sqrt{\operatorname{tr}\Sigma+d\delta^2}
\left(\frac{d\sigma^2H_\delta}{2}\right)^{1/4}.
\label{eq:appendix-population-smoothing-bias}
\end{equation}
When $\sigma=0$, $\rho_p=p^\delta$ and only the first bias term remains.
\end{proposition}

\begin{proof}
Lemma~\ref{lem:moment_realization_log_gmm} applies because
$\Sigma_\delta\succ0$, and gives existence and the projection property.
To verify finiteness and the entropy bound, Gaussian-mixture differentiation gives
\[
\nabla^2\log p^\delta(u)
=-\delta^{-2}I_d+\delta^{-4}\operatorname{Cov}(X\mid X+\delta G=u),
\qquad X\sim p,\quad G\sim\mathcal N(0,I_d)\text{ independent}.
\]
Its operator norm is at most $H_\delta$.
Taylor's formula and $\mathbb E\varepsilon=0$ imply
$|D(z)|\le H_\delta d\sigma^2/2$, where
$D(z)=\mathbb E\log p^\delta(z+\sigma\varepsilon)-\log p^\delta(z)$.
Thus $q_p=p^\delta e^D/\int p^\delta e^D$ and
$\mathrm{KL}(p^\delta\|q_p)\le H_\delta d\sigma^2$.
Both $p^\delta$ and $\rho_p$ have mean $\mu$ and covariance $\Sigma_\delta$.
The quadratic function $\log(\rho_p/q_p)$ therefore has the same finite
expectation under both laws. Subtracting it gives the entropy identity in
\eqref{eq:appendix-population-pythagoras}, with all three terms finite.

Writing $\operatorname{TV}(f,g)=\tfrac12\int|f-g|$, Wasserstein duality
and Cauchy--Schwarz give
\[
W_1(p^\delta,\rho_p)
\le\int\|z-\mu\||p^\delta-\rho_p|(z)\,dz
\le2\sqrt{\operatorname{tr}\Sigma_\delta}
\sqrt{\operatorname{TV}(p^\delta,\rho_p)}.
\]
Pinsker's inequality and \eqref{eq:appendix-population-pythagoras} bound this
by the second term in \eqref{eq:appendix-population-smoothing-bias}.
The coupling $(X,X+\delta G)$ gives
$W_1(p^\delta,p)\le\delta\mathbb E\|G\|\le\delta\sqrt d$;
the triangle inequality finishes the proof. At $\sigma=0$, $q_p=p^\delta$
already has the prescribed moments, so $\rho_p=p^\delta$.
\end{proof}

The bias estimate also applies to singular laws, including bounded
manifold-supported data. It is conservative and does not explain the
large $\sigma/\delta$ regime of the experiments.

\subsection{Stability and a finite-sample bound}

The next result makes the conditioning of the moment projection explicit.
Positive definite covariance of $z$ alone does not give a quantitative
lower bound for the covariance of $S(z)$.

\begin{proposition}[Stability of the moment-matched target]
\label{prop:moment_target_stability}
Let $\xi_0,\xi_1$ be supported in the ball $\{\|x\|\le R\}$, $R>0$,
and put $\xi_t=(1-t)\xi_0+t\xi_1$.
Suppose that for $0\le t\le1$ the moment-matched targets exist as
\[
r_t(z)=\exp\{-V_{\xi_t}(z)+\langle\vartheta_t,S(z)\rangle-A_t\},
\qquad \mathbb E_{r_t}S=\mathbb E_{\xi_t}S+(0,\delta^2I_d),
\]
with $\vartheta_t$ continuously differentiable. Assume constants $C,a,\kappa>0$
satisfy, uniformly in $t$,
\begin{equation}
r_t(z)\le Ce^{-a\|z\|^2},
\qquad H_t:=\operatorname{Cov}_{r_t}(S)\succeq\kappa I.
\label{eq:appendix-projection-regularity}
\end{equation}
Define the finite constants
\begin{align}
A_{\delta,\sigma,R}(z)
&=\mathbb E_\varepsilon\left[
\frac{\|z+\sigma\varepsilon\|+R}{\delta^2}
\exp\left\{\frac{2R\|z+\sigma\varepsilon\|}{\delta^2}
+\frac{R^2}{2\delta^2}\right\}\right], \notag\\
J^2&=C\int_{\mathbb R^d} A_{\delta,\sigma,R}(z)^2e^{-a\|z\|^2}\,dz,
\qquad
D_{\rm stab}=\sqrt{R^2+d\delta^2}\sqrt{J^2+(1+4R^2)/\kappa}.
\label{eq:appendix-stability-constant}
\end{align}
Then
\begin{equation}
W_1(r_0,r_1)\le D_{\rm stab}W_1(\xi_0,\xi_1).
\label{eq:appendix-target-lipschitz}
\end{equation}
\end{proposition}

\begin{proof}
Set $e=W_1(\xi_0,\xi_1)$ and $v_t=\partial_t V_{\xi_t}$.
For $k_u(x)=\exp(-\|u-x\|^2/(2\delta^2))$ on the ball,
\[
\operatorname{Lip}(k_u)\le
\frac{\|u\|+R}{\delta^2}\max_x k_u(x),
\qquad
\frac{\max_x k_u(x)}{\min_x k_u(x)}
\le e^{2R\|u\|/\delta^2+R^2/(2\delta^2)}.
\]
Wasserstein duality bounds
$|\int k_u\,d(\xi_1-\xi_0)|\le e\operatorname{Lip}(k_u)$.
Differentiating $\log\int k_u\,d\xi_t$ and averaging over
$u=z+\sigma\varepsilon$ therefore gives
$|v_t(z)|\le eA_{\delta,\sigma,R}(z)$.
This function grows at most as a polynomial times an exponential in
$\|z\|$, so $J<\infty$ and
$\operatorname{Var}_{r_t}(v_t)\le e^2J^2$.
To justify differentiation, choose a ball on which every $r_t$ has mass
at least $1/2$, using the common Gaussian envelope.
On this ball $V_{\xi_t}$ and $\vartheta_t$ are uniformly bounded, which bounds
the normalizers above and away from zero. Thus the unnormalized densities
also have a common Gaussian envelope. Their derivatives are bounded by
this envelope times a polynomial and $A_{\delta,\sigma,R}$, an integrable
bound. Differentiation of normalized expectations is consequently valid,
including for Lipschitz functions of at most linear growth.

Let $m_t=\mathbb E_{r_t}S$ and
$\Delta m=m'_t=\mathbb E_{\xi_1}S-\mathbb E_{\xi_0}S$;
the added covariance $\delta^2I_d$ is constant along the path.
Since $S$ is $\sqrt{1+4R^2}$-Lipschitz on the ball,
$\|\Delta m\|\le\sqrt{1+4R^2}\,e$.
Write $\ell_t=\partial_t\log r_t$ and
$u_t=\operatorname{Cov}_{r_t}(S,v_t)$. Differentiating the normalization
and moment equations yields
\[
\ell_t=\vartheta'_t\cdot(S-m_t)-(v_t-\mathbb E_{r_t}v_t),
\qquad H_t\vartheta'_t=\Delta m+u_t.
\]
The residual
$\widetilde v_t=v_t-\mathbb E_{r_t}v_t-u_t^\top H_t^{-1}(S-m_t)$
is orthogonal in $L^2(r_t)$ to every coordinate of $S-m_t$.
Consequently,
\[
\mathbb E_{r_t}\ell_t^2
=(\Delta m)^\top H_t^{-1}\Delta m+\mathbb E_{r_t}\widetilde v_t^2
\le \|\Delta m\|^2/\kappa+\operatorname{Var}_{r_t}(v_t).
\]
For any $1$-Lipschitz $f$,
$\operatorname{Var}_{r_t}(f)\le
\mathbb E_{r_t}\|Z-\mathbb E_{r_t}Z\|^2\le R^2+d\delta^2$,
by moment matching. Hence
\[
\left|\frac{d}{dt}\mathbb E_{r_t}f\right|
=|\operatorname{Cov}_{r_t}(f,\ell_t)|
\le D_{\rm stab}e.
\]
Integrate over $t$ and use Wasserstein duality.
\end{proof}

\begin{proposition}[A local finite-sample target bound]
\label{prop:finite_sample_target_bound}
Let $X_1,\ldots,X_N$ be independent with law $p$ supported in the radius-$R$
ball, $R>0$.
There are constants $e_0,D_{\rm stab}>0$, depending on $p,\delta,\sigma,R,d$,
and a constant $C_d$ depending only on $d$, with the following property.
For $0<\eta<1$, set
\[
u_{N,\eta}=R\left(C_d r_d(N)+\sqrt{2\log(1/\eta)/N}\right).
\]
Whenever $u_{N,\eta}\le e_0$, with probability at least $1-\eta$,
\begin{equation}
W_1(\rho_N,p)
\le B_{\delta,\sigma}+D_{\rm stab}u_{N,\eta},
\quad
r_d(N)=
\begin{cases}
N^{-1/2},&d=1,\\
N^{-1/2}\log(1+N),&d=2,\\
N^{-1/d},&d\ge3.
\end{cases}
\label{eq:appendix-finite-sample-target}
\end{equation}
$B_{\delta,\sigma}$ is bounded explicitly by
\eqref{eq:appendix-population-smoothing-bias}.
At $\sigma=0$, the sharper bound
$W_1(\rho_N,p)\le u_{N,\eta}+\delta\sqrt d$ holds with probability
at least $1-\eta$ for every $N$, without the local threshold.
\end{proposition}

\begin{proof}
Lemma~\ref{lem:moment_realization_log_gmm} gives a population tilt
$\vartheta_p=(b_p,\Theta_p)$ with $\delta^{-2}I_d-2\Theta_p\succ0$.
We verify local projection regularity. Uniformly over all probability
laws $\xi$ supported in the ball,
$V_\xi(z)=\|z\|^2/(2\delta^2)+O(1+\|z\|)$, with constants independent of
$\xi$. Choose a small closed parameter ball around $\vartheta_p$ preserving
the strict quadratic tail margin.
The unnormalized tilts in this ball have a common Gaussian upper envelope
and their normalizers have a common positive lower bound (integrate over
a fixed bounded ball).
Their normalized densities therefore satisfy the first condition in
\eqref{eq:appendix-projection-regularity}.
The kernel bound in Proposition~\ref{prop:moment_target_stability}, integrated
along the mixture path, gives
$|V_\xi(z)-V_p(z)|\le A_{\delta,\sigma,R}(z)W_1(\xi,p)$.
Thus $V_\xi\to V_p$ pointwise as $\xi\to p$ in $W_1$.
The supremum over the parameter ball of each integrand is bounded by
the Gaussian envelope times a polynomial of degree at most four.
Dominated convergence therefore gives uniform convergence
of the normalizers and these moments.
The covariance of $S$ under $\rho_p$ is positive definite: a nonconstant
linear--quadratic polynomial cannot be constant under a positive density
on $\mathbb R^d$.
Shrinking the parameter ball and then choosing $e_0>0$ small enough thus
gives a common lower bound $\kappa I$ for these covariance matrices whenever
$W_1(\xi,p)\le e_0$.

To localize the perturbed tilt, consider the convex dual
$D_\xi(\vartheta)=\log\int e^{-V_\xi+\langle\vartheta,S\rangle}
-\langle\vartheta,\mathbb E_\xi S+(0,\delta^2I_d)\rangle$.
If the parameter-ball radius is $r$, decrease $e_0$ so that
$\|\nabla D_\xi(\vartheta_p)\|<\kappa r/2$.
The Hessian bound implies that its outward radial derivative is positive
on the boundary. Its minimum over the ball is therefore interior, unique,
and has the required moments; the entropy identity identifies it with
the moment projection.
For the mixture path from $p$ to such a $\xi$, the same bounds apply since
$W_1((1-t)p+t\xi,p)\le W_1(\xi,p)$.
Uniqueness and compactness give continuity of the minimizer in $t$.
On $0<t<1$, the implicit function theorem gives
$\vartheta'_t=H_t^{-1}(\Delta m+\operatorname{Cov}_{r_t}(S,v_t))$, where $\Delta m=\mathbb E_\xi S-\mathbb E_pS$.
The Gaussian envelope and the bound for $v_t$ justify differentiation and
show that this derivative extends continuously to both endpoints.
Thus the path is $C^1$ with one-sided endpoint derivatives; no extension
to signed mixtures outside $[0,1]$ is required.
Proposition~\ref{prop:moment_target_stability} consequently supplies a common
$D_{\rm stab}$ on this neighborhood.

The bounded-support case of~\cite[Theorem~1]{fournier2015rate} gives
$\mathbb E W_1(p_N,p)\le C_dRr_d(N)$.
Replacing one observation changes this distance by at most $2R/N$.
The bounded-differences inequality therefore gives
\[
W_1(p_N,p)\le C_dRr_d(N)+R\sqrt{2\log(1/\eta)/N}
\]
with probability at least $1-\eta$.
When $u_{N,\eta}\le e_0$, apply the local regularity just established,
\eqref{eq:appendix-target-lipschitz} and
\eqref{eq:appendix-statistical-decomposition}.
When $\sigma=0$, $\rho_N=p_N*\mathcal N(0,\delta^2I_d)$.
Adding the same independent Gaussian noise to a coupling shows that
$W_1(\rho_N,p^\delta)\le W_1(p_N,p)$; combine this with
$W_1(p^\delta,p)\le\delta\sqrt d$.
\end{proof}

The local constants and sample-size threshold may depend on the data law
and bandwidths. In particular, the covariance of the quadratic
statistics can be poorly conditioned near a manifold.
The bound controls finite-data error for fixed bandwidths, while
$B_{\delta,\sigma}$ need not vanish as $N$ grows.
Consistency under shrinking bandwidths requires both this bias and the
statistical term (including its bandwidth-dependent constants) to vanish.
Adding the algorithmic error from Proposition~\ref{prop:approx_mmsold}
to \eqref{eq:appendix-finite-sample-target} gives a corresponding bound for
the full-score Euler sampler, conditional on the observed data.

\section{Proofs}
\label{app:proofs}

\subsection{Proof of Proposition~\ref{prop:large_particle_target}}
\label{app:proof_large_particle_target}

We prove Proposition~\ref{prop:large_particle_target} for the
Gaussian-smoothed log-GMM potential. The proof has six steps. First, we record
basic analytic properties of the potential. Second, we prove moment
realizability and identify the entropy projection in whitened coordinates.
Third, we prove the one-particle local density-ratio estimate needed for
equivalence of ensembles. Fourth, we identify the finite-\(P\) constrained
Gibbs law by coarea conditioning. Fifth, we prove convergence of the
one-particle marginal. Finally, we identify the limiting law in \(Z\)-space and
derive the tilting-parameter identities.

\paragraph{Setup.}
Recall that
\[
\hat p^\delta(z)
=
\frac1N\sum_{j=1}^N \mathcal N(z\mid x_j,\delta^2 I_d),
\qquad
V(z)
=
-\mathbb E_{\varepsilon\sim\mathcal N(0,I_d)}
\left[\log \hat p^\delta(z+\sigma\varepsilon)\right],
\]
where \(\delta>0\) and \(\sigma\ge0\). Since \(\Sigma^*\succ0\), choose an
invertible \(L^*\) satisfying
\[
\Sigma^*=L^*(L^*)^\top .
\]
We work in whitened coordinates
\[
z=\mu^*+L^*y,
\qquad
\widetilde V(y):=V(\mu^*+L^*y),
\qquad
\widetilde q(dy)=Z_{\widetilde q}^{-1}e^{-\widetilde V(y)}\,dy .
\]
The affine map \(y\mapsto z=\mu^*+L^*y\) pushes \(\widetilde q\) forward to
\(q(dz)=Z_q^{-1}e^{-V(z)}\,dz\), up to the constant Jacobian.

\paragraph{Step 1: Analytic properties of the Gaussian-smoothed log-GMM.}
Since \(\delta>0\), the GMM density \(\hat p^\delta\) is strictly positive and
real analytic on \(\mathbb R^d\). Hence \(\log\hat p^\delta\) is smooth, and
Gaussian smoothing permits differentiation under the expectation defining
\(V\).

Let
\[
R:=\max_{1\le j\le N}\|x_j\|.
\]
By log-sum-exp bounds, there are constants \(c_1,c_2,c_3,c_4>0\) such that,
for all \(u\in\mathbb R^d\),
\[
c_1\|u\|^2-c_2
\le
-\log\hat p^\delta(u)
\le
c_3\|u\|^2+c_4 .
\]
Indeed, up to additive constants,
\[
-\log\hat p^\delta(u)
=
-\log\sum_{j=1}^N
\exp\left(-\frac{\|u-x_j\|^2}{2\delta^2}\right),
\]
and this is controlled above and below by \(\min_j\|u-x_j\|^2/(2\delta^2)\), up to
additive constants. Since
\[
(\max\{\|u\|-R,0\})^2
\le
\min_j\|u-x_j\|^2
\le
(\|u\|+R)^2,
\]
the claimed bounds follow. Averaging over
\(u=z+\sigma\varepsilon\), with \(\varepsilon\sim\mathcal N(0,I_d)\), preserves
quadratic growth. Hence there exist constants \(\alpha_->0\),
\(\alpha_+>0\), and \(C<\infty\) such that
\[
\alpha_-\|z\|^2-C
\le
V(z)
\le
\alpha_+\|z\|^2+C .
\]
Since \(L^*\) is invertible, the whitened potential also satisfies
\[
\beta_-\|y\|^2-C
\le
\widetilde V(y)
\le
\beta_+\|y\|^2+C
\]
for some \(\beta_->0\), \(\beta_+>0\), and \(C<\infty\). Thus
\(e^{-\widetilde V(y)}dy\) has full support and Gaussian tails.

We also record a growth bound for \(g=\nabla V\). The GMM score satisfies
\[
\nabla\log\hat p^\delta(u)
=
\frac{c(u)-u}{\delta^2},
\qquad
c(u)=\sum_{j=1}^N w_j(u)x_j,
\]
where \(w_j(u)\ge0\) and \(\sum_jw_j(u)=1\). Hence \(c(u)\) lies in the convex
hull of the training samples and \(\|c(u)\|\le R\). Therefore
\[
\|\nabla\log\hat p^\delta(u)\|
\le
\frac{R+\|u\|}{\delta^2}.
\]
Since
\[
g(z)=\nabla V(z)
=
-\mathbb E_\varepsilon
\left[\nabla\log\hat p^\delta(z+\sigma\varepsilon)\right],
\]
we obtain
\[
\|g(z)\|\le C(1+\|z\|)
\]
for a constant \(C<\infty\).

\paragraph{Step 2: Moment realizability and entropy projection in \(Y\)-space.}
Let
\[
\widetilde T(y):=(y,yy^\top),
\qquad
t_0:=(0,I_d),
\]
and define
\[
\widetilde{\mathcal C}
=
\left\{
\rho:
\int y\,\rho(dy)=0,\;
\int yy^\top\,\rho(dy)=I_d
\right\}.
\]

\begin{lemma}[Moment realization for Gaussian-smoothed mixtures]
\label{lem:moment_realization_log_gmm}
Let $\xi$ be any bounded-support probability law, $\delta>0$, and $\sigma\ge0$.
For any prescribed mean $\mu$ and covariance $\Sigma=LL^\top\succ0$, set
$\widetilde V(y)=V_\xi(\mu+Ly)$ and $\widetilde q\propto e^{-\widetilde V}$.
There are unique $a_*\in\mathbb R^d$, $B_*\in\mathrm{Sym}_d$ such that
\[
\widetilde r(y)\propto
\exp\{a_*^\top y+\tfrac12y^\top B_*y-\widetilde V(y)\},
\qquad
\mathbb E_{\widetilde r}Y=0,\quad
\mathbb E_{\widetilde r}YY^\top=I_d.
\]
These parameters lie in the open integrability domain
$\mathcal O_{\rm w}=\{(a,B):Q-B\succ0\}$, where $Q=L^\top L/\delta^2$.
Moreover, $\widetilde r$ uniquely minimizes
$\mathrm{KL}(\rho\|\widetilde q)$ over $\widetilde{\mathcal C}$.
In particular, this applies to the empirical GMM with
$\mu=\mu^*$ and $\Sigma=\widehat\Sigma_N+\delta^2I_d$, without a rank assumption.
\end{lemma}

\begin{proof}
Expanding the Gaussian kernel and averaging over the smoothing noise gives
\[
\widetilde V(y)=\tfrac12y^\top Qy+a_0^\top y-\ell(y)+C,
\quad
\ell(y)=\mathbb E_\varepsilon\log\int
\exp\left\{\frac{x^\top(\mu+Ly+\sigma\varepsilon)}{\delta^2}
-\frac{\|x\|^2}{2\delta^2}\right\}\,\xi(dx).
\]
Bounded support makes $\ell$ smooth, convex, and globally Lipschitz:
its gradient is an average of $L^\top x/\delta^2$.
Consequently $|\ell(y)|\le C_1+C_2\|y\|$, while convexity gives the
supporting-plane bound $\ell(y)\ge\ell(0)+b^\top y$, $b=\nabla\ell(0)$.
Writing $A=Q-B$ and $\widetilde a=a-a_0$, the partition integral is finite
exactly when $A\succ0$. Sufficiency follows from the linear growth bound.
For necessity, the supporting plane bounds the integrand below by a
positive constant times
$\exp\{(\widetilde a+b)^\top y-\tfrac12y^\top Ay\}$.
Diagonalization and Tonelli's theorem show that this Gaussian integral is
infinite whenever $A$ has a nonpositive eigenvalue.

The dual objective for moments $(0,I_d)$, up to a constant, is
\[
\Phi(\widetilde a,A)
=\log\int e^{\widetilde a^\top y-\frac12y^\top Ay+\ell(y)}\,dy
+\tfrac12\operatorname{tr}A,\qquad A\succ0.
\]
The same supporting plane and the Gaussian integral formula yield
\[
\Phi(\widetilde a,A)\ge C_0+
\tfrac12\bigl(\operatorname{tr}A-\log\det A\bigr)
+\tfrac12(\widetilde a+b)^\top A^{-1}(\widetilde a+b).
\]
Since $t-\log t\to\infty$ as $t\downarrow0$ or $t\to\infty$, each sublevel
set has eigenvalues of $A$ bounded above and away from zero, and then
$\widetilde a$ bounded. Thus the continuous dual has compact sublevel sets
inside its domain and attains its minimum there.
Its Hessian in a direction $(h,K)\ne0$, $K=K^\top$, is
$\operatorname{Var}(h^\top Y-\tfrac12Y^\top KY)>0$ under the tilted density:
this polynomial cannot be constant on $\mathbb R^d$. The minimizer is
therefore unique. Differentiation under a locally uniform Gaussian envelope
gives $\nabla_{\widetilde a}\Phi=\mathbb E Y$ and
$\nabla_A\Phi=\tfrac12(I_d-\mathbb E YY^\top)$, proving the moment equations.
Finally, for every $\rho\in\widetilde{\mathcal C}$,
\[
\mathrm{KL}(\rho\|\widetilde q)
=\mathrm{KL}(\rho\|\widetilde r)
+\mathrm{KL}(\widetilde r\|\widetilde q),
\]
because $\log(\widetilde r/\widetilde q)$ is a quadratic polynomial whose
expectation is fixed by the constraints. The identity also holds with
infinite relative entropies, since that polynomial is integrable under
any feasible law. This proves the projection claim.
\end{proof}

Since \((a_*,B_*)\in\mathcal O_{\rm w}^\circ\) and \(\widetilde V\) is quadratically
confining, \(\widetilde r\) has Gaussian tails: there exists \(\eta>0\) such
that
\[
\int_{\mathbb R^d}e^{\eta\|y\|^2}\widetilde r(y)\,dy<\infty .
\]
For a finite common-covariance Gaussian mixture, derivatives of
\(\log\hat p^\delta\) have at most polynomial growth, and Gaussian smoothing
preserves these bounds. Since the tilted density has a Gaussian envelope,
differentiating \(\widetilde r\) produces only polynomial factors times a
Gaussian envelope. Hence, for every multi-index \(\alpha\), there exist
\(C_\alpha,c_\alpha>0\) such that
\[
|\partial^\alpha\widetilde r(y)|
\le
C_\alpha e^{-c_\alpha\|y\|^2}.
\]

\paragraph{Step 3: One-particle local density-ratio estimate.}
Identify \(\mathbb R^d\times\mathrm{Sym}_d\) with \(\mathbb R^m\), where
\[
m=d+\frac{d(d+1)}2 .
\]
Let
\[
W=\widetilde T(Y)=(Y,YY^\top),
\qquad
Y\sim\widetilde r .
\]

\begin{lemma}[Fourier smoothing for the mean--second-moment statistic]
\label{lem:fourier_smoothing_mean_cov}
Let \(W=(Y,YY^\top)\), with \(Y\sim\widetilde r\), and write its
characteristic function with quadratic frequency \(A/2\) as
\[
\phi_W(u,A)=
\int e^{iu^\top y+\frac i2y^\top Ay}\widetilde r(y)\,dy .
\]
There is \(C<\infty\) such that
\[
|\phi_W(u,A)|\le
C(1+\|u\|+\|A\|_{\mathrm{op}})^{-1/4}.
\]
In particular, \(\phi_W\to0\) at infinity and
\(\phi_W^{\,n_0}\in L^1(\mathbb R^m)\) for every integer \(n_0>4m\).
\end{lemma}

\begin{proof}
Put \(U=\|u\|\) and \(L=\|A\|_{\mathrm{op}}\).
If \(L\ge\max\{1,\sqrt U\}\), integrate first in a unit eigenvector
direction with eigenvalue of absolute value \(L\).
The one-dimensional second-derivative van der Corput bound
\cite[Chapter~VIII]{stein1993harmonic}, followed by integration over the
remaining coordinates, gives \(|\phi_W|\le CL^{-1/2}\).
The constant is uniform in the eigenvector because the orthogonal orbit
of \(\widetilde r\) has bounded Schwartz seminorms.
Since \(1+U+L\le3L^2\), this implies the stated bound.

Otherwise bounded frequencies are harmless, so assume \(U>1\) and
\(L<\sqrt U\). Take a smooth cutoff equal to one for
\(\|y\|\le\sqrt U/4\) and zero for \(\|y\|\ge\sqrt U/2\).
For the phase \(\varphi(y)=u^\top y+y^\top Ay/2\) and direction
\(e=u/U\), one has \(|\partial_e\varphi|\ge U/2\) on its support.
One integration by parts gives an \(O(U^{-1})\) bound for the cutoff
integral: the amplitude and its directional derivative have uniformly
bounded \(L^1\) norms, and
\(|\partial_e^2\varphi|\le L\).
The remaining integral is \(O(U^{-1})\) by Schwartz decay.
This proves the bound also in the second case.
Finally, \(\|u\|+\|A\|_{\mathrm{op}}\) is equivalent to a Euclidean norm
on the \(m\)-dimensional frequency space, so raising the bound to a power
\(n_0>4m\) gives integrability.
\end{proof}

\begin{lemma}[Local CLT and one-particle density ratio]
\label{lem:local_clt_density_ratio}
Let \(W\in\mathbb R^m\) have mean \(t_0\), finite nonsingular covariance \(\Gamma\),
and satisfy
$\sup_{\|t\|\ge r}|\phi_W(t)|<1$ for every $r>0$.
Suppose
\[
\phi_W^{\,n_0}\in L^1(\mathbb R^m)
\]
for some \(n_0\). For independent copies \(W_i\) of \(W\) and all
sufficiently large \(n\), the sum \(\sum_{i=1}^n W_i\) has a bounded
continuous density \(f_n\) satisfying the local Gaussian expansion
\[
f_n(nt_0+x)
=
\frac{1+o(1)}
{(2\pi n)^{m/2}\det(\Gamma)^{1/2}}
\exp\left(-\frac12x^\top(n\Gamma)^{-1}x\right)
\]
uniformly for \(\|x\|=o(\sqrt n)\). Moreover,
\[
\sup_x f_n(x)\le Cn^{-m/2},
\qquad
f_n(nt_0)\asymp n^{-m/2}.
\]
Consequently, for any fixed random vector \(\Delta\),
\[
\frac{f_{P-1}((P-1)t_0-\Delta)}{f_P(Pt_0)}
\to 1
\qquad
\text{in }L^1 .
\]
\end{lemma}

\begin{proof}[Proof of Lemma~\ref{lem:local_clt_density_ratio}]
Here is the Fourier argument for the local CLT; see also
\cite[Section~19]{bhattacharya1976normal}.
For the centered characteristic function \(\psi\),
\(\psi(t)=1-\tfrac12t^\top\Gamma t+o(\|t\|^2)\).
On a small neighborhood of zero,
\(|\psi(t)|\le e^{-c\|t\|^2}\).
Thus \(\psi(t/\sqrt n)^n\) converges in \(L^1\) on the growing region
\(\|t\|\le r\sqrt n\) to \(e^{-t^\top\Gamma t/2}\), by dominated
convergence. Outside that region its \(L^1\) norm is at most
\[
n^{m/2}q_r^{\,n-n_0}\int|\psi(s)|^{n_0}\,ds\longrightarrow0,
\qquad q_r=\sup_{\|s\|\ge r}|\psi(s)|<1.
\]
Fourier inversion gives a uniform Gaussian approximation for the density
of \(n^{-1/2}\sum_i(W_i-t_0)\), implying all the stated density bounds.
For every fixed value of \(\Delta\), these bounds give convergence of the
density ratio to one. The same bounds keep the ratio uniformly bounded
over all \(\Delta\) for large \(P\), so dominated convergence proves its
\(L^1\) convergence.
\end{proof}

By Lemma~\ref{lem:fourier_smoothing_mean_cov},
\[
\phi_W^{\,n_0}\in L^1(\mathbb R^m)
\]
for some \(n_0\). The statistic \(W\) has finite covariance because
\(\widetilde r\) has Gaussian tails. Its covariance is nonsingular: if
\[
u^\top Y+\frac12Y^\top AY
\]
has zero variance for \(A=A^\top\), then this quadratic polynomial is constant
on the support of \(\widetilde r\). Since \(\widetilde r\) is strictly
positive on all of \(\mathbb R^d\), the polynomial is constant on an open set,
so \(u=0\) and \(A=0\).

The required strict nonlattice property holds: if
\[
|\phi_W(u,A)|=1
\]
for some \((u,A)\ne(0,0)\), then
\[
u^\top Y+\frac12Y^\top AY
\]
is constant modulo \(2\pi\) almost surely. Since \(Y\) has a strictly positive
density on \(\mathbb R^d\), this polynomial would be constant modulo \(2\pi\)
on an open set, hence constant as an ordinary polynomial, forcing \(u=0\) and
\(A=0\), a contradiction.
Together with the decay at infinity in
Lemma~\ref{lem:fourier_smoothing_mean_cov} and continuity, this proves
$\sup_{\|t\|\ge r}|\phi_W(t)|<1$ for every $r>0$.

Thus Lemma~\ref{lem:local_clt_density_ratio} applies. Since
\[
\mathbb EW=t_0=(0,I_d),
\]
and since
\[
\Delta=\widetilde T(Y)-t_0
\]
is a fixed random vector, we obtain
\[
\frac{
f_{P-1}\left((P-1)t_0-(\widetilde T(Y)-t_0)\right)
}{
f_P(Pt_0)
}
\to 1
\qquad
\text{in }L^1(\widetilde r).
\]
This is the one-particle density-ratio estimate used below.

\paragraph{Step 4: Coarea conditioning and the constrained Gibbs measure.}
Writing $S_P(Y)=\sum_{i=1}^P\widetilde T(Y_i)$, the constraint manifold is
\[
\mathcal M_P
=
\left\{
Y\in\mathbb R^{P\times d}:
\sum_{i=1}^P Y_i=0,\;
\sum_{i=1}^P Y_iY_i^\top=PI_d
\right\}
=
\{Y:S_P=Pt_0\}.
\]
Define
\[
\Psi_P(Y)
=
\left(
\frac1P Y^\top\mathbf 1_P,\,
\frac1P Y^\top Y
\right).
\]
The first component is viewed as a vector in \(\mathbb R^d\). Its differential
in direction \(H\in\mathbb R^{P\times d}\) is
\[
D\Psi_P(Y)[H]
=
\left(
\frac1P H^\top\mathbf 1_P,\,
\frac1P(H^\top Y+Y^\top H)
\right).
\]
For \(P\ge d+1\), \((0,I_d)\) is a regular value of \(\Psi_P\). Indeed, for
arbitrary \(u\in\mathbb R^d\) and \(S=S^\top\in\mathrm{Sym}_d\), take
\[
H=\mathbf 1_Pu^\top+\frac12YS .
\]
Using \(\mathbf 1_P^\top Y=0\) and \(Y^\top Y=PI_d\), we get
\[
\frac1P H^\top\mathbf 1_P=u,
\qquad
\frac1P(H^\top Y+Y^\top H)=S .
\]
Thus \(D\Psi_P(Y)\) is surjective on \(\mathcal M_P\).

By the coarea formula, conditioning \(\widetilde q^{\otimes P}\) on
\(\Psi_P(Y)=(0,I_d)\) gives a surface measure on \(\mathcal M_P\) with density
proportional to
\[
\exp\left(-\sum_{i=1}^P\widetilde V(y_i)\right)
\]
times the inverse coarea Jacobian. This Jacobian is constant on
\(\mathcal M_P\). Indeed, if \(O\) is an orthogonal map on \(\mathbb R^P\)
satisfying \(O\mathbf 1_P=\mathbf 1_P\), then
\[
\Psi_P(OY)=\Psi_P(Y),
\]
and \(O\) preserves the Frobenius metric. Hence
\[
D\Psi_P(OY)\circ O=D\Psi_P(Y),
\]
so the coarea Jacobian is invariant.

The action is transitive on \(\mathcal M_P\). Let
\(U\in\mathbb R^{P\times(P-1)}\) have orthonormal columns spanning
\(\mathbf 1_P^\perp\). Since \(P-1\ge d\), any \(Y,Y'\in\mathcal M_P\) can be
written as
\[
Y=\sqrt P\,UX,
\qquad
Y'=\sqrt P\,UX',
\qquad
X,X'\in V_{P-1,d}.
\]
Since \(O(P-1)\) acts transitively on \(V_{P-1,d}\), there exists
\(Q\in O(P-1)\) such that \(QX=X'\). Setting
\[
O=UQU^\top+\frac1P\mathbf 1_P\mathbf 1_P^\top
\]
gives an orthogonal map fixing \(\mathbf 1_P\) and satisfying \(OY=Y'\).
Therefore the coarea Jacobian is constant, and the exact conditioned law is
\[
\mu_P(dy_{1:P})
=
\frac1{Z_P}
\exp\left(-\sum_{i=1}^P\widetilde V(y_i)\right)
\mathcal U_{\mathcal M_P}(dy_{1:P}) .
\]
We use this coarea disintegration as the version of the density-based regular
conditional law at the regular value \(S_P=Pt_0\).

\paragraph{Step 5: One-particle equivalence of ensembles.}
The Radon--Nikodym derivative of \(\widetilde r^{\otimes P}\) with respect to
\(\widetilde q^{\otimes P}\) is
\[
\frac{d\widetilde r^{\otimes P}}{d\widetilde q^{\otimes P}}(Y_{1:P})
=
\exp\left(
\sum_{i=1}^P a_*^\top Y_i
+
\frac12\sum_{i=1}^P Y_i^\top B_*Y_i
+\mathrm{const.}_P
\right).
\]
On \(\mathcal M_P\), this factor is constant because
\[
\sum_{i=1}^P Y_i=0,
\qquad
\sum_{i=1}^P Y_iY_i^\top=PI_d .
\]
Therefore the exact conditional laws of \(\widetilde q^{\otimes P}\) and
\(\widetilde r^{\otimes P}\) on \(\mathcal M_P\) coincide.

For sufficiently large \(P\), \(f_P,f_{P-1}\) are bounded and continuous,
and \(f_P(Pt_0)>0\). The conditional-density formula holds almost everywhere
in the constraint value. It also holds at \(Pt_0\): nearby fibers are
affine images of a fixed compact Stiefel manifold with positive definite
centered Gram matrix, so their coarea integrals are continuous in that
value. The density-ratio expression is continuous there by dominated
convergence. Thus, for bounded continuous
\(\varphi:\mathbb R^d\to\mathbb R\),
\[
\mathbb E_{\mu_P}[\varphi(Y_1)]
=
\mathbb E_{\widetilde r}
\left[
\varphi(Y)
\frac{
f_{P-1}\left(Pt_0-\widetilde T(Y)\right)
}{
f_P(Pt_0)
}
\right].
\]
Since
\[
Pt_0-\widetilde T(Y)
=
(P-1)t_0-(\widetilde T(Y)-t_0),
\]
the density-ratio estimate from Step 3 gives
\[
\mathbb E_{\mu_P}[\varphi(Y_1)]
\to
\mathbb E_{\widetilde r}[\varphi(Y)] .
\]
Hence
\[
\mu_P^{Y,(1)}\Rightarrow \widetilde r .
\]
Pushing forward by \(z=\mu^*+L^*y\), we obtain
\[
\mu_P^{Z,(1)}\Rightarrow
\rho_N:=(\mu^*+L^*\cdot)_\#\widetilde r .
\]
We verify its variational characterization next.

\paragraph{Step 6: Identification in \(Z\)-space and multiplier identities.}
The pushforward \(\rho_N=(\mu^*+L^*\cdot)_\#\widetilde r\) has mean
\(\mu^*\) and covariance \(\Sigma^*\). Moreover,
\[
\rho_N(z)
=
\frac1{Z(\lambda,\Lambda)}
\exp\left(
-V(z)-\lambda^\top(z-\mu^*)
-\frac12(z-\mu^*)^\top\Lambda(z-\mu^*)
\right),
\]
where
\[
\lambda=-(L^*)^{-\top}a_*,
\qquad
\Lambda=-(L^*)^{-\top}B_*(L^*)^{-1}.
\]
The constant Jacobian is absorbed into $Z(\lambda,\Lambda)$.

Let \(q(dz)=Z_q^{-1}e^{-V(z)}\,dz\). For any
\(\rho\in\mathcal C(\mu^*,\Sigma^*)\) with finite relative entropy with
respect to \(q\),
\[
\mathrm{KL}(\rho\,\|\,q)
=
\mathrm{KL}(\rho\,\|\,\rho_N)
+
\int\log\left(\frac{d\rho_N}{dq}\right)\,d\rho .
\]
Since
\[
\log\left(\frac{d\rho_N}{dq}\right)
=
-\lambda^\top(z-\mu^*)
-\frac12(z-\mu^*)^\top\Lambda(z-\mu^*)
+\mathrm{const.},
\]
the last integral is constant over \(\mathcal C(\mu^*,\Sigma^*)\). Indeed,
all laws in \(\mathcal C(\mu^*,\Sigma^*)\) have mean \(\mu^*\) and covariance
\(\Sigma^*\), so
\[
\int (z-\mu^*)^\top\Lambda(z-\mu^*)\,d\rho(z)
=
\operatorname{tr}(\Lambda\Sigma^*) .
\]
Thus
\[
\mathrm{KL}(\rho\,\|\,q)
=
\mathrm{KL}(\rho\,\|\,\rho_N)+\mathrm{const.},
\]
and the unique minimizer is \(\rho_N\). Since
\[
\mathrm{KL}(\rho\,\|\,q)
=
\int V(z)\rho(z)\,dz
+
\int \rho(z)\log\rho(z)\,dz
+
\log Z_q
\]
for absolutely continuous \(\rho\), and the entropy is \(+\infty\) for
singular measures, this is equivalent to minimizing \(\mathcal F\) over
\(\mathcal C(\mu^*,\Sigma^*)\).

It remains to derive the identities for \(\lambda\) and \(\Lambda\). The law
\(\rho_N\) has Gaussian tails, and \(g\) has at most linear growth. Therefore,
using compactly supported smooth cutoffs and then sending the cutoff radius to
infinity, the standard integration-by-parts identities hold:
\[
\mathbb E_{\rho_N}[\nabla\log\rho_N(Z)]=0,
\]
and
\[
\mathbb E_{\rho_N}
\left[
(Z-\mu^*)\nabla\log\rho_N(Z)^\top
\right]
=
-I_d .
\]
Since
\[
\nabla\log\rho_N(z)
=
-g(z)-\lambda-\Lambda(z-\mu^*),
\]
the first identity gives
\[
\lambda=-\mathbb E_{\rho_N}[g(Z)] .
\]
The second gives the stronger population relation
\[
\Sigma^*\Lambda=I_d-C,
\qquad
C=\mathbb E_{\rho_N}\left[(Z-\mu^*)g(Z)^\top\right].
\]
Since \(\Lambda\) is symmetric, transposing and adding gives
\[
\Sigma^*\Lambda+\Lambda\Sigma^*
=
2I_d-(C+C^\top)
=
2\left(I_d-\operatorname{sym}(C)\right).
\]
Finally, at $\sigma=0$ the base law is $q=\hat p^\delta$, which already
has mean $\mu^*$ and covariance $\Sigma^*=\widehat\Sigma_N+\delta^2I_d$.
Uniqueness of the moment projection then gives $\rho_N=\hat p^\delta$
and $\lambda=\Lambda=0$.
This completes the proof.

\subsection{A conditional local step-size scale}
\label{app:proof_stepsize_bound}
Let $U^\delta=-\log\hat p^\delta$ and let $\Gamma_z$ be the covariance of the centers under their posterior softmax weights at $z$. Differentiating the posterior mean gives
\begin{equation}
\nabla^2U^\delta(z)=\delta^{-2}I_d-\delta^{-4}\Gamma_z,
\qquad
\nabla^2V(z)=\delta^{-2}I_d-\delta^{-4}\mathbb E\Gamma_{z+\sigma\varepsilon}.
\label{eq:gmm_hessian_bound}
\end{equation}
The interchange is valid because the finite set of centers bounds all posterior covariance matrices. If, on a region of interest,
$\|\mathbb E\Gamma_{z+\sigma\varepsilon}\|_{\rm op}\le c_0\delta^2$
with $c_0$ independent of $\delta$, then
\begin{equation}
\|\nabla^2V(z)\|_{\rm op}\le(1+c_0)\delta^{-2}.
\label{eq:local_gmm_hessian_bound}
\end{equation}
This is an additional locality assumption, not a consequence of Gaussian smoothing: without it, the general bound is $\delta^{-2}+R^2\delta^{-4}$ for $R=\max_i\|x_i\|$.

For fixed data and $P$, and $\delta$ bounded above, the GMM convention keeps $\|L^*\|_{\rm op}$ bounded. The pullback multiplies Hessians by $L^*$ and its transpose. Differentiating the tangent projection adds a term controlled by $\|g(z)\|\le(\|z\|+R)/\delta^2$ on the compact constraint manifold. Thus the projected drift has local derivative scale $O(\delta^{-2})$ under the stated locality assumption, with geometry-dependent constants. For an independently sampled tilted target, the same scale requires additionally $\|\Lambda\|_{\rm op}=O(\delta^{-2})$. These calculations motivate the experimental choice $h=c\delta^2$; they do not establish a maximal stable step size or a stability theorem for the LM scheme.

\subsection{Estimation of the tilting parameters}
\label{app:tilting-parameter-estimation}

The tilting parameters $\lambda$ and $\Lambda$ appearing in Proposition~\ref{prop:large_particle_target} are defined by population identities under the moment-matched target. With $g(z)=\nabla V(z)$, they satisfy
\begin{equation}
\lambda
=
-\mathbb E_{\rho_N}
\left[g(Z)\right],
\qquad
C
=
\mathbb E_{\rho_N}
\left[
(Z-\mu^*)g(Z)^\top
\right],
\label{eq:appendix-population-tilting-parameters}
\end{equation}
with $\Lambda$ recovered from $C$ through the Lyapunov equation
\begin{equation}
\Sigma^*\Lambda+\Lambda\Sigma^*
=
2(I_d-\sym(C)).
\label{eq:appendix-population-lyapunov}
\end{equation}
The expectations in~\eqref{eq:appendix-population-tilting-parameters} are taken under $\rho_N$ itself, which is unavailable in practice. We give a consistent estimator based on safeguarded stochastic optimization of the convex dual, with moments estimated by BAOAB sampling. We then discuss a computationally cheap but generally biased closed-form plug-in estimator.

\paragraph{A consistent dual estimator.}
Consider the centered parametrization of the moment-matched target
\begin{equation}
p_{\lambda,\Lambda}(z)
=
\frac{1}{Z(\lambda,\Lambda)}
\exp\left\{
-V(z)
-\lambda^\top(z-\mu^*)
-\frac{1}{2}(z-\mu^*)^\top\Lambda(z-\mu^*)
\right\},
\label{eq:appendix-centered-tilted-density}
\end{equation}
using the same centered convention as~\eqref{eq:moment_matched_target}. The desired parameters are those for which
\begin{equation}
\mathbb E_{p_{\lambda,\Lambda}}[Z]=\mu^*,
\qquad
\mathbb E_{p_{\lambda,\Lambda}}
\left[
(Z-\mu^*)(Z-\mu^*)^\top
\right]
=
\Sigma^* .
\label{eq:appendix-moment-constraints-dual}
\end{equation}
Equivalently, they minimize the convex dual objective
\begin{equation}
\mathcal D(\lambda,\Lambda)
=
\log Z(\lambda,\Lambda)
+
\frac{1}{2}\operatorname{tr}(\Lambda\Sigma^*),
\label{eq:appendix-dual-objective}
\end{equation}
where
\begin{equation}
Z(\lambda,\Lambda)
=
\int
\exp\left\{
-V(z)
-\lambda^\top(z-\mu^*)
-\frac{1}{2}(z-\mu^*)^\top\Lambda(z-\mu^*)
\right\}
dz .
\label{eq:appendix-dual-normalizer}
\end{equation}
The gradients of the dual objective are
\begin{equation}
\nabla_\lambda \mathcal D
=
-\mathbb E_{p_{\lambda,\Lambda}}[Z-\mu^*],
\label{eq:appendix-dual-gradient-lambda}
\end{equation}
and
\begin{equation}
\nabla_\Lambda \mathcal D
=
\frac{1}{2}
\left[
\Sigma^*
-
\mathbb E_{p_{\lambda,\Lambda}}
\left[
(Z-\mu^*)(Z-\mu^*)^\top
\right]
\right].
\label{eq:appendix-dual-gradient-Lambda}
\end{equation}
Hence the first-order optimality conditions of~\eqref{eq:appendix-dual-objective} are exactly the required mean and covariance constraints in~\eqref{eq:appendix-moment-constraints-dual}. Lemma~\ref{lem:moment_realization_log_gmm} gives a unique interior optimizer for these GMM moments.

We now make the sampling and optimization limits precise, conditional on
the fixed training set. Use the Euclidean--Frobenius norm on
$\mathbb R^d\times\mathrm{Sym}_d$, and write
\[
\theta=(\lambda,\Lambda),\qquad
\mathsf S(z)=\left(z-\mu^*,\tfrac12(z-\mu^*)(z-\mu^*)^\top\right),
\qquad s_*=(0,\Sigma^*/2).
\]
Then $\nabla\mathcal D(\theta)=s_*-\mathbb E_{p_\theta}\mathsf S$.
Let $\theta_*$ be the unique optimizer, so $p_{\theta_*}=\rho_N$. Fix a compact convex parameter set
$\mathcal K$ containing $\theta_*$ in its interior and satisfying
\begin{equation}
\mathcal K\subset\mathcal O
:=\{(\lambda,\Lambda):\delta^{-2}I_d+\Lambda\succ0\}.
\label{eq:baoab-parameter-safeguard}
\end{equation}
Such a set exists by Lemma~\ref{lem:moment_realization_log_gmm} and the
change of coordinates in Proposition~\ref{prop:large_particle_target}.
Containment of $\theta_*$ is an assumption on the safeguard, not an
assertion that an arbitrary bounded parameter set suffices.

\paragraph{BAOAB approximation of the dual moments.}
Set $U_\theta(z)=V(z)+\lambda^\top(z-\mu^*)+
\tfrac12(z-\mu^*)^\top\Lambda(z-\mu^*)$ and fix friction $\gamma>0$.
The kinetic Langevin diffusion is
\[
dQ_t=P_t\,dt,\qquad
dP_t=-\nabla U_\theta(Q_t)\,dt-\gamma P_t\,dt+\sqrt{2\gamma}\,dW_t,
\]
with invariant law $p_\theta\otimes\mathcal N(0,I_d)$.
Initialize each chain independently from a fixed law $\chi$ on phase
space with finite fourth moment, for example the training GMM times
$\mathcal N(0,I_d)$. Set $a_h=e^{-\gamma h}$. One BAOAB
step~\cite{leimkuhler2013robust}, with a stored force $\widehat F_j$, is
\begin{align}
p_j^-&=p_j-\tfrac h2\widehat F_j,
&q_j^{1/2}&=q_j+\tfrac h2p_j^-,\notag\\
p_j^+&=a_hp_j^-+\sqrt{1-a_h^2}\,\xi_j,
&q_{j+1}&=q_j^{1/2}+\tfrac h2p_j^+,\notag\\
p_{j+1}&=p_j^+-\tfrac h2\widehat F_{j+1},
&\xi_j&\sim\mathcal N(0,I_d).
\label{eq:baoab-dual-step}
\end{align}
Here $\widehat F_0=\widehat F_\theta(q_0)$ and
$\widehat F_{j+1}=\widehat F_\theta(q_{j+1})$, where either
$\widehat F_\theta=\nabla U_\theta$ exactly, or
$\widehat F_\theta(z)=\widehat g_M(z)+\lambda+\Lambda(z-\mu^*)$,
with $\widehat g_M$ the full antithetic estimator of
Appendix~\ref{app:approx_finite_step_mmsold}, for fixed even $M\ge2$.
Every new force evaluation uses fresh smoothing draws, independent of
the initialization, all previous draws and the Ornstein--Uhlenbeck noises. The closing
force $\widehat F_{j+1}$ is retained for the next opening half-kick,
as in the experiments: there is one new force evaluation per step,
plus the initial evaluation. Let $\nu_{\theta,T,h}$ be the position law
after $\nsteps=T/h$ steps, and let $\overline{\mathsf S}_{B,T,h}$ average
$\mathsf S$ over $B$ independent endpoints.

\begin{lemma}[Uniform moment approximation for BAOAB]
\label{lem:baoab-dual-moments}
Under the assumptions of Proposition~\ref{prop:large_particle_target}
and~\eqref{eq:baoab-parameter-safeguard}, there are $C,c>0$ such that,
for every $T\ge1$, there exist $C_T<\infty$ and $h_T>0$ for which
\begin{align}
\sup_{\theta\in\mathcal K}
\mathbb E\|\overline{\mathsf S}_{B,T,h}-\mathbb E_{p_\theta}\mathsf S\|^2
&\le \frac CB+(Ce^{-cT}+C_T\sqrt h)^2,\notag\\
\sup_{\theta\in\mathcal K}W_2(\nu_{\theta,T,h},p_\theta)
&\le Ce^{-cT/2}+C_T\sqrt h
\label{eq:baoab-dual-errors}
\end{align}
whenever $h=T/\nsteps\le h_T$ and $B,\nsteps\ge1$ are integers.
The constants may depend on the fixed data, $\mathcal K,\delta,\sigma,\gamma,M$ and $\chi$;
no uniform complexity bound is asserted.
\end{lemma}

\begin{proof}
Write $A_\theta=\delta^{-2}I_d+\Lambda$. The GMM score formula gives
\[
\nabla U_\theta(q)=A_\theta q+r_\theta(q),\qquad
\sup_{\theta,q}\|r_\theta(q)\|<\infty,
\qquad aI_d\preceq A_\theta\preceq bI_d
\]
for some $a,b>0$. Its derivatives of positive order are uniformly bounded
on $\mathcal K\times\mathbb R^d$, since derivatives of the finite-mixture
posterior mean are bounded cumulants of the centers. In particular the
force is uniformly globally Lipschitz. The full antithetic force error is
conditionally centered and uniformly bounded by $2\max_i\|x_i\|/\delta^2$.

We first justify uniform mixing of the exact diffusion. The matrices
\[
B_\theta=\begin{pmatrix}0&I_d\\-A_\theta&-\gamma I_d\end{pmatrix}
\]
are uniformly stable on $\mathcal K$. The positive definite solutions
$H_\theta$ of $B_\theta^\top H_\theta+H_\theta B_\theta=-I_{2d}$
therefore have uniform positive lower and upper bounds.
For $x=(q,p)$ put $\mathcal V_\theta(x)=1+x^\top H_\theta x$.
The bounded remainder $r_\theta$ gives, for the kinetic generator,
\[
\mathcal L_\theta\mathcal V_\theta\le-c_1\mathcal V_\theta+C_1,
\qquad
\mathcal L_\theta\mathcal V_\theta^2\le-c_2\mathcal V_\theta^2+C_2,
\]
uniformly in $\theta$. Noise in every momentum coordinate and its
commutator with the drift span phase space. Moreover, any two phase
points can be joined in unit time by prescribing a smooth position path
and its velocity, with control
$(\ddot q+\gamma\dot q+\nabla U_\theta(q))/\sqrt{2\gamma}$.
The endpoint map has full rank: a smooth variation $\eta$ with
$\eta(0)=\dot\eta(0)=0$ and arbitrary $(\eta(1),\dot\eta(1))$ is realized
by the linearized control
$(\ddot\eta+\gamma\dot\eta+\nabla^2U_\theta(q)\eta)/\sqrt{2\gamma}$.
The hypoelliptic density and control argument therefore gives a positive transition density, jointly continuous
in the parameter and endpoints; compactness gives a common minorization
on each compact set. The Lyapunov--minorization theorem
\cite[Theorem~2.5 and Section~3]{mattingly2002ergodicity}, with these
uniform constants, yields
\[
\sup_{\theta\in\mathcal K}\sup_{|f(x)|\le1+\|x\|^2}
\left|\mathbb E_\chi f(Q_T,P_T)
-\mathbb E_{p_\theta\otimes\mathcal N(0,I_d)}f\right|
\le Ce^{-cT}.
\]
The second Lyapunov inequality also bounds fourth moments uniformly in
$T$. The weighted bound controls the bias of $\mathsf S$ and implies
$W_2(\mathcal L(Q_T),p_\theta)\le Ce^{-cT/2}$: couple the common mass
identically and use $\|x-y\|^2\le2\|x\|^2+2\|y\|^2$ on the remainder.

For completeness, finite-time numerical convergence also holds with
stored random forces. Set $v_0=p_0$ and, for $j\ge1$,
$v_j=p_{j-1}^+=p_j+h\widehat F_j/2$. The state $(q_j,v_j)$ is determined
before drawing the fresh force at $q_j$. Grouping adjacent half-kicks gives
\[
u_j=v_j-\alpha_jh\widehat F_j,\qquad
v_{j+1}=a_hu_j+\sqrt{1-a_h^2}\,\xi_j,\qquad
q_{j+1}=q_j+\tfrac h2(u_j+v_{j+1}),
\]
where $\alpha_0=1/2$ and $\alpha_j=1$ for $j\ge1$.
Thus $\eta_j=\widehat F_j-\nabla U_\theta(q_j)$ is centered conditional
on the history $\mathcal F_j$ just before this draw, and is uniformly
bounded. Couple the Ornstein--Uhlenbeck noises to the diffusion's
Brownian motion. Writing $X_j^h=(q_j,v_j)$,
$b_\theta(q,v)=(v,-\nabla U_\theta(q)-\gamma v)$ and
$G=(0,\sqrt{2\gamma}I_d)^\top$, expansion gives
\[
X^h_{j+1}=X^h_j+h b_\theta(X^h_j)+G\Delta W_j+h\zeta_j+R_j
             +\mathbf1_{\{j=0\}}D_0,
\]
where $\zeta_j=(0,-\alpha_j\eta_j)$,
$D_0=(0,\tfrac h2\nabla U_\theta(q_0))$, and
$\mathbb E[\zeta_j\mid\mathcal F_j]=0$.
For $h\le1$, global Lipschitzness, $a_h=1-\gamma h+O(h^2)$, and the
$O(h^{3/2})$ $L^4$ difference between the exact Ornstein--Uhlenbeck
noise and $\sqrt{2\gamma}\Delta W_j$ give
\[
\mathbb E[\|R_j\|^4\mid\mathcal F_j]^{1/4}
\le Ch^{3/2}(1+\|X_j^h\|).
\]
The initial correction satisfies
$\|D_0\|_{L^4}\le Ch(1+\|X_0^h\|_{L^4})$.
The grouped update is a linear recursion plus bounded force corrections
and Gaussian noise, so its fourth moments are bounded on each fixed
time interval, uniformly in $h$ and $\theta$.
The martingale maximal inequality bounds the accumulated $h\zeta_j$
in $L^4$ by $C_T\sqrt h$, and the sum of the $L^4$ remainder norms
has the same bound. Comparison with Euler--Maruyama and discrete
Gronwall therefore give
\[
\sup_{\theta\in\mathcal K}
\|X^h_{\nsteps}-(Q_T,P_T)\|_{L^4}\le C_T\sqrt h.
\]
The terminal half-kick changes $v_{\nsteps}$ to $p_{\nsteps}$ by only
$O(h)$ in $L^4$, so the same bound holds for the actual BAOAB endpoint.
Here the Euler bound follows directly from the $L^4$ increment estimate
$\|X_t-X_s\|_{L^4}\le C_T\sqrt{|t-s|}$ and Lipschitzness of $b_\theta$.
Since $\mathsf S$ is quadratic, H\"older's inequality gives the same
order for its expectation error. Choose $h_T$ small enough that the
endpoint fourth moments remain bounded by a constant independent of
$T$. Independence of the $B$ chains then bounds the variance of their
average by $C/B$. Combining the mixing, numerical and averaging bounds
proves~\eqref{eq:baoab-dual-errors}.
\end{proof}

\begin{proposition}[Consistent dual estimation with BAOAB]
\label{prop:baoab-dual-consistency}
Under the conditions of Lemma~\ref{lem:baoab-dual-moments}, there are
deterministic budgets $T_k\to\infty$, $h_k=T_k/n_{\mathrm{steps},k}\to0$ and $B_k\to\infty$
such that the following estimator is strongly consistent.
Starting from any $\theta_0\in\mathcal K$, at iteration $k$ run $B_k$
fresh independent BAOAB chains at $\theta_k$ from $\chi$, and set
\begin{equation}
\widehat g_k=s_*-\overline{\mathsf S}_{B_k,T_k,h_k},\qquad
\theta_{k+1}=\operatorname{proj}_{\mathcal K}
\left(\theta_k-\frac{\widehat g_k}{k+1}\right).
\label{eq:appendix-robbins-monro-update}
\end{equation}
Then $\theta_k\to\theta_*$ almost surely. Moreover, almost surely,
\[
W_2(p_{\theta_k},p_{\theta_*})\to0,
\qquad W_2(\nu_{\theta_k,T_k,h_k},p_{\theta_*})\to0.
\]
In particular the means and covariances of both laws converge to
$\mu^*$ and $\Sigma^*$, as do the empirical means and covariances of the
endpoint batches.
\end{proposition}

\begin{proof}
Uniformly on $\mathcal K$,
$U_\theta(z)=\tfrac12z^\top A_\theta z+O(1+\|z\|)$, and the normalizers
are bounded above and away from zero. Thus the tilted densities have a
common Gaussian envelope. Differentiating the dual is valid, and
$\nabla^2\mathcal D(\theta)=\operatorname{Cov}_{p_\theta}(\mathsf S)$.
This matrix is positive definite: a nonzero linear--quadratic polynomial
cannot be constant under a strictly positive density on $\mathbb R^d$.
Continuity and compactness give $\nabla^2\mathcal D\succeq\kappa I$ on
$\mathcal K$ for some $\kappa>0$, and a bounded dual gradient.

Use~\eqref{eq:baoab-dual-errors} to make its first right-hand side at most
$(k+1)^{-2}$ and its second at most $(k+1)^{-1}$, choosing $T_k\ge k+1$,
$h_k\le(k+1)^{-1}$ and $B_k\ge k+1$.
All choices can be made independently of the iterates because the bounds
are uniform on $\mathcal K$.
Conditional on the optimization history $\mathcal F_k$, this gives
\[
\mathbb E[\|\widehat g_k-\nabla\mathcal D(\theta_k)\|^2
\mid\mathcal F_k]\le(k+1)^{-2},\qquad
\mathbb E[\|\widehat g_k\|^2\mid\mathcal F_k]\le C.
\]
Let $e_k=\|\theta_k-\theta_*\|^2$.
Projection is nonexpansive and fixes $\theta_*$. Strong convexity,
$\nabla\mathcal D(\theta_*)=0$, and the conditional error bound imply
\[
\mathbb E[e_{k+1}\mid\mathcal F_k]
\le e_k-\frac{2\kappa}{k+1}e_k+\frac{C}{(k+1)^2}.
\]
Indeed, the bias cross term is at most
$2\operatorname{diam}(\mathcal K)/(k+1)^2$.
Adding the deterministic tail $\sum_{j\ge k}C/(j+1)^2$ makes a nonnegative
supermartingale. Hence $e_k$ converges almost surely, and summing the
expected drift gives $\sum_k e_k/(k+1)<\infty$ almost surely.
The limit must be zero.
Finally, dominated convergence with the common Gaussian envelope gives
weak convergence of $p_{\theta_k}$ and convergence of its second moments,
hence $W_2$ convergence. The second bound in
\eqref{eq:baoab-dual-errors} and the triangle inequality give convergence of the sampled
laws; $W_2$ convergence also implies convergence of means and covariances.
The summable mean-square estimation errors and Borel--Cantelli give
almost-sure convergence of the empirical moments as well.
\end{proof}

This proves consistency for fixed empirical data, with increasing physical
simulation time as well as increasing sample budgets and decreasing step
sizes. It gives no practical rate or a priori choice of a safeguard
containing $\theta_*$. The projection is essential to the stated argument;
symmetrizing $\Lambda$ alone does not ensure normalizability.
The experimental Lion updates and finite sampling budgets are not
covered by this asymptotic result. Nor is a fixed-budget nearest-neighbor
score: its persistent bias requires separate control. For the full
antithetic estimator with the stated force-storage rule, $M$ may remain fixed.
The numerical bound is a strong $L^4$ approximation used for consistency;
it does not assert second-order weak accuracy for fixed-$M$ stochastic
forces. Other consistent discretizations of kinetic Langevin could also
be used if analogous uniform moment-approximation bounds hold; the
algorithm and results here are stated for BAOAB only.

\paragraph{A closed-form plug-in estimator.}
Although the dual estimator is consistent, it requires repeatedly sampling from the current tilted density and optimizing $O(d^2)$ parameters in $\Lambda$. We therefore also consider a closed-form plug-in estimator. It replaces the population expectations in~\eqref{eq:appendix-population-tilting-parameters} by empirical averages over the training set:
\begin{equation}
\widehat\lambda
=
-\frac{1}{N}\sum_{i=1}^N g(x_i),
\qquad
\widehat C
=
\frac{1}{N}\sum_{i=1}^N (x_i-\mu^*)g(x_i)^\top .
\label{eq:appendix-empirical-tilting-parameters}
\end{equation}
Then $\widehat\Lambda$ is obtained by solving the regularized Lyapunov equation
\begin{equation}
(\Sigma^*+\zeta I_d)\widehat\Lambda
+
\widehat\Lambda(\Sigma^*+\zeta I_d)
=
2(I_d-\sym(\widehat C)),
\label{eq:appendix-plugin-lyapunov}
\end{equation}
where $\zeta>0$ is a small constant added for numerical stability. By Lemma~\ref{lem:moment_realization_log_gmm}, the estimated tilt is normalizable exactly when $\delta^{-2}I_d+\widehat\Lambda\succ0$; symmetrization or this Lyapunov regularization alone does not guarantee that condition. It must be checked before interpreting kinetic Langevin output as approximate equilibrium samples. This estimator is biased in general, since the training samples are drawn from the empirical data law rather than from $\rho_N$. Nevertheless, it is computationally cheap, requires only one pass of score evaluations on the training set, and can also serve as an initialization for the consistent dual refinement above.

\emph{Bias structure.}
Let $\widehat\pi_{\mathrm{data}}$ denote the empirical distribution of the training set. The plug-in estimator $\widehat\lambda$ targets $-\mathbb E_{\widehat\pi_{\mathrm{data}}}[g(X)]$, whereas the population parameter $\lambda$ is the expectation of $-g$ under $\rho_N$. The expectation discrepancy (the negative of the plug-in error for $\lambda$) is therefore
\begin{equation}
b_\lambda
:=
\mathbb E_{\widehat\pi_{\mathrm{data}}}[g(X)]
-
\mathbb E_{\rho_N}[g(Z)] .
\label{eq:appendix-lambda-bias-definition}
\end{equation}
The target matches the GMM moments, so its mean agrees with the training mean but its covariance includes the component variance:
\begin{equation}
\mathbb E_{\widehat\pi_{\mathrm{data}}}[X]
=
\mathbb E_{\rho_N}[Z]
=
\mu^*,
\qquad
\operatorname{Cov}_{\widehat\pi_{\mathrm{data}}}[X]
=
\widehat\Sigma_N,
\quad
\operatorname{Cov}_{\rho_N}[Z]
=
\Sigma^*=\widehat\Sigma_N+\delta^2I_d .
\label{eq:appendix-matched-moments}
\end{equation}
Decompose $g$ around $\mu^*$ into an affine part and a residual:
\begin{equation}
g(z)
=
g(\mu^*)
+
H(z-\mu^*)
+
r(z),
\qquad
H:=\nabla g(\mu^*),
\label{eq:appendix-g-affine-decomposition}
\end{equation}
where $r(z)=O(\|z-\mu^*\|^2)$ collects the higher-order terms. Substituting~\eqref{eq:appendix-g-affine-decomposition} into~\eqref{eq:appendix-lambda-bias-definition} yields
\begin{equation}
b_\lambda
=
\mathbb E_{\widehat\pi_{\mathrm{data}}}[r(X)]
-
\mathbb E_{\rho_N}[r(Z)] ,
\label{eq:appendix-lambda-bias-residual}
\end{equation}
because the means agree. For $(z-\mu^*)g(z)^\top$, the covariance difference contributes an additional term:
\begin{equation}
\begin{aligned}
b_C
&:=
\mathbb E_{\widehat\pi_{\mathrm{data}}}
\left[
(X-\mu^*)g(X)^\top
\right]
-
\mathbb E_{\rho_N}
\left[
(Z-\mu^*)g(Z)^\top
\right]
\\
&=-\delta^2H^\top+
\mathbb E_{\widehat\pi_{\mathrm{data}}}
\left[
(X-\mu^*)r(X)^\top
\right]
-
\mathbb E_{\rho_N}
\left[
(Z-\mu^*)r(Z)^\top
\right].
\end{aligned}
\label{eq:appendix-C-bias-residual}
\end{equation}
Thus $b_\lambda$ depends only on the non-affine residual, whereas $b_C$ also contains the explicit covariance-mismatch term $-\delta^2H^\top$. The covariance estimator has additional regularization bias from $\zeta$. There is no automatic $N^{-1/2}$ fluctuation bound here, since $g$ and the target both depend on the training set.

\emph{Self-evaluation at training points.}
A second issue specific to~\eqref{eq:appendix-empirical-tilting-parameters} is that $g$ itself depends on the training set through $\widehat p^\delta$, and is then evaluated at training points. When the component bandwidth $\delta$ is small relative to the typical inter-point spacing, the softmax weights at a query $y=z+\sigma\varepsilon$ with $z=x_i$ can place dominant mass on $x_i$ itself, biasing $g(x_i)$ toward the self-attraction value induced by the $x_i$-component of the GMM.

Smoothing may reduce this effect when it averages across several centers, but $\sigma\gg\delta$ alone does not guarantee this. A leave-one-out diagnostic is
\begin{equation}
g^{(-i)}(z)
:=
-\nabla_z
\mathbb E_{\varepsilon\sim\rho_\varepsilon}
\log\widehat p^\delta_{(-i)}(z+\sigma\varepsilon),
\label{eq:appendix-loo-score}
\end{equation}
where $\widehat p^\delta_{(-i)}$ is the GMM built from $\{x_j:j\ne i\}$. Replacing $g(x_i)$ by $g^{(-i)}(x_i)$ removes direct self-contribution, but changes the potential being evaluated and need not eliminate the plug-in bias. Our experiments use the original plug-in estimator.

\emph{Empirical evidence.}
The kinetic-Langevin ablation in Appendix~\ref{app:ablation_studies}, summarized in Figure~\ref{fig:ablation_particle_train}, provides a direct empirical assessment of the plug-in estimator. Kinetic Langevin samples the explicit moment-matched target~\eqref{eq:moment_matched_target} using $(\widehat\lambda,\widehat\Lambda)$ from~\eqref{eq:appendix-empirical-tilting-parameters}, with no particle-level moment matching. Its KID depends on tilting error, numerical sampling error, and the representation and evaluation procedure, so it does not isolate parameter-estimation error.

Figure~\ref{fig:ablation_particle_train} shows that, on the digit-8 task with fixed particle number $P=500$, the kinetic-Langevin KID decreases from $0.238$ at $N=100$ to approximately $0.14$ at $N=1000$. This decrease is consistent with reduced plug-in error as the training set grows. The gap to MM-SOLD at small $N$ further illustrates why particle-level moment matching is useful in sparse-data regimes: MM-SOLD avoids directly estimating all $O(d^2)$ entries of $\Lambda$, while still enforcing the prescribed GMM moments exactly at the particle level.

\section{Additional experimental details and results}
\label{app:exp_details}

\subsection{Handwritten digit classification details}
\label{app:handwritten_digit_details}

\paragraph{Dataset and preprocessing}
We use the handwritten digits dataset \cite{beaulac2022introducing}. The split is class-balanced, with $1{,}000$ training, $58$ validation, and $300$ test images per digit class. It is drawn once with a fixed random seed and shared by all models and experiments on this dataset, including the NRAE, the CNN feature extractor of Appendix~\ref{app:handwritten_generation_details} and the generation experiments.
Raw $500\times500$ grayscale images are converted to floating point and normalized to $[0,1]$. We then apply geometric normalization: an ink-weighted centroid is computed from the darkest pixels, the digit is PCA-aligned using its ink inertia matrix, and its scale is normalized by matching the root-mean-square ink distance to a target value. This target scale is estimated from the training set with an $18\%$ margin to prevent the digits from extending beyond the image boundaries and reused for validation and test images. Finally, images are bilinearly resized to $64\times64$ and intensity-inverted so that ink has value $1$ on a dark background.

\paragraph{NRAE encoder}
All experiments are performed in the latent space of a pretrained Nuclear Norm-Regularized AutoEncoder (NRAE) \cite{scarvelis2024nuclear}. The encoder first applies a two-dimensional discrete cosine transform (DCT), which represents the image by spatial-frequency coefficients, and keeps the low-frequency $20\times20$ block. This gives a compact $400$-dimensional representation that preserves the coarse digit shape while discarding high-frequency noise. The retained coefficients are then mapped to a $100$-dimensional latent code by a $400\to2048\to100$ MLP, with an ELU activation \cite{clevert2015fast} after the first layer. The decoder maps the latent code back to $400$ DCT coefficients through a $100\to2048\to400$ MLP, reconstructs a coarse $64\times64$ image by inverse DCT, and refines it with a small U-Net using base channel width $32$ and skip connections across resolutions. 
The NRAE is trained only on the $10{,}000$ training images for $100$ epochs with AdamW~\cite{loshchilov2017decoupled} (learning rate $10^{-4}$, no weight decay, batch size $64$), and we keep the checkpoint with the lowest loss on the $580$ validation images (the final epoch); test images are never used. For a mini-batch of $B$ images $x_b$ with reconstructions $\hat x_b$, the loss combines a log-cosh reconstruction term~\cite{chen2018log} and a finite-difference noise-sensitivity regularizer~\cite{scarvelis2024nuclear}:
\begin{equation}
\begin{split}
\mathcal L={}&\frac{1}{2\alpha B}\sum_{b=1}^{B}\sum_{p}\log\cosh\big(\alpha(\hat x_{b,p}-x_{b,p})\big)\\
&+\frac{\sigma_{\mathrm{NRAE}}^2}{2\epsilon^2B}\sum_{b=1}^{B}\Big(\|e_1(c_b+\epsilon\xi_b)-e_1(c_b)\|_2^2+\|e_2(h_b+\epsilon\xi_b')-e_2(h_b)\|_2^2\Big),
\end{split}
\end{equation}
where $p$ runs over the $64\times64$ pixels, $c_b$ are the $400$ retained DCT coefficients, $e_1$ and $e_2$ are the $400\to2048$ and $2048\to100$ encoder stages, $h_b=e_1(c_b)$, and $\xi_b,\xi_b'$ are independent standard Gaussian vectors. We use $\alpha=100$, $\sigma_{\mathrm{NRAE}}=2$ and $\epsilon=10^{-3}$, and both terms have unit weight.

\paragraph{Moment-matched classifiers}

For each digit class $c$, we fit two class-conditional moment-matched densities of the form in \eqref{eq:moment_matched_target} directly in the $100$-dimensional NRAE latent space. The first classifier uses the closed-form plug-in estimator for $(\lambda_c,\Lambda_c)$. The second classifier uses a practical stochastic dual refinement, which optimizes the convex moment-matching dual objective using samples from the current tilted density, as discussed in Appendix~\ref{app:tilting-parameter-estimation}.

For each class $c$, let $\mathrm{NN}_c$ denote the average nearest-neighbor distance among the class-$c$ training latents:
\begin{equation}
\mathrm{NN}_c
=
\frac{1}{N_c}
\sum_{i=1}^{N_c}
\min_{j\ne i}
\|x_i^{(c)}-x_j^{(c)}\|_2 .
\end{equation}
We parameterize the GMM component standard deviation and score-smoothing bandwidth as
\begin{equation}
\delta_c
=
c_\delta\frac{\mathrm{NN}_c}{\sqrt{100}},
\qquad
\sigma_c
=
c_\sigma\mathrm{NN}_c .
\end{equation}
The Monte Carlo budget is fixed to $M=32$, and the smoothed score uses all $1{,}000$ class-$c$ training latents. For the plug-in classifier, $(c_\delta,c_\sigma)$ is selected from the grid $\{0.05,0.10,\ldots,0.50\}^2$ by validation macro-F1. The same cell is used for all ten classes, so the class-specific scales differ only through $\mathrm{NN}_c$. Because the dual estimator is much more expensive, its cell was selected by validation macro-F1 in a search with $500$ instead of $2{,}000$ chains per iteration and otherwise the same numbers of BAOAB steps and Lion iterations, and the final dual estimates are computed once at the selected cell $(c_\delta,c_\sigma)=(0.35,0.25)$ with the budget given below.
For both moment-matched ECM classifiers, we learn scalar biases $b_c$ on the validation set by minimizing multiclass cross-entropy for logits $-E_c(x)-b_c$ with L-BFGS initialized at $b_c=0$, which absorbs the unknown class normalizing constants. The class energies are only evaluated at validation and test latents and are never sampled, so the classification rule is well defined whether or not an estimated tilt is normalizable. The resulting classifier is the minimum ECM rule used in the main text.

\paragraph{Practical dual estimator implementation}

For the dual-estimated moment-matched classifier, we estimate $(\lambda_c,\Lambda_c)$ by stochastic optimization of the convex dual moment-matching objective described in Appendix~\ref{app:tilting-parameter-estimation}. At each dual-optimization iteration, samples are drawn from the current tilted density using a BAOAB kinetic Langevin integrator in the $100$-dimensional NRAE latent space. For class $c$, the BAOAB step size is set to
\begin{equation}
h_{\mathrm{BAOAB},c}
=
\frac{\delta_c}{20},
\end{equation}
and the friction parameter is set to
\begin{equation}
\gamma_c
=
\frac{1}{\sqrt{\lambda_{\max}(\widehat\Sigma_c)}} ,
\end{equation}
where $\widehat\Sigma_c$ is the class-$c$ sample covariance matrix in the $100$-dimensional NRAE latent space. This choice normalizes the damping by the standard deviation of the most spread-out principal direction, which helps the sampler explore the class manifold and mix across directions. The dual parameters are initialized at the plug-in estimate and optimized with Lion for $300$ iterations, with no other stopping rule. At each iteration, $2{,}000$ independent BAOAB chains are started afresh from the class-$c$ training GMM and run for $500$ steps targeting the current tilted density, and the stochastic dual gradient is formed from their endpoints. The relative Lion step size is $0.01$, cosine-annealed to zero, with $\beta_1=0.9$ and $\beta_2=0.99$. The remaining conventions (step-size scaling and the projection of $\Lambda$) are those described for SG-BAOAB in Appendix~\ref{app:handwritten_generation_details}.

\paragraph{MLP baselines}

Both discriminative baselines use the same MLP classifier trained on the $100$-dimensional NRAE latent codes. The network maps the $100$-dimensional latent code through two hidden layers of widths $256$ and $128$, followed by a $10$-way linear output layer. Each hidden layer uses GELU activation \cite{hendrycks2016gaussian}. We train with cross-entropy loss using AdamW, batch size $256$, and a warmup-cosine learning-rate schedule: the learning rate warms up from $0$ to $10^{-3}$ during the first $10\%$ of training steps and then decays to $10^{-5}$. Each model is trained for $100$ epochs. For model selection, we grid search the weight decay over $\{10^{-3},10^{-2},5\times10^{-2},10^{-1}\}$ and keep the models with the best validation accuracy. The final test accuracy is reported for the best validation model.

\paragraph{MM-SOLD augmentation details}

For latent augmentation, MM-SOLD is run separately for each digit class directly in the $100$-dimensional NRAE latent space. We compute the class mean $\mu_c^*$ and raw covariance $\widehat\Sigma_c$ from its training latents (divisor $N_c$), and enforce the GMM covariance $\Sigma_c^*=\widehat\Sigma_c+\delta_c^2I_{100}$ at the particle level. For each class, we use $P=9{,}000$ particles to generate $9{,}000$ augmented samples. The Monte Carlo budget is fixed to $M=32$, and the GMM and smoothing scales are parameterized by
\begin{equation}
\delta_c
=
c_\delta\frac{\mathrm{NN}_c}{\sqrt{100}},
\qquad
\sigma_c
=
c_\sigma\mathrm{NN}_c .
\end{equation}
We reuse the cell selected on the validation set for the dual-estimated classifier, $(c_\delta,c_\sigma)=(0.35,0.25)$.
The overdamped Langevin step size is set to
\begin{equation}
h_c
=
0.2\delta_c^2 .
\end{equation}
We use $500$ Langevin iterations and the LM discretization. The generated samples are kept in the original $100$-dimensional NRAE latent space and combined with the original latent codes to train the augmented MLP.

\paragraph{Metrics}

We report per-class and overall accuracy on the test set. Since the test set has $300$ samples per class, the overall accuracy is also the class-balanced average accuracy up to rounding. All reported classifiers operate on the $100$-dimensional NRAE latent codes. Each classifier is fitted once with a fixed seed (each MLP once per weight-decay value), and its accuracy is a single evaluation on the test set; we did not repeat the fits with different seeds. Differences of a few tenths of a percentage point between the augmented MLP and the two moment-matched classifiers should therefore not be read as a reliable ranking.

\subsection{Handwritten digit generation details}
\label{app:handwritten_generation_details}

\paragraph{Task and latent representation}

We use only digit ``8'' from the handwritten digits dataset. The preprocessing and NRAE encoder are the same as in Appendix~\ref{app:handwritten_digit_details}. In particular, raw images are geometrically normalized, resized to $64\times64$, intensity-inverted so that ink has value $1$, and then encoded by the fixed NRAE encoder. The experiment uses $1{,}000$ training images and $300$ held-out test images.
For this generation experiment, we compute a PCA basis from the pooled within-class covariance of the $10{,}000$ NRAE training latents of all digit classes (each class centred at its own mean; validation and test latents are not used) and keep the top $d_{\mathrm{sub}}=40$ principal directions, which explain $99.86\%$ of this variance. For a digit-8 latent $z\in\mathbb R^{100}$, we write $u\in\mathbb R^{40}$ for its top-$40$ PCA coordinates. All methods in this generation experiment operate on these digit-8 top-$40$ coordinates, and no whitening is applied. Generated top-$40$ latents are mapped back to the original $100$-dimensional NRAE latent space by setting the discarded $60$ PCA coordinates to the digit-8 sample mean, and are then decoded by the fixed NRAE decoder. Thus differences in the results are due to the latent-space generative methods rather than the decoder.

\paragraph{Common scale parametrization}

Let $u_1,\ldots,u_N\in\mathbb R^{40}$ be the digit-8 training latents in the top-$40$ PCA subspace. We define the average nearest-neighbor distance
\begin{equation}
\mathrm{NN}
=
\frac{1}{N}
\sum_{i=1}^N
\min_{j\ne i}
\|u_i-u_j\|_2 .
\end{equation}
For all methods based on the GMM potential or score smoothing, we parameterize the GMM component standard deviation and score-smoothing bandwidth as
\begin{equation}
\delta
=
c_\delta\frac{\mathrm{NN}}{\sqrt{40}},
\qquad
\sigma
=
c_\sigma\mathrm{NN}.
\end{equation}
The scaling of $\delta$ accounts for the typical norm of a $40$-dimensional Gaussian component perturbation, while the scaling of $\sigma$ makes the smoothed score average over multiple neighbouring GMM centers rather than only the nearest one. For MM-SOLD, $\sigma$-CFDM, Plug-in BAOAB, and SG-BAOAB, we search over
\begin{equation}
c_\delta,c_\sigma\in\{0.10,0.20,\ldots,1.0\},
\qquad
M\in\{2,32\}.
\end{equation}
For each grid cell, we run each sampler five times with independent seeds; each run generates and decodes $300$ latent samples, and the reported $\pm$ is the standard deviation over the five runs. A run redraws the initialization, the Langevin or BAOAB noise and the smoothing perturbations. For Plug-in BAOAB and SG-BAOAB, the tilting parameters are estimated once per cell and kept fixed across runs, and the latent DDPM is trained once, so their error bars reflect sampling variability only. Within each run, KID is averaged over $50$ random subsets of $100$ test and $100$ generated features, while Recall and DupRate use all samples. No separate validation split is used for generation: the main text reports, for every method including the baselines, the best value over the grid for each metric, selected with the same $300$ held-out test images. Here the best value is the best mean over the five runs (lowest KID and DupRate, highest Recall), chosen separately for each metric, so different entries of the same method can come from different configurations; exact ties are broken by grid order, and the selected configuration of every entry is written to the output of the released summary script. These best-over-grid values are therefore exploratory comparisons rather than estimates on an untouched test set.

\paragraph{MM-SOLD sampler}

MM-SOLD is initialized from the digit-8 GMM in the top-$40$ latent subspace and evolved on the moment-constrained manifold using the LM discretization. We use $P=300$ particles, thereby generating $300$ samples per run. For each value of $\delta$, the step size is set to
\begin{equation}
h
=
0.2\delta^2 .
\end{equation}
The sampler is run for $1{,}000$ LM steps. At each Langevin step, the smoothed score is estimated using the full set of $1{,}000$ digit-8 training latents, rather than the nearest-neighbor estimator, because this dataset is small.

\paragraph{$\sigma$-CFDM baseline}

The $\sigma$-CFDM baseline is implemented in the same top-$40$ latent subspace and uses the same $(c_\delta,c_\sigma,M)$ grid as MM-SOLD. Particles are initialized from $\mathcal N(0,I_{40})$ and evolved with the deterministic closed-form diffusion ODE~\cite{scarvelis2023closed}. We use a uniform grid with spacing $\Delta t=10^{-3}$ and perform $999$ explicit Euler updates $z\leftarrow z+\Delta t\,v(z,t_k)$, with the velocity evaluated at $t_k=k\Delta t$, $k=1,\ldots,999$. The initial Gaussian sample is assigned to $t_1=0.001$, so the final state corresponds to $t=1$, where the velocity is singular and is never evaluated.

\paragraph{KNN-PFlow baseline}

KNN-PFlow is a nearest-neighbor empirical-score probability-flow ODE sampler. At noise level $s$, it uses the nonparametric denoiser
\begin{equation}
D_s(z)
=
\sum_{i\in\mathcal I_K(z)}
w_i(z,s)u_i,
\qquad
w_i(z,s)
=
\frac{
\exp\left(-\|z-u_i\|^2/(2s^2)\right)
}{
\sum_{j\in\mathcal I_K(z)}
\exp\left(-\|z-u_j\|^2/(2s^2)\right)
},
\end{equation}
where $\mathcal I_K(z)$ is either the set of the $K$ nearest training latents to $z$ or the full training set. It then integrates the variance-exploding probability-flow ODE
\begin{equation}
\frac{dz}{ds}
=
\frac{z-D_s(z)}{s}
\end{equation}
from $s_{\max}$ to $s_{\min}$. We initialize from $z=s_{\max}\xi$ with $\xi\sim\mathcal N(0,I_{40})$; since the operating coordinates are centred at the training mean, this is a Gaussian centred at the training mean, the usual initialization of variance-exploding probability-flow samplers. To make the search scale-free, we set
\begin{equation}
s_{\max}=c_{\max}\mathrm{NN},
\qquad
s_{\min}=c_{\min}\frac{\mathrm{NN}}{\sqrt{40}},
\end{equation}
and search over
\begin{equation}
c_{\max}\in\{1,2,4\},
\qquad
c_{\min}\in\{0.025,0.05,0.10\},
\qquad
K\in\{32,128,N\}.
\end{equation}
Here $K=N$ denotes exact full-training-set evaluation. We use $1{,}000$ ODE steps on the noise levels $s_k=\big(s_{\max}^{1/\rho}+\tfrac{k}{1000}(s_{\min}^{1/\rho}-s_{\max}^{1/\rho})\big)^{\rho}$, $k=0,\ldots,1000$, with $\rho=7$, and search over Euler and Heun solvers; the Heun correction is applied at every step, and the output is the denoised point $D_{s_{\min}}(z)$ after the last step. This baseline tests whether a training-free empirical-score sampler is sufficient for generation.

\paragraph{Sliced-Wasserstein Flow baseline}

Sliced-Wasserstein Flow (SWF) is a training-free particle-flow baseline based on sliced optimal transport. Let $Z=\{z_i\}_{i=1}^P\subset\mathbb R^{40}$ denote generated particles and let $U=\{u_j\}_{j=1}^N\subset\mathbb R^{40}$ denote the digit-8 training latents. At each iteration, SWF samples random directions $\theta_1,\ldots,\theta_R$ on the unit sphere. For each direction $\theta_r$, we project both empirical distributions:
\begin{equation}
a_i^{(r)}
=
\theta_r^\top z_i,
\qquad
b_j^{(r)}
=
\theta_r^\top u_j .
\end{equation}
Let $F_{Z,\theta_r}$ and $F_{U,\theta_r}$ be the empirical CDFs of $\{a_i^{(r)}\}_{i=1}^P$ and $\{b_j^{(r)}\}_{j=1}^N$, respectively. The one-dimensional quantile-matching update is
\begin{equation}
T_{\theta_r}
=
F_{U,\theta_r}^{-1}\circ F_{Z,\theta_r},
\end{equation}
where $F_{U,\theta_r}^{-1}$ linearly interpolates the sorted values $b^{(r)}_{(1)}\le\cdots\le b^{(r)}_{(N)}$ at the levels $(j-1)/(N-1)$, and a particle whose projection has rank $k$ among the $P$ projections is mapped to this interpolant at level $(k-1)/(P-1)$; this also covers $P\ne N$. The particles are then updated by
\begin{equation}
z_i
\leftarrow
z_i
+
\frac{\eta}{R}
\sum_{r=1}^R
\left(
T_{\theta_r}(\theta_r^\top z_i)
-
\theta_r^\top z_i
\right)\theta_r,
\qquad
i=1,\ldots,P .
\end{equation}
We run SWF for $n_{\mathrm{SWF}}$ particle-flow iterations and search over
\begin{equation}
n_{\mathrm{SWF}}\in\{300,1000\},
\qquad
R\in\{64,128,256\},
\qquad
\eta\in\{0.05,0.10\}.
\end{equation}
Here $n_{\mathrm{SWF}}$ denotes the number of sliced-Wasserstein particle-flow updates, $R$ is the number of random projection directions per update, and $\eta$ is the update step size. We use two initializations: a fitted Gaussian $\mathcal N(\mu_8,\widehat\Sigma_8)$ in the top-$40$ subspace and a noised empirical initialization $u_i+\mathrm{NN}\xi$ with $\xi\sim\mathcal N(0,I_{40})$. This baseline represents a different class of nonparametric training-free samplers from empirical-score diffusion samplers.

\paragraph{Plug-in BAOAB sampler}

Plug-in BAOAB samples independently from the explicit moment-matched target~\eqref{eq:moment_matched_target} using tilting parameters estimated by the closed-form plug-in estimator in Appendix~\ref{app:tilting-parameter-estimation}. For each $(c_\delta,c_\sigma,M)$ grid cell, we estimate $(\widehat\lambda,\widehat\Lambda)$ from the $1{,}000$ digit-8 training latents and then run independent BAOAB kinetic Langevin chains targeting
\begin{equation}
\widehat p_{\mathrm{plug}}(z)
\propto
\exp\left\{
-V(z)
-\widehat\lambda^\top(z-\mu^*)
-\frac{1}{2}(z-\mu^*)^\top\widehat\Lambda(z-\mu^*)
\right\}.
\end{equation}
Positions are initialized from the same digit-8 GMM used by MM-SOLD, and momenta are initialized from $\mathcal N(0,I_{40})$. We use $1{,}000$ BAOAB steps with step size
\begin{equation}
h_{\mathrm{BAOAB}}
=
\frac{\delta}{20}.
\end{equation}
The friction parameter is set to
\begin{equation}
\gamma
=
\frac{1}{\sqrt{\lambda_{\max}(\widehat\Sigma_8)}} ,
\end{equation}
where $\widehat\Sigma_8$ is the digit-8 sample covariance matrix in the top-$40$ subspace.

\paragraph{SG-BAOAB sampler}

SG-BAOAB uses a practical variant of the dual estimator from Appendix~\ref{app:tilting-parameter-estimation} before sampling independently from the estimated moment-matched target. The asymptotic guarantee in Proposition~\ref{prop:baoab-dual-consistency} applies to the safeguarded update specified there. For each $(c_\delta,c_\sigma,M)$ grid cell, we initialize the dual parameters from the plug-in estimator and perform $300$ stochastic-gradient iterations with a Lion optimizer. Each stochastic-gradient iteration estimates the dual moments using $100$ parallel BAOAB chains, each run for $200$ BAOAB steps targeting the current tilted density. The relative Lion step size is $0.02$ and is cosine-annealed to $0$; the momentum coefficients are $\beta_1=0.9$ and $\beta_2=0.99$. After optimization, we sample $300$ generated latents from the final estimated tilted density using BAOAB with the same step size $h_{\mathrm{BAOAB}}=\delta/20$, $1{,}000$ steps and friction $\gamma=1/\sqrt{\lambda_{\max}(\widehat\Sigma_8)}$ as in Plug-in BAOAB. Since SG-BAOAB optimizes the tilting parameters before sampling, we report its optimization time separately from sampling time.

\paragraph{Implementation of the BAOAB-based samplers}
In all experiments, the prescribed moments of MM-SOLD, of the BAOAB-based samplers and of the moment-matched classifiers are those of the training GMM, $\mu^*=\bar x$ and $\Sigma^*=\widehat\Sigma_N+\delta^2I$, with $\widehat\Sigma_N$ the $1/N$ empirical covariance in the operating coordinates. Each BAOAB chain starts from the training GMM (a uniformly chosen training point plus $\mathcal N(0,\delta^2I)$ noise) with momentum drawn from $\mathcal N(0,I)$, and all chains are independent. All kinetic samplers in our experiments, including the ablation in Appendix~\ref{app:ablation_studies}, use BAOAB as in~\eqref{eq:baoab-dual-step}. The force $\widehat F_{j+1}$ computed at $q_{j+1}$ is used in the closing half-kick of step $j$ and retained for the opening half-kick of step $j+1$. Thus each step requires one new stochastic force evaluation, plus one initial evaluation at $q_0$. Every new evaluation uses fresh smoothing perturbations, independent of previous draws and the Ornstein--Uhlenbeck noises. Adjacent half-kicks combine into a full kick, leaving a drift--Ornstein--Uhlenbeck--drift update between successive force evaluations; the initial and final kicks remain half-kicks. 
Every reported BAOAB sample is the endpoint of its own chain; trajectories are neither thinned nor averaged. During SG-BAOAB optimization, chains are restarted from the training GMM at every Lion iteration, and the dual gradient $\big(\mu^*-\bar q,\ \tfrac12[\Sigma^*-P^{-1}\sum_i(q_i-\mu^*)(q_i-\mu^*)^\top]\big)$ is formed from the $P$ endpoints $q_i$. Lion updates $\lambda$ and the upper triangle of $\Lambda$ (so $\Lambda$ stays symmetric) with step sizes $r\,\mathrm{rms}(\lambda_0)$ and $r\,\mathrm{rms}(\Lambda_0)/(2\sqrt d)$, where $r$ is the relative step size and $\mathrm{rms}$ is the root-mean-square entry of the plug-in initialization; no weight decay is used. The initialization and every iterate are projected by eigenvalue clipping, $\Lambda\leftarrow U\max\{\Omega,-0.95\,\delta^{-2}\}U^\top$ for $\Lambda=U\Omega U^\top$, so that $\lambda_{\min}(\delta^{-2}I+\Lambda)\ge0.05\,\delta^{-2}$ and every intermediate and final tilted target is normalizable. This safeguard bounds neither $\lambda$ nor the largest eigenvalue of $\Lambda$, so it is not the compact set assumed in Proposition~\ref{prop:baoab-dual-consistency}. No other safeguard, penalty or rejection rule is used; the implementation stops with an error if a chain endpoint or the moment residual becomes non-finite. The plug-in estimate uses $\zeta=\max\{10^{-6}\lambda_{\max}(\Sigma^*),10^{-6}\}$ in~\eqref{eq:appendix-plugin-lyapunov} and is used by Plug-in BAOAB without projection, so its normalizability has to be checked. We evaluated $a_{\min}=\delta^2\lambda_{\min}(\delta^{-2}I+\widehat\Lambda)$ for these unprojected estimates with the released estimator: it lies in $[0.12,1.03]$ over all $200$ cells of the digit-8 grid, lies in $[0.08,1.04]$ over the $58$ CelebA-HQ cells with $M=32$ that we evaluated (all cells with $c_\delta\le0.5$ and eight cells with $c_\delta\in\{0.6,1.0\}$), and equals $0.58$ and $0.51$ in the configurations of Tables~\ref{tab:digit8_posterior_moment_validation} and~\ref{tab:celebahq_posterior_moment_validation}, respectively. No evaluated fit has $a_{\min}\le0$. We did not evaluate the remaining CelebA-HQ cells; for these, the Plug-in BAOAB results should be read as outputs of the finite-step sampler with the estimated parameters rather than as samples from a target verified to be normalizable.

\paragraph{Latent DDPM baseline}

The latent DDPM is trained on the same $1{,}000$ digit-8 latents in the top-$40$ PCA subspace. For neural diffusion training only, we standardize each top-$40$ coordinate using the training-set mean and standard deviation, and unstandardize generated samples before evaluation and decoding. The denoiser is an MLP with DiT-style adaptive LayerNorm conditioning~\cite{peebles2023scalable}. A sinusoidal time embedding of dimension $128$ is mapped to a $256$-dimensional conditioning vector. The noisy latent is first projected to width $256$, then passed through $4$ AdaLN residual blocks. Each block predicts scale and shift parameters from the time embedding, applies adaptive LayerNorm, Swish activations, dense layers, and a residual connection. A final AdaLN head maps the hidden state back to a $40$-dimensional noise prediction.
We train the DDPM with a cosine noise schedule~\cite{nichol2021improved} and $\nsteps=1{,}000$ diffusion steps. The loss is the standard noise-prediction objective as in~\cite{ho2020denoising}. The optimizer is AdamW with learning rate $10^{-4}$, weight decay $10^{-3}$, batch size $128$, and a warmup-cosine learning-rate schedule. We train this model for $50{,}000$ epochs, where an epoch is one pass over a random permutation of the $1{,}000$ training latents in $7$ full batches of $128$ (the last $104$ latents of each permutation are skipped), i.e.\ $350{,}000$ optimizer updates; the learning rate warms up linearly over the first $2{,}000$ updates and then decays with a cosine schedule to $10^{-5}$. We sample with the parameters of the epoch with the lowest average training loss, using deterministic DDIM with $100$ reverse steps.
The generated standardized top-$40$ latents are first unstandardized, then mapped back to the original $100$-dimensional NRAE latent space by setting the discarded $60$ PCA coordinates to the digit-8 sample mean before decoding.

\paragraph{Feature extractor for KID and Recall}

KID and Recall are computed in the feature space of a pretrained digit classifier. We use the penultimate $256$-dimensional feature of a CNN classifier trained on the same $10{,}000$ training images as the NRAE, with weight decay and training epoch selected by accuracy on the $580$ validation images; the $3{,}000$ test images, including the $300$ digit-8 test images used for evaluation, are used neither for training nor for model selection. The classifier is used only as a fixed feature extractor shared by all methods. The feature extractor has a convolutional stem with base channel width $32$, four downsampling stages with residual blocks, global average pooling, LayerNorm, and a $256$-dimensional Swish-activated dense layer. The final classification layer is discarded. These features are used only for evaluation.

\paragraph{Evaluation metrics}

KID measures fidelity to the held-out test distribution. Given classifier features $X$ from test images and $Y$ from generated images, we use the usual polynomial-kernel MMD$^2$ U-statistic with
\begin{equation}
k(x,y)
=
\left(\frac{x^\top y}{d_{\mathrm{feat}}}+1\right)^3,\qquad d_{\mathrm{feat}}=256.
\end{equation}

Recall measures coverage of the test distribution. We compute the radius of each test feature as the distance to its third nearest neighbor among test features. A test sample is covered if at least one generated sample lies inside this ball. The final Recall is the fraction of covered test samples.

DupRate measures memorization in the top-$40$ latent subspace. We first compute the within-training nearest-neighbor distances among digit-8 training latents in this subspace and set $\tau$ to the $1$st percentile of these distances. A generated latent $z_i^{\mathrm{gen}}$ is counted as a duplicate if
\begin{equation}
\min_j
\left\|
z_i^{\mathrm{gen}}-z_j^{\mathrm{train}}
\right\|_2
<
\tau .
\end{equation}
The duplication rate is the fraction of generated samples satisfying this condition. Hence low DupRate indicates that the method is not simply reproducing training latents.

\paragraph{Timing}

Sampling times are wall-clock times of one complete call of each sampler on the JAX CPU backend, at a fixed lightweight timing configuration rather than at the selected quality configurations: $(c_\delta,c_\sigma,M)=(0.05,0.05,2)$ for MM-SOLD, $\sigma$-CFDM, Plug-in BAOAB and SG-BAOAB; $c_{\max}=1$, $c_{\min}=0.10$, $K=32$ and Euler integration for KNN-PFlow; and $n_{\mathrm{SWF}}=300$, $R=64$, $\eta=0.05$ and Gaussian initialization for SWF. The numbers of steps are those used in the grid, so the cost does not depend on $(c_\delta,c_\sigma)$; it does increase with $M$. Each call generates the $300$ latents of one run as a single batch and returns them as a NumPy array, so all computation has finished when the timer stops. We perform one warm-up call, which also excludes JIT compilation, and average five timed calls; the reported time per sample is the call time divided by $300$, i.e.\ an amortized cost for a batch of $300$. A call includes initialization, the computation of the target moments and all iterations; it excludes data loading, the PCA projection, the nearest-neighbor scale, decoding and, for Plug-in BAOAB and SG-BAOAB, the estimation of the tilting parameters. The device reported for each method refers to sampling. Training and optimization times are reported separately and are measured on GPU: for SG-BAOAB, one run of the dual optimization, including JIT compilation; for the DDPM, its neural-network training. DDPM sampling is timed on an NVIDIA RTX 6000 Ada Generation GPU. MM-SOLD, $\sigma$-CFDM, KNN-PFlow, SWF and Plug-in BAOAB require no iterative training. The closed-form plug-in fit used by Plug-in BAOAB is a one-off preparation step, one smoothed-score evaluation at each training latent followed by one Lyapunov solve, and is included in neither time. CPU and GPU times are measured on different hardware and are not directly comparable.

\paragraph{Additional results}

We provide additional digit-8 generation results to assess hyperparameter robustness and to compare particle-level moment matching with direct independent sampling from the estimated moment-matched target.

Figure~\ref{fig:digit8_appendix_grids} first compares MM-SOLD and $\sigma$-CFDM under different smoothing bandwidths and numbers of Monte Carlo noise samples. We fix a representative value of $c_\delta=0.5$ and show generated samples across different values of $(c_\sigma,M)$. 
The behavior of $\sigma$-CFDM is visibly sensitive to these parameters. With small $c_\sigma$, many samples remain close to training-like digits, while larger $c_\sigma$ produces more irregular shapes and reduces digit-specific structure. Increasing $M$ makes the samples cleaner but also more concentrated around a few similar loop patterns, reflecting the barycentric nature of $\sigma$-CFDM. In contrast, MM-SOLD is more stable across the same settings: it preserves recognizable digit-8 structure while maintaining variation in slant, loop size, stroke width, and local writing style. This mirrors the robustness observed in the 2D experiments in Figure~\ref{fig:2d_checkerboard_spirals}.
Figure~\ref{fig:digit8_generation_heatmaps} gives the corresponding grid-level comparison over the full $2\times 10$ grid. The heatmaps show metric differences between MM-SOLD and $\sigma$-CFDM for KID, Recall, and DupRate.  In all heatmaps of this appendix, each cell compares the run-averaged values $a$ of the first-named method and $b$ of the second-named method at the same configuration. For KID, latent MMD$^2$ and Recall, the cell shows the relative difference $100\,(a-b)/(|b|+10^{-6})$ in percent; for DupRate, it shows the absolute difference $100\,(a-b)$ in percentage points. Because the denominator is $|b|$, the sign of every entry is the sign of $a-b$, so the color always follows the raw ordering, also when an MMD$^2$ estimate is negative. Green means that the first-named method is better (lower KID, MMD$^2$ or DupRate, higher Recall) and red that it is worse. The color scale is continuous and symmetric about zero and saturates at the $95$th percentile of the absolute entries of the panel; yellow cells are those with entries close to zero, and no separate tolerance is applied. The heatmaps confirm the qualitative trend in Figure~\ref{fig:digit8_appendix_grids}. MM-SOLD improves KID over $\sigma$-CFDM across most of the grid, with the largest gains appearing at larger smoothing scales where $\sigma$-CFDM is most affected by barycentric averaging. Recall is also generally comparable or better for MM-SOLD, showing that the improvement in fidelity does not come from reduced coverage of the digit-8 distribution. The DupRate heatmap remains mostly comparable and low, indicating that MM-SOLD's improvement is not obtained by moving samples closer to the training set. 

\begin{figure}[t]
\centering
\setlength{\tabcolsep}{1.5pt}
\renewcommand{\arraystretch}{0.8}
\begin{tabular}{cccc}
\includegraphics[width=0.24\textwidth]{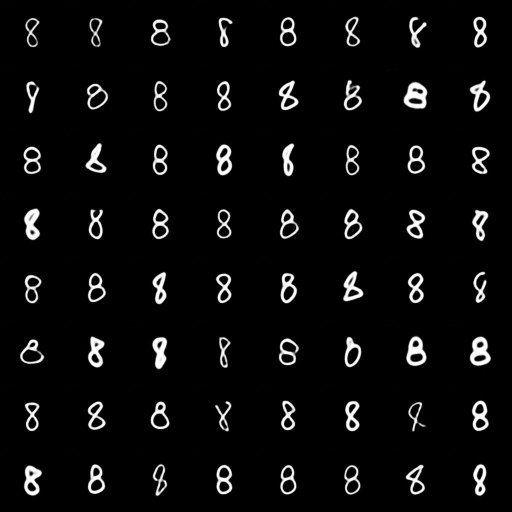} &
\includegraphics[width=0.24\textwidth]{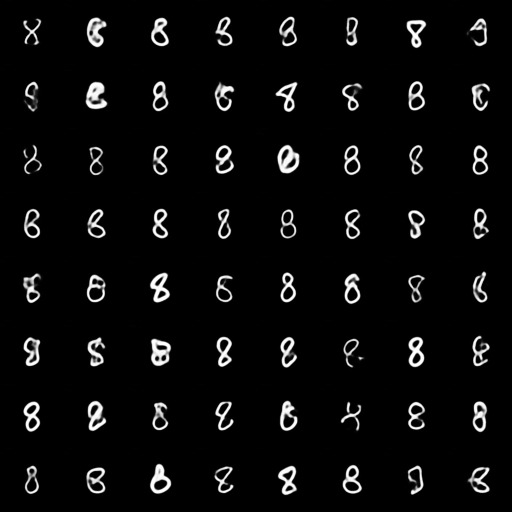} &
\includegraphics[width=0.24\textwidth]{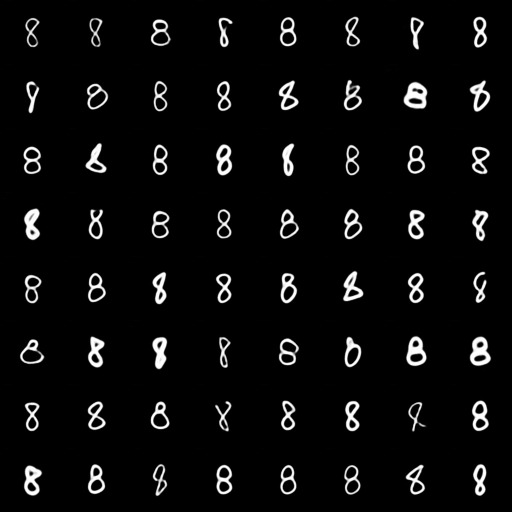} &
\includegraphics[width=0.24\textwidth]{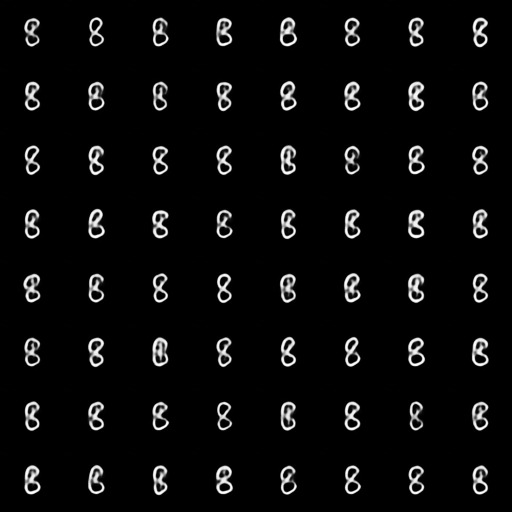} \\
{\scriptsize $M=2,c_\sigma=0.1$} &
{\scriptsize $M=2,c_\sigma=1.0$} &
{\scriptsize $M=32,c_\sigma=0.1$} &
{\scriptsize $M=32,c_\sigma=1.0$} \\
\multicolumn{4}{c}{\scriptsize $\sigma$-CFDM} \\[3pt]
\includegraphics[width=0.24\textwidth]{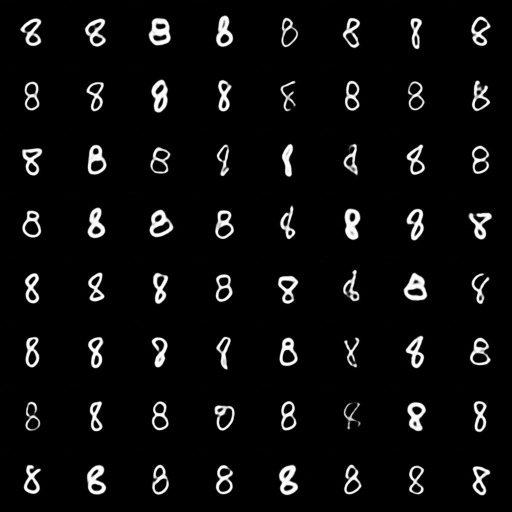} &
\includegraphics[width=0.24\textwidth]{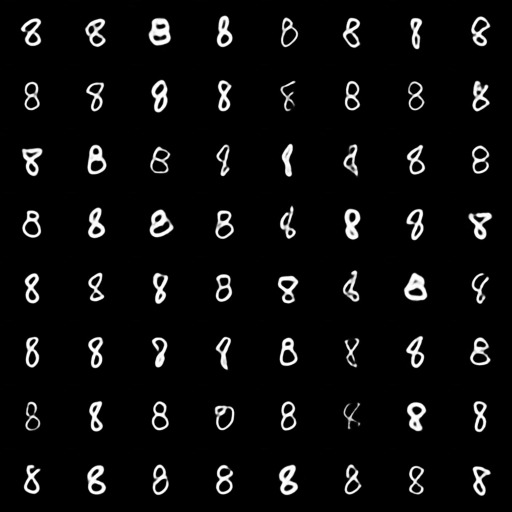} &
\includegraphics[width=0.24\textwidth]{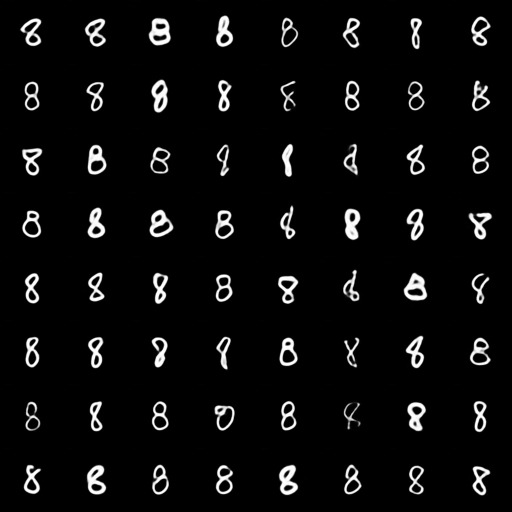} &
\includegraphics[width=0.24\textwidth]{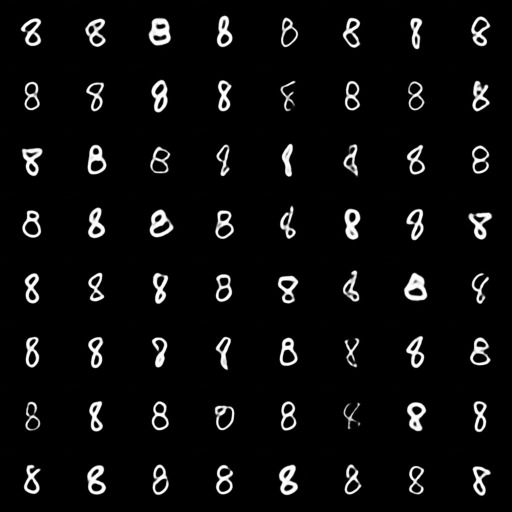} \\
{\scriptsize $M=2,c_\sigma=0.1$} &
{\scriptsize $M=2,c_\sigma=1.0$} &
{\scriptsize $M=32,c_\sigma=0.1$} &
{\scriptsize $M=32,c_\sigma=1.0$} \\
\multicolumn{4}{c}{\scriptsize MM-SOLD}
\end{tabular}
\caption{Additional sample grids for handwritten digit-8 generation under different score-smoothing settings, with a fixed representative value of $c_\delta$. Top row: $\sigma$-CFDM; bottom row: MM-SOLD.}
\label{fig:digit8_appendix_grids}
\end{figure}

\begin{figure}[t]
\centering
\setlength{\tabcolsep}{0pt}
\renewcommand{\arraystretch}{0.9}
\makebox[\textwidth][c]{%
\begin{tabular}{@{}c@{\hspace{2pt}}c@{\hspace{2pt}}c@{}}
\includegraphics[width=0.32\textwidth, trim=5 5 53 5, clip]{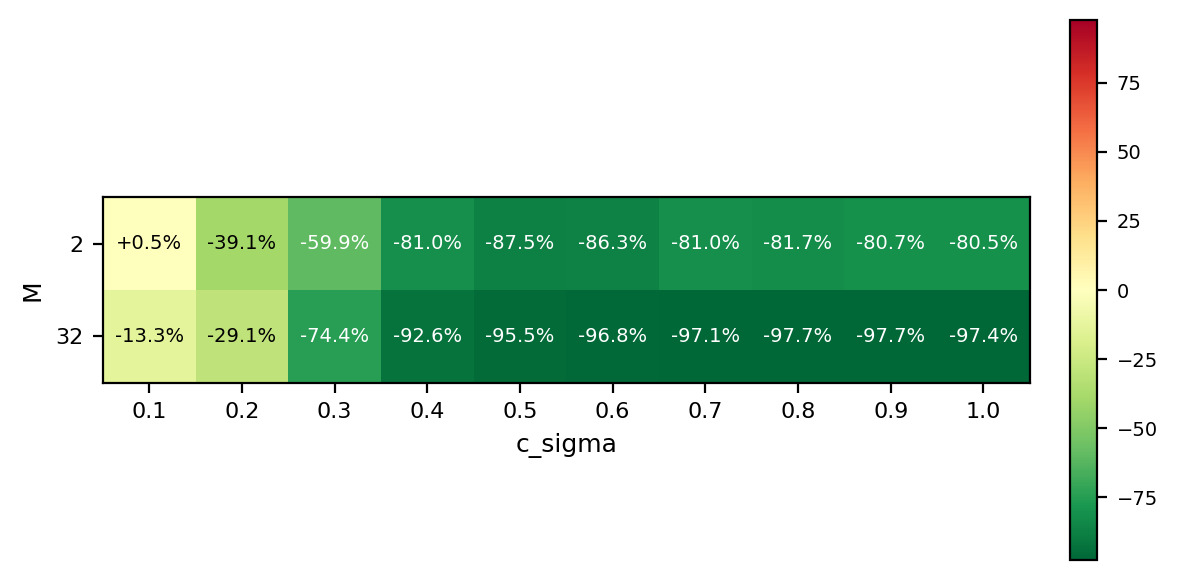} &
\includegraphics[width=0.32\textwidth, trim=5 5 53 5, clip]{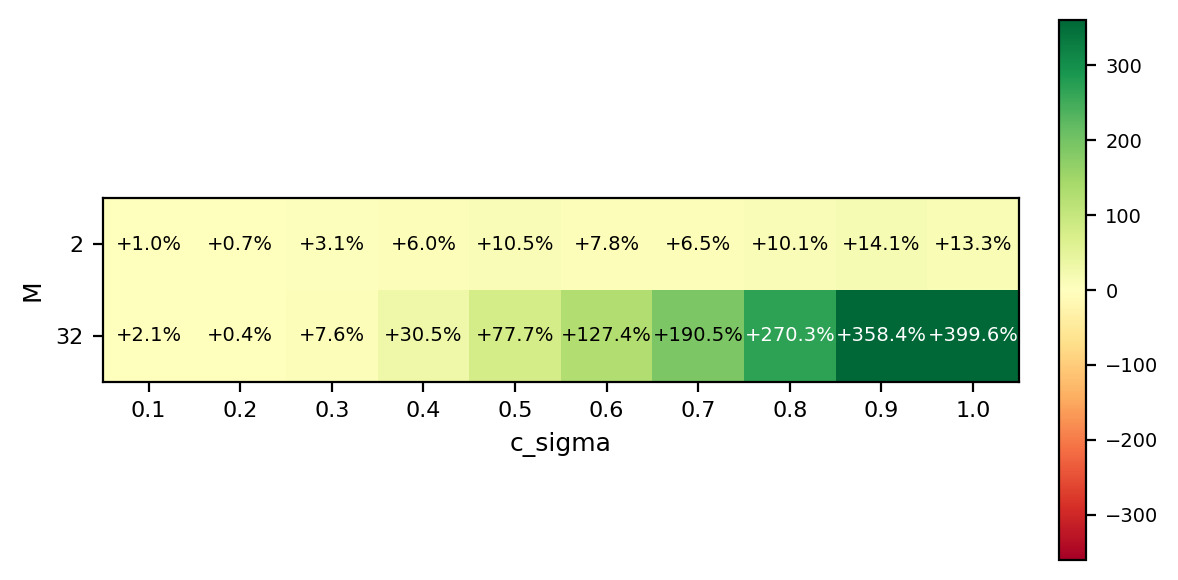} &
\includegraphics[width=0.32\textwidth, trim=5 5 53 5, clip]{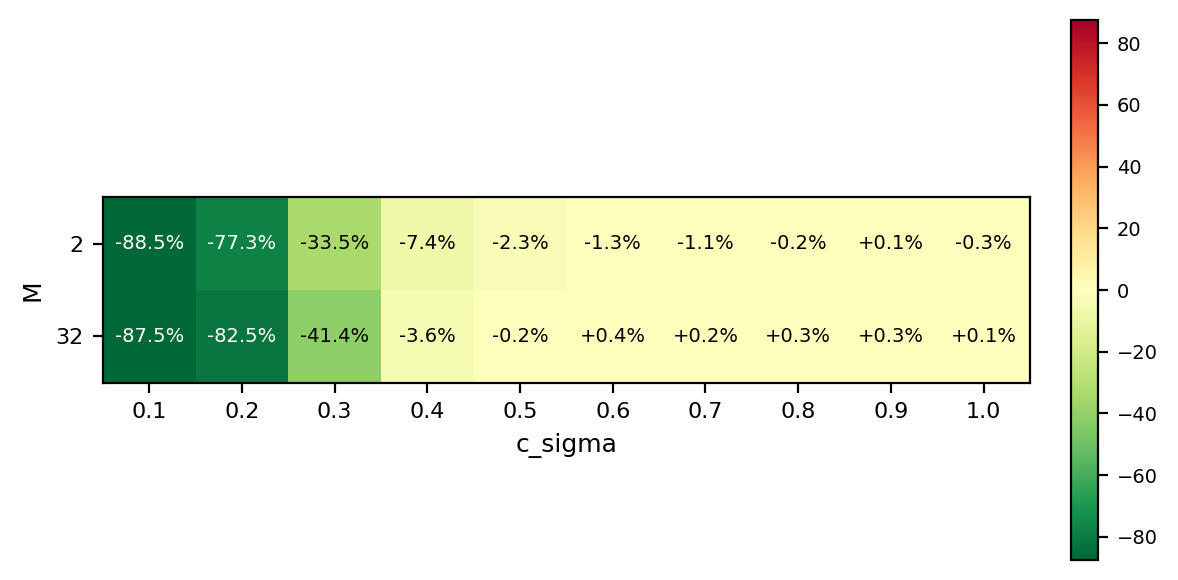} \\
{\scriptsize KID} &
{\scriptsize Recall} &
{\scriptsize DupRate}
\end{tabular}%
}
\caption{Grid-level comparison between MM-SOLD and $\sigma$-CFDM over $c_\sigma\in\{0.10,0.20,\ldots,1.0\}$ and $M\in\{2,32\}$, with the same fixed representative value of $c_\delta$ as in Figure~\ref{fig:digit8_appendix_grids}. Green indicates MM-SOLD is better, red indicates $\sigma$-CFDM is better, and yellow indicates comparable performance.}
\label{fig:digit8_generation_heatmaps}
\end{figure}

We next compare MM-SOLD with direct sampling from the estimated moment-matched limiting target. Figure~\ref{fig:digit8_nearest_neighbor_pairs} fixes $M=32$ and shows samples from MM-SOLD, Plug-in BAOAB, and SG-BAOAB under representative values of $(c_\delta,c_\sigma)=(0.4,0.6)$. For each generated digit, we also show the nearest training latent in the top-$40$ subspace, decoded by the same NRAE decoder.
Figure~\ref{fig:digit8_nearest_neighbor_pairs} confirms that the generated samples are not simple copies of their nearest training examples. For all three moment-matched samplers, the generated digits often differ from their nearest neighbors in stroke thickness, slant, loop size, local curvature, or writing style, while still remaining natural and recognizable as digit ``8''. This suggests that moment-matched score smoothing fits the digit manifold well enough to interpolate between training examples. MM-SOLD and SG-BAOAB show particularly clear examples of such manifold interpolation, whereas Plug-in BAOAB is more concentrated but still produces non-identical variants.

\begin{figure}[t]
\centering
\setlength{\tabcolsep}{2pt}
\renewcommand{\arraystretch}{0.9}
\begin{tabular}{ccc}
\includegraphics[width=0.32\textwidth]{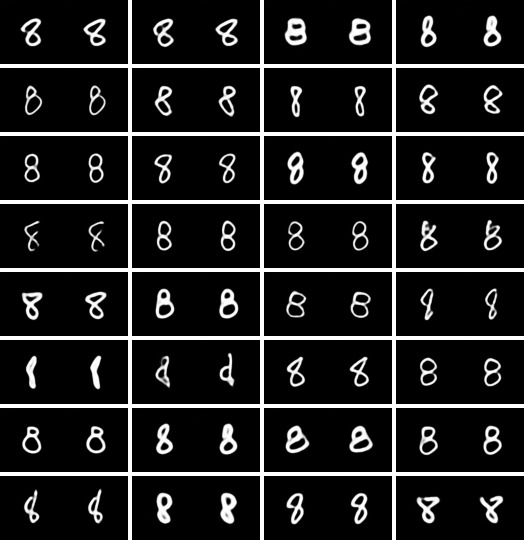} &
\includegraphics[width=0.32\textwidth]{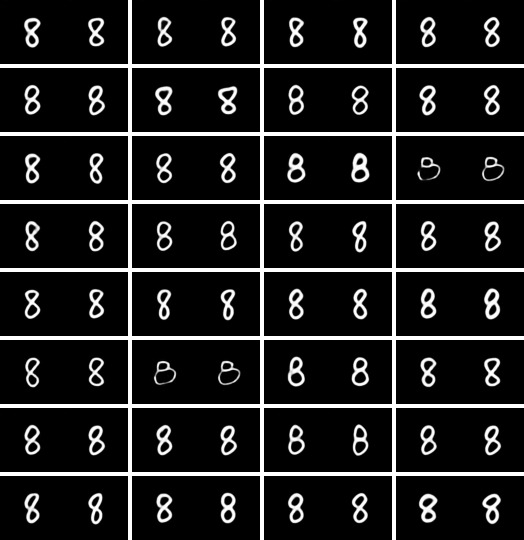} &
\includegraphics[width=0.32\textwidth]{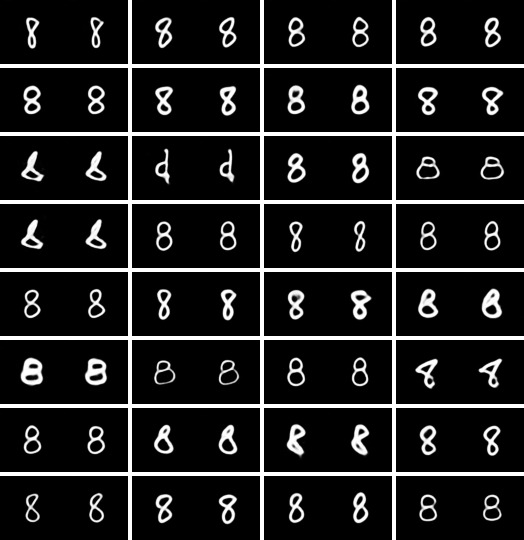} \\
{\scriptsize MM-SOLD} &
{\scriptsize Plug-in BAOAB} &
{\scriptsize SG-BAOAB}
\end{tabular}
\caption{Generated digit-8 samples and their nearest training-set reconstructions in the top-$40$ latent subspace. In each panel, every generated sample is paired with the decoded training latent closest to it in Euclidean distance.}
\label{fig:digit8_nearest_neighbor_pairs}
\end{figure}

Figure~\ref{fig:digit8_baoab_heatmaps} gives a full grid-level comparison among MM-SOLD, Plug-in BAOAB, and SG-BAOAB at $M=32$. The horizontal and vertical axes are $c_\sigma$ and $c_\delta$, respectively. The first row compares MM-SOLD against Plug-in BAOAB, the second row compares MM-SOLD against SG-BAOAB, and the third row compares Plug-in BAOAB against SG-BAOAB. Each row reports KID, Recall, and DupRate.
The pairwise heatmaps show that MM-SOLD is the most robust of the three moment-matched samplers. It improves over Plug-in BAOAB in KID and Recall in almost every grid cell, and over SG-BAOAB in KID and Recall across almost the whole grid, indicating that enforcing moments directly at the particle level is more stable than independently sampling from an estimated tilted target. Its DupRate is comparable for $c_\delta\ge0.5$, but higher than that of both BAOAB samplers when a small $c_\delta\le0.3$ is combined with a larger $c_\sigma$. The comparison between Plug-in BAOAB and SG-BAOAB shows the benefit of optimizing the tilting parameters: SG-BAOAB substantially improves KID and Recall over the plug-in estimator across most settings, confirming that the stochastic dual refinement moves the explicit target closer to the desired moment-matched distribution.

\begin{figure}[t]
\centering
\setlength{\tabcolsep}{0pt}
\renewcommand{\arraystretch}{0.9}
\makebox[\textwidth][c]{%
\begin{tabular}{@{}c@{\hspace{2pt}}c@{\hspace{2pt}}c@{}}

\multicolumn{3}{c}{\small\textbf{MM-SOLD vs Plug-in BAOAB}} \\[-1pt]
\includegraphics[width=0.32\textwidth, trim=5 5 58 5, clip]{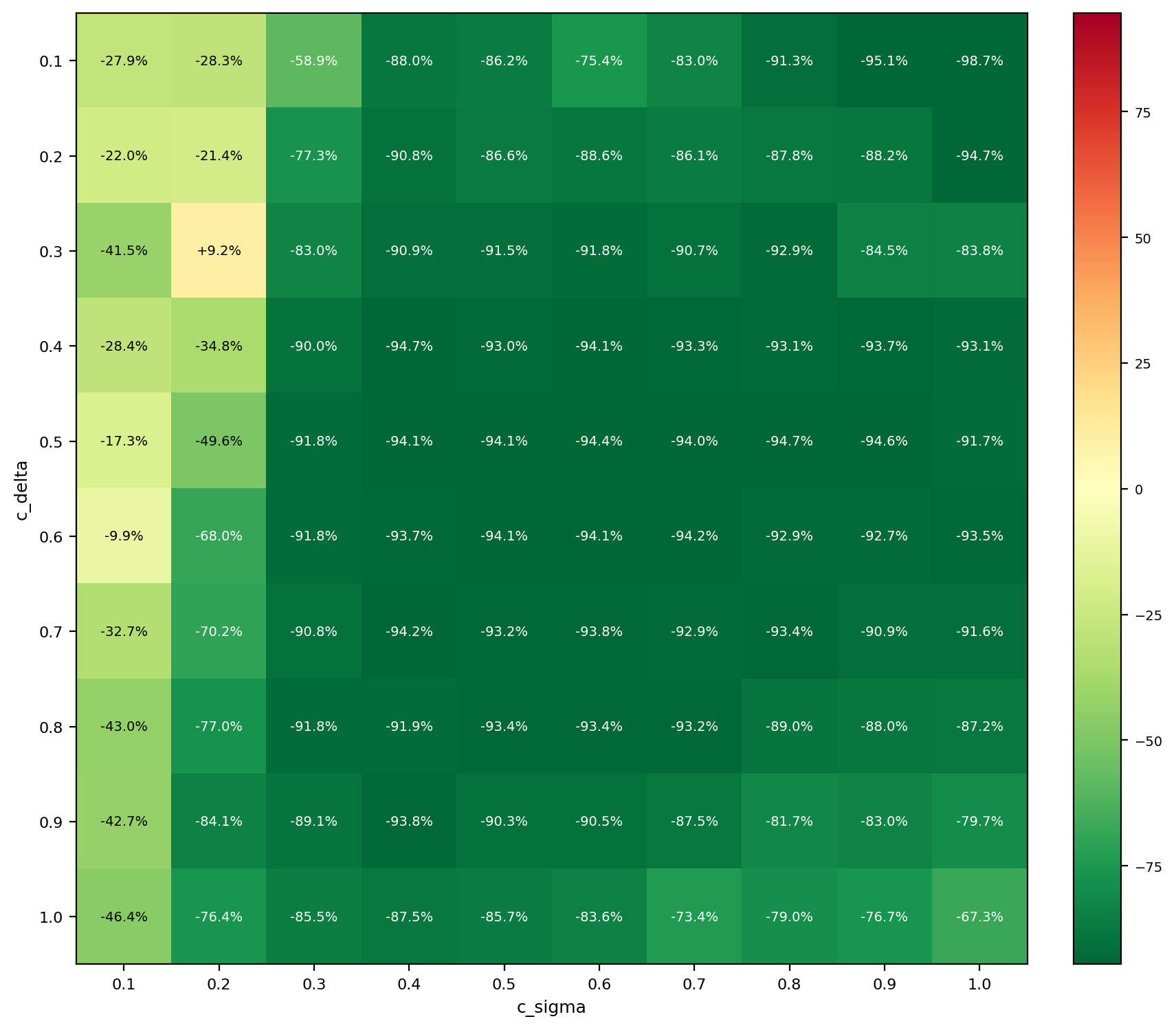} &
\includegraphics[width=0.32\textwidth, trim=5 5 58 5, clip]{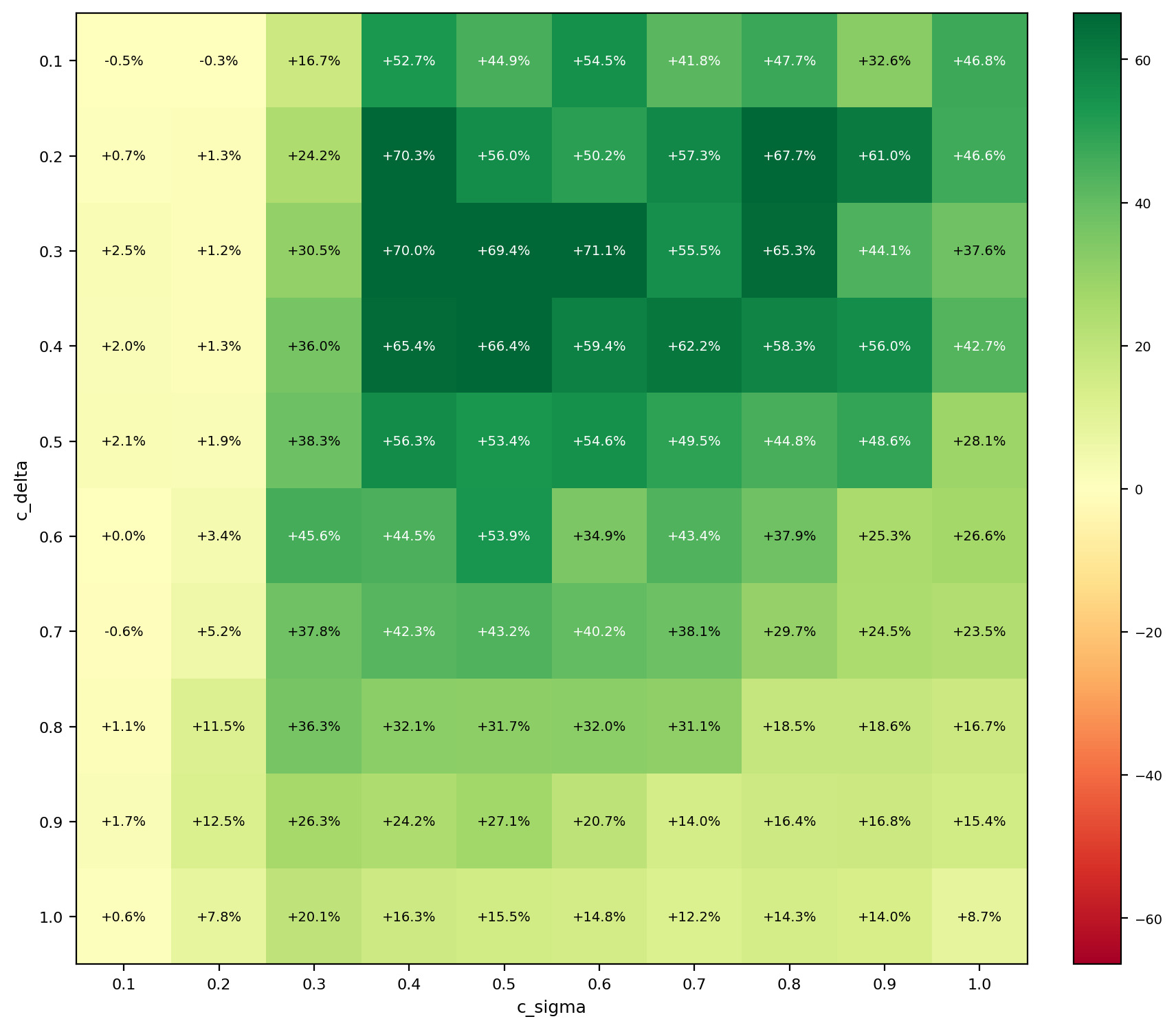} &
\includegraphics[width=0.32\textwidth, trim=5 5 58 5, clip]{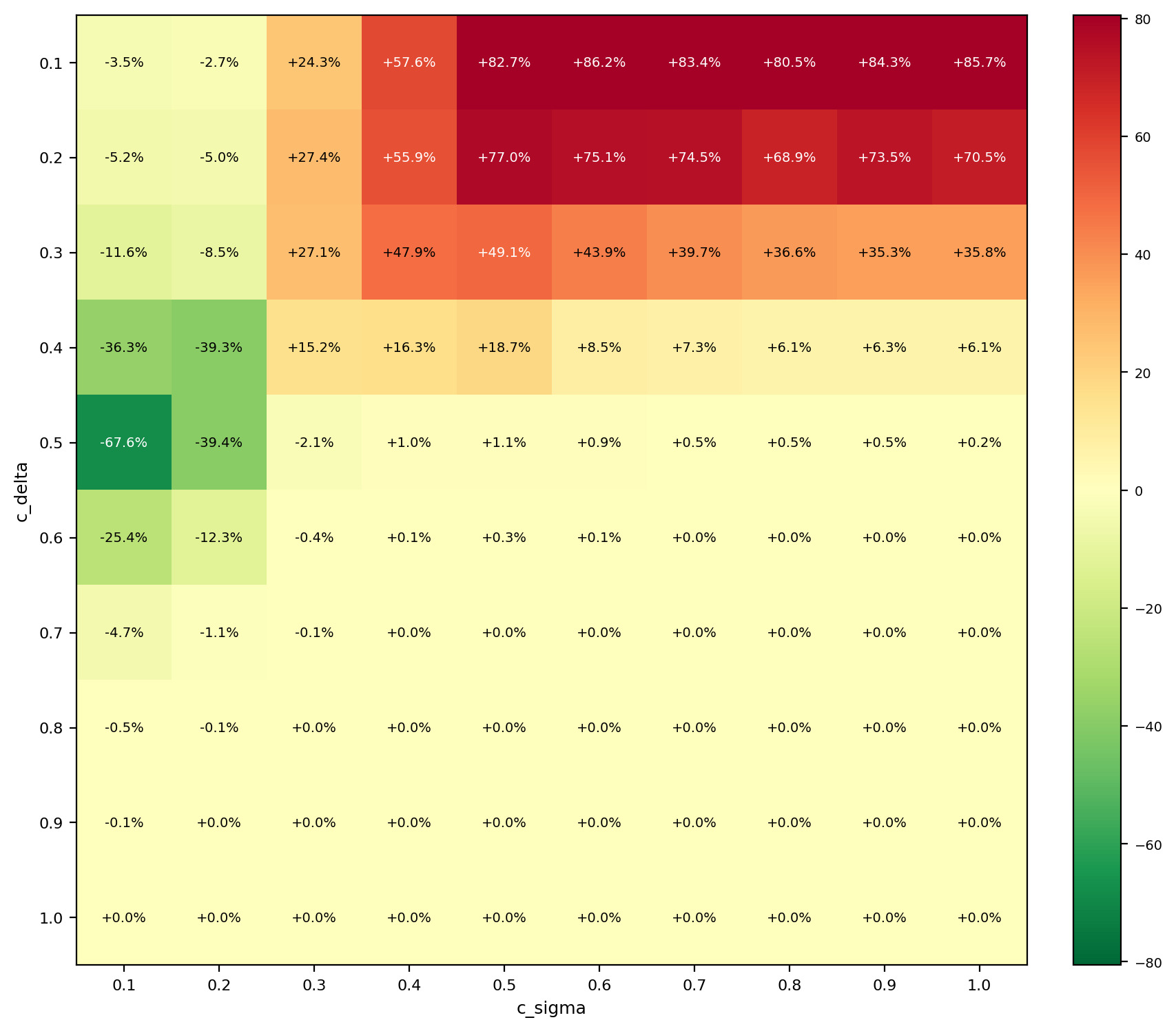} \\
{\scriptsize KID} &
{\scriptsize Recall} &
{\scriptsize DupRate} \\[5pt]

\multicolumn{3}{c}{\small\textbf{MM-SOLD vs SG-BAOAB}} \\[-1pt]
\includegraphics[width=0.32\textwidth, trim=5 5 58 5, clip]{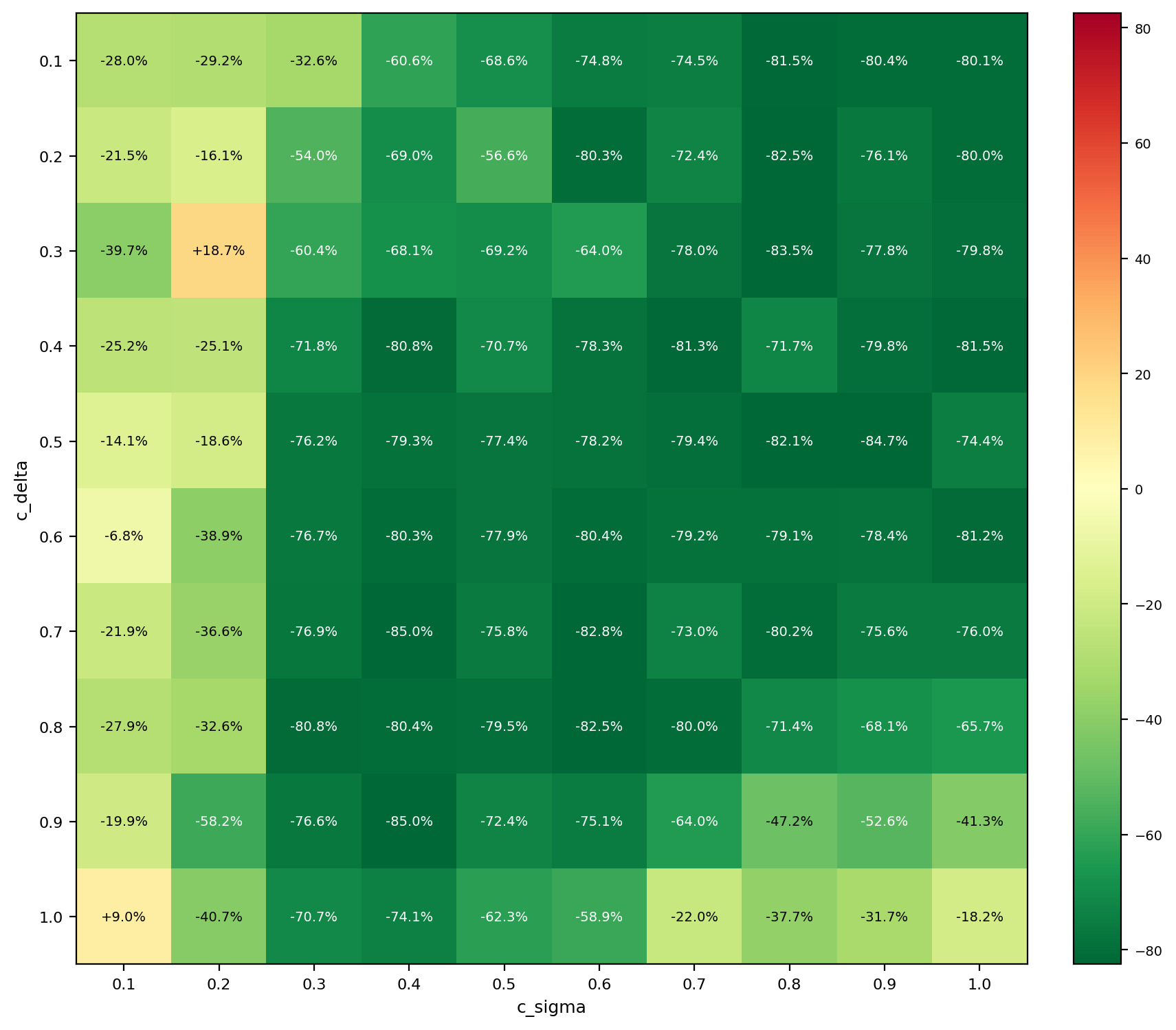} &
\includegraphics[width=0.32\textwidth, trim=5 5 58 5, clip]{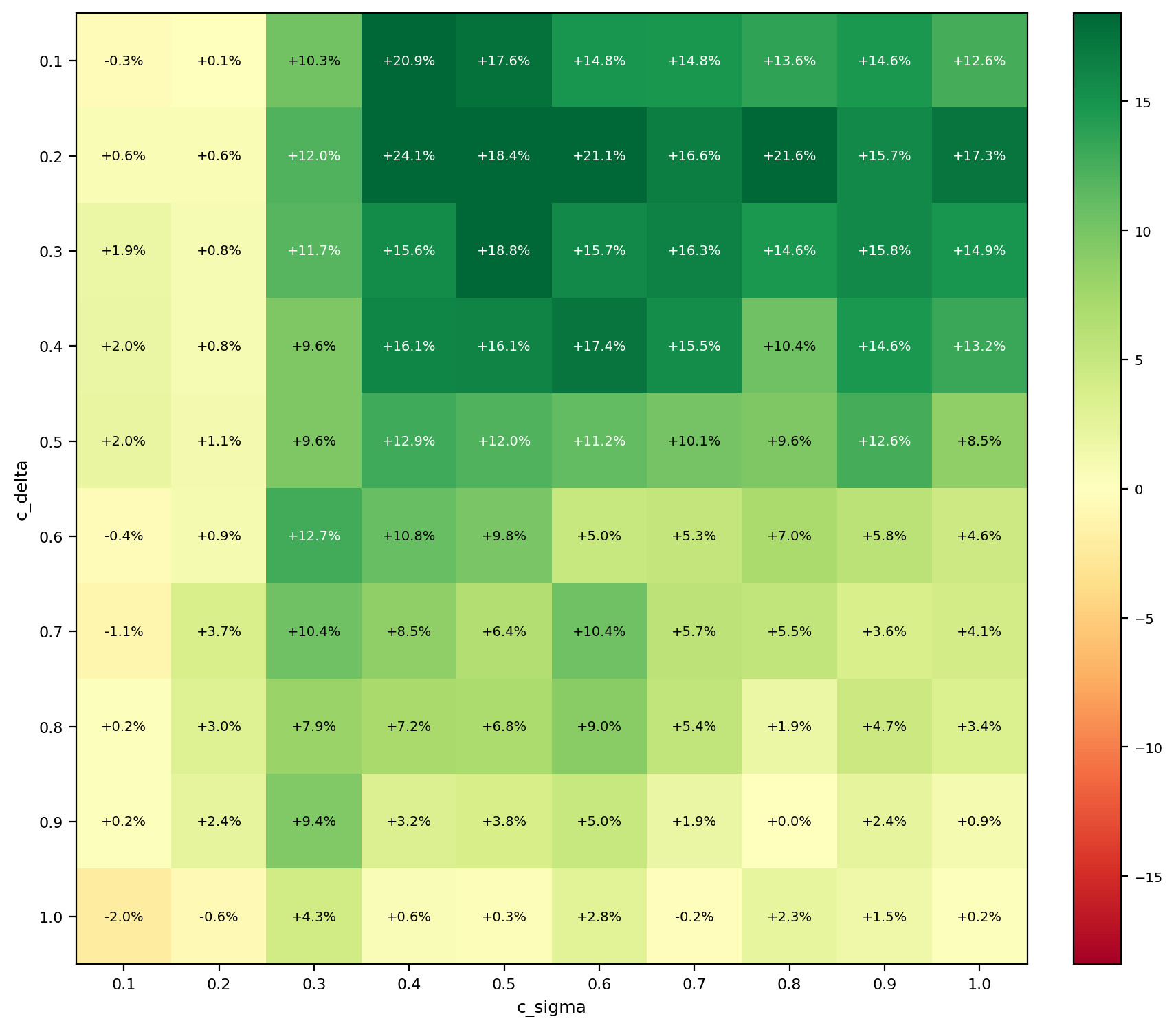} &
\includegraphics[width=0.32\textwidth, trim=5 5 58 5, clip]{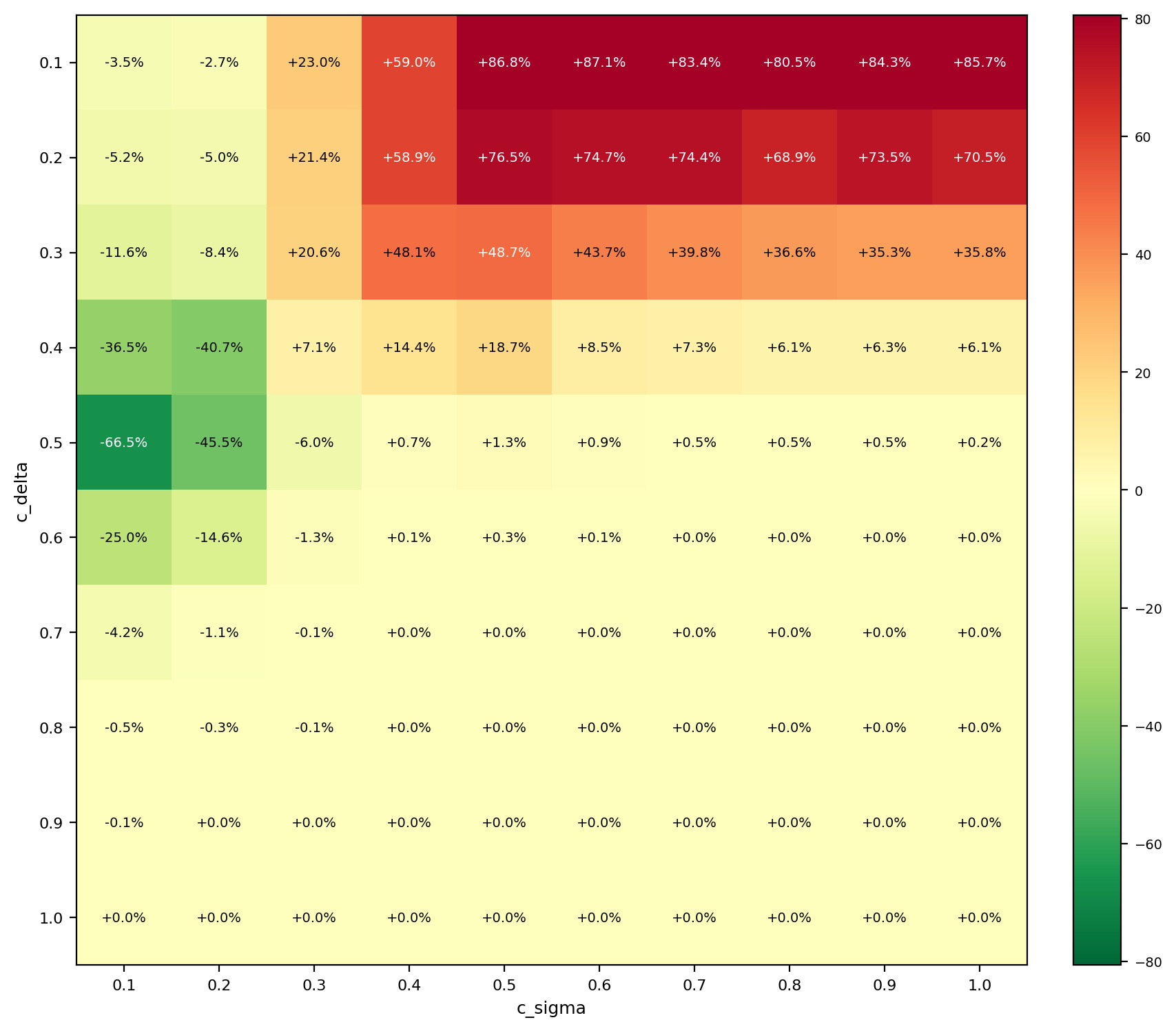} \\
{\scriptsize KID} &
{\scriptsize Recall} &
{\scriptsize DupRate} \\[5pt]

\multicolumn{3}{c}{\small\textbf{Plug-in BAOAB vs SG-BAOAB}} \\[-1pt]
\includegraphics[width=0.32\textwidth, trim=5 5 60 5, clip]{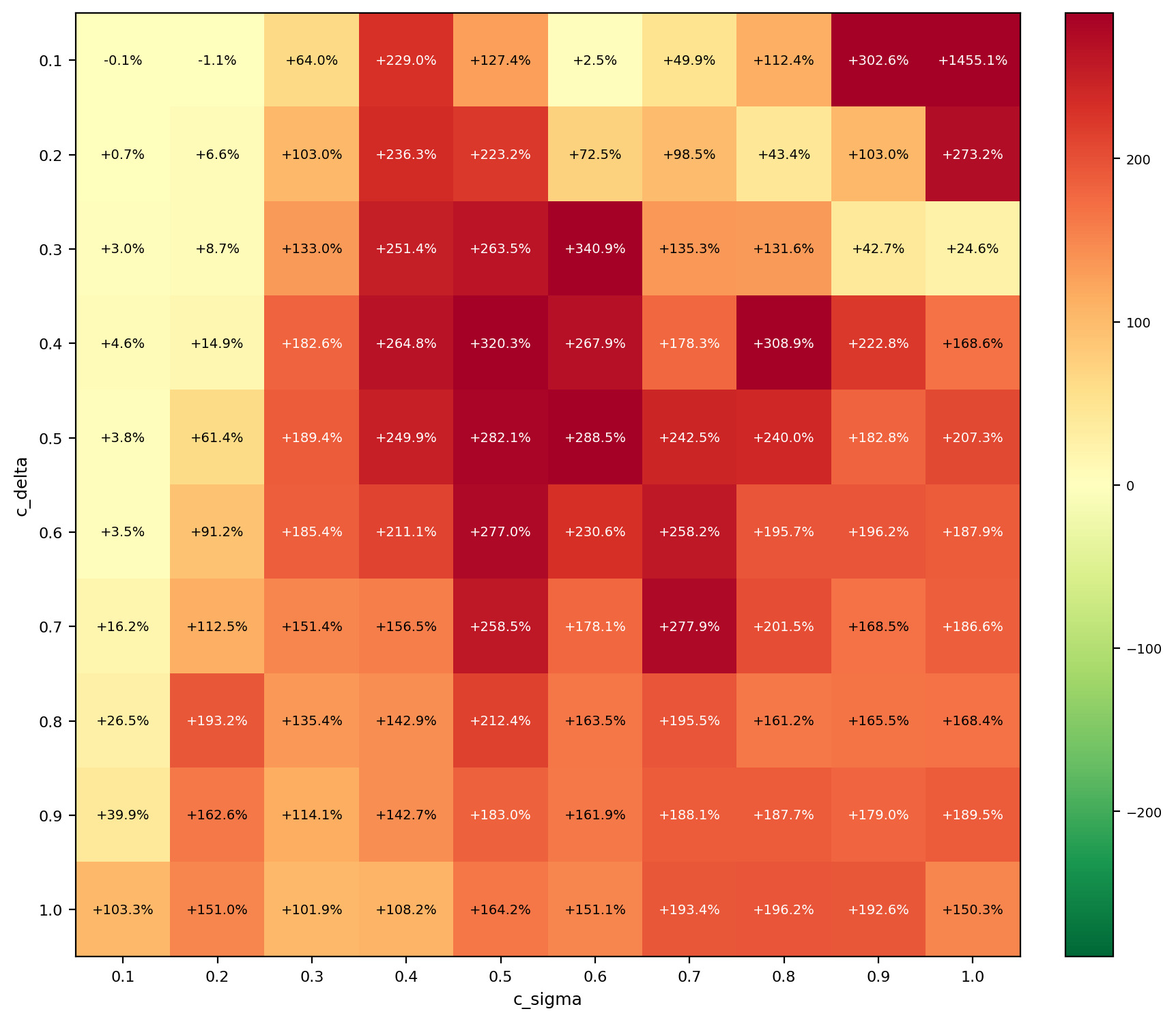} &
\includegraphics[width=0.32\textwidth, trim=5 5 58 5, clip]{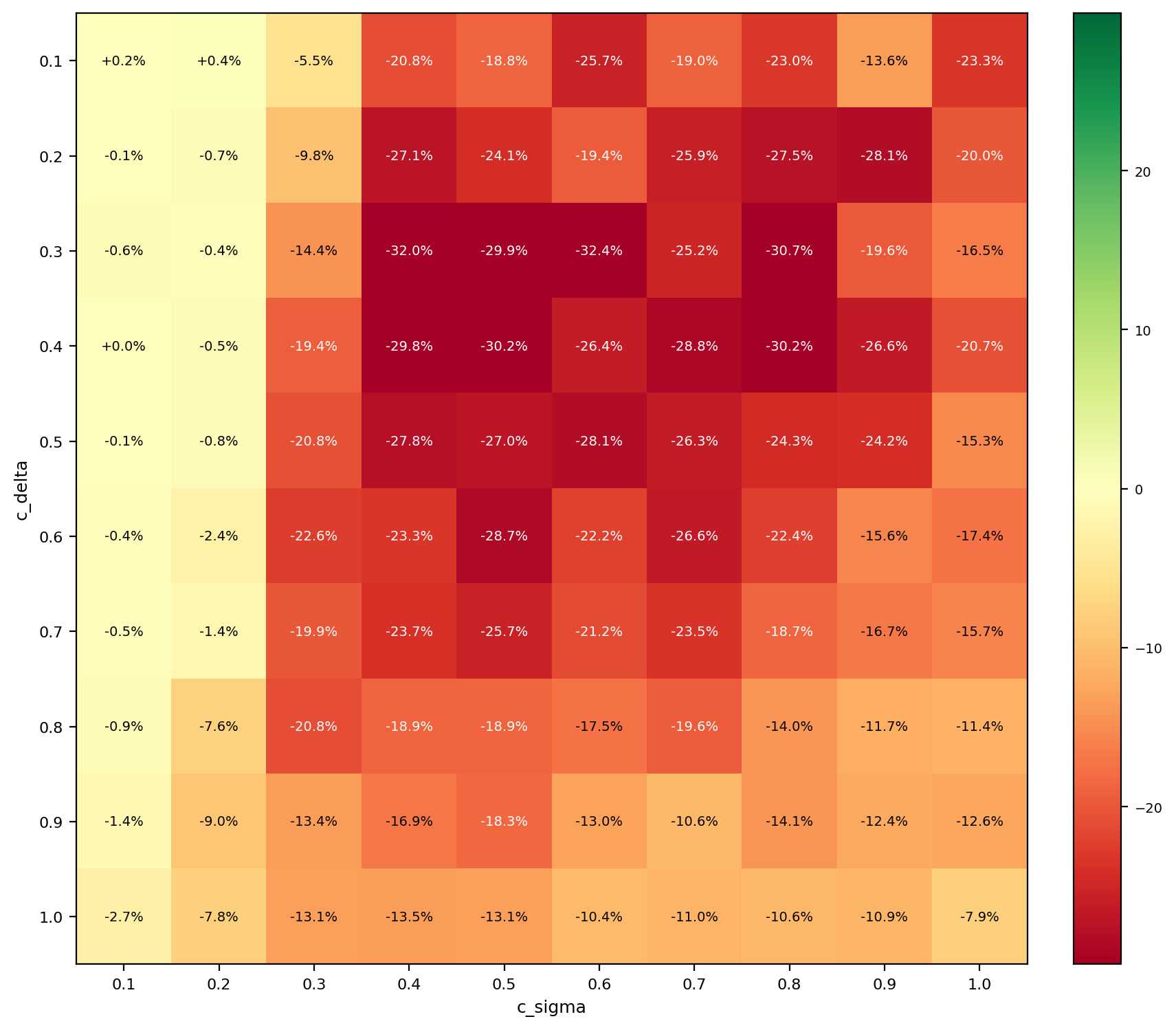} &
\includegraphics[width=0.32\textwidth, trim=5 5 58 5, clip]{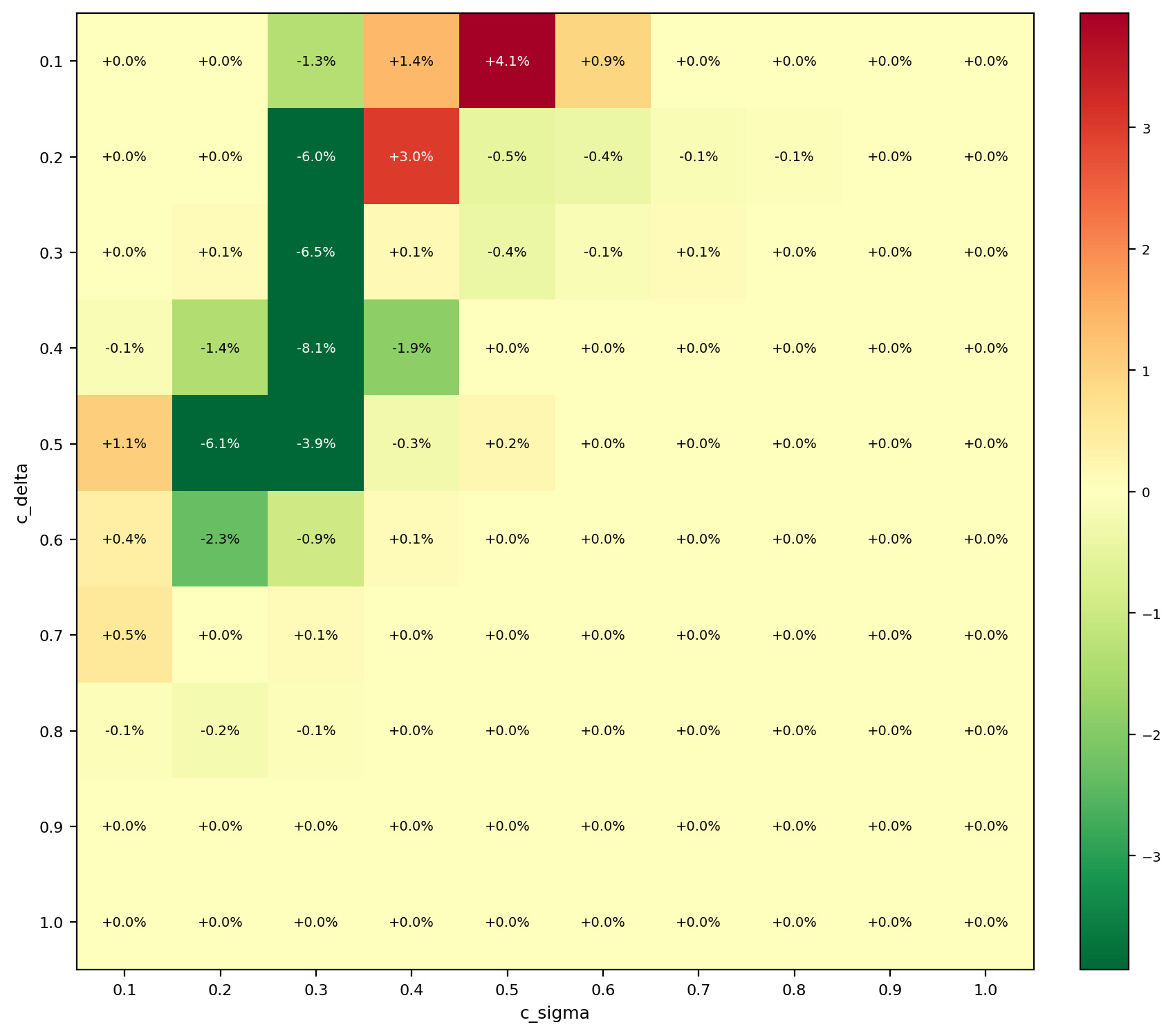} \\
{\scriptsize KID} &
{\scriptsize Recall} &
{\scriptsize DupRate}

\end{tabular}%
}
\caption{Pairwise grid-level comparisons among MM-SOLD, Plug-in BAOAB, and SG-BAOAB at $M=32$. Axes correspond to $c_\sigma$ and $c_\delta$ in the $10\times10$ search grid. In each comparison block, green indicates that the first named method is better, red indicates that the second named method is better, and yellow indicates comparable performance.}
\label{fig:digit8_baoab_heatmaps}
\end{figure}

Finally, Table~\ref{tab:digit8_posterior_moment_validation} validates whether the estimated moment-matched targets reproduce the prescribed GMM moments. We fix a representative $(c_\delta,c_\sigma)=(0.4,0.6)$ and set $M=32$. We compare three explicit BAOAB targets: the naive score-smoothed target with $\lambda=0$ and $\Lambda=0$, the plug-in tilted target, and the SG-refined tilted target. For each target, we run $10{,}000$ independent BAOAB chains from the digit-8 GMM for $10{,}000$ steps each ($h_{\mathrm{BAOAB}}=\delta/20$), compute the empirical mean and covariance of their endpoints in the top-$40$ subspace, and compare them with the GMM moments $\mu^*=\bar x$ and $\Sigma^*=\widehat\Sigma_N+\delta^2I$. KID, Recall, and DupRate are computed from the endpoints of a separate set of $300$ independent chains with the same settings. For this table and for the BAOAB samples in Figure~\ref{fig:digit8_nearest_neighbor_pairs}, chains are run for $10{,}000$ steps and the SG-BAOAB parameters are fitted with a larger budget than in the grid ($600$ Lion iterations with $300$ BAOAB steps per iteration and $100$ chains).
The result confirms that tilting the naive score-smoothed target substantially improves posterior moment matching and generation quality. The naive target has a large covariance error, $\|\widehat\Sigma-\Sigma^*\|_F/\|\Sigma^*\|_F=0.919$, and poor generation metrics, with KID $1.745$ and Recall $0.220$. Plug-in BAOAB reduces the relative covariance error to $0.633$ and improves KID and Recall to $0.719$ and $0.523$, respectively. SG-BAOAB further reduces the relative covariance error to $0.411$ and gives the best explicit-target performance, with KID $0.105$ and Recall $0.770$, while keeping DupRate at zero. The absolute mean error is nonmonotone across these estimators. These finite-budget measurements combine parameter error, numerical bias, incomplete mixing, and Monte Carlo fluctuations; they demonstrate improved moment matching, not convergence to exact moments.

\begin{table}[t]
\centering
\caption{Posterior moment validation for explicit BAOAB sampling on digit-8 latents. All quantities are computed in the top-$40$ latent subspace. The moment errors use $10{,}000$ independent BAOAB samples; KID, Recall, and DupRate are computed from $300$ generated samples.}
\label{tab:digit8_posterior_moment_validation}
\scriptsize
\setlength{\tabcolsep}{3.0pt}
\renewcommand{\arraystretch}{0.9}
\begin{tabular}{lcccccc}
\toprule
Setting
& $\|\bar z-\mu^*\|_2$
& $\|\widehat\Sigma-\Sigma^*\|_F$
& $\|\widehat\Sigma-\Sigma^*\|_F/\|\Sigma^*\|_F$
& KID $\downarrow$
& Recall $\uparrow$
& DupRate $\downarrow$ \\
\midrule
Naive BAOAB $(\lambda=0,\Lambda=0)$
& 0.0435 & 0.0374 & 0.919 & 1.745 & 0.220 & 0.003 \\
Plug-in BAOAB
& 0.0816 & 0.0258 & 0.633 & 0.719 & 0.523 & \textbf{0.000} \\
SG-BAOAB
& \textbf{0.0339} & \textbf{0.0167} & \textbf{0.411} & \textbf{0.105} & \textbf{0.770} & \textbf{0.000} \\
\bottomrule
\end{tabular}
\end{table}

\subsection{CelebA-HQ generation details}
\label{app:celebahq_generation_details}

\paragraph{Dataset and DiffAE representation}

We use the CelebA-HQ dataset~\cite{karras2017progressive} resized to $256\times256$ RGB images. The experiment uses $27{,}000$ training images and $3{,}000$ held-out test images.
Images are encoded by a fixed pretrained Diffusion Autoencoder (DiffAE)~\cite{preechakul2022diffusion} into $512$-dimensional latent vectors. We use the publicly released DiffAE autoencoder trained on FFHQ at $256\times256$ resolution, without any training or fine-tuning on CelebA-HQ. Specifically, we use the official \texttt{ffhq256\_autoenc} checkpoint. We did not check for image-level overlap between FFHQ and CelebA-HQ, so we only state that CelebA-HQ was not used to train the representation. The $3{,}000$ held-out images are never used by any latent sampler; they enter only through the evaluation metrics and the hyperparameter selection described below. The DiffAE decoder maps a latent code back to image space by running the conditional diffusion decoder from a fresh Gaussian initial noise. In all image-space evaluations, we use a $20$-step DPM-Solver++(2M) integrator on a uniform log-SNR grid. In our implementation, the DiffAE decoder is accelerated with a TensorRT fp16 engine, but decoding remains substantially more expensive than latent-space sampling. This motivates the latent-space model-selection protocol described below.

\paragraph{Top-$190$ latent subspace and Gaussian tail completion}

Let $z_1,\ldots,z_N\in\mathbb R^{512}$ be the $27{,}000$ DiffAE training latents and let $\bar z$ be their empirical mean. We compute the empirical covariance of the centered latents and write its eigendecomposition as
\begin{equation}
\widehat{\Sigma}_{512}
=
V
\operatorname{diag}(\lambda_1,\ldots,\lambda_{512})
V^\top,
\qquad
\lambda_1\ge \lambda_2\ge \cdots \ge \lambda_{512}.
\end{equation}
We retain the top $d_{\mathrm{sub}}=190$ principal directions, denoted by $V_{\mathrm{top}}\in\mathbb R^{512\times 190}$, and write the discarded directions as $V_{\mathrm{tail}}\in\mathbb R^{512\times 322}$. The top-$190$ subspace explains $90.1\%$ of the training latent variance. All training-free samplers and BAOAB-based samplers operate on the unwhitened top-$190$ coordinates
\begin{equation}
u_i
=
V_{\mathrm{top}}^\top(z_i-\bar z)
\in\mathbb R^{190}.
\end{equation}
No whitening is applied: the retained PCA coordinates keep their empirical scales. Here ``no whitening'' refers to the latent coordinate system used for the potential, score estimator, metrics, and baselines; MM-SOLD still uses its internal moment-normalized coordinates when enforcing the mean--covariance constraint.

After a method generates a top-$190$ latent coordinate $u$, we reconstruct a full $512$-dimensional DiffAE latent by sampling the discarded PCA coordinates independently:
\begin{equation}
b
\sim
\mathcal N
\left(
0,
\operatorname{diag}(\lambda_{191},\ldots,\lambda_{512})
\right),
\qquad
\widetilde z
=
\bar z
+
V_{\mathrm{top}}u
+
V_{\mathrm{tail}}b .
\end{equation}
The reconstructed latent $\widetilde z$ is then decoded by the fixed DiffAE decoder. This construction retains variation in the discarded $322$ directions through the empirical Gaussian tail.

\paragraph{Nearest-neighbor score estimation}

Evaluating the smoothed GMM score against all $27{,}000$ training samples for every particle, Monte Carlo perturbation, and iteration is computationally and memory costly. We therefore use the nearest-neighbor-plus-remainder score estimator in Appendix~\ref{app:nn_score}. For MM-SOLD and $\sigma$-CFDM, each query keeps the $K=50$ nearest training latents and samples $L=50$ additional points from the remaining training set. For the BAOAB-based tilted-target samplers, we use the same estimator with $K=L=200$ in the CelebA-HQ top-$190$ subspace. In all cases, the score estimator is applied in the operating PCA subspace before Gaussian tail completion.

\paragraph{Common scale parametrization}

Let $u_1,\ldots,u_N\in\mathbb R^{190}$ be the CelebA-HQ training latents in the top-$190$ PCA subspace. We estimate the nearest-neighbor scale by the average nearest-neighbor distance over $4{,}000$ randomly selected training latents:
\begin{equation}
\mathrm{NN}
=
\frac{1}{|\mathcal S|}
\sum_{i\in\mathcal S}
\min_{j\ne i}
\|u_i-u_j\|_2,
\qquad
|\mathcal S|=4{,}000.
\end{equation}
For all methods based on the GMM potential or score smoothing, we parameterize the GMM component standard deviation and score-smoothing bandwidth as
\begin{equation}
\delta
=
c_\delta\frac{\mathrm{NN}}{\sqrt{190}},
\qquad
\sigma
=
c_\sigma\mathrm{NN}.
\end{equation}

For each method, the configuration with the lowest latent-space MMD$^2$, averaged over the three runs, is used for the latent MMD$^2$ entry and for all decoded-image entries (image-space and ArcFace FID, KID and MMD); the Recall and DupRate entries report the best three-run mean over the grid for each metric separately. Exact ties are broken by grid order. The decoded metrics use three new runs at the selected configuration rather than the batches generated during the grid search.

\paragraph{MM-SOLD sampler}

MM-SOLD uses the prescribed GMM covariance $\Sigma^*=\widehat\Sigma_{190}+\delta^2 I_{190}$. It is initialized from the GMM, normalized onto the corresponding moment manifold, and evolved using the LM discretization. We use $P=3{,}000$ particles and obtain $3{,}000$ generated samples for evaluation. For each value of $\delta$, the step size is set to
\begin{equation}
h=0.2\delta^2.
\end{equation}
The sampler is run for $1{,}000$ LM steps. The smoothed score is estimated with the nearest-neighbor-plus-remainder estimator using $K=L=50$. After sampling in the top-$190$ subspace, generated latents are completed by the independent Gaussian tail described above and decoded by the fixed DiffAE decoder.

\paragraph{$\sigma$-CFDM baseline}

The $\sigma$-CFDM baseline is implemented in the same unwhitened top-$190$ latent subspace and uses the same scale parametrization for $\delta$ and $\sigma$. Particles are evolved with the deterministic closed-form diffusion ODE~\cite{scarvelis2023closed}. We use a uniform grid with spacing $\Delta t=10^{-3}$ and perform $999$ explicit Euler updates $z\leftarrow z+\Delta t\,v(z,t_k)$, with the velocity evaluated at $t_k=k\Delta t$, $k=1,\ldots,999$. The initial Gaussian sample is assigned to $t_1=0.001$, so the final state corresponds to $t=1$, where the velocity is singular and is never evaluated. The smoothed score is estimated with the same $K=L=50$ nearest-neighbor-plus-remainder estimator as MM-SOLD. Generated top-$190$ latents are completed with the independent Gaussian tail before DiffAE decoding.

\paragraph{KNN-PFlow baseline}

KNN-PFlow is a nearest-neighbor empirical-score probability-flow ODE sampler. It is run in the same top-$190$ latent subspace. At noise level $s$, it uses the nonparametric denoiser
\begin{equation}
D_s(z)
=
\sum_{i\in\mathcal I_K(z)}
w_i(z,s)u_i,
\qquad
w_i(z,s)
=
\frac{
\exp\left(-\|z-u_i\|^2/(2s^2)\right)
}{
\sum_{j\in\mathcal I_K(z)}
\exp\left(-\|z-u_j\|^2/(2s^2)\right)
},
\end{equation}
where $\mathcal I_K(z)$ is the set of the $K$ nearest training latents to $z$. It then integrates the variance-exploding probability-flow ODE
\begin{equation}
\frac{dz}{ds}
=
\frac{z-D_s(z)}{s}
\end{equation}
from $s_{\max}$ to $s_{\min}$. We initialize from $z=s_{\max}\xi$ with $\xi\sim\mathcal N(0,I_{190})$; since the operating coordinates are centred at the training mean, this is a Gaussian centred at the training mean, the usual initialization of variance-exploding probability-flow samplers. To make the search scale-free, we set
\begin{equation}
s_{\max}=c_{\max}\mathrm{NN},
\qquad
s_{\min}=c_{\min}\frac{\mathrm{NN}}{\sqrt{190}},
\end{equation}
and search over
\begin{equation}
c_{\max}\in\{1,2,4\},
\qquad
c_{\min}\in\{0.025,0.05,0.10\},
\qquad
K\in\{50,100,250\}.
\end{equation}
We use the same $1{,}000$-step noise schedule, solvers and final denoising step as for the digits. Generated top-$190$ latents are completed with the same Gaussian tail before decoding.

\paragraph{Sliced-Wasserstein Flow baseline}

Sliced-Wasserstein Flow (SWF) is a training-free particle-flow baseline based on sliced optimal transport. Let $Z\in\mathbb R^{P\times 190}$ denote generated particles and let $U\in\mathbb R^{N\times 190}$ denote the CelebA-HQ training latents in the top-$190$ subspace. At each iteration, SWF samples random directions $\theta_1,\ldots,\theta_R$ on the unit sphere, projects generated and training latents onto each direction, computes the one-dimensional quantile-matching update $T_{\theta_r}$ toward the empirical distribution of $U\theta_r$, and updates
\begin{equation}
Z
\leftarrow
Z
+
\frac{\eta}{R}
\sum_{r=1}^R
\left(
T_{\theta_r}(Z\theta_r)-Z\theta_r
\right)\theta_r^\top .
\end{equation}
We search over
\begin{equation}
n_{\mathrm{SWF}}\in\{300,1000\},
\qquad
R\in\{64,128,256\},
\qquad
\eta\in\{0.05,0.10\}.
\end{equation}
We use two initializations: a fitted Gaussian $\mathcal N(\mu_{\mathrm{celeb}},\Sigma_{\mathrm{celeb}})$ in the top-$190$ subspace and a noised empirical initialization $u_i+\mathrm{NN}\xi$ with $\xi\sim\mathcal N(0,I_{190})$. Generated top-$190$ latents are completed with the same Gaussian tail before decoding.

\paragraph{Plug-in BAOAB sampler}

Plug-in BAOAB samples independently from the explicit moment-matched target~\eqref{eq:moment_matched_target} using tilting parameters estimated by the closed-form plug-in estimator in Appendix~\ref{app:tilting-parameter-estimation}. The target moments are $\mu^*$ and $\Sigma^*=\widehat{\Sigma}_{190}+\delta^2 I_{190}$, following the same GMM-covariance convention as MM-SOLD. For each selected $(c_\delta,c_\sigma,M)$ setting, we estimate $(\widehat\lambda,\widehat\Lambda)$ from the $27{,}000$ CelebA-HQ training latents in the top-$190$ subspace and then run independent BAOAB kinetic Langevin chains targeting
\begin{equation}
\widehat p_{\mathrm{plug}}(z)
\propto
\exp\left\{
-V(z)
-\widehat\lambda^\top(z-\mu^*)
-\frac{1}{2}(z-\mu^*)^\top\widehat\Lambda(z-\mu^*)
\right\}.
\end{equation}
Positions are initialized from the same CelebA-HQ GMM used by MM-SOLD, and momenta are initialized from $\mathcal N(0,I_{190})$. We use $1{,}000$ BAOAB steps with step size
\begin{equation}
h_{\mathrm{BAOAB}}
=
\frac{\delta}{4},
\end{equation}
and friction
\begin{equation}
\gamma
=
\frac{1}{\sqrt{\lambda_{\max}(\widehat\Sigma_{190})}} .
\end{equation}
The smoothed score in the BAOAB force evaluations is estimated using the nearest-neighbor-plus-remainder estimator with $K=L=200$. Generated top-$190$ latents are completed with the independent Gaussian tail before DiffAE decoding.

\paragraph{SG-BAOAB sampler}

SG-BAOAB uses a practical variant of the dual estimator from Appendix~\ref{app:tilting-parameter-estimation} before sampling independently from the estimated moment-matched target. The asymptotic guarantee in Proposition~\ref{prop:baoab-dual-consistency} applies to the safeguarded update specified there. We initialize the dual parameters at the plug-in estimate and optimize them with Lion, using the projection and the other conventions described in Appendix~\ref{app:handwritten_generation_details}.
The default CelebA-HQ setting uses $300$ Lion iterations. Each iteration uses $P_{\mathrm{Lion}}=1{,}000$ parallel BAOAB chains, each run for $100$ BAOAB steps targeting the current tilted density. The BAOAB step size is $h_{\mathrm{BAOAB}}=\delta/4$, the friction is $\gamma=1/\sqrt{\lambda_{\max}(\widehat\Sigma_{190})}$, and the score estimator uses $K=L=200$ nearest-neighbor-plus-remainder samples. The relative Lion step size is $1.0$ and is cosine-annealed to zero; the momentum coefficients are $\beta_1=0.9$ and $\beta_2=0.99$. After optimization, each run draws $3{,}000$ samples from the final estimated tilted density as the endpoints of independent BAOAB chains run for $1{,}000$ steps from the training GMM, with the same step size, friction and score estimator; we complete the discarded $322$ PCA coordinates with the independent Gaussian tail and decode the resulting $512$-dimensional DiffAE latents.

\paragraph{Latent DDPM baseline}

The latent DDPM baseline is trained in the full $512$-dimensional DiffAE latent space rather than in the top-$190$ PCA subspace. We use the DiffAE latent-network architecture implemented by \texttt{MLPSkipNetConfig}: a $10$-layer MLP skip network with hidden width $2048$ and skip layers $1,\ldots,9$. Training uses coordinate-wise normalized DiffAE latents, $\nsteps=1{,}000$ diffusion steps with constant beta $0.008$, epsilon prediction, and an $L^1$ noise-prediction loss. The optimizer is AdamW with learning rate $10^{-4}$, weight decay $0.01$, batch size $256$, and exponential moving average rate $0.9999$. The training budget is approximately $101$ million latent samples. Sampling in latent space uses deterministic DDIM with $20$ reverse steps. The generated latents are then unnormalized and decoded by the same fixed DiffAE image decoder used for all other methods, namely the $20$-step DPM-Solver++(2M) decoder on a uniform log-SNR grid.

\paragraph{Metrics and timing}

We report both image-space and latent-space metrics. Image-space FID and KID use the $2{,}048$-dimensional pool features of the clean-fid InceptionV3 network for the $3{,}000$ held-out images and the $3{,}000$ decoded samples of each run; decoded images are scaled to $[0,255]$ and bilinearly resized to $299\times299$, and the features of the held-out images are computed once with the same network. ArcFace metrics use the $512$-dimensional, $\ell_2$-normalized embeddings of the insightface \texttt{buffalo\_l} recognition model (\texttt{w600k\_r50})~\cite{deng2019arcface}; since CelebA-HQ faces are already aligned, we skip face detection, take the central $200\times200$ crop and resize it to $112\times112$. In both feature spaces, KID uses the polynomial kernel $(x^\top y/d_{\mathrm{feat}}+1)^3$ and is averaged over $50$ random subsets of $100$ real and $100$ generated features. The MMD entries report the unbiased MMD$^2$ U-statistic itself (not its square root, and without clipping) for the single RBF kernel $k(x,y)=\exp(-\|x-y\|_2^2/(2\tau_{\mathrm{med}}^2))$, averaged over $100$ random subsets of $1{,}000$ real and $1{,}000$ generated features. In each feature space, the bandwidth is fixed once by the median heuristic, $\tau_{\mathrm{med}}^2=\tfrac12\,\mathrm{median}\,\|x-x'\|_2^2$ over pairs from two random samples of $2{,}000$ held-out reference features, and is shared by all methods and runs. Because the DiffAE decoder is slow, these decoded-image metrics are not computed for every hyperparameter grid cell. Instead, for each method we first select hyperparameters using latent-space MMD$^2$, and then decode only the selected $3{,}000$ samples to compute image-space and ArcFace-space metrics.

Latent-space MMD$^2$ is computed directly in the raw $512$-dimensional DiffAE latent space against the $3{,}000$ held-out test latents, using the MMD$^2$ U-statistic with the RBF-mixture kernel
\begin{equation}
k_{\mathrm{lat}}(x,y)
=
\sum_{c\in\{0.25,0.5,1,2,4\}}
\exp\left(
-\frac{\|x-y\|_2^2}{2(c\tau_{\mathrm{med}})^2}
\right),
\end{equation}
where $\tau_{\mathrm{med}}$ is fixed once by the median heuristic on randomly sampled pairs of training DiffAE latents. We report $10^3$ times the resulting MMD$^2$ statistic, averaged over $100$ random subsets of size $1{,}000$. This RBF-mixture statistic is not polynomial-kernel KID. MMD$^2$ U-statistics can be negative. Their usual unbiasedness statement assumes independent samples within each set and does not apply directly to the coupled MM-SOLD particles. We therefore interpret these and the digit KID values as descriptive sample-comparison statistics. The same held-out set is used for grid selection and reporting, so these are exploratory comparisons, not estimates on an untouched final test set.

Recall and DupRate are computed in the top-$190$ operating subspace without image decoding. For Recall, we use the same nearest-neighbor coverage criterion as in the image-space experiments, but apply it to top-$190$ latent coordinates. For DupRate, we use the $1$st percentile of within-training nearest-neighbor distances in the top-$190$ subspace as the duplicate threshold. These latent-space metrics can therefore be evaluated over the full hyperparameter grid. For each hyperparameter setting, we generate $3{,}000$ samples in each of three independent runs. Latent-space MMD$^2$, Recall, and DupRate are computed for every run and reported as mean $\pm$ standard deviation. Each run redraws the sampler randomness, the Gaussian tail completion and, for decoded metrics, the decoder noise $x_T$; the Plug-in BAOAB and SG-BAOAB tilting parameters are estimated once per configuration and the latent DDPM is trained once. For image-space FID and KID, we first select the configuration by latent-space MMD$^2$ and then decode $3{,}000$ samples from each of the three independent runs under the selected configuration. Sampling time is reported for latent generation; DiffAE decoding time is common to all decoded methods and is discussed separately when reporting image-space metrics.  Sampling times follow the digit protocol (one warm-up call and five timed calls on the JAX CPU backend, with the output returned as a NumPy array), at the fixed timing configuration $(c_\delta,c_\sigma,M)=(0.10,0.10,2)$ for MM-SOLD, $\sigma$-CFDM, Plug-in BAOAB and SG-BAOAB, at $c_{\max}=1$, $c_{\min}=0.10$, $K=50$ and Euler integration for KNN-PFlow, and at $n_{\mathrm{SWF}}=300$, $R=64$, $\eta=0.05$ and Gaussian initialization for SWF. These are not the selected quality configurations; the cost does not depend on $(c_\delta,c_\sigma)$ but increases with $M$. Each call generates $3{,}000$ latents ($\sigma$-CFDM and KNN-PFlow in batches of $500$), and the time per sample is the call time divided by $3{,}000$. Gaussian tail completion and decoding are excluded. The SG-BAOAB optimization time is one GPU run including JIT compilation, and the latent DDPM is trained and sampled on GPU.

\paragraph{Compute resources and software}
All experiments were run as single-GPU jobs with an NVIDIA RTX 6000 Ada GPU, a SLURM cluster with $4$--$8$ CPU cores and $16$--$64$\,GB of RAM per job; the CPU timings use $8$ cores on an Intel Xeon Gold 6426Y. We used Python~3.10, JAX~0.6.2 (CUDA~12), Flax~0.10.7, Optax~0.2.6, NumPy~2.2.6, SciPy~1.15.3 and PyTorch~2.9.1. The largest costs were the CelebA-HQ grid searches, the dual optimizations and DiffAE decoding; we did not record the aggregate GPU hours of all searches. The DiffAE autoencoder is an external pretrained model, and its training cost is not included.

\paragraph{Additional results}

We provide additional CelebA-HQ generation results to assess hyperparameter robustness and to compare particle-level moment matching with direct independent sampling from the estimated moment-matched target. All training-free and BAOAB-based results in this subsection are computed in the unwhitened top-$190$ DiffAE latent subspace. Generated top-$190$ latents are completed by the independent Gaussian tail before decoding.

Figures~\ref{fig:celebahq_appendix_grids_cfdm} and~\ref{fig:celebahq_appendix_grids_mmsold} first compare $\sigma$-CFDM and MM-SOLD under different smoothing bandwidths and numbers of Monte Carlo noise samples. We fix a representative value of $c_\delta=0.3$ and show generated samples across selected values of $(c_\sigma,M)$, with $c_\sigma\in\{0.1,1.0\}$ and $M\in\{2,32\}$. The behavior of $\sigma$-CFDM is visibly sensitive to these parameters. Increasing $c_\sigma$ produces more unnatural feature combinations and image artifacts, while increasing $M$ makes the samples substantially less diverse and more concentrated around an averaged facial appearance, reflecting the barycentric nature of $\sigma$-CFDM. In contrast, MM-SOLD is more stable across the same settings: it maintains high-fidelity face structure while preserving variation in identity, pose, hair, background, and local facial details. This is consistent with the behavior observed in the 2D experiments in Figure~\ref{fig:2d_checkerboard_spirals} and in the handwritten digit-8 generation experiment in Figure~\ref{fig:digit8_appendix_grids}.

Figure~\ref{fig:celebahq_generation_heatmaps} gives the corresponding grid-level comparison between MM-SOLD and $\sigma$-CFDM. Since image-space FID and KID are not evaluated over the full grid, the heatmaps use latent-space MMD$^2$ together with latent-space Recall and DupRate diagnostics. They follow the conventions described in Appendix~\ref{app:handwritten_generation_details}. Some latent MMD$^2$ estimates are close to zero, so some relative differences are very large in magnitude (for example $-16{,}506.5\%$); for these cells only the sign of the entry is informative.
The heatmaps confirm the qualitative trend in Figures~\ref{fig:celebahq_appendix_grids_cfdm} and~\ref{fig:celebahq_appendix_grids_mmsold}: MM-SOLD improves latent-space MMD$^2$ over most settings, especially in regimes where $\sigma$-CFDM suffers from barycentric averaging. Recall is broadly comparable across many settings, while the DupRate comparison shows that MM-SOLD's improvement in latent-space fidelity is not obtained by moving generated samples closer to the training set.

\begin{figure}[!t]
\centering
\setlength{\tabcolsep}{1.5pt}
\renewcommand{\arraystretch}{0.8}
\begin{tabular}{cc}
\includegraphics[width=0.47\textwidth]{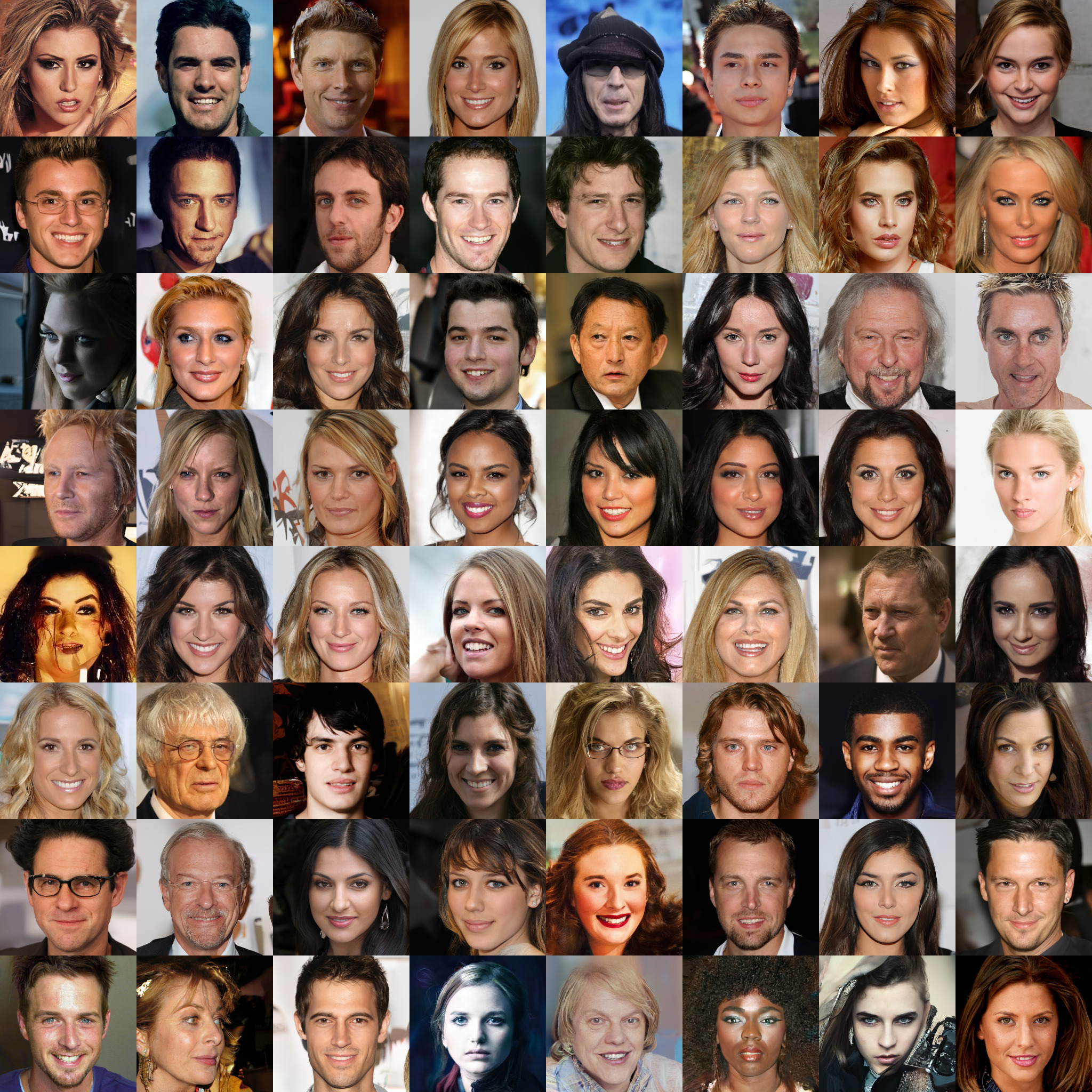} &
\includegraphics[width=0.47\textwidth]{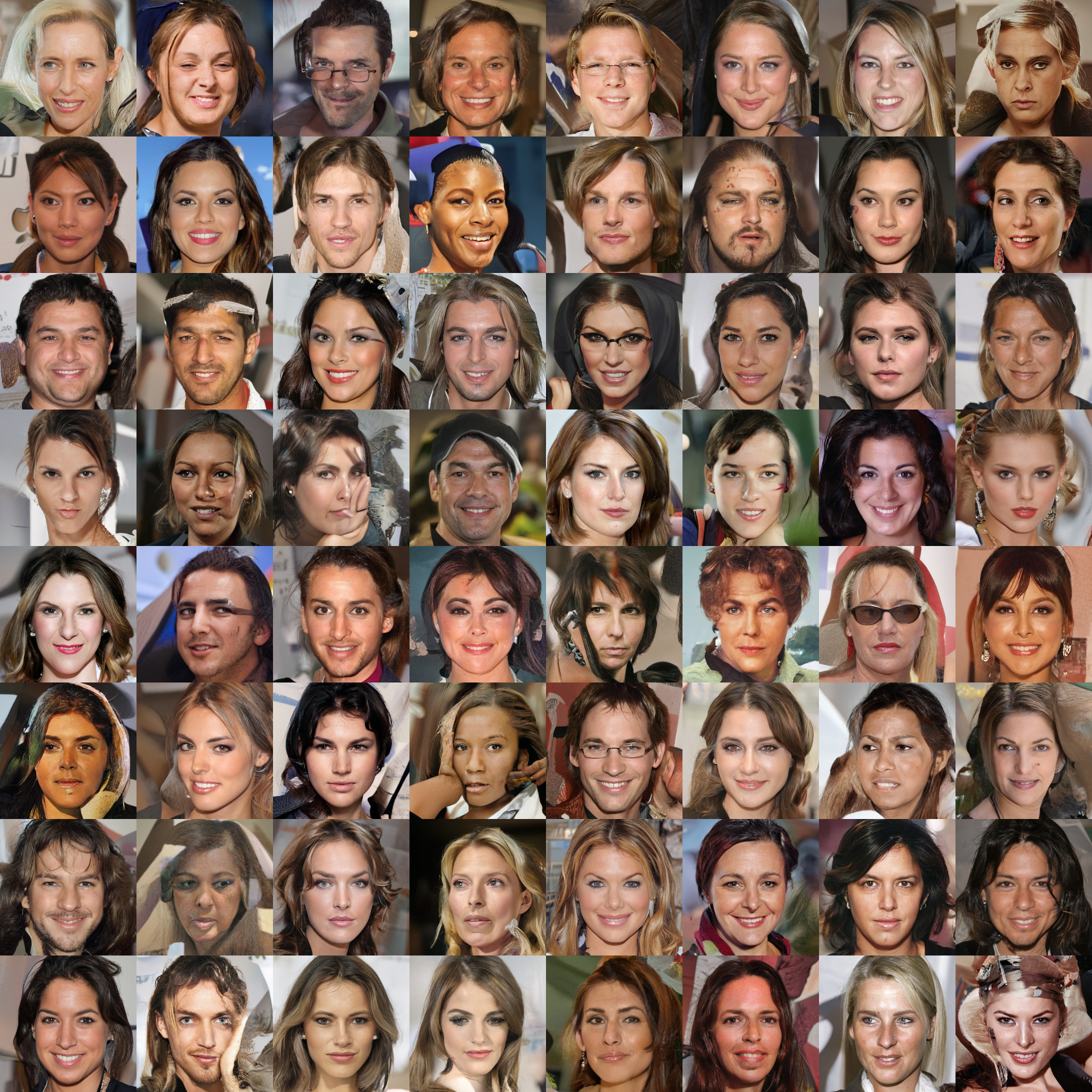} \\
{\scriptsize $M=2,c_\sigma=0.1$} &
{\scriptsize $M=2,c_\sigma=1.0$} \\[3pt]
\includegraphics[width=0.47\textwidth]{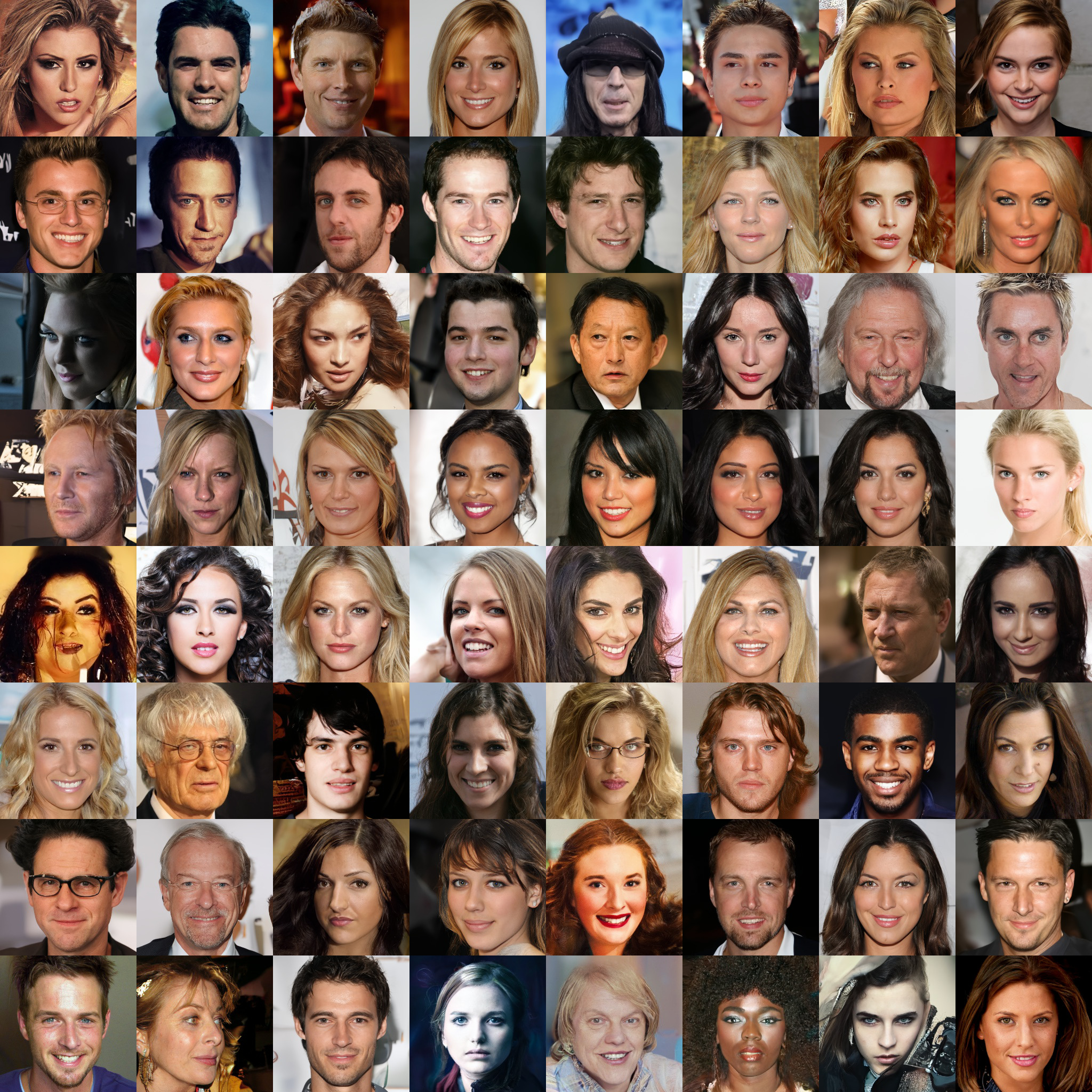} &
\includegraphics[width=0.47\textwidth]{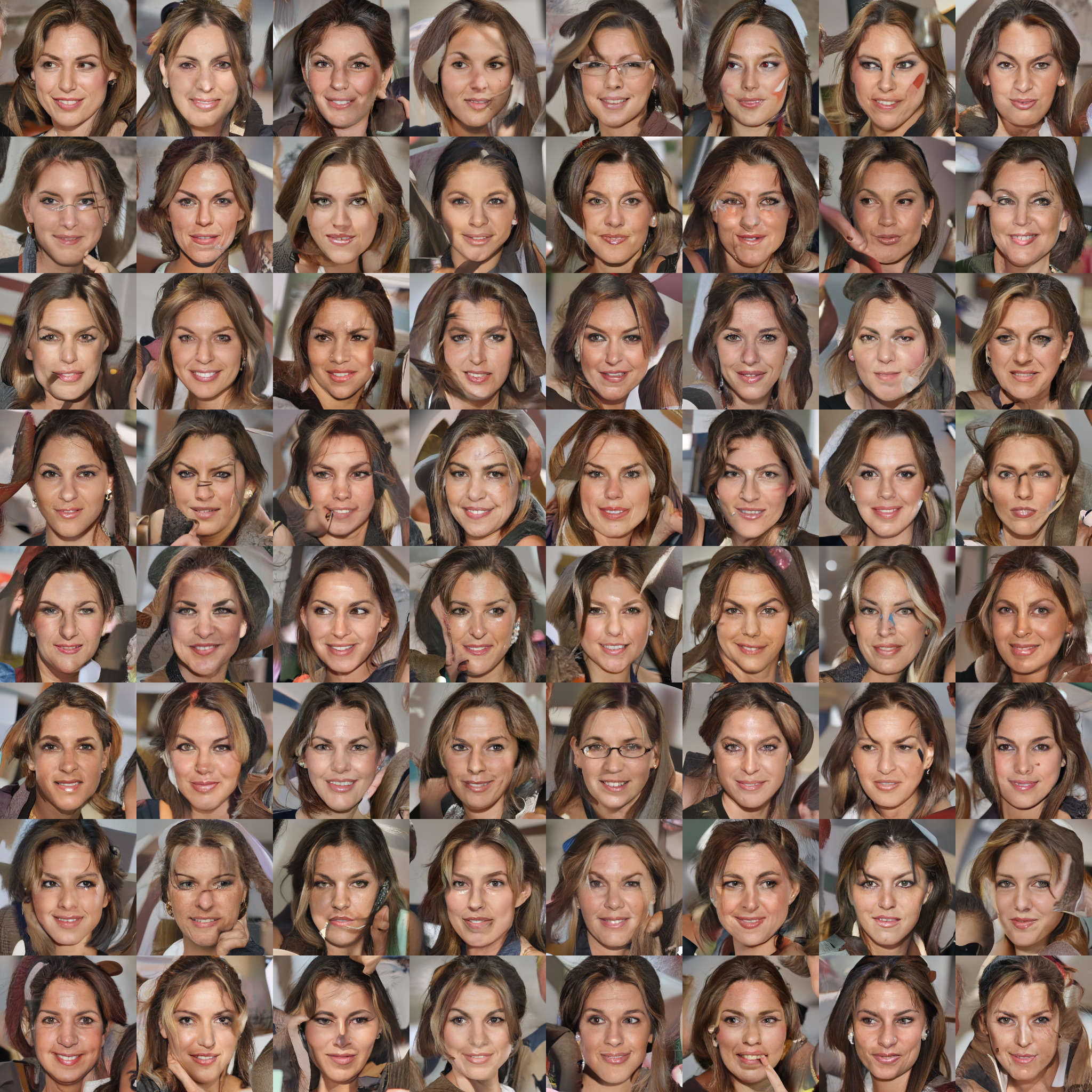} \\
{\scriptsize $M=32,c_\sigma=0.1$} &
{\scriptsize $M=32,c_\sigma=1.0$}
\end{tabular}
\caption{Additional CelebA-HQ samples from $\sigma$-CFDM under selected score-smoothing settings, with fixed $c_\delta=0.3$. Samples are generated in the top-$190$ DiffAE latent subspace and completed with the empirical Gaussian tail before decoding.}
\label{fig:celebahq_appendix_grids_cfdm}
\end{figure}

\begin{figure}[!t]
\centering
\setlength{\tabcolsep}{1.5pt}
\renewcommand{\arraystretch}{0.8}
\begin{tabular}{cc}
\includegraphics[width=0.47\textwidth]{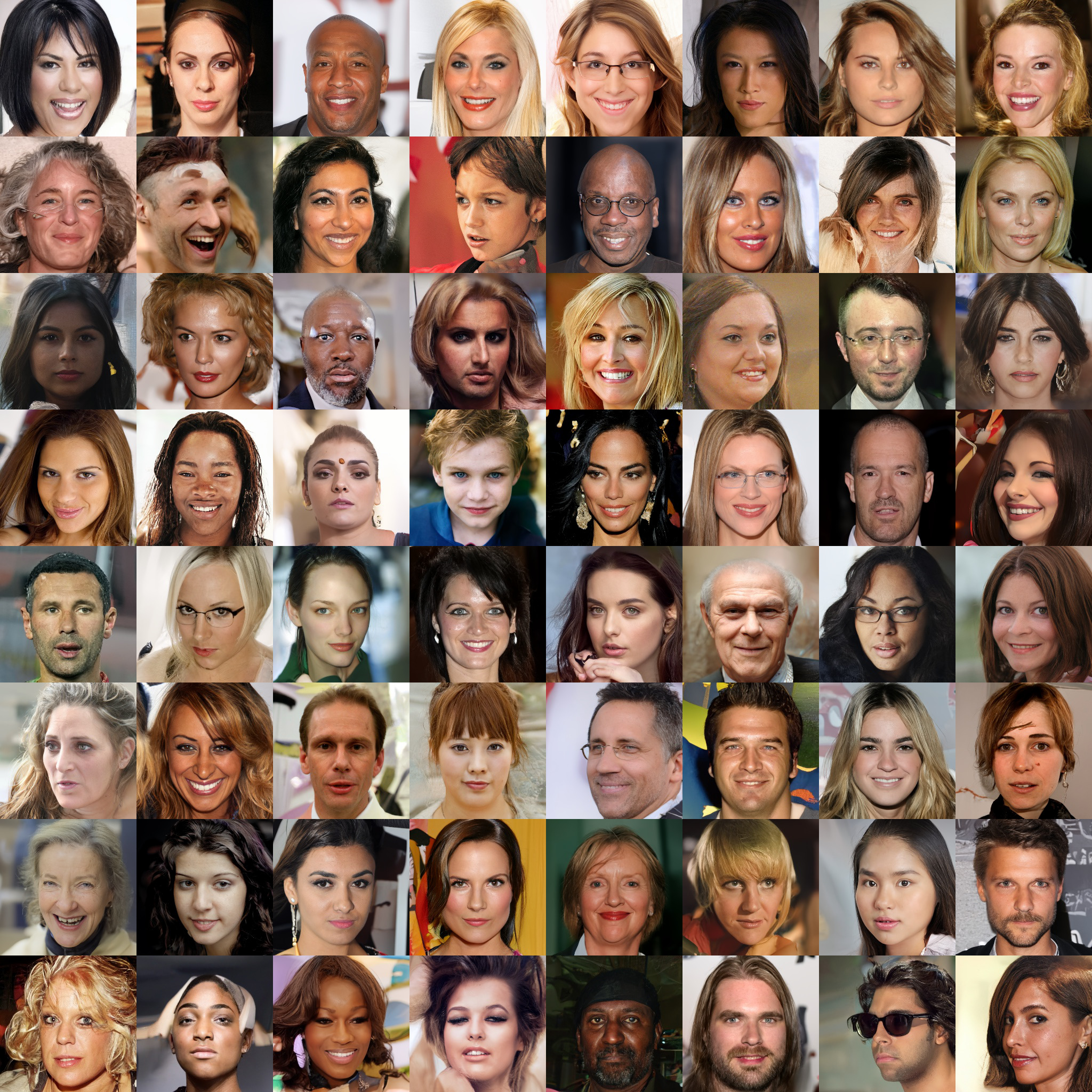} &
\includegraphics[width=0.47\textwidth]{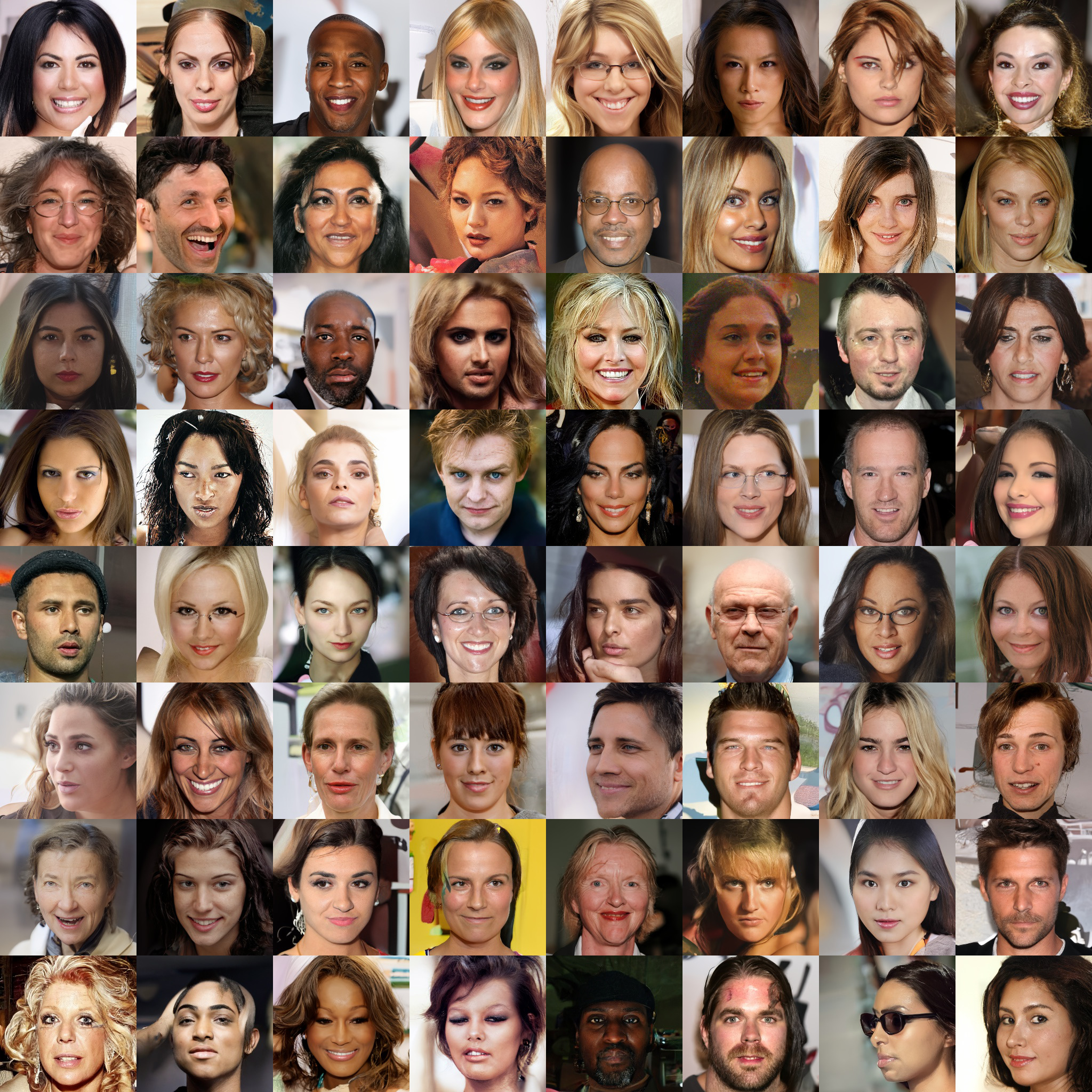} \\
{\scriptsize $M=2,c_\sigma=0.1$} &
{\scriptsize $M=2,c_\sigma=1.0$} \\[3pt]
\includegraphics[width=0.47\textwidth]{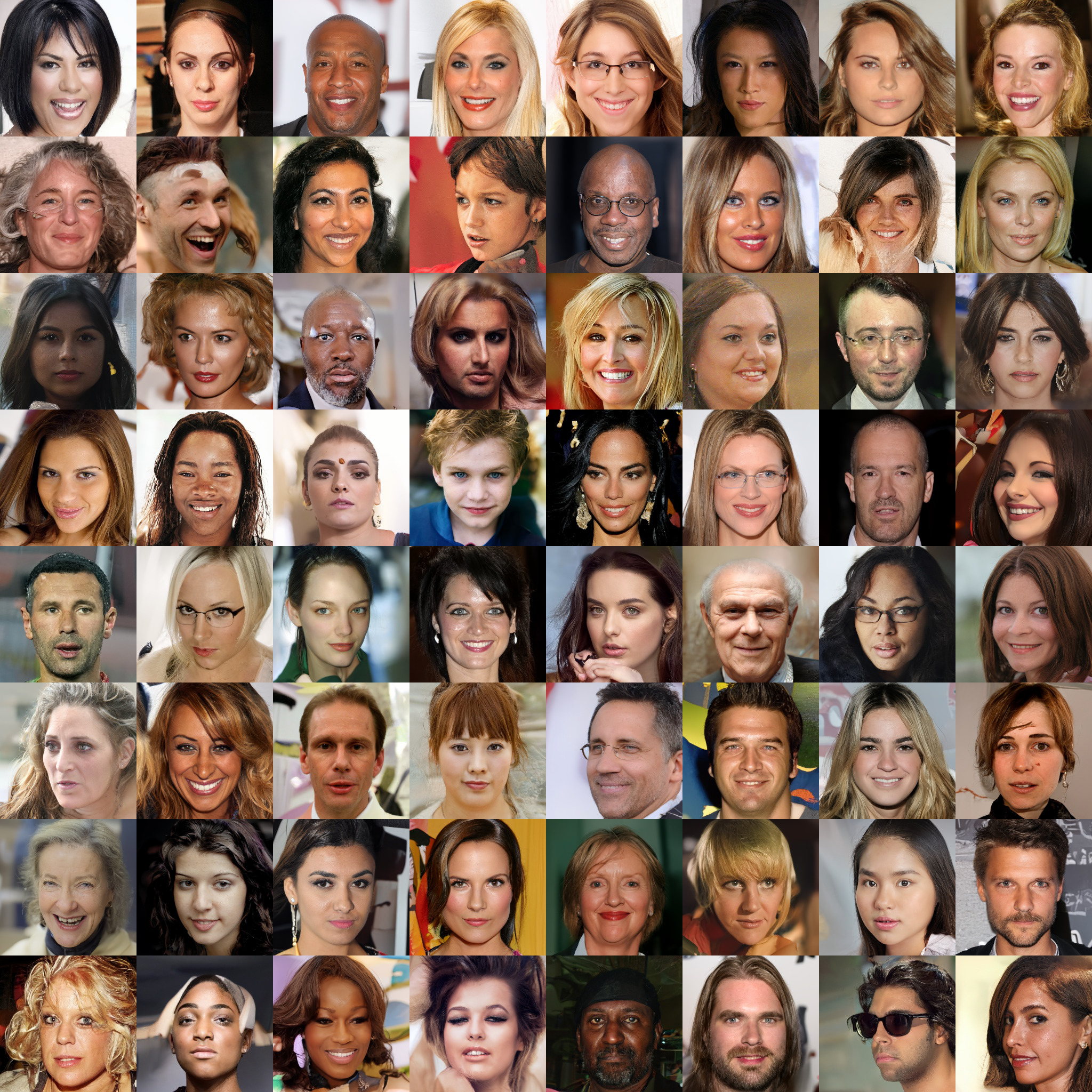} &
\includegraphics[width=0.47\textwidth]{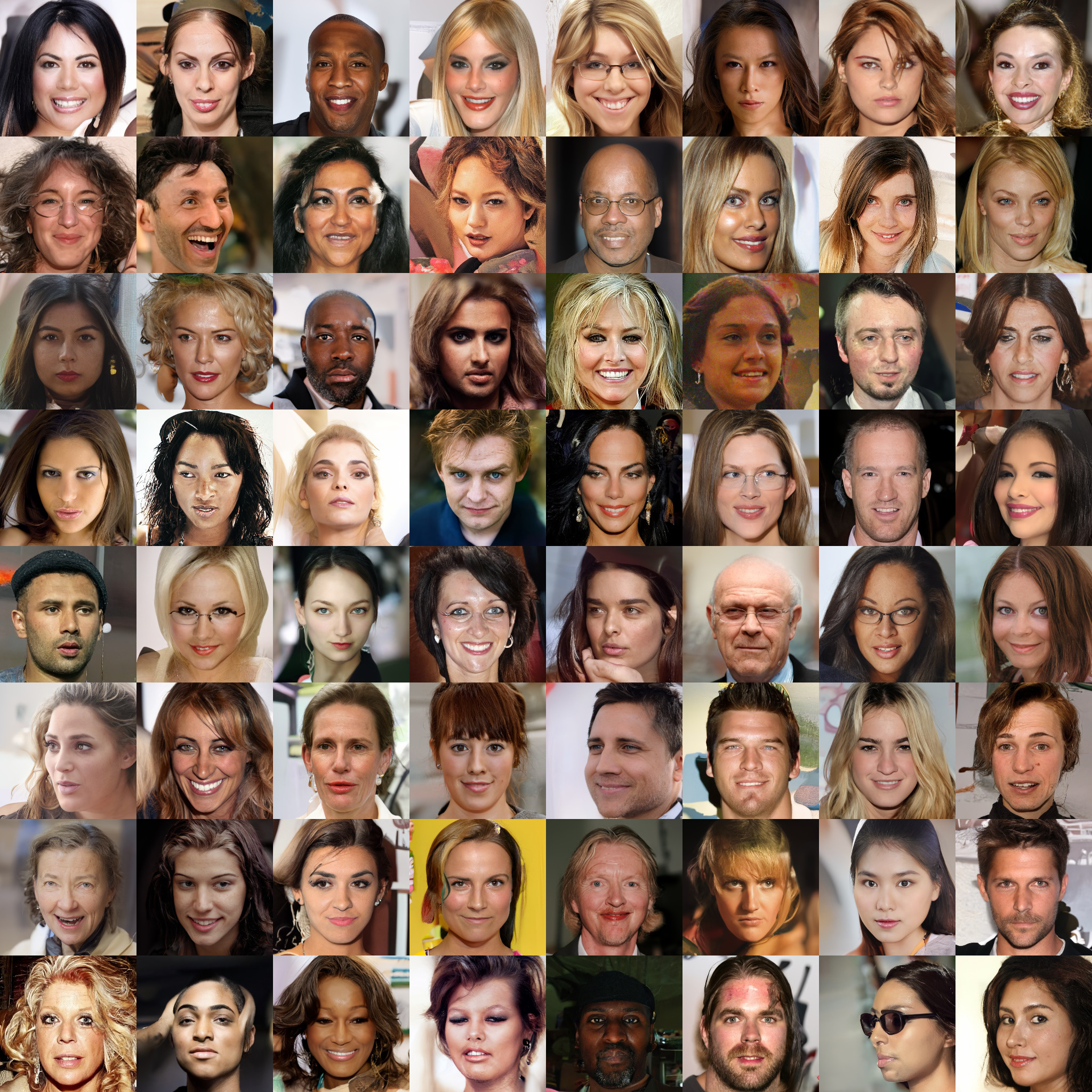} \\
{\scriptsize $M=32,c_\sigma=0.1$} &
{\scriptsize $M=32,c_\sigma=1.0$}
\end{tabular}
\caption{Additional CelebA-HQ samples from MM-SOLD under the same selected score-smoothing settings as Figure~\ref{fig:celebahq_appendix_grids_cfdm}, with fixed $c_\delta=0.3$. Compared with $\sigma$-CFDM, MM-SOLD remains visually more stable across smoothing bandwidths and Monte Carlo budgets.}
\label{fig:celebahq_appendix_grids_mmsold}
\end{figure}

\begin{figure}[t]
\centering
\setlength{\tabcolsep}{0pt}
\renewcommand{\arraystretch}{0.9}
\makebox[\textwidth][c]{%
\begin{tabular}{@{}c@{\hspace{2pt}}c@{\hspace{2pt}}c@{}}
\includegraphics[width=0.32\textwidth, trim=7 7 50 7, clip]{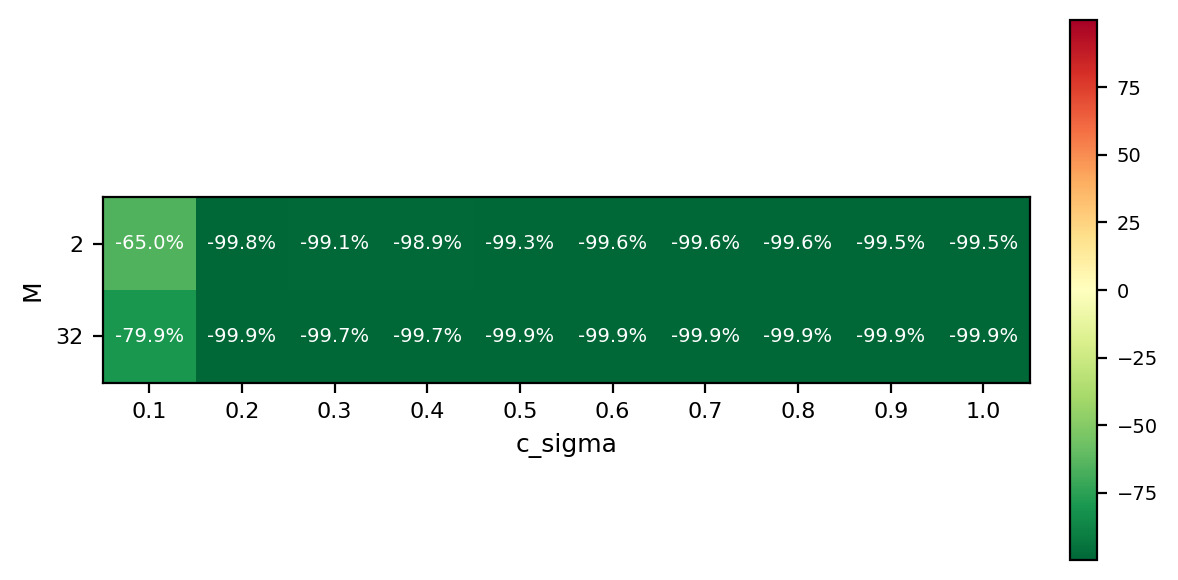} &
\includegraphics[width=0.32\textwidth, trim=7 7 50 7, clip]{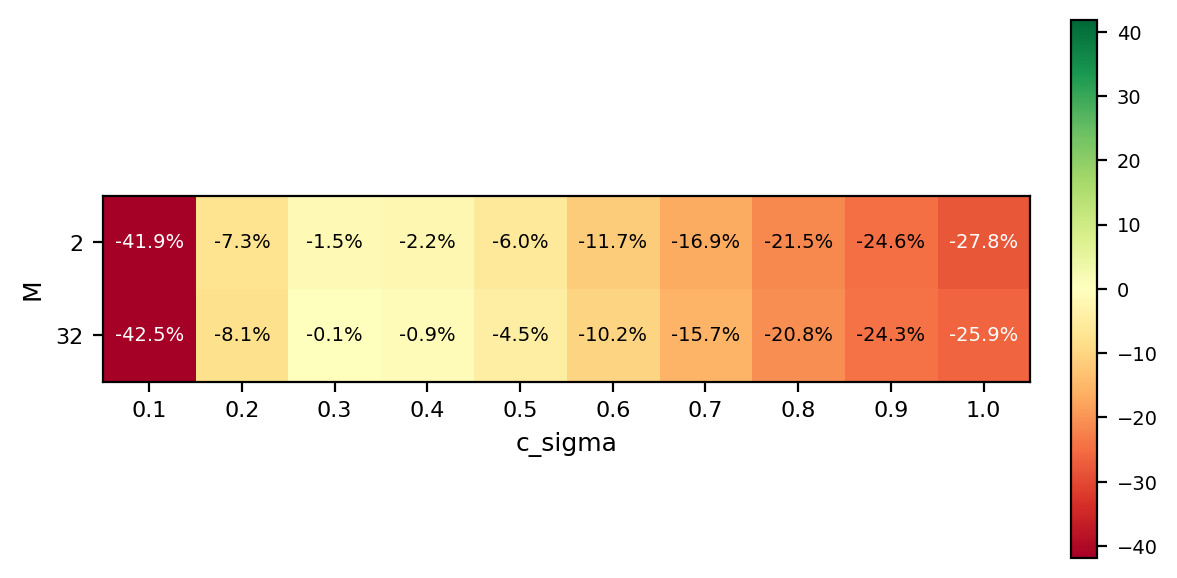} &
\includegraphics[width=0.32\textwidth, trim=7 7 50 7, clip]{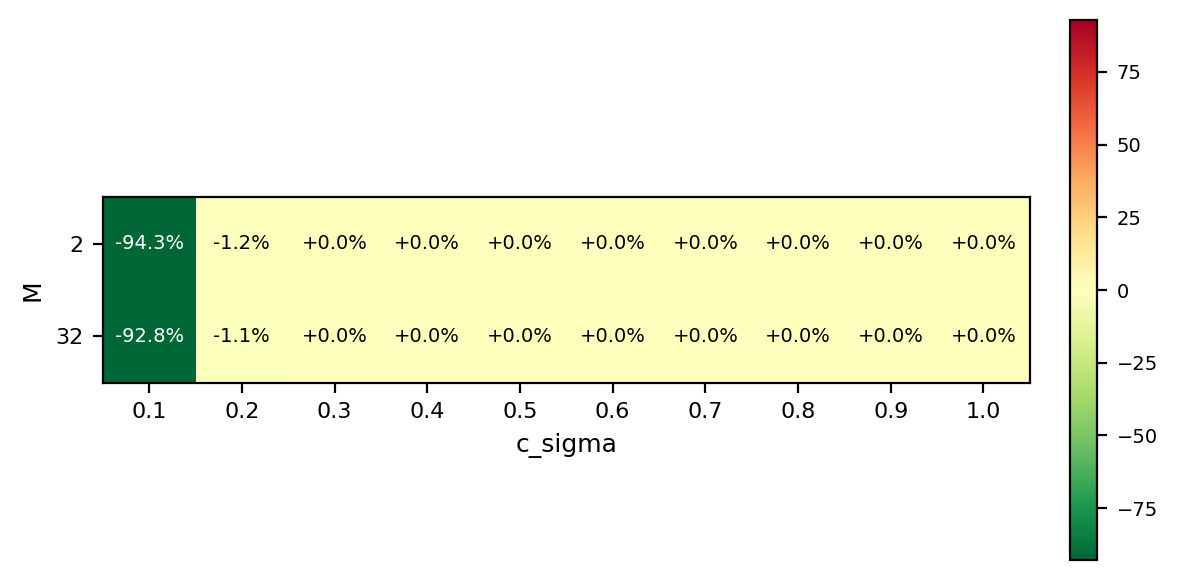} \\
{\scriptsize Latent MMD$^2$} &
{\scriptsize Recall} &
{\scriptsize DupRate}
\end{tabular}%
}
\caption{Grid-level comparison between MM-SOLD and $\sigma$-CFDM on CelebA-HQ. The heatmaps are computed in latent space to avoid decoding every hyperparameter cell. Green indicates MM-SOLD is better, red indicates $\sigma$-CFDM is better, and yellow indicates comparable performance.}
\label{fig:celebahq_generation_heatmaps}
\end{figure}

We next compare MM-SOLD with direct sampling from the estimated moment-matched limiting target. Figure~\ref{fig:celebahq_nearest_neighbor_pairs} fixes $(c_\delta,c_\sigma,M)=(0.4,0.4,32)$ and shows samples from MM-SOLD, Plug-in BAOAB, and SG-BAOAB. For each generated face, we also show the nearest training latent in the top-$190$ subspace, decoded by the same DiffAE decoder.
Figure~\ref{fig:celebahq_nearest_neighbor_pairs} confirms that the generated samples are not simple copies of their nearest training examples. The generated faces often differ from their nearest neighbors in identity, pose, hairstyle, background, facial expression, or local facial details, while still remaining realistic face images. This suggests that moment-matched score smoothing fits the DiffAE latent face manifold well enough to interpolate between training examples. MM-SOLD and SG-BAOAB show particularly clear examples of such manifold interpolation, whereas Plug-in BAOAB produces non-identical variants but explores a more concentrated region of the latent space, resulting in visibly lower sample diversity.

\begin{figure}[t]
\centering
\setlength{\tabcolsep}{2pt}
\renewcommand{\arraystretch}{0.9}
\begin{tabular}{ccc}
\includegraphics[width=0.32\textwidth]{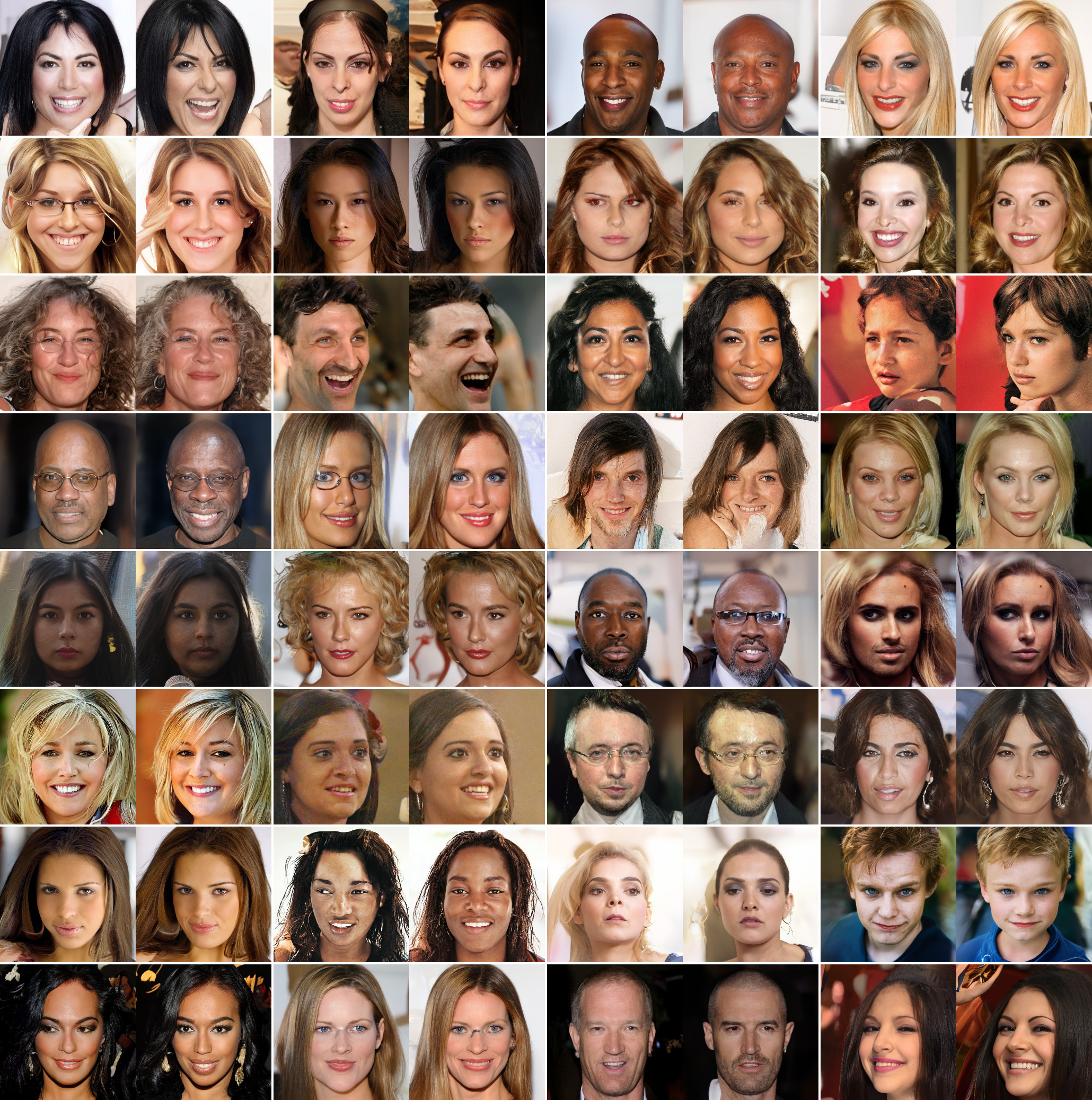} &
\includegraphics[width=0.32\textwidth]{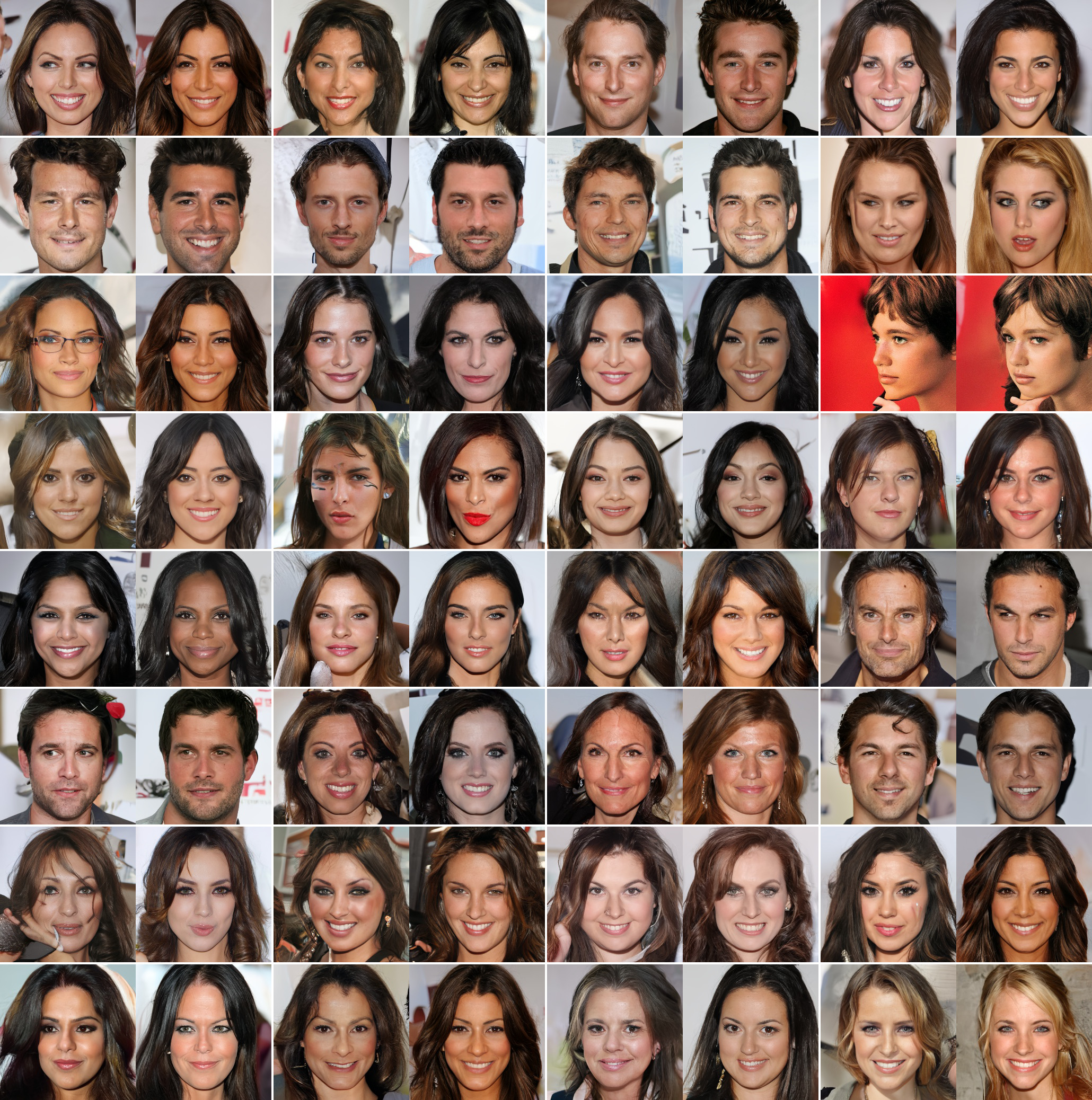} &
\includegraphics[width=0.32\textwidth]{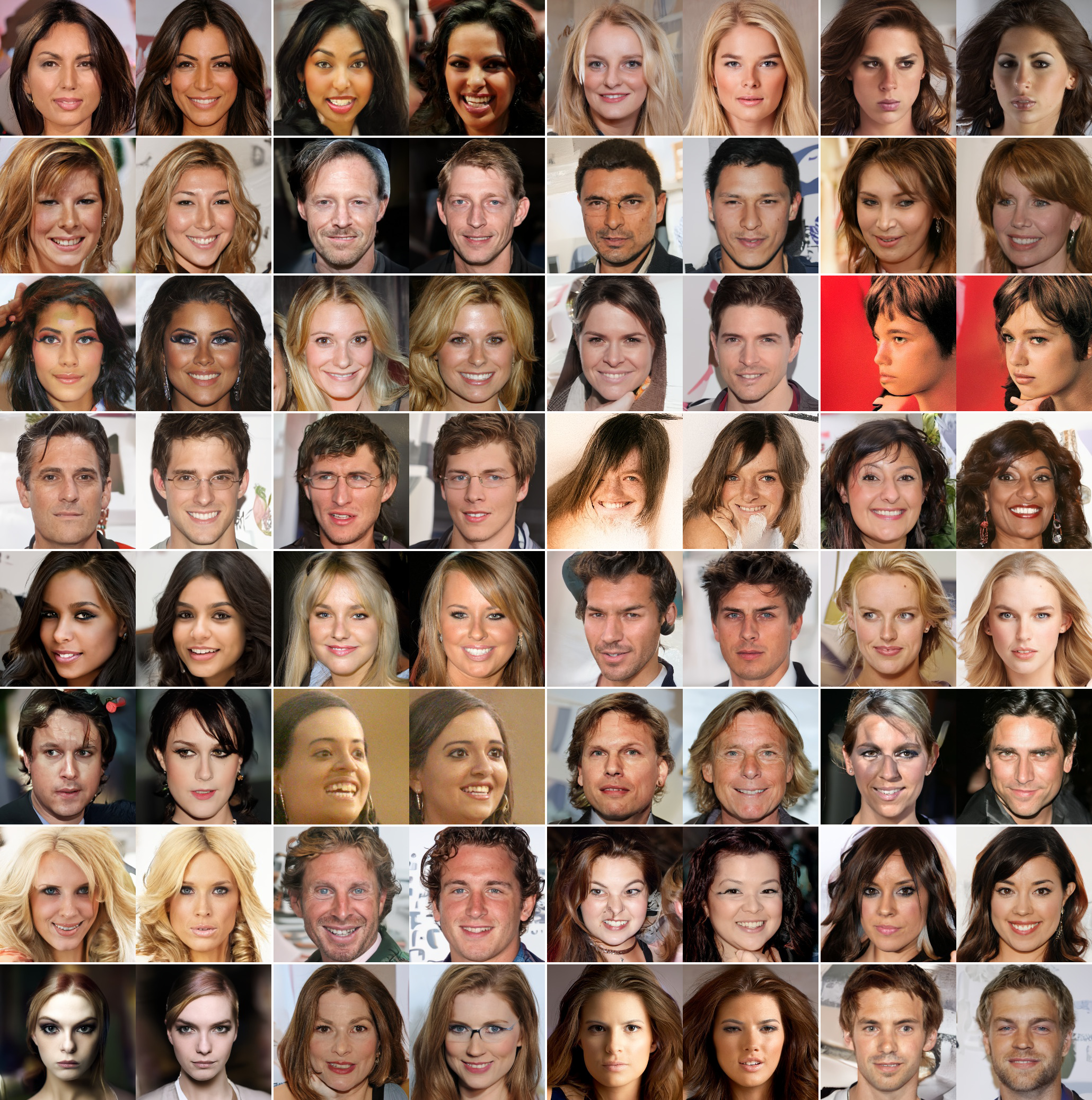} \\
{\scriptsize MM-SOLD} &
{\scriptsize Plug-in BAOAB} &
{\scriptsize SG-BAOAB}
\end{tabular}
\caption{Generated CelebA-HQ samples and their nearest training-set reconstructions in the top-$190$ latent subspace. In each panel, every generated sample is paired with the decoded training latent closest to it in Euclidean distance.}
\label{fig:celebahq_nearest_neighbor_pairs}
\end{figure}

Figure~\ref{fig:celebahq_baoab_heatmaps} gives a full grid-level comparison among MM-SOLD, Plug-in BAOAB, and SG-BAOAB at $M=32$. The horizontal and vertical axes are $c_\sigma$ and $c_\delta$, respectively. The first row compares MM-SOLD against Plug-in BAOAB, the second row compares MM-SOLD against SG-BAOAB, and the third row compares Plug-in BAOAB against SG-BAOAB. Since image decoding is too expensive for full-grid evaluation, each row reports latent-space MMD$^2$, latent-space Recall, and DupRate.
The pairwise heatmaps show that MM-SOLD is the most robust of the three moment-matched samplers across the searched grid. Compared with Plug-in BAOAB, MM-SOLD gives better latent-space MMD$^2$ over a large part of the grid while maintaining comparable or better duplication behavior. Compared with SG-BAOAB, MM-SOLD remains competitive despite requiring no stochastic dual optimization. The comparison between Plug-in BAOAB and SG-BAOAB shows the benefit of refining the tilting parameters: SG-BAOAB improves over the plug-in estimator across many settings, especially in latent-space MMD$^2$ and Recall, confirming that stochastic dual optimization moves the explicit tilted target closer to the desired moment-matched distribution.

\begin{figure}[t]
\centering
\setlength{\tabcolsep}{0pt}
\renewcommand{\arraystretch}{0.9}
\makebox[\textwidth][c]{%
\begin{tabular}{@{}c@{\hspace{2pt}}c@{\hspace{2pt}}c@{}}
\multicolumn{3}{c}{\small\textbf{MM-SOLD vs Plug-in BAOAB}} \\[-1pt]
\includegraphics[width=0.32\textwidth, trim=5 5 70 5, clip]{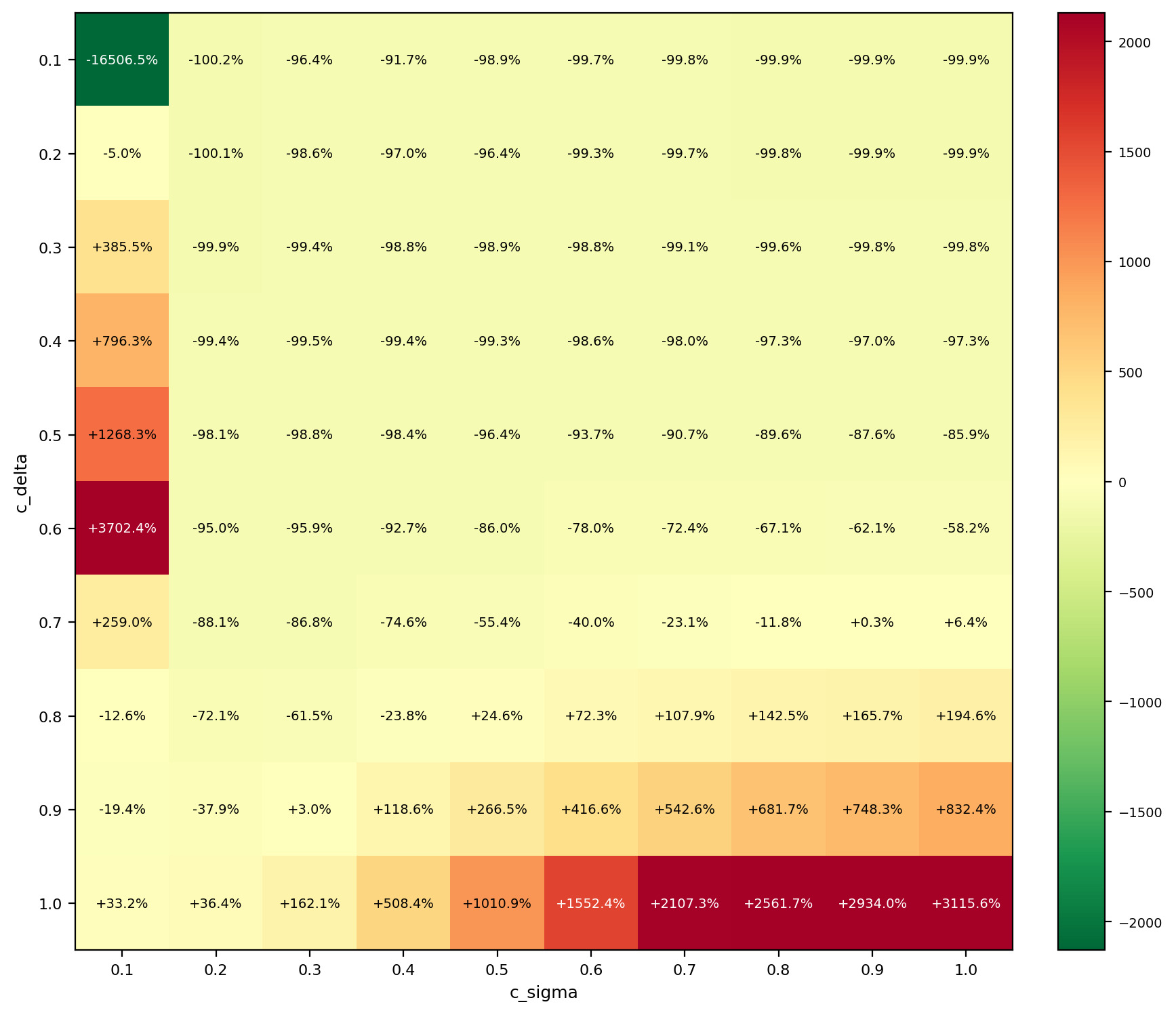} &
\includegraphics[width=0.32\textwidth, trim=5 5 70 5, clip]{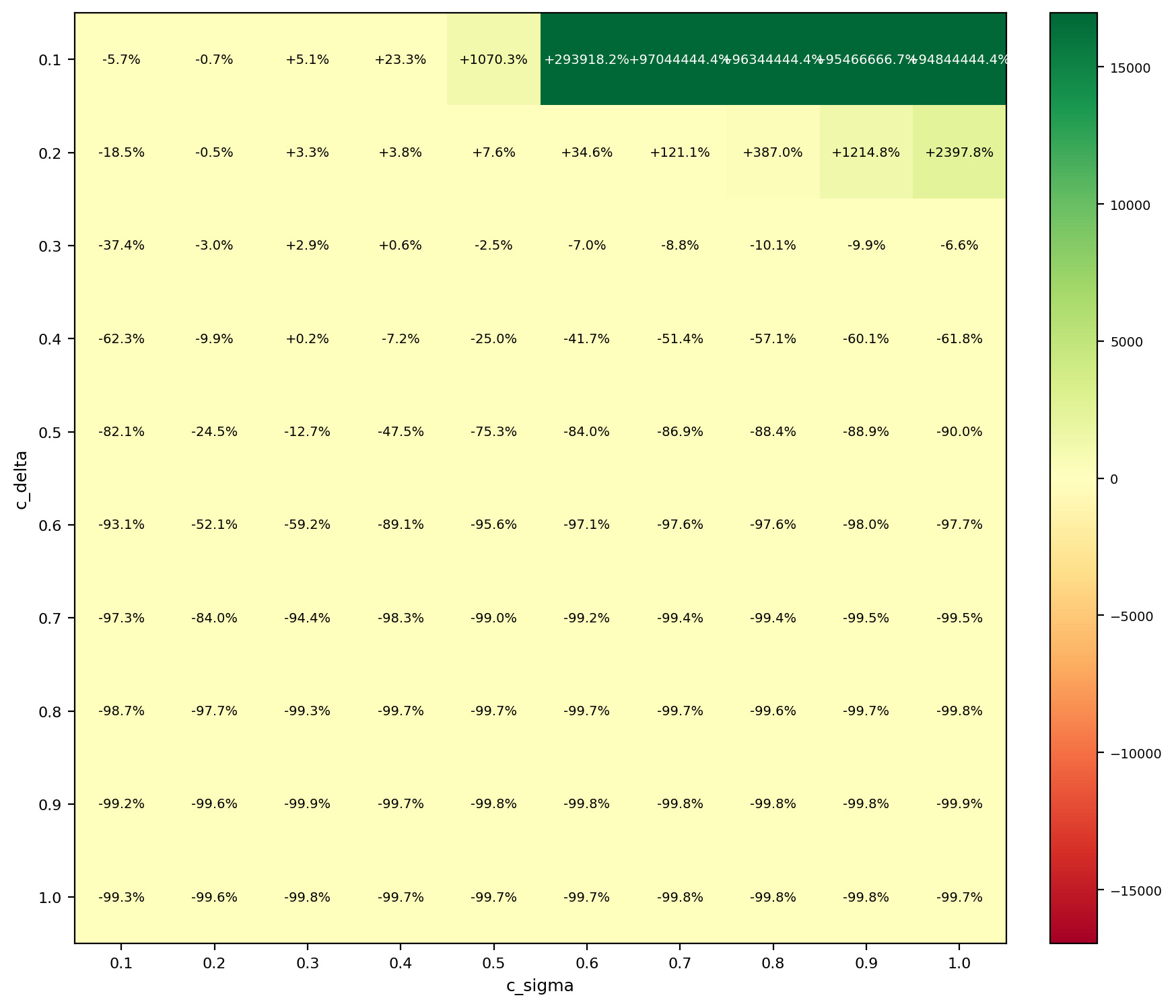} &
\includegraphics[width=0.32\textwidth, trim=5 5 70 5, clip]{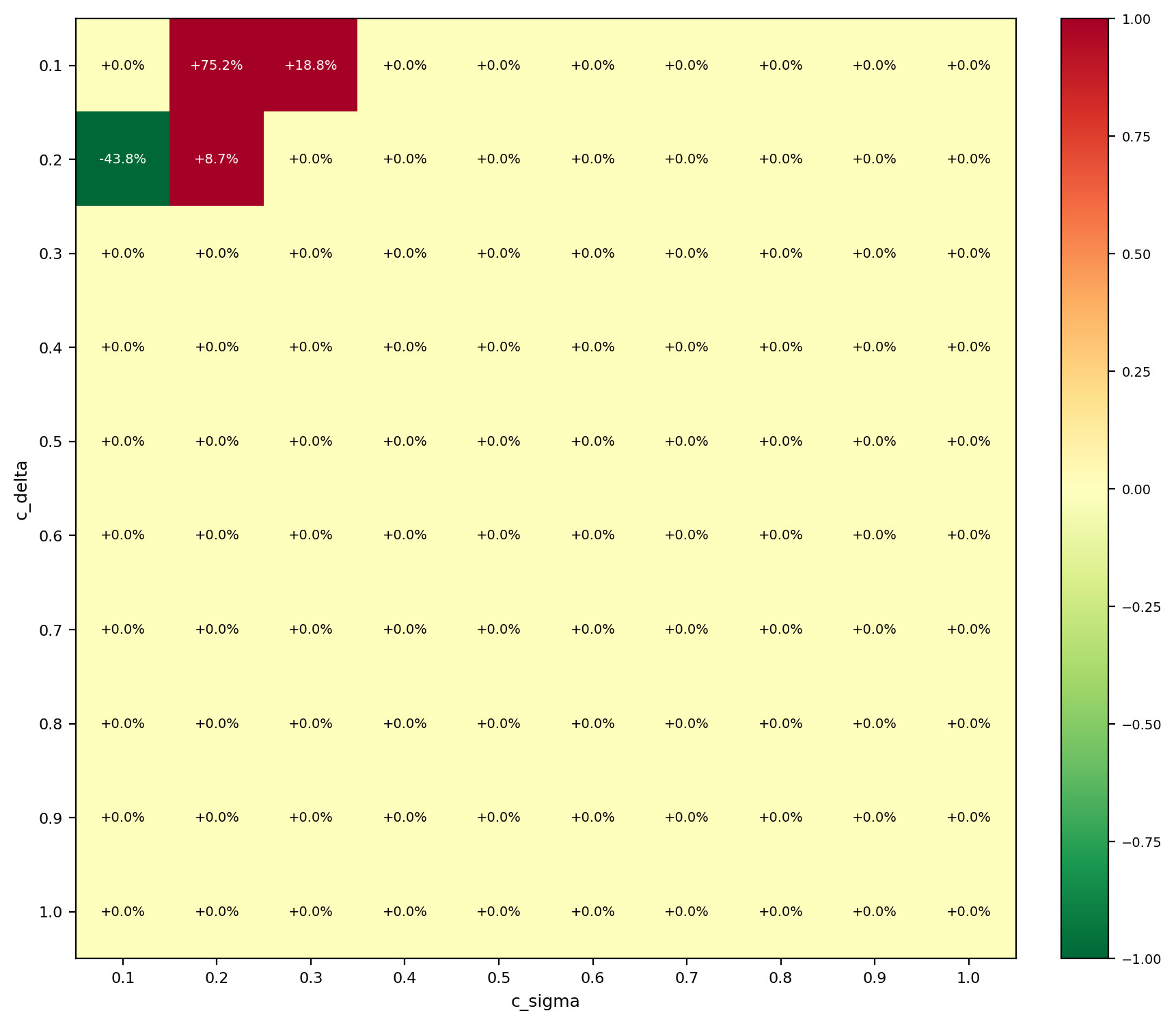} \\
{\scriptsize Latent MMD$^2$} &
{\scriptsize Recall} &
{\scriptsize DupRate} \\[3pt]
\multicolumn{3}{c}{\small\textbf{MM-SOLD vs SG-BAOAB}} \\[-1pt] 
\includegraphics[width=0.32\textwidth, trim=5 5 70 5, clip]{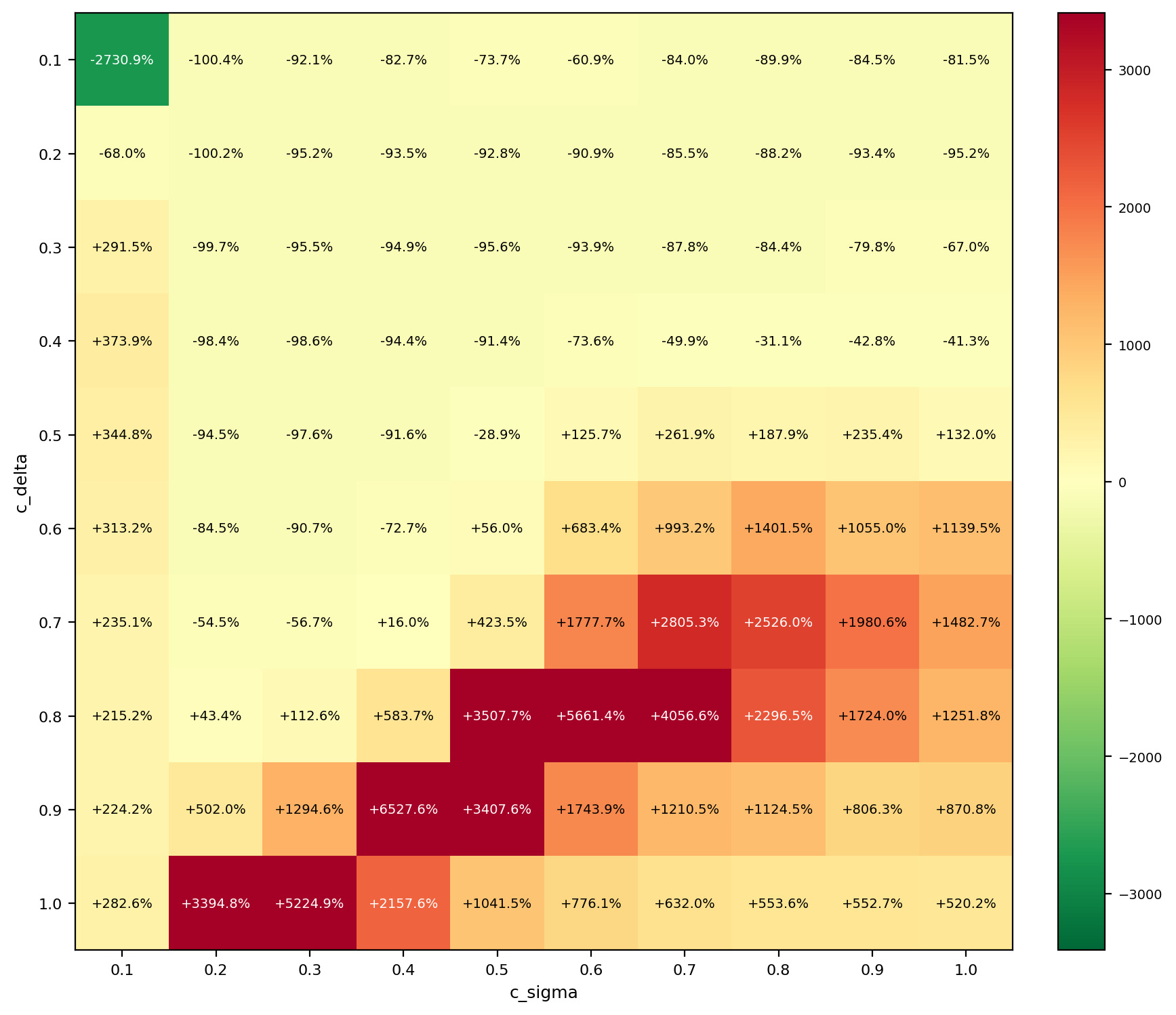} &
\includegraphics[width=0.32\textwidth, trim=5 5 70 5, clip]{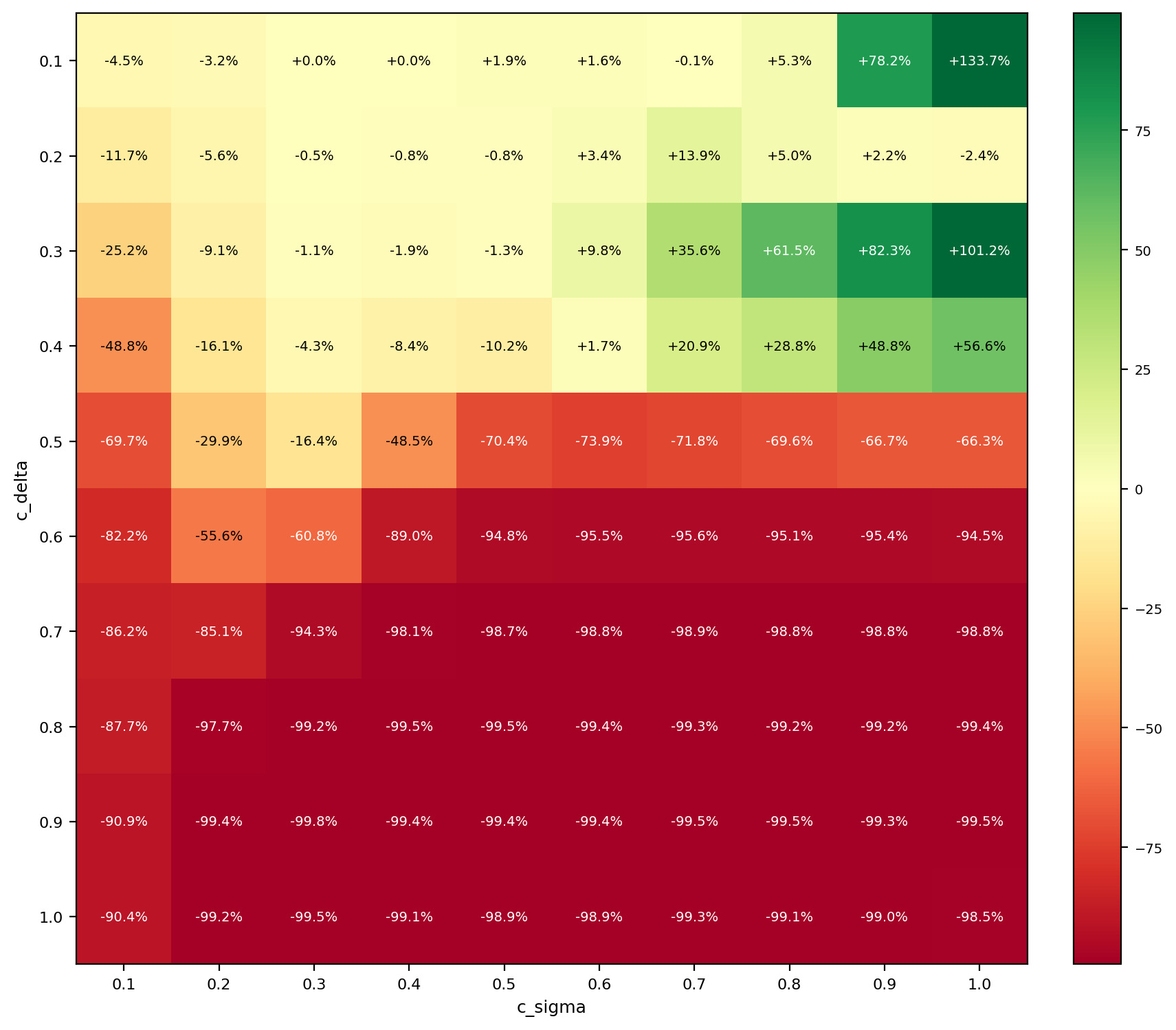} &
\includegraphics[width=0.32\textwidth, trim=5 5 70 5, clip]{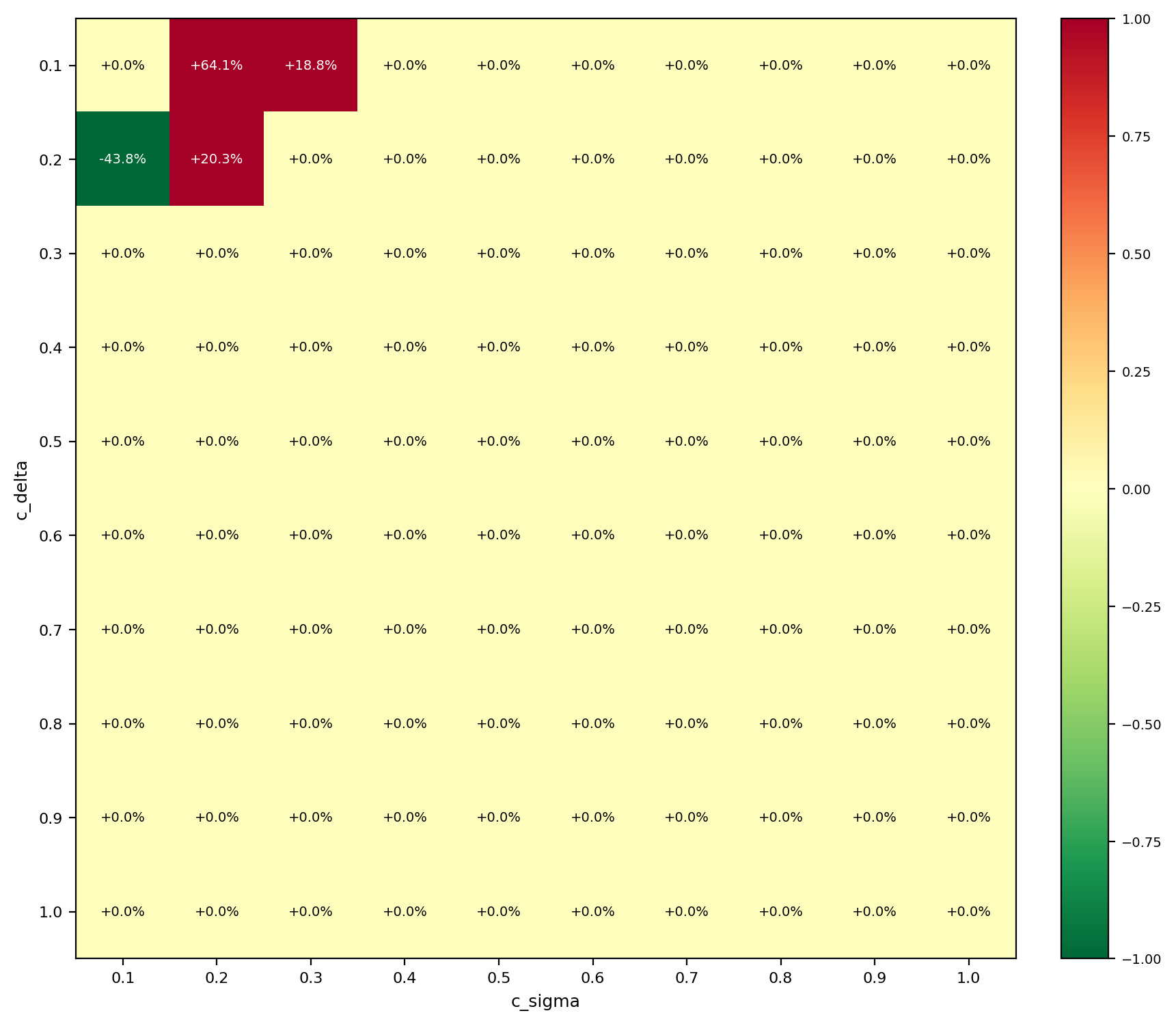} \\
{\scriptsize Latent MMD$^2$} &
{\scriptsize Recall} &
{\scriptsize DupRate} \\[3pt]
\multicolumn{3}{c}{\small\textbf{Plug-in BAOAB vs SG-BAOAB}} \\[-1pt]
\includegraphics[width=0.32\textwidth, trim=5 5 70 5, clip]{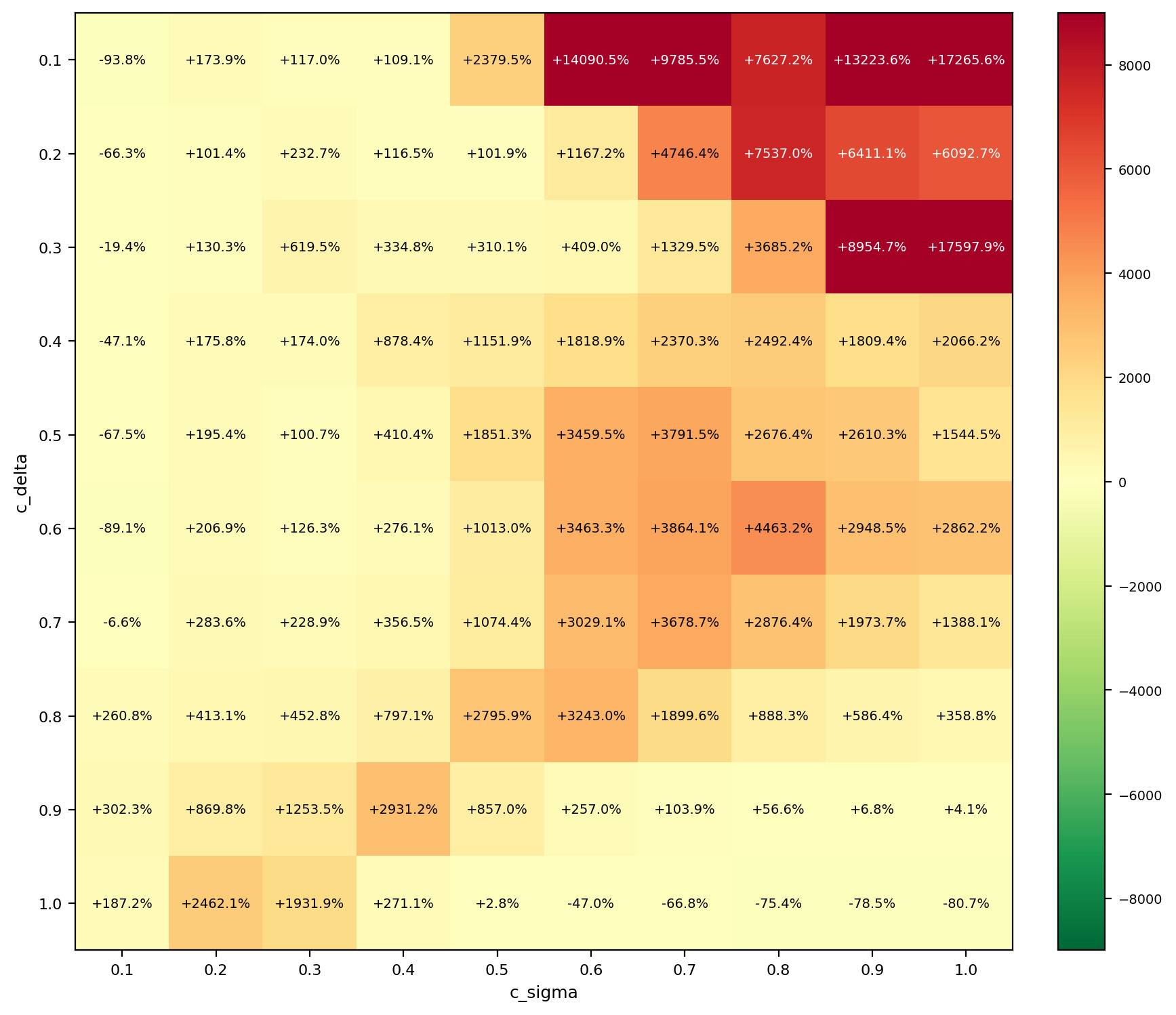} &
\includegraphics[width=0.32\textwidth, trim=5 5 70 5, clip]{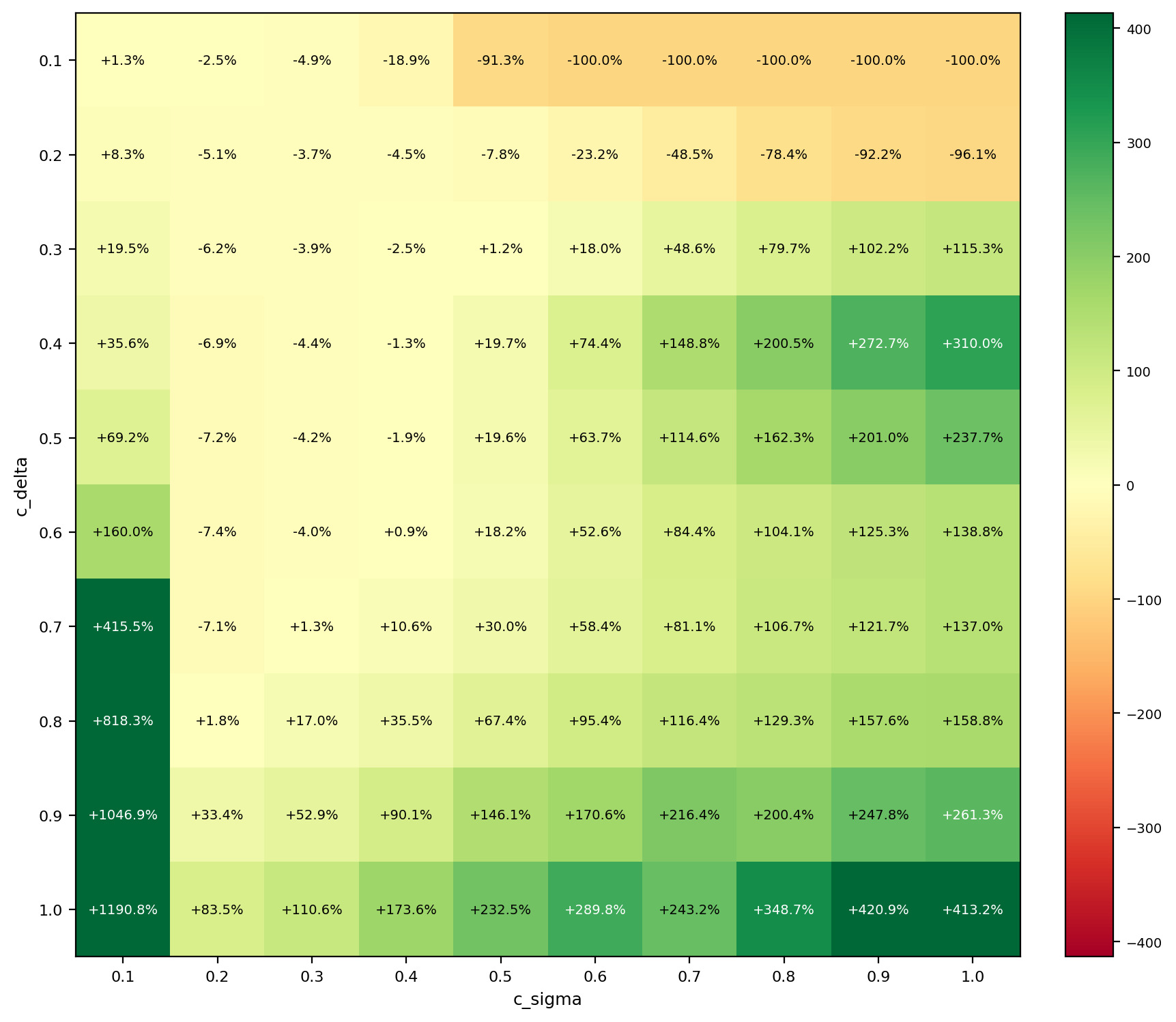} &
\includegraphics[width=0.32\textwidth, trim=5 5 70 5, clip]{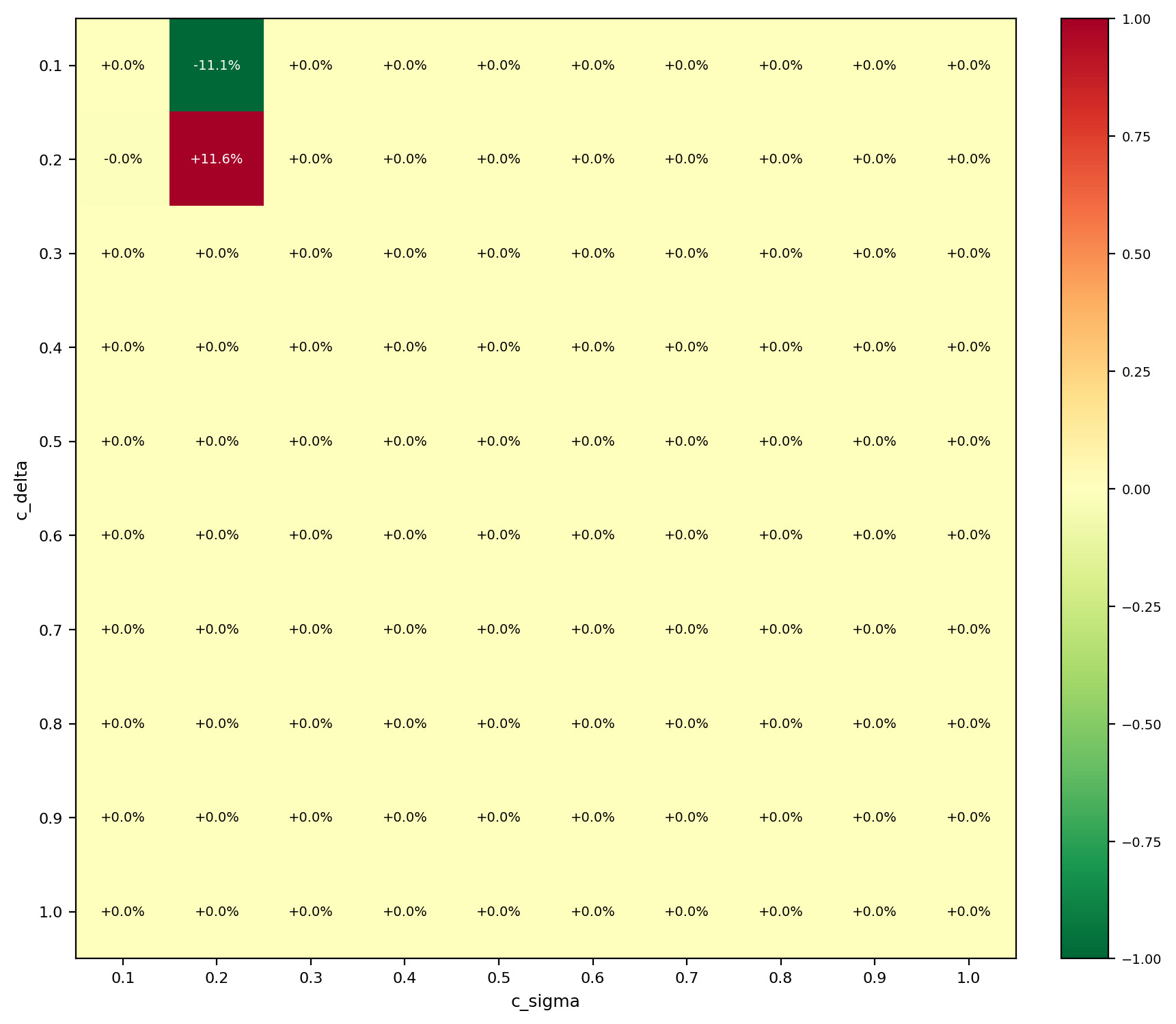} \\
{\scriptsize Latent MMD$^2$} &
{\scriptsize Recall} &
{\scriptsize DupRate}
\end{tabular}%
}
\caption{Pairwise grid-level comparisons among MM-SOLD, Plug-in BAOAB, and SG-BAOAB on CelebA-HQ at $M=32$. Axes correspond to $c_\sigma$ and $c_\delta$ in the search grid. In each row, green indicates that the first named method is better, red indicates that the second named method is better, and yellow indicates comparable performance.}
\label{fig:celebahq_baoab_heatmaps}
\end{figure}

Finally, Table~\ref{tab:celebahq_posterior_moment_validation} validates whether the estimated moment-matched targets reproduce the prescribed GMM moments in the top-$190$ CelebA-HQ latent subspace. We fix $(c_\delta,c_\sigma,M)=(0.4,0.4,32)$. We compare three explicit BAOAB targets: the naive score-smoothed target with $\lambda=0$ and $\Lambda=0$, the plug-in tilted target, and the SG-refined tilted target. For each target, we run $10{,}000$ independent BAOAB chains from the CelebA-HQ GMM for $1{,}000$ steps each, compute the empirical mean and covariance of their endpoints in the top-$190$ subspace, and compare them with the target mean $\mu^*$ and covariance $\Sigma^*$. Recall and DupRate are computed before decoding from the endpoints of a separate set of $3{,}000$ independent chains with the same settings; image-space FID, KID, and median-RBF MMD are computed from the same $3{,}000$ samples after decoding. We also report FID, KID, and median-RBF MMD in ArcFace feature space~\cite{deng2019arcface}.
The result confirms that tilting the naive score-smoothed target substantially improves posterior moment matching and image quality. The naive BAOAB target has a large relative covariance error, $\|\widehat\Sigma-\Sigma^*\|_F/\|\Sigma^*\|_F=0.8776$, and poor image-space metrics, with FID $65.628$ and KID $53.960$. Plug-in BAOAB reduces the relative covariance error to $0.6889$ and improves FID and KID to $39.479$ and $27.473$, respectively. SG-BAOAB further reduces the relative covariance error to $0.2518$ and gives the best explicit-target performance, with image-space FID $21.751$, KID $9.595$, and median-RBF MMD $15.43$, ArcFace FID $0.103$, KID $0.080$, and median-RBF MMD $5.30$, Recall $0.9770$, and DupRate equal to zero. These finite-budget measurements combine parameter error, numerical bias from the finite step size and chain length, incomplete mixing, nearest-neighbor score bias and Monte Carlo fluctuations. They show that the SG-refined target reproduces the prescribed moments much better than the naive and plug-in targets, but they do not isolate the accuracy of the estimated tilting parameters. MM-SOLD avoids this high-dimensional estimation problem by enforcing the moments directly at the particle level.

\begin{table}[t]
\centering
\caption{Posterior moment validation for explicit BAOAB sampling on CelebA-HQ DiffAE latents. Moment errors are computed in the top-$190$ latent subspace using $10{,}000$ independent BAOAB samples. Image-space metrics are computed after DiffAE decoding.}
\label{tab:celebahq_posterior_moment_validation}
\scriptsize
\setlength{\tabcolsep}{3.0pt}
\renewcommand{\arraystretch}{0.9}
\begin{tabular*}{\textwidth}{@{\extracolsep{\fill}}lccc@{}}
\toprule
Metric
& Naive BAOAB $(\lambda=0,\Lambda=0)$
& Plug-in BAOAB
& SG-BAOAB \\
\midrule
$\|\bar z-\mu^*\|_2$
& 0.2764 & 0.3961 & \textbf{0.0835} \\
$\|\widehat\Sigma-\Sigma^*\|_F$
& 0.4443 & 0.3488 & \textbf{0.1275} \\
$\|\widehat\Sigma-\Sigma^*\|_F/\|\Sigma^*\|_F$
& 0.8776 & 0.6889 & \textbf{0.2518} \\
FID (image) $\downarrow$
& 65.628 & 39.479 & \textbf{21.751} \\
KID (image, $\times10^3$) $\downarrow$
& 53.960 & 27.473 & \textbf{9.595} \\
MMD (image, $\times10^3$) $\downarrow$
& 87.89 & 51.39 & \textbf{15.43} \\
FID (ArcFace) $\downarrow$
& 0.226 & 0.151 & \textbf{0.103} \\
KID (ArcFace, $\times10^3$) $\downarrow$
& 0.408 & 0.198 & \textbf{0.080} \\
MMD (ArcFace, $\times10^3$) $\downarrow$
& 26.97 & 13.01 & \textbf{5.30} \\
Recall $\uparrow$
& 0.9657 & 0.9670 & \textbf{0.9770} \\
DupRate $\downarrow$
& \textbf{0.000} & \textbf{0.000} & \textbf{0.000} \\
\bottomrule
\end{tabular*}
\end{table}

\subsection{Ablation studies}
\label{app:ablation_studies}

\paragraph{Effect of step size and number of steps}
We study the sensitivity of MM-SOLD to the Langevin step size $h$ and number of steps $\nsteps$ using the 2D checkerboard distribution. We use the same setting as in Section \ref{sec:eval_on_2D}: $500$ training samples and $P=5{,}000$ particles, thereby generating $5{,}000$ samples. The checkerboard data are generated by drawing $x_1\sim\mathcal U[-2,2]$ and $x_2=u-2b+(\lfloor x_1\rfloor\bmod 2)$ with $u\sim\mathcal U[0,1)$ and $b\sim\mathrm{Bernoulli}(1/2)$, and multiplying both coordinates by $2$. The training and reference sets are not i.i.d.\ samples: we draw ten times as many points as needed (seed $1$ for the training set and seed $0$ for the reference set) and keep $500$ and $5{,}000$ of them, respectively, by farthest-point sampling, so that they cover the support evenly. The 2D experiments in Section~\ref{sec:eval_on_2D} construct their training and reference sets in the same way; the spirals data consist of two arms $\pm\big(-\theta\cos\theta+u_1,\ \theta\sin\theta+u_2\big)/3$ with $\theta=3\pi\sqrt v$, $v\sim\mathcal U[0,1)$ and $u_1,u_2\sim\mathcal U[0,0.5)$, plus $\mathcal N(0,0.1^2I_2)$ noise.
We fix $M=8$, smoothing bandwidth $\sigma=0.2$, and GMM component standard deviation $\delta=0.1$, and vary
$\nsteps\in\{1,5,10,25,50,100,200,500\}$ and
$h\in\{10^{-5},10^{-4},5\!\times\!10^{-4},10^{-3},2\!\times\!10^{-3},5\!\times\!10^{-3},8\!\times\!10^{-3}\}$.
We evaluate sliced Wasserstein-2 distance (SW2) \cite{bonneel2015sliced} with $512$ random projection directions both to the $5{,}000$ reference points and to the training set. Each $(h,\nsteps)$ configuration is run once with seed $0$, all SW2 values use the same $512$ directions, and when the two point sets have different sizes, the sorted projections of both are linearly interpolated onto a common uniform grid of quantile levels whose size is that of the larger set.

Figure \ref{fig:ablation_stepsize_steps} shows that MM-SOLD is robust over a broad range of $h$ and $\nsteps$. With the training-GMM initialization, SW2-to-target values range from $0.083$ to $0.108$ across most configurations; this metric alone does not establish equilibration. The performance degrades only when the step size or the number of steps is too small: particles do not move far from their initial GMM centers, leading to low SW2-to-train values. The stable range is consistent with the local Langevin step-size scale $h=O(\delta^2)$; here $\delta=0.1$, so $\delta^2=10^{-2}$, and all tested step sizes remain at or below this scale. The conditional local calculation motivating this choice is in Appendix~\ref{app:proof_stepsize_bound}.

\begin{figure}[t]
\centering
\begin{minipage}{0.49\textwidth}
\centering
\includegraphics[width=\textwidth]{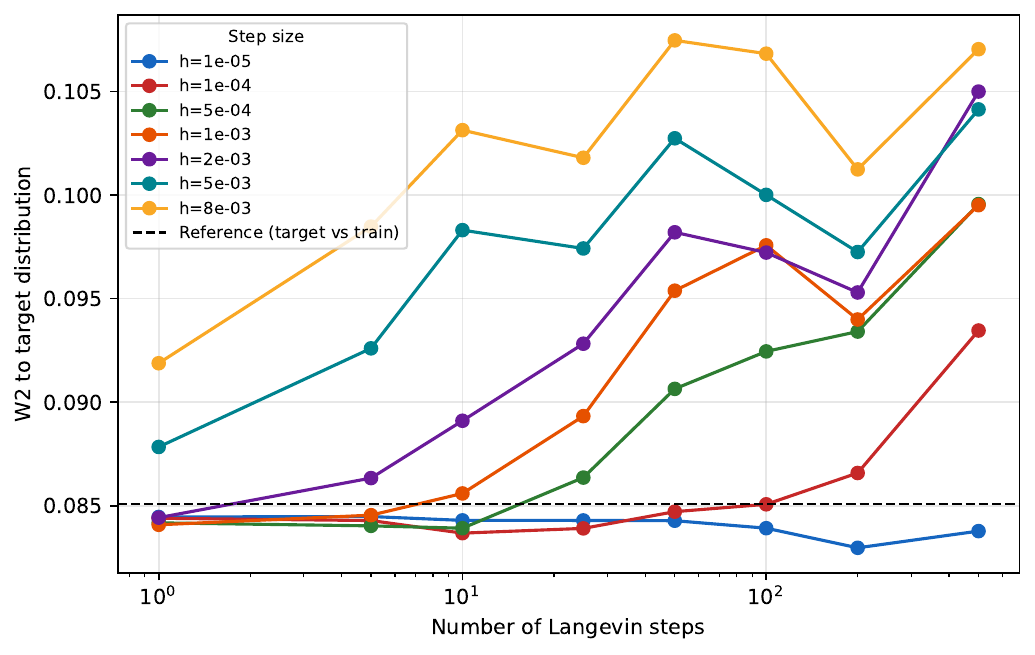}
\end{minipage}
\hfill
\begin{minipage}{0.49\textwidth}
\centering
\includegraphics[width=\textwidth]{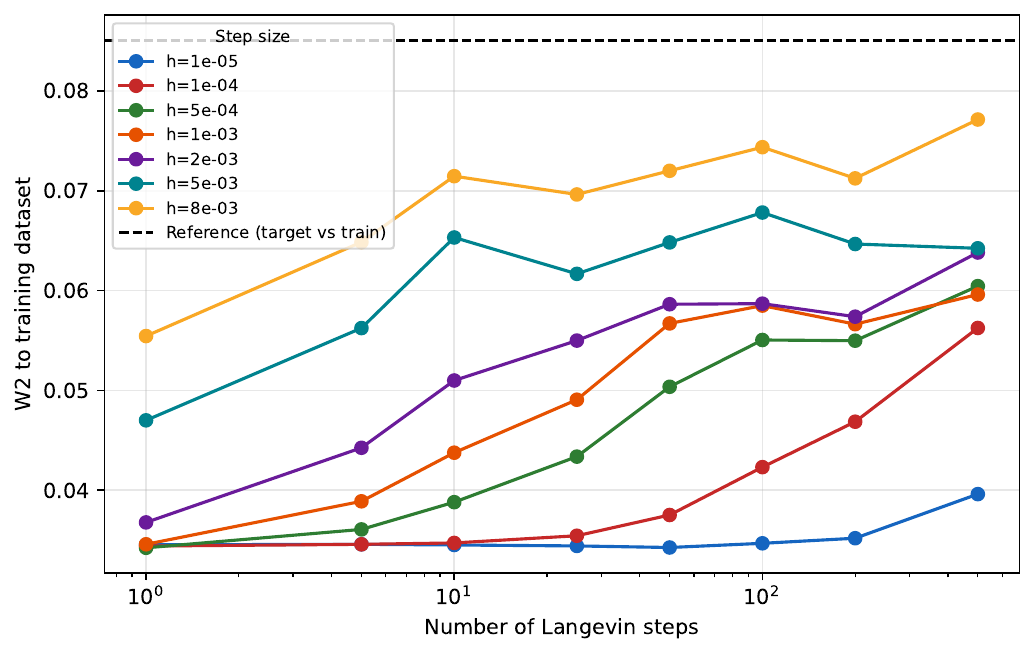}
\end{minipage}
\caption{Effect of Langevin step size $h$ and number of steps $\nsteps$ on MM-SOLD for the 2D checkerboard distribution. Left: SW2 to the target distribution. Right: SW2 to the training set. The dashed line is the finite-sample reference distance between the target reference samples and the training set.}
\label{fig:ablation_stepsize_steps}
\end{figure}

\paragraph{Effect of particle count and training-set size} 
We study how the number of particles and the training-set size affect the generation quality of MM-SOLD. We use a separate full-dimensional version of the handwritten digit-8 experiment in Section \ref{exp:handwritten_digit_generation}: samples are generated in the $100$-dimensional NRAE latent space, decoded by the fixed NRAE decoder, and evaluated by KID in the pretrained digit-classifier feature space. For the particle-count study, we fix $N=1{,}000$ and vary $P\in\{101,200,300,500,800,1{,}000,1{,}500,2{,}000\}$. For the training-set-size study, we fix $P=500$ and vary $N\in\{50,100,200,500,700,1{,}000\}$. As a reference, we compare with an independent kinetic Langevin sampler using a BAOAB integrator \cite{leimkuhler2013robust} for the explicit moment-matched target \eqref{eq:moment_matched_target}. This baseline estimates $(\lambda,\Lambda)$ from the training latents using \eqref{eq:tilting_parameter_identities}. Positions are initialized from the same training-data GMM as MM-SOLD, while momenta are initialized from $\mathcal N(0,I)$. Both samplers in this ablation work in partially whitened coordinates $y=W(z-\bar z)$ computed from the same $N$ training latents. Writing their unbiased covariance as $V\operatorname{diag}(s_1,\ldots,s_d)V^\top$ with $s_1\ge\cdots\ge s_d$, we set $\tilde s_j=\max(s_j,0)+10^{-5}$, $\tau=\tilde s_{21}$ and $W=V\operatorname{diag}\big(\min(\tilde s_j,\tau)/\tau\big)^{-1/2}V^\top$, which is invertible for every $N$. In these coordinates both samplers use the GMM moments, $\Sigma^*=\widehat\Sigma^{y}_N+\delta^2I_{100}\succeq\delta^2I_{100}$, so the Cholesky factor in Algorithm~\ref{alg:mmsold-lm} is well defined even for $N=50$ and $N=100$, where $\operatorname{rank}\widehat\Sigma_N\le N-1<100$. In the original NRAE coordinates the enforced covariance is $\widehat\Sigma_N+\delta^2W^{-2}$, whose variance along the null space of $\widehat\Sigma_N$ is only $10^{-5}\delta^2/\tau$, so the particles stay essentially in the affine span of the training latents. Both samplers use $\delta=0.03$, $\sigma=0.2$, $M=32$ and $100$ steps of size $h=5\times10^{-4}$; the kinetic sampler uses friction $\gamma=1$ and BAOAB with one stochastic force evaluation per step, as in Appendix~\ref{app:handwritten_generation_details}. Its plug-in estimate uses the same ridge $\zeta$ as in Appendix~\ref{app:handwritten_generation_details} and no projection; $\delta^2\lambda_{\min}(\delta^{-2}I+\widehat\Lambda)$ lies in $[0.52,0.55]$ for every training-set size, so all estimated targets are normalizable.
Training subsets for $N<1{,}000$ are drawn without replacement with a fixed seed. KID is the polynomial-kernel MMD$^2$ U-statistic between all $300$ held-out digit-8 test features and all generated features, interpreted descriptively because MM-SOLD particles are dependent, and the curves show the mean and standard deviation over five runs with different sampler seeds. This ablation uses an earlier NRAE checkpoint with the same architecture and loss as in Appendix~\ref{app:handwritten_digit_details}, trained on the same $10{,}000$ training images but selected by the lowest training loss rather than the validation loss, and a CNN feature classifier with the architecture of Appendix~\ref{app:handwritten_generation_details}, trained and selected on the same training and validation images. Test images are used by neither model. Because the representation differs, the KID values in Figure~\ref{fig:ablation_particle_train} are not comparable with those of the main digit-generation experiment.

Figure \ref{fig:ablation_particle_train} shows the results. With $N=1{,}000$ fixed, MM-SOLD improves as the number of particles grows and overtakes kinetic Langevin from about $200$ particles on. This is consistent with the large-particle equilibrium characterization, although the theorem does not guarantee monotone improvement of finite-step KID. By contrast, adding more independent kinetic Langevin chains does not improve the estimated target itself. When $P=500$ is fixed and $N$ varies, MM-SOLD is more reliable in the small-data regime. For example, at $N=100$, MM-SOLD obtains KID $0.144$, while kinetic Langevin gives KID $0.238$. This gap is consistent with the difficulty of estimating the linear and quadratic tilting parameters from sparse data, but it cannot be attributed to parameter estimation alone, since the two samplers also differ in their dynamics and finite-step behaviour. In these experiments the gap is largest for the smallest training sets.
\begin{figure}[t]
\centering
\begin{minipage}{0.49\textwidth}
\centering
\includegraphics[width=\textwidth]{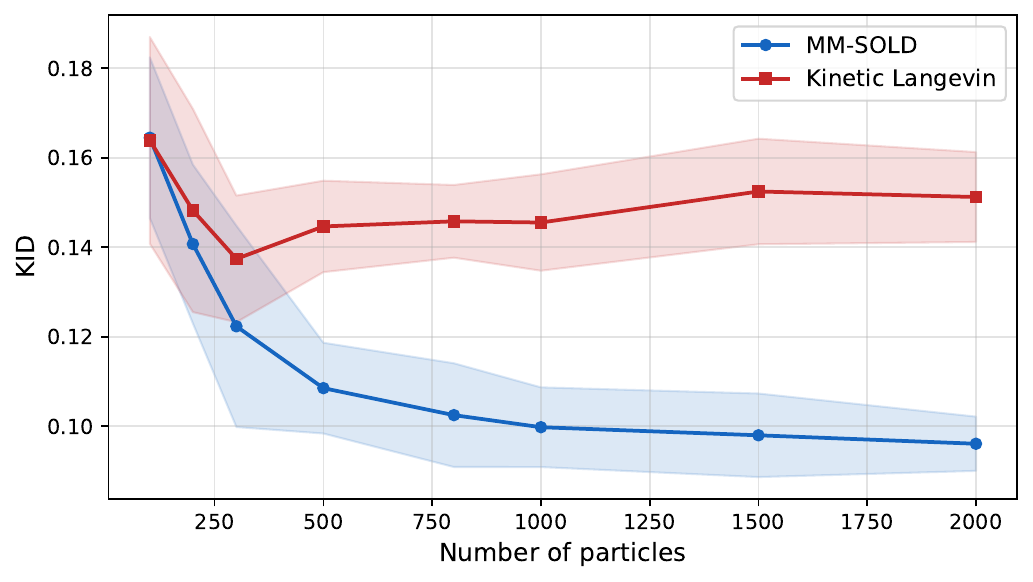}
\end{minipage}
\hfill
\begin{minipage}{0.49\textwidth}
\centering
\includegraphics[width=\textwidth]{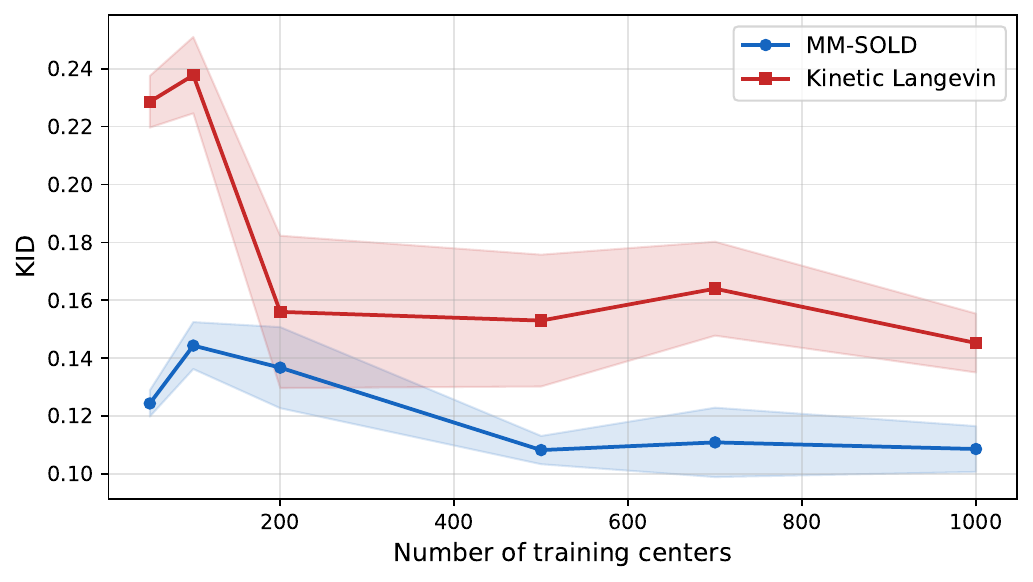}
\end{minipage}
\caption{Effect of particle count and training-set size on digit-8 generation. Left: KID versus the number of generated particles with $N=1{,}000$. Right: KID versus the number of training samples with fixed particle number. Kinetic Langevin denotes independent BAOAB sampling from the estimated moment-matched limiting target.}
\label{fig:ablation_particle_train}
\end{figure}

\clearpage
\section*{NeurIPS Paper Checklist}

\begin{enumerate}

\Needspace{12\baselineskip}
\begin{samepage}
\item \textbf{Claims}

Question: Do the main claims made in the abstract and introduction accurately reflect the paper's contributions and scope?

Answer: \answerYes{}

Justification: The abstract and introduction state the proposed MM-SOLD sampler, the large-particle limiting target, and the empirical claims on training-free, robust, CPU-based sampling. These claims are supported by the theoretical results in the main text and by the experiments in Section \ref{sec:Experiments}.

\end{samepage}

\Needspace{12\baselineskip}
\begin{samepage}
\item \textbf{Limitations}

Question: Does the paper discuss the limitations of the work performed by the authors?

Answer: \answerYes{}

Justification: Section~\ref{sec:Conclusion_and_discussion} discusses the assumptions and residual bias in our local finite-sample bound, the gap between the analyzed and experimental samplers, covariance and particle-count requirements, and limitations of the experimental comparisons.

\end{samepage}

\Needspace{12\baselineskip}
\begin{samepage}
\item \textbf{Theory assumptions and proofs}

Question: For each theoretical result, does the paper provide the full set of assumptions and a complete (and correct) proof?

Answer: \answerYes{}

Justification: Assumptions are stated with each result. Proofs are in Appendices~\ref{app:approx_finite_step_mmsold}, \ref{app:statistical_target_error}, and~\ref{app:proofs}; the step-size discussion is explicitly a conditional calculation, not a stability theorem.

\end{samepage}

\Needspace{12\baselineskip}
\begin{samepage}
\item \textbf{Experimental result reproducibility}

Question: Does the paper fully disclose all the information needed to reproduce the main experimental results of the paper to the extent that it affects the main claims and/or conclusions of the paper, regardless of whether the code and data are provided or not?

Answer: \answerNo{}

Justification: The main text and Appendix~\ref{app:exp_details} specify the datasets, baselines, metrics, and most settings. The GMM covariance and small-sample transformation are specified, but exact evaluation-checkpoint identifiers and some selected configurations remain undocumented.

\end{samepage}

\Needspace{12\baselineskip}
\begin{samepage}
\item \textbf{Open access to data and code}

Question: Does the paper provide open access to the data and code, with sufficient instructions to faithfully reproduce the main experimental results, as described in supplemental material?

Answer: \answerYes{}

Justification: The experiments use publicly available datasets. Code for reproducing the main experiments is included in the supplementary material, together with the experimental settings described in the paper and Appendix \ref{app:exp_details}.

\end{samepage}

\Needspace{12\baselineskip}
\begin{samepage}
\item \textbf{Experimental setting/details}

Question: Does the paper specify all the training and test details (e.g., data splits, hyperparameters, how they were chosen, type of optimizer) necessary to understand the results?

Answer: \answerNo{}

Justification: The main text and Appendix~\ref{app:exp_details} report data splits, searched hyperparameters, and metrics, including reuse of held-out sets for generation selection. Some evaluation-split and checkpoint details remain incomplete.

\end{samepage}

\Needspace{12\baselineskip}
\begin{samepage}
\item \textbf{Experiment statistical significance}

Question: Does the paper report error bars suitably and correctly defined or other appropriate information about the statistical significance of the experiments?

Answer: \answerYes{}

Justification: Tables~\ref{tab:digit8_generation_results} and~\ref{tab:celebahq_generation_results} report mean $\pm$ standard deviation over repeated runs using independently generated samples. Additional generated images and heatmaps are provided in Appendix \ref{app:exp_details} to show behavior across hyperparameter choices.

\end{samepage}

\Needspace{12\baselineskip}
\begin{samepage}
\item \textbf{Experiments compute resources}

Question: For each experiment, does the paper provide sufficient information on the computer resources (type of compute workers, memory, time of execution) needed to reproduce the experiments?

Answer: \answerNo{}

Justification: Tables~\ref{tab:digit8_generation_results} and~\ref{tab:celebahq_generation_results} report training or optimization time, per-sample sampling time, and the device used for each method. Appendix~\ref{app:exp_details} gives implementation settings, but CPU model, memory specifications, and total project compute are not fully documented.

\end{samepage}

\Needspace{12\baselineskip}
\begin{samepage}
\item \textbf{Code of ethics}

Question: Does the research conducted in the paper conform, in every respect, with the NeurIPS Code of Ethics \url{https://neurips.cc/public/EthicsGuidelines}?

Answer: \answerYes{}

Justification: The work uses public datasets and standard generative-model evaluation metrics. It does not involve private data, deceptive data collection, human-subject experiments, or high-risk deployment. The research conforms to the NeurIPS Code of Ethics.

\end{samepage}

\Needspace{12\baselineskip}
\begin{samepage}
\item \textbf{Broader impacts}

Question: Does the paper discuss both potential positive societal impacts and negative societal impacts of the work performed?

Answer: \answerYes{}

Justification: The method may reduce the computational cost and hardware requirements of generative sampling by avoiding neural score training and enabling CPU-based sampling. As with other image-generation methods, it could in principle be used to generate synthetic media if scaled further. However, this paper studies a training-free research prototype, does not release a high-capacity pretrained generator, and evaluates only on standard public research datasets.

\end{samepage}

\Needspace{12\baselineskip}
\begin{samepage}
\item \textbf{Safeguards}

Question: Does the paper describe safeguards that have been put in place for responsible release of data or models that have a high risk for misuse (e.g., pretrained language models, image generators, or scraped datasets)?

Answer: \answerNA{}

Justification: The paper does not release a pretrained high-capacity generative model or a newly scraped dataset. The experiments use existing public research datasets.

\end{samepage}

\Needspace{12\baselineskip}
\begin{samepage}
\item \textbf{Licenses for existing assets}

Question: Are the creators or original owners of assets (e.g., code, data, models), used in the paper, properly credited and are the license and terms of use explicitly mentioned and properly respected?

Answer: \answerNo{}

Justification: The original sources of the datasets, baselines, and models are cited, but the manuscript does not provide a complete asset-by-asset account of licenses and terms of use. 

\end{samepage}

\Needspace{12\baselineskip}
\begin{samepage}
\item \textbf{New assets}

Question: Are new assets introduced in the paper well documented and is the documentation provided alongside the assets?

Answer: \answerNo{}

Justification: The released implementation is a new code asset. Its algorithms and experimental settings are described in the paper and appendix, but a complete account of asset documentation is not provided here.

\end{samepage}

\Needspace{12\baselineskip}
\begin{samepage}
\item \textbf{Crowdsourcing and research with human subjects}

Question: For crowdsourcing experiments and research with human subjects, does the paper include the full text of instructions given to participants and screenshots, if applicable, as well as details about compensation (if any)?

Answer: \answerNA{}

Justification: The paper does not involve crowdsourcing or research with human subjects.

\end{samepage}

\Needspace{12\baselineskip}
\begin{samepage}
\item \textbf{Institutional review board (IRB) approvals or equivalent for research with human subjects}

Question: Does the paper describe potential risks incurred by study participants, whether such risks were disclosed to the subjects, and whether Institutional Review Board (IRB) approvals (or an equivalent approval/review based on the requirements of your country or institution) were obtained?

Answer: \answerNA{}

Justification: The paper does not involve human-subject studies, so IRB approval is not applicable.

\end{samepage}

\Needspace{12\baselineskip}
\begin{samepage}
\item \textbf{Declaration of LLM usage}

Question: Does the paper describe the usage of LLMs if it is an important, original, or non-standard component of the core methods in this research? Note that if the LLM is used only for writing, editing, or formatting purposes and does not impact the core methodology, scientific rigorousness, or originality of the research, declaration is not required.

Answer: \answerYes{}

Justification: LLM assistance was used to develop, draft, and check the additional appendix theory and to edit the manuscript. The LLM is not a component of the proposed sampler; responsibility for the mathematical claims and reported results remains with the authors.

\end{samepage}
\end{enumerate}

\end{document}